\PassOptionsToPackage{dvipsnames}{xcolor}
\documentclass{article}

\PassOptionsToPackage{authoryear}{natbib}

 \usepackage[preprint]{neurips_2026}
\usepackage{bbm}
\usepackage{booktabs}
\usepackage{multirow}

\usepackage{microtype}
\usepackage{graphicx}
\usepackage{subcaption}
\usepackage{booktabs} %

\usepackage{amsmath}
\usepackage{nccmath}
\usepackage{comment}

\usepackage{tikz}
\usetikzlibrary{arrows}

\usetikzlibrary{calc} 

\usepackage{titlesec}
\usepackage{titletoc}

\usepackage{adjustbox}

\usepackage{multicol}

\usepackage{booktabs}         %

\usepackage{amsmath}          %
\usepackage{siunitx}          %
\usepackage{colortbl}

\usepackage{comment}
\usepackage[framemethod=TikZ]{mdframed}
\usepackage{makecell} 
\usepackage[normalem]{ulem}
\useunder{\uline}{\ul}{}

\usepackage[hidelinks]{hyperref}

\usepackage{enumitem}

\makeatletter
\newcommand{\shorteq}{\mathrel{\mkern0.2mu\mathpalette\shorteq@\relax\mkern0.2mu}}
\newcommand{\shorteq@}[2]{\scalebox{0.5}[1]{$\m@th#1=$}}

\usepackage[utf8]{inputenc}
\usepackage[most]{tcolorbox}
\usepackage{listings}
\usepackage{xcolor}
\usepackage{caption}
\usepackage{subcaption}

\makeatletter
\newcommand{\shortminus}{\mathrel{\mkern0.2mu\mathpalette\shortminus@\relax\mkern0.2mu}}
\newcommand{\shortminus@}[2]{\scalebox{0.5}[1]{$\m@th#1-$}}

\makeatletter
\newcommand{\shortnabla}{\mathrel{\mkern0.2mu\mathpalette\shortnabla@\relax\mkern0.2mu}}
\newcommand{\shortnabla@}[2]{\scalebox{0.8}[1]{$\m@th#1\nabla$}}

\usepackage{algorithm}
\usepackage{algorithm}                 %
\usepackage[noend]{algpseudocode}      %
\usepackage{algpseudocodex}    %

\usepackage{amsmath}
\usepackage{amssymb}
\usepackage{mathtools}
\usepackage{amsthm}

\usepackage[capitalize,noabbrev]{cleveref}

\usepackage{multicol}

\theoremstyle{plain}
\newtheorem{theorem}{Theorem}[section]
\newtheorem{proposition}[theorem]{Proposition}
\newtheorem{remark}[theorem]{Remark}
\newtheorem{lemma}[theorem]{Lemma}

\newtheorem{corollary}[theorem]{Corollary}
\theoremstyle{definition}
\newtheorem{definition}[theorem]{Definition}

 \usepackage{booktabs}
\usepackage{multirow}

\makeatletter
\newtheorem*{rep@theorem}{\rep@title}
\newcommand{\newreptheorem}[2]{%
\newenvironment{rep#1}[1]{%
 \def\rep@title{\bf #2 \ref{##1}}%
 \begin{rep@theorem}}%
 {\end{rep@theorem}}}
\makeatother

\newreptheorem{proposition}{Proposition}
\newreptheorem{corollary}{Corollary}
\newreptheorem{theorem}{Theorem}
\newreptheorem{lemma}{Lemma}
\newreptheorem{observation}{Observation}
\newreptheorem{remark}{Remark}

\mdfdefinestyle{highlightedBox}{
    linecolor=Periwinkle!40,          %
    backgroundcolor=Periwinkle!4,   %
    innertopmargin=5pt,       %
    innerbottommargin=5pt,    %
    roundcorner=5pt,           %
    innerleftmargin=7pt,
    innerrightmargin=7pt,
    rightmargin=-0pt,
    leftmargin=-0pt
}

\usepackage[most]{tcolorbox}
\colorlet{LightRed}{BrickRed!15!}
\colorlet{LightGreen}{ForestGreen!10!}
\colorlet{Lightgray}{gray!15!}
\tcbset{
        mygraybox/.style={
        on line, boxsep=4pt, left=0pt, right=0pt, top=0pt, bottom=0pt,colframe=white,,boxrule=0.1pt,colback=Lightgray, highlight math style={enhanced}
    }}
    \tcbset{
        mygreenbox/.style={
        on line, boxsep=4pt, left=0pt, right=0pt, top=0pt, bottom=0pt,colframe=white,,boxrule=0.1pt,colback=LightGreen, highlight math style={enhanced}
    }
}

\usepackage{stmaryrd}
\usepackage{trimclip}

\makeatletter
\DeclareRobustCommand{\shortto}{%
  \mathrel{\mathpalette\short@to\relax}%
}

\newcommand{\short@to}[2]{%
  \mkern2mu
  \clipbox{{.5\width} 0 0 0}{$\m@th#1\vphantom{+}{\shortrightarrow}$}%
  }
\makeatother

\usepackage{amsmath,amsfonts,bm}

\def\1{\bm{1}}

\tikzset{
  singlemathdoublearrow/.style={
    draw,
    line width=0.5pt,
    <->,
    >={[scale=0.6]Triangle} 
  }
}

\def\d{\mathrm{d}}

\DeclareMathAlphabet{\mathsfit}{\encodingdefault}{\sfdefault}{m}{sl}
\SetMathAlphabet{\mathsfit}{bold}{\encodingdefault}{\sfdefault}{bx}{n}

\def\gL{{\gL}}

\newcommand{\E}{\mathbb{E}}

\let\log\relax
\DeclareMathOperator{\log}{log}

\newcommand{\Var}{\mathrm{Var}}

\makeatletter
\newcommand{\bwd}{\mathpalette{\overarrowsmall@\leftarrowfill@}}
\newcommand{\overarrowsmall@}[3]{%
  \vbox{%
    \ialign{%
      ##\crcr
      #1{\smaller@style{#2}}\crcr
      \noalign{\nointerlineskip}%
      $\m@th\hfil#2#3\hfil$\crcr
    }%
  }%
}
\def\smaller@style#1{%
  \ifx#1\displaystyle\scriptstyle\else
    \ifx#1\textstyle\scriptstyle\else
      \scriptscriptstyle
    \fi
  \fi
}
\makeatother

\makeatletter
\newcommand{\fwd}{\mathpalette{\overrightarrowsmall@\rightarrowfill@}}
\newcommand{\overrightarrowsmall@}[3]{%
  \vbox{%
    \ialign{%
      ##\crcr
      #1{\smaller@style{#2}}\crcr
      \noalign{\nointerlineskip}%
      $\m@th\hfil#2#3\hfil$\crcr
    }%
  }%
}
\def\smaller@style#1{%
  \ifx#1\displaystyle\scriptstyle\else
    \ifx#1\textstyle\scriptstyle\else
      \scriptscriptstyle
    \fi
  \fi
}
\makeatother

\usepackage{fontawesome5}
\hypersetup{
   breaklinks=true,   
   colorlinks=true,  
   linkcolor=BrickRed,
   citecolor=Blue,
   urlcolor=Blue
}
\definecolor{red}{HTML}{ca0020}
\definecolor{lightred}{HTML}{f4a582}
\definecolor{lightblue}{HTML}{92c5de}
\definecolor{green}{HTML}{008837}
\definecolor{blue}{HTML}{2c7bb6}
\usepackage{tcolorbox}
\usepackage[threshold=1]{csquotes}
\definecolor{myred}{HTML}{e83722}
\newtcolorbox{coloredquote}{
  colback=gray!0,      %
  colframe=gray!0,     %
  left=6pt,             %
  right=6pt,            %
  top=6pt,              %
  bottom=6pt,           %
  breakable             %
}
\usepackage{wrapfig,lipsum,booktabs}
\usepackage{cancel}

\usepackage{algorithm}
\usepackage{svg}
\usepackage[hidelinks]{hyperref}
\usepackage{fontawesome5}
\newcommand{\paperlinks}[2]{%
\vspace{-0.35em}
\begin{center}
\faGithub\ {\color{MidnightBlue}\href{#1}{\textsc{code}}}
\hspace{1.2em}
\faGlobe\ {\color{MidnightBlue}\href{#2}{\textsc{website}}}
\end{center}
\vspace{-0.6em}
}

\crefname{proposition}{proposition}{propositions}
\Crefname{proposition}{Proposition}{Propositions}

\usepackage{makecell}
\usepackage[most]{tcolorbox}
\tcbuselibrary{theorems,breakable}
\usepackage{tikz}
\usetikzlibrary{arrows.meta,calc}
\newtcbtheorem[number within=section]{boxedprop}{Proposition}{
  enhanced,
  breakable,
  colback=LimeGreen!3,
  colframe=LimeGreen!20,
  coltitle=black,
  colbacktitle=LimeGreen!20,
  boxrule=0.7pt,
  arc=3pt,
  top=3pt,
  bottom=3pt,
  left=4pt,
  right=4pt,
  fonttitle=\bfseries,
  boxed title style={
    size=small,
    boxrule=0pt,
    colframe=LimeGreen!75,
    colback=LimeGreen!75
  },
  before skip=5pt,
  after skip=5pt
}{prop}
\newcolumntype{Y}{>{\raggedright\arraybackslash}X}
\defcitealias{he2025feat}{He and Du et al., 2025}
\usepackage{wrapfig}
\usepackage{tabularx}
\title{Free energy Estimation on Any State Space}

\author{%
Jiajun He$^{1,\dagger}$ \And Zijing Ou$^{2}$    \And Francisco Vargas$^{3}$   \And Yingzhen Li$^{2}$  \NEWLINE José Miguel Hernández-Lobato$^{1}$  \And Carles Domingo-Enrich$^{4}$   \And Yuanqi Du$^{4,\dagger}$ \AFF 
$^1$University of Cambridge, $^2$Imperial College London, \\$^3$Xaira Therapeutics, $^4$Microsoft Research New England\CORR
$^\dagger$Corresponding to \texttt{jh2383@cam.ac.uk} and \texttt{yuanqidu@microsoft.com}.
}

\begin{document}

\maketitle

\begin{abstract}
Free energy estimation is a fundamental yet challenging problem, from physics to statistics. Classical approaches rely on thermodynamic transformations, ranging from direct estimation, quasistatic integration, to finite-time averaging. Recent work \citepalias{he2025feat} learns neural transports to significantly accelerate the efficiency in the finite-time regime. 
In this paper, we generalize this framework to arbitrary state spaces. 
Building on this view, we develop a generalized neural transport learning approach for efficient estimation.
Experiments validate the effectiveness and efficiency of the proposed method beyond continuous settings, extending to discrete and multimodal spaces as well as autoregressive settings.
Beyond free energy estimation, we establish algebraic identities and reveal a group-theoretic structure linking infinitesimal time reversal and generalized Doob's $h$-transforms, showing that their compositions form a generalized dihedral group.
\end{abstract}

\section{Introduction}

Free energy estimation is a fundamental problem across physics, statistics, and machine learning. Defined as the negative logarithm of the normalization constant for an unnormalized density $\tilde{\pi}$,
\begin{equation}
\label{eq:fe}
F = - \log Z = -\log \int_\mathcal{X} \tilde{\pi}(x) \text{d}x.
\end{equation}
\looseness=-1
It plays a central role in characterizing probabilistic models and underlies a wide range of applications from molecular thermodynamics to Bayesian inference \citep{lelievre2010free,tuckerman2023statistical}. 

However, brute force integration of \Cref{eq:fe} in high dimensions is generally intractable.
A standard strategy is to recast the problem as a Monte Carlo (MC) estimation for relative free energy.
Let
$\tilde P$ and $\tilde Q$ be two finite, nonnegative measures on $\mathcal X$, with normalizing constants
$Z_P=\tilde P(\mathcal X)=\int_{\mathcal{X}} \tilde{P}(\text{d}x)$ and $Z_Q=\tilde Q(\mathcal X)=\int_{\mathcal{X}} \tilde{Q}(\text{d}x)$, 
and normalized
probability measures $P(\mathrm d x)=\tilde P(\mathrm d x)/Z_P$ and
$Q(\mathrm d x)=\tilde Q(\mathrm d x)/Z_Q$. Assuming $\tilde Q(\text{d}x)\ll \tilde P(\text{d}x)$,
the relative free energy is
\begin{equation}
\label{eq:general_fe}
\Delta F
= -\log (Z_Q/Z_P)
= -\log \mathbb E_{X\sim P}
\left[
{\mathrm d\tilde Q}/{\mathrm d\tilde P}(X)
\right].
\end{equation}
Here, $\mathrm d\tilde Q/\mathrm d\tilde P$ denotes the Radon-Nikodym
derivative between the two unnormalized measures. 
We keep this formulation
measure-theoretic, since the state space $\mathcal X$ does not always need to have a Lebesgue measure, and $P$ and $Q$ may therefore
not admit a canonical notion of density, while \Cref{eq:fe} is only the special case where $Z_P=1$ and $Q$ admit density to Lebesgue measure.
This perspective gives rise to a broad family of free energy estimators \citep{lelievre2010free}, ranging from direct sample reweighting to finite-time path averaging. Representative examples include free energy perturbation (FEP) and the Jarzynski equality (JE), which we discuss in \Cref{sec:is}.

Unfortunately, both FEP and JE face well-known limitations. 
Direct reweighting becomes unreliable in high dimensions when the overlap between the reference and target distributions is small, often necessitating a sequence of intermediate distributions to bridge the gap.
Finite-time nonequilibrium estimators such as JE typically suffer from high variance unless the underlying process is close to the quasistatic limit. In practice, this favorable regime is only attained when the transformation is infinitely slow, or when special time-reversible structures can be exploited.

To address this challenge, \citet{vaikuntanathan2008escorted} introduced an extension of JE based on a controlled transport, now known as the escorted Jarzynski estimator (EJE), to reduce estimation variance. 
More recently, this idea has attracted renewed attention in machine learning, where computational methods have been developed to learn such escorts efficiently \citep{vargas2023transport,zhong2024time,albergo2024nets}. These methods seek to construct a transport through a sequence of prescribed intermediate states. 
Existing approaches mainly differ in how these intermediate states are defined: some specify them through their log densities, while others define them through interpolating paths in sample space. 
Among these methods, FEAT \citepalias{he2025feat} provides a notable example of learning transport-based finite-time estimators efficiently by interpolating in the sample space.

However, these recent advances mostly focus on continuous state spaces with SDE-based transport, while the concept of free energy can appear in any state space.
To close this gap, in this paper,
\begin{enumerate}[leftmargin=*]
    \item We extend JE and EJE beyond the classical thermodynamic definition to general Markov processes, while preserving their underlying thermodynamic interpretation (\Cref{sec:EJE_crook}).
    \item We describe frameworks to learn efficient neural transport between distributions over arbitrary state spaces and to use them for the estimation of free energy, generalizing the central principle behind the success of FEAT estimators (\Cref{sec:learn,sec:inf_est_deltaF}).
    \item We empirically verify our proposed framework on different modalities, from discrete to multimodal distributions, and further broaden the notion of non-equilibrium transport to autoregressive models, substantially expanding the scope of these methods beyond traditional thermodynamic and continuous-state settings (\Cref{sec:exa_exp}).
    \item Beyond free energy estimation, we reveal a broader algebraic structure connecting infinitesimal time reversal and generalized Doob's $h$-transforms.
In particular, any finite composition of these maps reduces to either a single generalized Doob transform or a single infinitesimal reversal.
This property exposes an underlying group-theoretic structure: the transformations generated by these two operations form a group isomorphic to a generalized dihedral group (\Cref{sec:reversal_doob}).
\end{enumerate}

\vspace{-4pt}
\section{Background}
\label{sec:bg}
\vspace{-4pt}

In this section, we review the nonequilibrium approaches to free energy estimation based on finite-time stochastic transformations between two distributions. Without loss of generality, from now on, we assume the unnormalized density is a Boltzmann distribribution $\tilde{\pi}(x) \propto \exp(-U(x))$ to recover the thermodynamic picture of free energy. 
We further ignore the Boltzmann constant, i.e. $k_B T =1$.\vspace{-3pt}

\paragraph{Jarzynski Equality and Crooks Fluctuation Theorem.}
The Jarzynski equality  \citep[JE, ][]
{jarzynski1997nonequilibrium} expresses the free energy difference as an exponential average of work:\vspace{-3pt}
\begin{equation}
\label{eq:je}
\exp(-\Delta F)
=
\mathbb{E}_{\fwd{{P}}}\!\left[\exp\!\left(-W(X_{[0, 1]})\right)\right],
\qquad
W(X_{[0, 1]}) = \int_0^1 \partial_t U_t(X_t)\,\mathrm{d}t,
\end{equation}
where \(\fwd{{P}}\) denotes the path measure of a finite-time stochastic process (annealed Langevin) connecting two systems, and  \(W \) is the accumulated work along the process.
The Crooks fluctuation theorem \citep[CFT,][]{crooks1999entropy} generalizes JE by relating the forward and backward processes with work:
\begin{equation}
\label{eq:crooks}
{\mathrm{d}\bwd{{Q}}}/{\mathrm{d}\fwd{P}}(X_{[0, 1]})
=
\exp\!\left(-W(X_{[0, 1]})+\Delta F\right)
\end{equation}
with JE recovered when taking an expectation on both sides of the equality.

\paragraph{Escorted JE (EJE) and Crooks.}
To reduce the variance of JE estimators, \citet{vaikuntanathan2008escorted} introduced the escorted Jarzynski estimator, which augments the dynamics with an auxiliary transport $v_t$, or \emph{escort}, that guides trajectories towards the prefixed intermediate states. The resulting estimator preserves the JE form with a modification to the work,
\begin{equation}
\label{eq:escorted_je}
\exp(-\Delta F)
=
\mathbb{E}_{\fwd{P}^{v}}\!\left[\exp\!(-\widetilde{W})\right], \; \widetilde{W} = \int_0^1 \Big(\partial_t U_t   - \nabla \cdot v_t  + \nabla U_t  \cdot v_t \Big)\text{d}t  
\end{equation}
where \(\fwd{P}^{v}\) denotes the escorted forward path measure and \(\widetilde{W}\) is a generalization of the work. 
An analogous Crooks fluctuation theorem also holds,
$
{\mathrm{d}\bwd{Q}^{v}}/{\mathrm{d}\fwd{P}^{v}}(X_{[0, 1]})
=
\exp\!\left(-\widetilde{W}(X_{[0, 1]})+\Delta F\right).
$\vspace{-3pt}

\paragraph{Generator of Markov Process.}
Our work mainly uses generators to characterize Markov processes, so we
briefly provide their intuition here; see \Cref{app:assumption_definitation} for a more detailed
discussion.
For a time-inhomogeneous Markov process, the time-dependent generator is the infinitesimal object that determines the local evolution of the process.
More precisely, if $P_{s,t}$ denotes the Markov transition operator from time $s$ to $t$, then for a small time increment $\delta > 0$, 
\begin{align}
    P_{t,t+\delta} f(x)
    =
    f(x) + \delta\,\mathcal L_t f(x) + o(\delta).
\end{align}
Thus, $\mathcal L_t f(x)$ represents the instantaneous rate of change of the observable
$f$. By adjointness\footnote{We consider the adjoint with respect to the Lebesgue measure, $\int (\mathcal L_t f)(x)\,\rho(x)  \mathrm d x
=
\int f(x)\,(\mathcal L_t^\dagger \rho)(x) \mathrm d x.$},
the adjoint operator $\mathcal L_t^\dagger$ describes the corresponding infinitesimal evolution of
densities.

\section{Generalized Escorted Jarzynski and Crooks Fluctuation Theorem}\label{sec:EJE_crook}
In this section, we present the escorted Jarzynski equality and the Crooks fluctuation theorem for a general Markov process, aiming to estimate the free-energy difference between $U_A$ and $U_B$.
We start with a (time-inhomogeneous) continuous-time Markov process $(X_t)_{t\ge 0}$ with generator $(\mathcal L_t)_{t\ge 0}$.
Denote the law of $X_t$ as $p_t$.
The following conclusion generalizes the escorted JE for the free energy difference between $U_A$ and $U_B$:

\begin{proposition}[Generalized Escorted Jarzynski Equality]\label{prop:jarzynski}
Let \((U_t)_{t\in[0,1]}\) be an energy path interpolating between $U_A=:U_0$ and $U_B=:U_1$. Define
\begin{align}
    \gamma_t(x) := e^{-U_t(x)},
    \qquad
    Z_t := \int \gamma_t(x)\d x,
    \qquad
    \pi_t(x) := {\gamma_t(x)}/{Z_t}.
\end{align}
Let \(\mathcal L_t^F\) be the generator of a forward continuous-time Markov process, with adjoint \((\mathcal L_t^F)^\dagger\). Let \((X_t)_{t\in[0,1]}\) be the corresponding process whose marginal law \(p_t\) satisfies
\begin{align}\label{eq:P}
    \partial_t p_t = (\mathcal L_t^F)^\dagger p_t,
    \qquad
    p_0 = \pi_A,
    \qquad
    \pi_A \propto e^{-U_A},
\end{align}
and denote its path law by \(\fwd{P}\). Then
\begin{align}\label{eq:jarzynski}
    \Delta F
    =
    -\log {Z_B}/{Z_A}
    =
    -\log \E_{\fwd{P}}\!\left[\exp({-\widetilde{W}(X_{[0,1]})})\right],
\end{align}
where the ``generalized" work is defined as 
\begingroup
\refstepcounter{equation}\label{eq:general_work_cont}
\setlength{\fboxsep}{0pt}

\noindent
\colorbox{LimeGreen!25}{%
  \makebox[\textwidth][l]{%
    \makebox[\dimexpr\textwidth-4em\relax][c]{%
      $\displaystyle
      \widetilde{W}(X_{[0,1]})
      :=
      \underbrace{\int_0^1 \partial_t U_t(X_t)\,\mathrm{d}t}_{\mathrm{work}}
      +
      \underbrace{\int_0^1 \frac{((\mathcal L_t^F)^\dagger \gamma_t)(X_t)}{\gamma_t(X_t)}\,\mathrm{d}t}_{\mathrm{escorting}}.
      $}%
    \makebox[4em][r]{\textup{(\theequation)}}%
  }%
}
\par
\endgroup

\end{proposition}
Similarly, we can generalize the Crooks fluctuation theorem as follows:
\begin{proposition}[Generalized Crooks Fluctuation Theorem]\label{prop:crooks}
Define the generator and its adjoint for a backward continuous-time Markov process\vspace{-5pt}
\begingroup
\refstepcounter{equation}\label{eq:bwd_crooks_cont}
\setlength{\fboxsep}{0pt} %

\noindent
\colorbox{Periwinkle!25}{%
  \makebox[\textwidth][l]{%
    \makebox[\dimexpr\textwidth-4em\relax][c]{%
      {$\displaystyle
      \mathcal L_t^B f
      :=
      \frac{1}{\gamma_t}(\mathcal L_t^F)^\dagger(\gamma_t f)
      -
      \frac{1}{\gamma_t}(\mathcal L_t^F)^\dagger(\gamma_t)\,f,
      \qquad
      (\mathcal L_t^B)^\dagger p
      =
      \gamma_t\,\mathcal L_t^F\!\left(\frac{p}{\gamma_t}\right)
      -
      \frac{p}{\gamma_t}\,(\mathcal L_t^F)^\dagger \gamma_t
      $}%
    }%
    \makebox[4em][r]{\textup{(\theequation)}}
  }%
}
\par
\endgroup
Let \(\bwd{Q}\) be the path law of the backward process with marginal law \(q_t\) satisfying
\begin{align}
    \partial_t q_t = -(\mathcal L_t^B)^\dagger q_t,
    \qquad
    q_1 = \pi_B,
    \qquad
    \pi_B \propto e^{-U_B}.
\end{align}
The RND between the forward ($\fwd{P}$ in \Cref{eq:P}) and this backward path laws is
\begin{align}\label{eq:crooks_eq}
    \frac{\mathrm{d}\bwd{Q}}{\mathrm{d}\fwd{P}}(X_{[0,1]})
    =
    \exp \bigl(- \widetilde{W}(X_{[0,1]}) + \Delta F\bigr).
\end{align}
\end{proposition}

\begin{remark}\label{remark:special_case}We note two special cases of Propositions~\ref{prop:jarzynski} and~\ref{prop:crooks}.
\begin{enumerate}[leftmargin=*]
    \item When the marginal of the forward process satisfies $p_t=\pi_t$, the integrand in \Cref{eq:general_work_cont} becomes constant:
    $\partial_t U_t(X_t)
        +
       \gamma_t^{-1}(X_t){((\mathcal L_t^F)^\dagger \gamma_t)(X_t)}
        =
        \partial_t U_t(X_t)
        +
        {p_t^{-1}(X_t)}{\partial_t p_t(X_t)}
        =
        -\partial_t \log Z_t$. 
    Moreover, the backward process in \Cref{eq:bwd_crooks_cont} is exactly the \emph{time-reversal} of the forward process, and its marginal satisfies $ 
        q_t = p_t = \pi_t $. 
    This corresponds to the equilibrium setting.
    Also, when $\pi_t$ is marginal, \Cref{eq:bwd_crooks_cont} recovers Nelson's relation \citep{9a350bc3-5552-3c07-bec5-2119235ec937} in generator form.

    \item When the forward dynamics is stationary at $\pi_t$ for each $t$, namely, $(\mathcal L_t^F)^\dagger \pi_t = 0$,  the escort term vanishes, which reduces to the standard Jarzynski equality. In addition, the backward generator is also stationary at $\pi_t$, $(\mathcal L_t^B)^\dagger \pi_t = 0$,   yielding the standard Crooks fluctuation theorem.
\end{enumerate}
\end{remark}
Additionally, the generalized work in \Cref{eq:general_work_cont} is not specific to a single forward-backward pair. 
In particular, there are infinitely many forward-backward pairs that give rise to the same functional:
\begin{remark}\label{remark:non-uniqueness}
    If $\mathcal K_t$ is any generator satisfying detailed balance with respect to $\pi_t$, namely $\mathcal K_t^\dagger(\pi_t f)=\pi_t \mathcal K_t f$ for all $f$, then for any $\alpha\in\mathbb R$ we can define modified dynamics by $\partial_t p_t=(\mathcal L_t^F+\alpha\mathcal K_t)^\dagger p_t$ and $\partial_t q_t=-(\mathcal L_t^B+\alpha\mathcal K_t)^\dagger q_t$, and the generalized work functional remains unchanged.
\end{remark}

In the following, we present two perspectives that provide intuition for the conclusion above. 
First, we present the corresponding discrete-time formulation,  which makes clear what the escort represents and how the choice of backward kernel should be interpreted.
Second, we interpret the result through the generalized Doob's $h$-transform, which provides a path-weight perspective to sketch the proof.

\paragraph{\faLightbulb[regular] Generalized EJE and Crooks in Discrete time} 
We now present the discrete-time EJE and Crooks. We state a compact version here and provide the full version in
Proposition~\ref{prop:discretetime_restate}.
\begin{proposition}[Discrete-time Jarzynski and Crooks, short version]\label{prop:discretetime}
Let $\{U_n\}_{n=0}^N$ be an discrete-time energy path, and define
$\gamma_n, Z_n, 
    \pi_n(x) 
$ in the same way as in  Proposition~\ref{prop:jarzynski}.
Assuming the forward processes have transition kernels $\{P_{F,n}\}_{n=1}^N$, the discrete-time generalized work is given by 
\begingroup
\refstepcounter{equation}\label{eq:general_work_disc}
\setlength{\fboxsep}{0pt}

\noindent
\colorbox{LimeGreen!25}{%
  \makebox[\textwidth][l]{%
    \makebox[\dimexpr\textwidth-4em\relax][c]{%
      $\displaystyle
      \widetilde{W}(X_{0:N})
      :=
      \sum_{n=1}^N
      \left[
          U_n(X_n)-U_{n-1}(X_n)
          +
          \log \frac{(P_{F,n}\#\pi_{n-1})(X_n)}{\pi_{n-1}(X_n)}
      \right]
      $}%
    \makebox[4em][r]{\textup{(\theequation)}}%
  }%
}
\par
\endgroup
Assume the backward process has transition kernels $\{P_{B,n}\}_{n=1}^N$ and 
\begingroup
\refstepcounter{equation}\label{eq:discrete_requirement}
\setlength{\fboxsep}{0pt}

\noindent
\colorbox{Periwinkle!25}{%
  \makebox[\textwidth][l]{%
    \makebox[\dimexpr\textwidth-4em\relax][c]{%
      $\displaystyle
      \pi_{n-1}(x_{n-1})\,P_{F,n}(x_n\mid x_{n-1})
      =
      (P_{F,n}\#\pi_{n-1})(x_n)\,P_{B,n}(x_{n-1}\mid x_n)
      .$}%
    \makebox[4em][r]{\textup{(\theequation)}}%
  }%
}
\par
\endgroup
Then the Crooks Fluctuation Theorem holds:
\begin{align}
    \frac{\d\bwd{Q}}{\d\fwd{P}}(X_{0:N})= \frac{
        \pi_B(x_N)\,\prod_{n=1}^N P_{B,n}(x_{n-1}\mid x_n)
    }{
        \pi_A(x_0)\,\prod_{n=1}^N P_{F,n}(x_n\mid x_{n-1})
    }
    =
    \exp\bigl(-\widetilde{W}(X_{0:N})+\Delta F\bigr).
\end{align}
\end{proposition}

From this proposition, we see that the generalized work can be interpreted as the \emph{unnormalized density ratio} between the forward and backward path measures, while the backward path is constructed by reversing the forward kernel at each step. We summarize the correspondence between the discrete- and continuous-time settings in \Cref{tab:escort_summary_compact}.
We also note that the continuous-time escorting term can be obtained by taking a first-order expansion:
\begin{align}
\label{eq:generator_expansion}
\log \frac{(P_{F,t+\delta}\#\pi_t)(x)}{\pi_t(x)}
&= \log  \frac{\pi_t(x) + \delta (\mathcal L_t^F)^\dagger \pi_t(x) + o(\delta) }{\pi_t(x)} = \delta \frac{(\mathcal L_t^F)^\dagger \pi_t(x)}{\pi_t(x)} + o(\delta).
\end{align}
The forward-backward relation can be interpreted following the same Taylor argument, which we make explicit in the proof of Proposition~\ref{prop:pi_reversal} in the Appendix.

\begin{figure}[t]
\centering

\begin{minipage}[t]{0.41\textwidth}
\centering
\vspace{4pt}

{\fontsize{6.5pt}{7.2pt}\selectfont
\setlength{\tabcolsep}{2.5pt}
\renewcommand{\arraystretch}{0.82}

\begin{tabularx}{0.95\linewidth}{
@{}
>{\raggedright\arraybackslash\bfseries}p{0.15\linewidth}
>{\columncolor{Periwinkle!25}}X
>{\columncolor{LimeGreen!25}}p{0.3\linewidth}
@{}
}
\toprule
&
\textbf{Crooks Relation}
&
\textbf{Escort}
\\
\midrule
Discrete time
&
\(
\begin{aligned}
&\pi_{n-1}P_{F,n}
\\[-1pt]
&=
(P_{F,n}\#\pi_{n-1})P_{B,n}
\end{aligned}
\)
&
\(
\log
\dfrac{P_{F,n}\#\pi_{n-1}}{\pi_{n-1}}
\)
\\\midrule 
Generator form
&
\(
\begin{aligned}
&\mathcal L_t^B f
=
\gamma_t^{-1}(\mathcal L_t^F)^\dagger(\gamma_t f)
\\[-2pt]
& 
-\gamma_t^{-1}(\mathcal L_t^F)^\dagger(\gamma_t)f
\end{aligned}
\)
&
\(
\dfrac{
((\mathcal L_t^F)^\dagger\gamma_t)(X_t)
}{
\gamma_t(X_t)
}
\)
\\
\bottomrule
\end{tabularx}
}

\captionof{table}{Discrete-time and generator forms of the escorted relation.}
\label{tab:escort_summary_compact}
\end{minipage}
\hfill
\begin{minipage}[t]{0.58\textwidth}
\centering
\vspace{0pt}
\includegraphics[width=\linewidth]{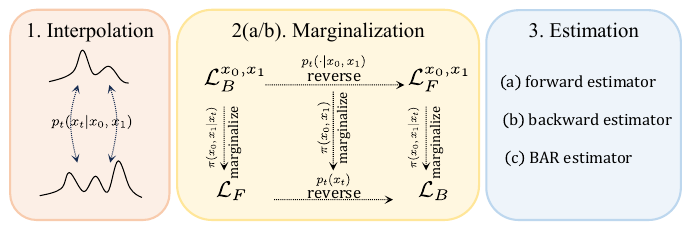}
\captionof{figure}{Pipeline of learning the transport.}
\label{fig:your_figure}
\end{minipage}

\vspace{-13pt}
\end{figure}

\paragraph{\faLightbulb[regular] Generalized Doob's $h$-transform and Generalized EJE} 
Generalized Doob's $h$-transform provides a general perspective and approach to handle path weight. We state 
generalized Doob's $h$ in  \Cref{thm:conservative_doob_h_transform}, and connect it with the adjoint and time reversal in \Cref{appendix:theorem}.
In the following theorem, we establish a connection between the generalized Doob's \(h\)-transform and its marginal prescription.
\begin{tcolorbox}[
    colback=gray!8,
    colframe=gray!40,
    boxrule=0.5pt,
    arc=4pt,
    left=1pt,
    right=1pt,
    top=2pt,
    bottom=2pt,
    breakable
]
\begin{theorem}[Generalized Doob's $h$-transform and Marginal prescription for it]
Let $(\mathcal L_t)_{t\in[0,1]}$ be a family of time-inhomogeneous Markov
generators, with adjoints $(\mathcal L_t^\dagger)_{t\in[0,1]}$.
Let
$h_t:E\to(0,\infty)$ be sufficiently regular. 
The generalized Doob's $h$-transform generator is given by
\begin{align}
    \mathcal L_t^h f
    =
    h_t^{-1}\mathcal L_t(h_t f)
    -
    h_t^{-1}f\,\mathcal L_t h_t .
\end{align}
Let $\pi_t$ be a strictly positive function and define $\rho_t^h = h_t \pi_t$.
Then $\rho_t^h$ solves the forward equation of the $h$-transformed process, i.e., 
\begin{align}
    \partial_t \rho_t^h = (\mathcal L_t^h)^\dagger \rho_t^h , \qquad \textbf{if and only if} \qquad
    {h_t}^{-1}\Big(\partial_t h_t+\mathcal L_th_t\Big) 
    =
    \pi_t^{-1}\Big(\mathcal L_t^\dagger \pi_t-\partial_t \pi_t\Big).
\end{align}
\end{theorem}
\end{tcolorbox}
This result is an extension of results derived in previous work in the scenario of marginal prescription \citep{chetrite2013nonequilibrium,denker2024deft,heng2025simulating}. 
This proposition allows us to write the unknown marginal density $\rho^h$ in the form of the known predefined $\pi_t$.
The generalized EJE and Crooks can be derived from this proposition, as we can view the backward process with the $h$-transformed process with known marginal.
We discuss more conclusions and intuition revolving this theorem in detail in \Cref{appendix:theorem}.

In summary, this section generalizes the escorted Jarzynski and Crooks in both continuous and discrete time to the general state space. 
However, although \Cref{eq:jarzynski}   holds for any energy path \(U_t\) and any dynamics, in practice, when estimating it using Monte Carlo samples, different choices lead to different statistical efficiencies.
More precisely, by \Cref{eq:crooks_eq}, we can further define the dissipated work \(W_{\mathrm{diss}}:=\widetilde{W}-\Delta F\), and
\( \E_{\fwd{P}}[W_{\mathrm{diss}}]
= \E_{\fwd{P}}\!\left[\log \frac{\d \fwd{P}}{\d \bwd{Q}}\right]
= D_{\mathrm{KL}}(\fwd{P}\,\|\,\bwd{Q}) \ge 0\),  \( \E_{\bwd{Q}}[-W_{\mathrm{diss}}]
= \E_{\bwd{Q}}\!\left[\log \frac{\d \bwd{Q}}{\d \fwd{P}}\right]
= D_{\mathrm{KL}}(\bwd{Q}\,\|\,\fwd{P}) \ge 0\),
which is \emph{the second law of thermodynamics} in non-equilibrium thermodynamics.
In fact, the KL lower bounds the variance of the estimator:
\begin{remark}[Variance-dissipation Relation]
\label{remark:variance_bound}
    Denote the  free energy estimate  as  \(\widehat {Z} := \widehat {\exp(-\Delta F)} = \frac1N \sum_{i=1}^N \exp({-\widetilde{W}(X^{(i)}_{[0,1]})})\) with \(X^{(i)}_{[0,1]} \sim \fwd{P}\). Then
    \begin{align}
        \Var(\widehat Z)=
\mathcal{O}\left(\tfrac1N
\chi^2(\bwd{Q}\,\|\,\fwd{P})  
\right)\geq \mathcal{O}\left(\tfrac1ND_{\mathrm{KL}}(\bwd{Q}\,\|\,\fwd{P}) \right).
    \end{align}
\end{remark}
Therefore, our goal is to reduce the gap between the forward and backward process, and thus improve the variance of the free-energy estimator. 
Ideally, we want to find dynamics that directly generate samples from \(\pi_t \propto e^{-U_t}\) at each time \(t\). Equivalently, we seek forward and backward dynamics that are time-reversals and bridges between $U_A$ and $U_B$. 
In this ideal setup, the dissipated work is 0, and the estimate has zero variance.
In \Cref{sec:learn}, we discuss how to construct such dynamics by learning a neural transport.
However, before discussing this, we first present a broader relation for the transformations of generators in the next section for completeness.
\vspace{-4pt}
\section{Beyond Free Energy: Infinitesimal Reversal and Generalized Doob's $h$}\label{sec:reversal_doob}
\vspace{-4pt}
In this section, we discuss results that go beyond the generalized EJE and Crooks fluctuation theorem, highlighting a connection between a generalized notion of time reversal and generalized Doob's $h$-transform.
These broader results arise from unifying the two perspectives developed above for understanding the generalized EJE: the Doob's $h$-transform perspective and the discrete-time perspective.
This unification further admits a more general and abstract algebraic interpretation.

\subsection{Algebraic Identities}
We first introduce the definition of infinitesimal time reversal.
\begin{tcolorbox}[
    colback=gray!8,
    colframe=gray!40,
    boxrule=0.5pt,
    arc=4pt,
    left=1pt,
    right=1pt,
    top=2pt,
    bottom=2pt,
    breakable
]
\begin{definition}[Infinitesimal $\pi_t$-reversal]
Let $(\mathcal L_t)_{t\in[0,1]}$ be a time-inhomogeneous Markov generator, and let $\mathcal L_t^\dagger$ denote its adjoint with
respect to a reference measure $\mu$. Let $\pi_t$ be a strictly positive
density with respect to $\mu$. The infinitesimal reverse generator with
respect to $\pi_t$ is defined by
\begin{align}\label{eq:11111}
    \bwd{\mathcal L}_t^\pi f
    &=
  {\pi_t}^{-1}\mathcal L_t^\dagger(\pi_t f)
    -
  {\pi_t}^{-1}{\mathcal L_t^\dagger \pi_t}f .
\end{align}
\end{definition}
\end{tcolorbox}

Here, we introduce \Cref{eq:11111} as a definition directly for simplification.
In Proposition~\ref{prop:pi_reversal}, we derive this form and show that $ \bwd{\mathcal L}_t^\pi$ is exactly the infinitesimal generator associated with the backward operator $\bwd {P}^{\pi_t}_{t+\delta,t}$ which satisfies the following equation:
\begin{align}
    \pi_t(x) \, P_{t,t+\delta}(x,\mathrm dy)\,\mu(\mathrm dx)
    =
    (\pi_t P_{t,t+\delta})(y)\,\mu(\mathrm dy)\,
    \bwd {P}^{\pi_t}_{t+\delta,t}(y,\mathrm dx).
\end{align}
where $P_{t,t+\delta}$ is the forward operator associated with $\mathcal L_t$.
\par
For simplicity, we define the infinitesimal time reversal and the generalized Doob's $h$-transform stated in  \Cref{thm:conservative_doob_h_transform} as two maps on generators.
For strictly positive function $h_t$ and  $\pi_t$, define
\begin{align}
   R_{\pi}: \mathcal L_t \mapsto \bwd{\mathcal L}_t^\pi, \quad H_{h}: \mathcal L_t \mapsto \mathcal L_t^h,
\end{align}
Then, we have the following conclusions:

\begin{tcolorbox}[
    colback=gray!8,
    colframe=gray!40,
    boxrule=0.5pt,
    arc=4pt,
    left=1pt,
    right=1pt,
    top=2pt,
    bottom=2pt,
    breakable
]
\begin{theorem}[Algebraic identities for reversal and Doob transforms]\label{theorem:algebra}
Let $p,q,h,h'$ be strictly positive densities/functions with respect to the reference measure $\mu$.
Then, as maps on generators,
\begin{align}
    R_p \circ R_q
    &=
    H_{h:=p/q}.
\end{align}
In particular, taking $p=q$ gives
\begin{align}
    R_p \circ R_p
    &=
    H_{h:=1}
    =
    I,
\end{align}
where $I$ denotes the identity map on generators.

Using the associativity of composition,  the reversal and generalized Doob transforms satisfy
\begin{align}
    R_p \circ H_{h}
    =
    R_{p/h},\quad
    H_{h} \circ R_q
    =
    R_{qh},
   \quad
    H_h \circ H_{h'}
    =
    H_{hh'}.
\end{align}
\end{theorem}
\end{tcolorbox}
These identities can be abstracted by the following diagram.
We drop the subscripts for simplicity.
\[
\begin{tikzpicture}[
    line width=0.3pt,
    every node/.style={font=\normalsize},
    arr/.style={-{Stealth[length=3mm,width=1.3mm]}}
]

\begin{scope}[xshift=0cm]
    \coordinate (A) at (0,1.2);
    \coordinate (B) at (1.2,1.2);
    \coordinate (C) at (0,0);
    \coordinate (D) at (1.2,0);

    \draw[arr] (A) -- (B);
    \draw[arr] (A) -- (C);
    \draw[arr] (B) -- (D);
    \draw[arr] (C) -- (D);
    \draw[arr] (A) -- (D);

    \node[above] at ($(A)!0.5!(B)$) {$R$};
    \node[left]  at ($(A)!0.5!(C)$) {$H$};
    \node[right] at ($(B)!0.5!(D)$) {$R$};
    \node[below] at ($(C)!0.5!(D)$) {$H$};
    \node at ($(A)!0.5!(D)$) {$H$};
\end{scope}

\begin{scope}[xshift=5cm]
    \coordinate (A) at (0,1.2);
    \coordinate (B) at (1.2,1.2);
    \coordinate (C) at (0,0);
    \coordinate (D) at (1.2,0);

    \draw[arr] (A) -- (B);
    \draw[arr] (A) -- (C);
    \draw[arr] (B) -- (D);
    \draw[arr] (C) -- (D);
    \draw[arr] (A) -- (D);

    \node[above] at ($(A)!0.5!(B)$) {$R$};
    \node[left]  at ($(A)!0.5!(C)$) {$H$};
    \node[right] at ($(B)!0.5!(D)$) {$H$};
    \node[below] at ($(C)!0.5!(D)$) {$R$};
    \node at ($(A)!0.5!(D)$) {$R$};
\end{scope}

\end{tikzpicture}
\]
Here, the nodes represent generators, and the edges represent maps between generators.
Since we suppress the subscripts in the diagram, the same symbol denotes the same family of maps rather than the same specific map.
By composing maps, we obtain the following general conclusions.
\begin{tcolorbox}[
    colback=gray!8,
    colframe=gray!40,
    boxrule=0.5pt,
    arc=4pt,
    left=1pt,
    right=1pt,
    top=2pt,
    bottom=2pt,
    breakable
]
\begin{corollary}[Reduction of finite compositions]
Any finite composition of reversal maps $R$ and generalized Doob transforms $H$ reduces to either a single generalized Doob transform or a single reversal map.
More precisely, let
\begin{align}
    T
    =
    A_n\circ A_{n-1}\circ \cdots \circ A_1,
\end{align}
where each $A_i$ is either of the form $R_{p_i}$ or $H_{h_i}$.
Then $T$ can be written as
\begin{align}
    T
    =
    \begin{cases}
        H_g, & \text{if the number of reversal maps appearing in } T \text{ is even},\\
        R_g, & \text{if the number of reversal maps appearing in } T \text{ is odd},
    \end{cases}
\end{align}
for some strictly positive function $g$.
\end{corollary}

\end{tcolorbox}

\subsection{Group Theory of Reversals and Generalized Doob's $h$-Transforms}

The identities above show that the collection of reversal maps and generalized Doob transforms forms a group under composition.
Let
\begin{align}\label{eq:G}
    \mathcal G
    =
    \{H_h : h>0\}
    \cup
    \{R_p : p>0\},
\end{align}
Then $(\mathcal G, \circ)$ is a group, since:
\begin{itemize}[leftmargin=*]
    \item \textbf{Closure.}
    \Cref{theorem:algebra} shows that any composition of two elements of $\mathcal G$ is again in $\mathcal G$.

    \item \textbf{Associativity.}
    The operation is a composition of maps on generators, which is associative.

    \item \textbf{Identity.}
    The identity element is $H_1=I$.

    \item \textbf{Inverses.}
    Each element has an inverse:
    $
        H_h^{-1} = H_{1/h},  
        R_p^{-1} = R_p,
    $
    as $H_h\circ H_{1/h}=R_p\circ R_p=I$.
\end{itemize}
The generalized Doob transforms $\{H_h:h>0\}$ form a subgroup of $\mathcal G$, while the reversals $\{R_p:p>0\}$ do not.
Additionally, as proved in \Cref{app:group_iso_proof}, $\mathcal G$ is isomorphic to the
generalized dihedral group, where the generalized $h$-transform is analogous to rotation and reversal is analogous to reflection:
\begin{tcolorbox}[
    colback=gray!8,
    colframe=gray!40,
    boxrule=0.5pt,
    arc=4pt,
    left=1pt,
    right=1pt,
    top=2pt,
    bottom=2pt,
    breakable
]
\begin{proposition}[Generalized dihedral structure]\label{thm:group_iso}
Let $\mathcal F_+$ denote the multiplicative group of strictly positive functions, and let
\begin{align}
    A = \mathcal F_+/\mathbb R_{>0}
\end{align}
be the quotient by positive constants.
Then $\mathcal G$ defined in \Cref{eq:G} is isomorphic to the generalized dihedral group
\begin{align}
  \mathcal G  \cong\operatorname{Dih}(A)
    =
    A\rtimes \mathbb Z_2,
\end{align}
where the nontrivial element of $\mathbb Z_2$ acts on $A$ by inversion.
\end{proposition}
\end{tcolorbox}
The isomorphism lets us import properties from the generalized dihedral group directly into $\mathcal{G}$.

In summary, this section introduces a general structural view of infinitesimal reversal and generalized Doob's $h$-transform.
By choosing one of the reversal marginals to coincide with the energy path, the abstract identities above recover the generalized EJE and Crooks fluctuation theorem introduced in the previous section.
In the next section, we return to free energy estimation and develop learning objectives for generative models that make the generalized EJE practically effective.

\section{Learning the Transport Models}\label{sec:learn}

Propositions~\ref{prop:jarzynski} and~\ref{prop:crooks} provide us with two identities for free energy estimation.
We can either (1) choose an energy path $U_t$ and a forward process, and calculate work with \Cref{eq:general_work_cont}, or (2) choose the forward-backward process pair, and calculate the work with the forward-backward Radon-Nikodym derivative in \Cref{eq:crooks_eq}.
Similar to \citetalias{he2025feat}, we consider the second way to estimate the free energy with a forward-backward pair. 
This allows us to skip defining or learning the energy path $U_t$ explicitly, yielding a more robust and scalable estimator.

There are also two main paradigms for learning a forward-backward pair: (1) we can specify the energy path $U_t$ first, and learn the forward and backward generators so that the resulting paths follow this energy in optimality; (2) we first obtain samples drawn from the two endpoint distributions, define an interpolating path between samples, and then learn the forward and backward generators that match this interpolant, with the corresponding energy defined only implicitly.
While the former allows us to define more thermodynamic-related energy paths, the path may suffer from teleportation \citep{mate2023learning}, and learning can be more challenging and expensive.
The latter, on the other hand, is generally more scalable and compatible with recent advances in diffusion and flow models \citep{song2020score,lipman2022flow,benton2024denoising,holderrieth2024generator,albergo2025stochastic}.
Therefore, in the following, we will focus our discussion on the second option.

We start by defining the interpolant as a conditional distribution $x_t\sim p_t(x_t|x_0, x_1)$ given a pair of samples from a coupling of two boundary distributions $(x_0, x_1)\sim \pi(x_0, x_1)$.
Examples include 
\begin{align}\label{eq:int1}
    &{\text{Average: } }x_t = \alpha_t x_0 + \beta_t x_1 + \gamma_t \epsilon, \quad \epsilon \perp (x_0, x_1),\quad \epsilon \sim \mathcal{N}(0, I) \\&{\text{Mixture: } }x_t \sim  \alpha_t \delta_{x_0}(x_t) + \beta_t \delta_{x_1}(x_t) ,\quad \text{or}\quad x_t \sim  \alpha_t \delta_{x_0}(x_t) + \beta_t \delta_{x_1}(x_t) + \gamma_t \delta_{M}(x_t)\label{eq:int2}
\end{align}
We then seek to learn a forward process and a backward process such that, at optimum, (i) both generate the marginal density \(p_t(x_t)=\mathbb{E}_{(x_0,x_1)\sim\pi}[p_t(x_t\mid x_0,x_1)]\) for every \(t\in[0,1]\), and (ii) the two processes are time-reversal.
Note that (ii) is a stronger requirement than (i).

\begingroup
\setlength{\fboxsep}{3pt}%
\noindent\makebox[\linewidth][c]{%
  \colorbox{gray!20}{%
    \parbox{0.94\linewidth}{%
      \raggedright\normalfont
      \textbf{Step 1:} \textit{Define a coupling $\pi(x_0, x_1)$ and an interpolant path
      $x_t \sim p_t(x_t | x_0, x_1)$, and learn a pair of time-reversal dynamics with marginal $p_t(x_t)=\mathbb{E}_{(x_0,x_1)\sim\pi}[p_t(x_t| x_0,x_1)]$.}%
    }%
  }%
}%
\endgroup

We discuss two frameworks to achieve this in the following.
In the first method, we directly learn the forward and backward processes, while in the second method, we learn one direction only and learn the marginal statistics to reverse it.
We now explain these two frameworks respectively:

\paragraph{\faHandPointRight[regular] Learning forward and backward processes}

We extend generator matching \citep{holderrieth2024generator} to the bridge setting, where we learn the transport process directly between two endpoint distributions.
Specifically, we first choose a coupling $\pi(x_0,x_1)$ between the two endpoint distributions, and then define conditional forward and backward generators,
$\mathcal{L}_{F,t}^{x_0,x_1}$ and $\mathcal{L}_{B,t}^{x_0,x_1}$,
associated with the interpolant path introduced in \Cref{eq:int1,eq:int2}. That is,
\begin{align}
    \partial_t p_t(x_t |x_0,x_1)
    = \bigl(\mathcal{L}_{F,t}^{x_0,x_1}\bigr)^\dagger p_t(x_t | x_0,x_1), \quad 
    \partial_t p_t(x_t |x_0,x_1)
    = -\bigl(\mathcal{L}_{B,t}^{x_0,x_1}\bigr)^\dagger p_t(x_t |x_0,x_1).
\end{align}
We then aim to learn the marginal forward and backward generators by averaging the conditional generators over the posterior of the endpoints given the current state:
\begin{align}\label{eq:marginal_generator}
    \mathcal{L}_t^F f(x_t)
   = \mathbb{E}_{x_0, x_1|x_t}\!\left[
        \mathcal{L}_{F,t}^{x_0,x_1} f(x_t)
    \right], \quad 
    \mathcal{L}_t^B f(x_t)
   = \mathbb{E}_{x_0, x_1|x_t}\!\left[
        \mathcal{L}_{B,t}^{x_0,x_1} f(x_t)
    \right].
\end{align}
One can verify that the marginal density
\(
p_t(x_t)=\mathbb{E}_{(x_0,x_1)\sim\pi}[p_t(x_t | x_0,x_1)]
\)
evolves according to 
\begin{align}\label{eq:marginal}
    \partial_t p_t(x_t)
    = (\mathcal{L}_t^F)^\dagger p_t(x_t), \quad 
    \partial_t p_t(x_t)
    = -(\mathcal{L}_t^B)^\dagger p_t(x_t).
\end{align}
Following \citet{holderrieth2024generator}, we can parameterize our network for the marginal generator and define a Bregman divergence $D$ as the objective.

However, this workflow has an important caveat. \emph{Even if two generators both satisfy \Cref{eq:marginal}, so that they induce the same marginal density at every time \(t\), this does NOT imply that they are time reversals of one another}.  
Fortunately, the following proposition shows that if we define the conditional generators to be time reversals of each other, then the corresponding marginal generators are also time reversals, which we prove in \Cref{app:proof_marginal_time_rev}. 
This saves us from having to verify the marginal construction directly: it suffices to design the conditional generators to be time-reversal pairs, which is fully tractable.
\begin{tcolorbox}[
    colback=gray!8,
    colframe=gray!40,
    boxrule=0.5pt,
    arc=4pt,
    left=1pt,
    right=1pt,
    top=2pt,
    bottom=2pt,
    breakable
]
\begin{proposition}[Marginalization maintains time-reversal]\label{prop:marginal_reversal}
Let \(p_t(x_t|x_0,x_1)\) be a conditional interpolant, and let
\(\mathcal{L}_{F,t}^{x_0,x_1}\) and \(\mathcal{L}_{B,t}^{x_0,x_1}\)
be the corresponding forward and backward conditional generators. Assume that, for every \((x_0,x_1)\), they are time reversals of each other, i.e.
\begin{align}
    \resizebox{0.93\linewidth}{!}{$
    \mathcal{L}_{B,t}^{x_0,x_1} f
    =
    \frac{1}{p_t(x_t|x_0,x_1)}
    (\mathcal{L}_{F,t}^{x_0,x_1})^\dagger
    \!\left(
        p_t(x_t|x_0,x_1) f
    \right)
    -
    \frac{1}{p_t(x_t| x_0,x_1)}
    (\mathcal{L}_{F,t}^{x_0,x_1})^\dagger
    p_t(x_t|x_0,x_1)\, f
    $}
\end{align}
Then \(\mathcal{L}_{F,t}\) and \(\mathcal{L}_{B,t}\) defined by \Cref{eq:marginal_generator} are also time-reversal with respect to \(p_t(x_t)\), namely
\begin{align}\label{eq:time-reversal}
    \mathcal{L}_{B,t} f
    =
 {p_t}^{-1}
    \mathcal{L}_{F,t}^\dagger (p_t f)
    -
   {p_t}^{-1}
    \mathcal{L}_{F,t}^\dagger p_t \, f .
\end{align}
\end{proposition}
\end{tcolorbox}
This proposition leads to the following learning strategy:

\begingroup
\setlength{\fboxsep}{3pt}%
\noindent\makebox[\linewidth][c]{%
  \colorbox{gray!20}{%
    \parbox{0.94\linewidth}{%
      \raggedright\normalfont
      \textbf{Step 2(A):} \textit{Define a pair of forward and backward conditional generators which are time-reversal with
      $p_t(x_t | x_0, x_1)$, and learn to marginalize them.}%
    }%
  }%
}%
\endgroup

\vspace{-10pt}
\paragraph{\faHandPointRight[regular] Learning one direction plus its marginal statistics}
Another way to learn the transports is to learn a forward process, plus some statistics that allows we to obtain the backward process from the forward process.
From \Cref{eq:time-reversal}, we can see that we only need to learn the marginal $p_t$\footnote{Here, the marginal does not only mean the marginal energy or density. Instead, it can be any sufficient statistics to reverse a process. For example, for Ito's SDE, we only need the score; and for discrete CTMC, we only need the concrete score, i.e., density ratio between two states.}. 
\par
Since we can directly sample from $p_t$, one could in principle use implicit score matching \citep[ISM;][]{hyvarinen2005estimation}; see also \citet{benton2024denoising} for a generalization to arbitrary state spaces. However, ISM typically requires expensive computations, such as evaluating divergences or summing/integrating over the entire state space. A more practical alternative is denoising score matching \citep{vincent2011connection,benton2024denoising}, adapted to our interpolant setting, as discribed next.

Assume we are given an interpolant such that the conditional distribution $p(x_t | x_0, x_1)$ is tractable. For example, one may first construct an interpolating state $I_t$, such as a linear interpolation $I_t = \alpha_t x_0 + \beta_t x_1$ or a mixture $I_t \sim \alpha_t \delta_{x_0} + \beta_t \delta_{x_1}$, and then apply a tractable corruption process such as Gaussian corruption or random masking, described by $\partial_\tau q_\tau(\cdot |I) = \mathcal{A} q_\tau(\cdot| I)$ with $q_0(\cdot|I) = \delta_I$. 
Note that the original interpolant $p(x_t  |x_0, x_1)$ can then be written as $p(x_t |x_0, x_1) = \int q_{\tau = f(t)}(x_t |I)\, p_t^I(I \mid x_0, x_1)\,\mathrm{d}I$, where $p_t^I$ denotes the law of the uncorrupted interpolant and $f(t)$ controls the amount of corruption at each time $t$, with $f(0)=f(1)=0$.
Then, the DSM objective is given by
\begin{align}
 \mathcal{J}_{\text{DSM}}(\varphi_t) =  \mathbb{E}_{x_t, I_t} \left[
    \frac{\mathcal{A} ( q_{\tau}(x_t|I_t) / \varphi_t(x_t) )}{q_{\tau}(x_t|I_t) / \varphi_t(x_t)} - \mathcal{A} \log (
q_{\tau}(x_t|I_t) / \varphi_t(x_t)
    )\Big|_{\tau=f(t)}
    \right] 
\end{align}
and $\arg\min_{\varphi_t}   \mathcal{J}_{\text{DSM}} = p_t$, which leads to the following learning strategy:

\begingroup
\setlength{\fboxsep}{3pt}%
\noindent\makebox[\linewidth][c]{%
  \colorbox{gray!20}{%
    \parbox{0.94\linewidth}{%
      \raggedright\normalfont
      \textbf{Step 2(B):} \textit{Learn a forward generator and the marginal statistics to obtain the time-reversal.}%
    }%
  }%
}%
\endgroup

We note that, although both approaches are valid, their practical difficulty can vary substantially across settings. 
For example, in continuous state spaces, both strategies can be feasible in practice, while directly learning the score is a more natural choice \citep{albergo2025stochastic}. 
In contrast, in discrete state spaces, the second strategy of learning the marginal (in the form of the concrete score) can be much more expensive, since constructing the backward generator from the forward and the marginal concrete score generally requires summing over all states unless additional structure is imposed, e.g., local equivariance \citep{holderrieth2025leaps}.

\section{Inference and Estimation} \label{sec:inf_est_deltaF}

After training the forward and backward transports, we sample trajectories and evaluate the generalized work \(\widetilde W\) using the forward-backward Radon-Nikodym derivative in \Cref{eq:crooks_eq}.
As we have two dynamics, we can sample from each of them and obtain two estimators for the free energy:
with \(X^{F,(i)}_{[0,1]}\sim \fwd P\) and \(X^{B,(j)}_{[0,1]}\sim \bwd Q\), the free energy can be estimated by the following estimators
\begin{align}\label{eq:fwd_bwd_estimator}
    \widehat{\Delta F}_{F}
    =
    -\log
    \frac1{N}
    \sum_{i=1}^{N}
    \exp\!\left(-\widetilde W(X^{F,(i)}_{[0,1]})\right),
   \quad
    \widehat{\Delta F}_{B}
    =
    \log
    \frac1{N}
    \sum_{j=1}^{N}
    \exp\!\left(\widetilde W(X^{B,(j)}_{[0,1]})\right).
\end{align}
Following the same principle by \citetalias{he2025feat}, we can combine both directions for a Bennett Acceptance Ratio (BAR)-style estimator \citep{BENNETT1976245}
\begin{align}\label{eq:bar}
    \widehat{\Delta F}_{\mathrm{BAR}}(C)
    =
    C+
    \log
    \frac{
        \sum_{j=1}^{N}
        \phi\!\left(-\widetilde W(X^{B,(j)}_{[0,1]})+C\right)
    }{
        \sum_{i=1}^{N}
        \phi\!\left(\widetilde W(X^{F,(i)}_{[0,1]})-C\right)
    },
    \qquad
    \phi(z)=\frac{1}{1+\exp(z)} .
\end{align}
When $C=\Delta F$, this estimator minimizes the MSE error, thus often referred as \textit{the minimum variance estimator}.
In practice, we initialize \(C\) by averaging the two one-sided estimates and iterate, setting $C$ to be the current estimate of $\Delta F$, until convergence. 
This leads to the final step of the pipeline:

\begingroup
\setlength{\fboxsep}{3pt}%
\noindent\makebox[\linewidth][c]{%
  \colorbox{gray!20}{%
    \parbox{0.94\linewidth}{%
      \raggedright\normalfont
      \textbf{Step 3:} \textit{Calculate generalized work for both forward and backward trajectories, and obtain the minimal variance estimator using \Cref{eq:fwd_bwd_estimator} or \Cref{eq:bar}.}%
    }%
  }%
}%
\endgroup

\section{Examples and Experiments}\label{sec:exa_exp}

In this section, we make the pipeline described above concrete. We provide examples in Euclidean space, discrete space, and autoregressive models.
We also illustrate that this framework is compatible with multimodal space and can be applied to transport with momentum. 

In this section, we present the main conclusions and experimental results.
In \Cref{app:everything_here}, we explain in detail how the general conclusions and learning framework introduced in \Cref{sec:EJE_crook,sec:learn} lead to the dynamics and the corresponding learning objectives.
In \Cref{sec:hyperparameters}, we describe the systems we work on, the experimental settings we choose, and the hyperparameters.

\paragraph{Euclidean State Space}
Proosition~\ref{prop:crooks} in the form of SDE recovers the escorted Langevin dynamics:
\begin{align}
\text{d}X_t &= v_t(X_t) \text{d}t +\sigma_t^2 s_t(X_t) \text{d}t + \sqrt{2}\sigma_t \fwd{\text{d}W_t}, \qquad \mathcal{L}_{F,t} f = (v_t + \sigma_t^2 s_t) \cdot \nabla f + \sigma_t^2 \Delta f \\
\text{d}X_t &= v_t(X_t) \text{d}t - \sigma_t^2s_t(X_t) \text{d}t + \sqrt{2}\sigma_t \bwd{\text{d}W_t}, \qquad \mathcal{L}_{B,t} f = (-v_t + \sigma_t^2 s_t) \cdot \nabla f + \sigma_t^2 \Delta f
\end{align} 
where the score $s = \nabla \log \pi_t$ in \Cref{prop:jarzynski}.
Choosing interpolant $x_t \sim \alpha_t x_0 + \beta_t x_1+\gamma_t \epsilon$, and learn the score and vector field for this path, our framework leads to the standard stochastic interpolant \citep{albergo2025stochastic} and FEAT \citepalias{he2025feat}.
As this is the setup for FEAT and has been verified by \citetalias{he2025feat} and \citet{schebek2025assessing}, we do not repeat their experiments here.

Below, we consider cases beyond those studied in previous works.
Since our method can be viewed as a generalization of FEAT~\citepalias{he2025feat}, we keep the same naming convention, prefixing FEAT with the corresponding modality to specify the setting.

\begin{table}[t]
\centering
\begin{minipage}[t]{0.52\textwidth}
\centering
\caption{Ising models using CTMC-FEAT. 
The standard deviation is obtained by training three networks and running estimators for each.}
\label{tab:estimated_values}
\scriptsize
\setlength{\tabcolsep}{4pt}
\renewcommand{\arraystretch}{0.88}
\resizebox{\linewidth}{!}{%
\begin{tabular}{llcc}
\toprule
\textbf{Dim} $D$ & \textbf{Method} & \makecell{\textbf{Estimates}\\($\Delta F / D \times 10^3$)}& \makecell{\textbf{Reference value}\\($\Delta F / D \times 10^3$)} \\
\midrule
\multirow{3}{*}{15$\times$15 (225)}
    & DFM         & -145.44 {\tiny $\pm 0.44$} & \multirow{3}{*}{-145.41} \\
    & GM-no mask  & -145.41 {\tiny $\pm 0.33$} &  \\
    & GM-w.\ mask & -145.14 {\tiny $\pm 0.57$} &  \\
\midrule
\multirow{3}{*}{25$\times$25 (625)}
    & DFM         & -145.30 {\tiny $\pm 0.28$} & \multirow{3}{*}{-144.88} \\
    & GM-no mask  & -144.45 {\tiny $\pm 0.91$} &  \\
    & GM-w.\ mask & -144.74 {\tiny $\pm 1.92$} &  \\
\midrule
\multirow{3}{*}{32$\times$32 (1024)}
    & DFM         & -143.89 {\tiny $\pm 0.60$} & \multirow{3}{*}{-144.84} \\
    & GM-no mask  & -143.93 {\tiny $\pm 0.75$} &  \\
    & GM-w.\ mask & -144.62 {\tiny $\pm 1.93$} &  \\
\bottomrule
\end{tabular}%
}
\end{minipage}
\hfill
\begin{minipage}[t]{0.45\textwidth}
\centering
\caption{Ising models w. AR-FEAT. The standard deviation is obtained by training three networks and running estimators for each.}
\label{tab:ar_feat_results}
\scriptsize
\setlength{\tabcolsep}{4pt}
\renewcommand{\arraystretch}{0.88}
\resizebox{\linewidth}{!}{%
\begin{tabular}{llcc}
\toprule
\textbf{Dim} $D$ & \textbf{$\beta$} & \makecell{\textbf{AR FEAT}\\($\Delta F / D \times 10^3$)} & \makecell{\textbf{Reference value}\\($\Delta F / D \times 10^3$)} \\
\midrule
\multirow{3}{*}{15$\times$15 (225)}
    & $0.2$ & -734.52 $\pm$  0.01 &  -734.53 \\
    & $0.4$ & -879.90 $\pm$ 0.03 &  -879.94 \\
    & $0.6$ & -1213.22 $\pm$  0.01 & -1213.21 \\
\midrule
\multirow{3}{*}{25$\times$25 (625) }
    & $0.2$ & -734.51 $\pm$  0.01 &  -734.53\\
    & $0.4$ & -879.43 $\pm$  0.02 &  -879.41 \\
    & $0.6$ & -1211.25 $\pm$  0.01 &  -1211.24 \\
\midrule
\multirow{3}{*}{32$\times$32 (1024)}
    & $0.2$ & -734.52  $\pm$ 0.04 & -734.53 \\
    & $0.4$ & -879.40  $\pm$ 0.04 & -879.37\\
    & $0.6$ & -1210.73 $\pm$ 0.01 & -1210.81 \\
\bottomrule
\end{tabular}%
}
\end{minipage}\vspace{-10pt}
\end{table}

\paragraph{Discrete State Space}  Proosition~\ref{prop:crooks} gives us the following forward and backward process:
\begin{align}
p^F_t(x, y) = \delta_x(y) + R^F_t(x,y)\Delta t + o(\Delta t), \qquad \mathcal{L}_{F,t} f(x) = \sum_y R_t^F(x, y)f(y)\\
p^B_t(x, y) = \delta_x(y) + R^B_t(x,y)\Delta t+ o(\Delta t), \qquad \mathcal{L}_{B,t} f(x) = \sum_y R_t^B(x, y)f(y)
\end{align}
where the forward and backward rate matrices are related by the concrete score of $\pi_t$.

Choosing a mixture-style interpolant $x_t \!\sim\! \alpha_t \delta_{x_0}(x_t) \!+\! \beta_t \delta_{x_1}(x_t) \!+\! \gamma_t \delta_M (x_t)$, we can learn both the forward and backward with generator matching.
As there exist many different paths that could realize the same interpolant, we have a big design space, as we discuss in \Cref{sec:appendix_ctmc_feat_theory}.
Here, we consider three variants for learning the generator: (1) discrete flow matching \citep{gat2024discrete} without masking (i.e., $\gamma_t = 0$), which we discussed in \hyperref[par:ctmc_case1]{CTMC Case 1} with loss described in \Cref{eq:appendix_case1_celoss_1,eq:appendix_case1_celoss_2}; (2) Generator Matching without masking, which we discussed in \hyperref[par:ctmc_case1]{CTMC Case 1} with loss described in \Cref{eq:case1_gm,eq:case1_gm2}; and (3) Generator Matching with masking (i.e., $\gamma_t \neq 0, t\in(0,1)$), which we discussed in \hyperref[par:ctmc_case2]{CTMC Case 2} with loss described in \Cref{eq:appendix_case2_GMF,eq:appendix_case2_GMB}.

We evaluate these three settings by estimating free energy differences between Ising models with inverse temperature $\beta=0.2$ and $\beta=0.4$ across different lattice dimensions.
For fair comparisons, all methods employ similar architectures with different hyperparameters (see \cref{sec:appendix_disc_feat_ising} for details).
Reference values are computed analytically using the method in \citet{ferdinand1969bounded}.
Results in \cref{tab:estimated_values} demonstrate that all three methods achieve a relatively reliable estimation of the free energy difference, with the no-mask choices slightly more robust.
In \Cref{tab:discrete_feat_0.2_0.6_estimated_value}, we further consider more challenging transport settings with $\beta=0.2$ and $\beta=0.6$, where the discrepancy between the masking and no-mask choices becomes more pronounced.

\paragraph{Autoregressive Transport}
As discussed in the discrete-time Crooks in Proposition~\ref{prop:discretetime} and the importance-sampling perspective in \Cref{sec:is}, AR models can also be used for free-energy estimation.
We here evaluate AR transport on Ising models across different dimensions and inverse temperatures.
In contrast to the transports between two target distributions, an AR model transports an empty sequence to a full sequence.
Consequently, it estimates the absolute free energy rather than the free energy difference.
As shown in \Cref{tab:ar_feat_results}, AR transport achieves very strong performance.
We hypothesize that this is because AR training provides a particularly strong learning signal, allowing the model to more easily capture dependencies across positions in the target distribution.

Besides these standard settings, we further demonstrate the flexibility of our framework. 
In particular, we show that it is compatible with multi-modal state spaces and can also be extended to momentum-augmented dynamics. 
These additional validations broaden the applicability of our framework.

 \begin{figure}[t]
    \centering
 \begin{minipage}[t]{0.38\textwidth}
\centering
\caption{Estimated and reference free energy difference for Ising fluid with number of particles $N=55$ and different attraction parameters $R$.
The standard deviation is computed by training the network once but running the estimator 3 times using 2,000 samples.}
\label{tab:ising_fluid_res}
\vspace{1pt}
\resizebox{\linewidth}{!}{%
\begin{tabular}{lc c c}
\toprule
$N$
& $R$ 
& \begin{tabular}{c}Estimated\\values\end{tabular}
& \begin{tabular}{c}Reference\\values\end{tabular} \\
\midrule
\multirow{3}{*}{55}
& $0.5 \leftrightarrow 1.0$ & 319.72 $\pm$ 0.23 & 319.84 \\
& $0.5 \leftrightarrow 1.5$ & 412.64 $\pm$ 0.24 &  412.82 \\
& $0.5 \leftrightarrow 2.0$ & 455.16  $\pm$ 0.16 & 455.36 \\
\bottomrule
\end{tabular}%
}
    \end{minipage}
    \begin{minipage}[t]{0.255\textwidth}
        \centering
        \vspace{-5pt}
        \includegraphics[width=\linewidth]{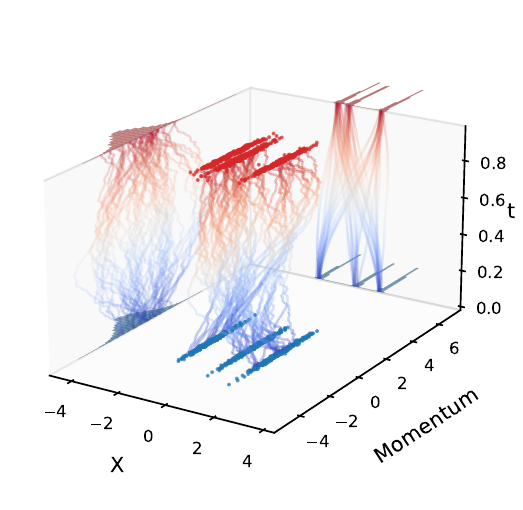}\vspace{0pt}
        \captionof{figure}{Illustration of transport with momentum between two 1D GMMs.  }
        \label{fig:underdamped_vis_1D}
    \end{minipage}
    \begin{minipage}[t]{0.35\textwidth}
        \centering
        \vspace{0pt}
        \includegraphics[width=\linewidth]{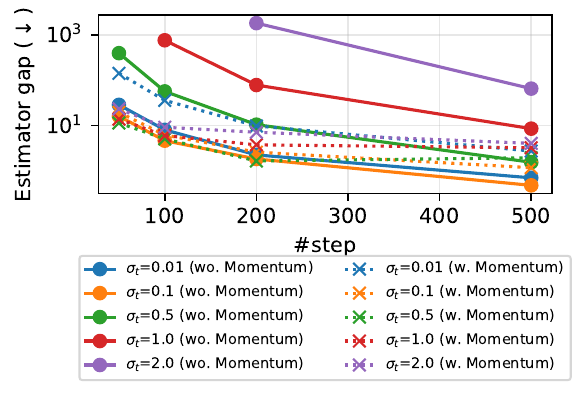}
        \captionof{figure}{Comparing standard SDE  and transport with momentum for free energy of 40D GMMs. }
        \label{fig:underdamped_vis_40D}\vspace{-20pt}
    \end{minipage}
    \vspace{-5pt}
\end{figure}

\paragraph{Multi-modal State Space} 
Another interesting and important case for free-energy estimation is when the system contains more than one modality.
Here, we demonstrate the applicability of our framework in this regime.
Specifically, we consider a component-wise construction that combines Langevin dynamics for the continuous variables with a CTMC for the discrete variables.
The resulting Markov process has a generator as the sum of the two component generators, as detailed in \Cref{appendix:composition}.

We consider an Ising fluid system \citep{omelyan2004ising} with different attraction parameters $R$. The continuous component represents the particles' coordinates, while the discrete part represents the spin.
For simplicity, we learn the transport between a common prior, consisting of a centered Gaussian and a fully masked spin, to each system, and obtain the free energy difference by subtracting the value for each system.
See \Cref{app_details_ising_fluid} for more details on the system, and \Cref{fig:IsingFluid_55,fig:spin-distance-lambda} for visualization of the systems.
We evaluate $N=55$ in \Cref{tab:ising_fluid_res}.
The reference values are obtained with MBAR.
We can see that our method provides a reliable estimate.

\paragraph{Transport with Momentum} 
Besides the flexibility of modality, our framework also supports different transport.
Here, we consider an ``underdamped" Langevin as the Markov process.
In this class of transport, we augment the state space with a momentum space and inject stochasticity into the momentum.
In \Cref{fig:underdamped_vis_1D}, we illustrate this transport between two one-dimensional target distributions.
From Proposition \ref{prop:crooks}, we obtain the forward and backward process in the following form:
\begin{align}
&\text{d}X_t = Y_t \text{d}t, \qquad
\text{d}Y_t = v_t(X_t, Y_t) \text{d}t +\sigma_t^2 s_t(X_t, Y_t) \text{d}t + \sqrt{2}\sigma_t \fwd{\text{d}W_t}\\
&\text{d}X_t = Y_t\bwd{\text{d}t},\qquad\text{d}Y_t = v_t(X_t, Y_t) \text{d}t - \sigma_t^2s_t(X_t, Y_t) \text{d}t + \sqrt{2}\sigma_t \bwd{\text{d}W_t} 
\end{align}
To learn this transport, as we describe in 
\Cref{app:learn_underdamp}, we take a cubic Hermite interpolant:
\begin{align}
    x_t
    =
    h_{00}(t)x_0
    +
    h_{10}(t)y_0
    +
    h_{01}(t)x_1
    +
    h_{11}(t)y_1 , \quad  y_t=\dot x_t.
\end{align}
with $ 
    h_{00}(t)=2t^3-3t^2+1,  
    h_{10}(t)=t^3-2t^2+t,  
    h_{01}(t)=-2t^3+3t^2, 
    h_{11}(t)=t^3-t^2 $ and objective defined in \Cref{loss_underdamp1,loss_underdamp2}.

To verify the effectiveness of this method, we evaluate it for the free energy difference between two 40-dimensional Gaussian Mixture models and compare standard SDE transport in FEAT.
Since both standard SDE transport and momentum-augmented transport can be run with different choices of \(\sigma_t>0\) using the same learned network, we study their performance across $\sigma_t$ levels and different numbers of discretization steps. 
As both yield good estimates of the free energy, we report the gap between the forward and backward estimators in \Cref{eq:fwd_bwd_estimator}.
As shown,  both achieve comparable performance, while momentum-augmented transport is slightly more robust to larger noises.

\vspace{-5pt}
\section{Conclusion}
\vspace{-5pt}
In this paper, we generalize the escorted Jarzynski equality and the Crooks fluctuation theorem to arbitrary state spaces, and develop a general framework for learning transport models in this setting.
Our framework enables the use of generative models for efficient free energy estimation, significantly broadening the class of transport processes for nonequilibrium free energy computation.

\paragraph{Limitations and future works. } While efficient and broadly applicable, our framework is still limited by the expressiveness the chosen network parameterization and may require domain-specific tuning in practice. 
In addition, as discussed in \Cref{app:heavy_tailed}, heavy-tailed target distributions (e.g., Cauchy) can pose difficulties for finite-time transport models. 
Although such heavy-tailed behavior is less typical for Boltzmann distributions, it may arise in Bayesian inference settings.

\section*{Acknowledgment}
JH acknowledges support from the University
of the Cambridge Harding Distinguished Postgraduate Scholar Programme.
ZO is supported by the Lee Family Scholarship and would like to thank Alexander Ganose for insightful discussions on potential applications in materials design.
JMHL acknowledges funding from AI Hub in Generative Models, under grant EP/Y028805/1.

\bibliographystyle{abbrvnat}
\bibliography{ref}

\clearpage
\appendix
\vbox{%
    \hsize\textwidth
    \linewidth\hsize
    \vskip 0.1in
    \hrule height 4pt
  \vskip 0.25in
  \vskip -\parskip%
    \centering
    {\LARGE\bf {Appendix:\\
    Free energy Estimation on Any State Space} \par}
     \vskip 0.29in
  \vskip -\parskip
  \hrule height 1pt
  \vskip 0.09in%
  }

\startcontents[sections]
\printcontents[sections]{l}{1}{\setcounter{tocdepth}{3}}
\newpage

\section{Related Works on Neural Free Energy Estimation}

Probabilistic methods in machine learning have been increasingly used to accelerate free energy estimation. In this section, we review several major directions.

Targeted FEP seeks to learn an invertible map between two states so as to reduce the variance of free energy estimates. Both discrete-time normalizing flows and continuous-time flow matching methods have been used for this purpose, relying on the change-of-variables formula to account for the work associated with the transformation \citep{wirnsberger2020targeted,zhao2023bounding,deepbar}.\looseness=-1

Thermodynamic integration (TI) estimates free energy differences by integrating along an interpolation between two boundary states in the quasi-static limit. NeuralTI \citep{mate2024neural,mate2024solvation} parameterizes this interpolation using diffusion models or stochastic interpolants, while simultaneously learning the energy associated with each intermediate state.

The main class of methods considered in this paper is based on finite-time averaging, or Jarzynski equalities, which use continuous-time non-equilibrium processes to bridge the two boundary states. Within this class, one line of work fixes the non-equilibrium process and seeks to reduce dissipation along the process \citep{doucet2022score,rosa2025nonadiabatic}. A complementary line of work learns an additional escort to reduce dissipation not only along the prefixed non-equilibrium process, but also across the intermediate states during the transformation. Most approaches construct this transformation along an annealing path between the energies of the two boundary states \citep{vargas2023transport,zhong2024time,albergo2024nets}. 
FEAT constructs the transformation through a stochastic interpolant between samples from the two boundary states \citepalias{he2025feat}.
\citet{zhang2025accelerated} combine parallel tempering on the path with escorted Jarzynski for free energy estimation.

When estimating the relative free energy between two equilibrium states from samples drawn from each state, the Bennett acceptance ratio (BAR) \citep{BENNETT1976245} provides an estimator, minimizing the mean-squared error. 
\Cref{eq:bar} extends this principle to non-equilibrium path ensembles, , which was later developed by \citet{shirts2003equilibrium,PhysRevE.80.031111,minh2009optimal,vaikuntanathan2011escorted} and recently \citetalias{he2025feat}. 
For settings involving more than two states, the multistate Bennett acceptance ratio (MBAR) \citep{shirts2008statistically} further generalizes BAR by combining samples from all states into a statistically efficient estimate of their relative free energies.

Recently, \citet{guo2025complexity} conducted complexity analysis for the standard Jarzynski equality and anneal importance sampling, providing theoretical guidance. 
A future direction of our work is to conduct a similar analysis for our generalized formulation.

\section{Generalization of Free Energy estimation from Importance Sampling}
\label{sec:is}

We expand the general perspective for estimating free energy difference between two systems in the perspective of importance sampling from \Cref{eq:general_fe}. 
\begin{equation}
\Delta F
= -\log (Z_Q/Z_P)
= -\log \mathbb E_{X\sim P}
\left[
{\mathrm d\tilde Q}/{\mathrm d\tilde P}(X)
\right].
\end{equation}
Free energy perturbation (FEP) directly reweights the proposal samples with the target density
\begin{tcolorbox}[
  colback=orange!5,
  colframe=orange!60!black,
  boxrule=0.2pt,
  arc=2mm,
  left=1mm,right=2mm,top=1mm,bottom=1mm
]
\begin{align}
&\text{FEP:}\qquad \mathrm d \tilde{Q} /\mathrm d x= e^{-U_B}, \mathrm d \tilde{P}/\mathrm d x = e^{-U_A}, \mathrm d P/\mathrm d x = \pi_A,\\
&\Delta F
=
-\log \mathbb E_{X\sim \pi_A}
\left[
\exp\bigl(-U_B(X)+U_A(X)\bigr)
\right]
\end{align}
\end{tcolorbox}

Targeted FEP learns an (invertible) transport map $T$ from the proposal distribution to the target distribution, then reweight the transported samples under the target density while taking into account the work induced by the transport
\begin{tcolorbox}[
  colback=orange!5,
  colframe=orange!60!black,
  boxrule=0.2pt,
  arc=2mm,
  left=1mm,right=2mm,top=1mm,bottom=1mm
]
\begin{align}
&\text{Targeted FEP:}\qquad \mathrm d \tilde{Q}/\mathrm d x = e^{-U_B}, \mathrm d \tilde{P}/ \mathrm d x = e^{-U_A}|\det \nabla T|^{-1}, \mathrm d P/\mathrm d  x = \pi_A \\
&\Delta F
=
-\log \mathbb E_{X\sim \pi_A}
\left[
\exp\bigl(-U_B(T(X))+U_A(X)\bigr)
\left|\det \nabla T(X)\right|
\right]
\end{align}
\end{tcolorbox}
Jarzynski equality and escorted Jarzynski equality augment the proposal with a path measure induced by a finite-time stochastic process, then reweight path samples following the weights determined by the Crooks as in \Cref{eq:crooks}

\begin{tcolorbox}[
  colback=orange!5,
  colframe=orange!60!black,
  boxrule=0.2pt,
  arc=2mm,
  left=1mm,right=2mm,top=1mm,bottom=1mm
]
\begin{align}
\text{JE and EJE:}\qquad 
 \mathrm d\tilde{Q} &= e^{-U_B(X_1)}\text{d}\bwd{Q}(\cdot|X_1), \;\; \mathrm d \tilde{P} = e^{-U_A(X_0)}\text{d}\fwd{P}(\cdot|X_0), \;\; P = \overrightarrow{P}\\
\Delta F&
=
-\log \mathbb E_{X_{[0,1]}\sim \overrightarrow{P}}
\left[
e^{-U_B(X_1)}\text{d}\bwd{Q}(\cdot|X_1)/e^{-U_A(X_0)}\text{d}\fwd{P}(\cdot|X_0))
\right]
\end{align}
\end{tcolorbox}
while the key difference between JE and EJE is that EJE uses a different stochastic process with \textit{escort} for the transport.

Now we consider autoregressive models (AR) as transport. As AR begins with all empty tokens, we consider the absolute free energy case or the relative free energy between an unnormalized density and a normalized density. 

We define the vocabulary set $\mathcal{V}$ and extended vocabulary set $\bar{\mathcal{V}} := \mathcal{V} \, \cup \{\varnothing\}$ including the empty token. Now we define an extended state space $\mathcal{X}^N := \bar{\mathcal{V}}^N$ with maximum of $N$ tokens.  The autoregressive models factorize as 
\begin{equation}
p(X_{N}) = p(X_0)\prod_{i=1}^N p(X_i|X_{i-1})
\end{equation}
where each state grows from previous state by filling empty tokens from the vocabulary set, i.e. $X_i \in \mathcal{V}^i \times \{\varnothing\}^{N-i} \subset \bar{\mathcal{V}}^N$.

\begin{tcolorbox}[
  colback=orange!5,
  colframe=orange!60!black,
  boxrule=0.2pt,
  arc=2mm,
  left=1mm,right=2mm,top=1mm,bottom=1mm
]
\begin{align}
&\text{AR transport:} \;\; \tilde Q= e^{-U_B(X_N)}\prod_{i=1}^N q(X_{i-1}|X_i),\; \tilde P = p(X_0)\prod_{i=1}^N p(X_i|X_{i-1}) ,\; P = p
\\
& \Delta F = -\log Z_B
=
-\log \mathbb E_{X_N\sim p}
\left[
\frac{\exp(-U_B(X_{N}))}{p(X_0)}\frac{\prod_{i=1}^N q(X_{i-1}|X_i)}{\prod_{i=1}^N p(X_i|X_{i-1})}\right]
\end{align}
\end{tcolorbox}
Now, we need to consider the choice of the ``backward kernel'', i.e. $q(X_{i-1}|X_{i})$, for AR. It turns out the optimal ``backward kernel'' is the deterministic deletion kernel, $D_i(x_{0:i},\varnothing_{i+1:N}) = (x_{0:i-1}, \varnothing_{i:N})$.
\begin{equation}
q(X_{i-1}|X_{i}) = \bold{1}\{X_{i-1} = D_i(X_i)\}
\end{equation}
This follows from the optimality of the AR model, where the free energy estimator holds not only in expectation but almost surely. It can also be noted by rewriting the free energy estimator as
\begin{equation}
\Delta F = -\log Z_Q
= -\log \mathbb{E}_{X_N\sim p} \left[\frac{\exp(-U_B(X_{N}))}{p(X_N)} \right]
\end{equation}
where with the optimal AR $Z_B = \exp(-U_B(X_{N}))/p^*(X_N) $.

\section{Assumptions, Definitions, and Preliminaries}\label{app:assumption_definitation}
In our paper, we consider transport as a Markov process. 
The key concept is the Markov semigroup, which we defined below.
\begin{definition}[Markov semigroup]
Let $(E,\mathcal B)$ be a measurable state space, and let $(P_t)_{t\ge 0}$ be a family of linear operators acting on a space $\mathcal F$ of bounded measurable functions $f:E\to\mathbb R$. We say that $(P_t)_{t\ge 0}$ is a Markov semigroup if, for all $s,t\ge 0$,
\begin{enumerate}
    \item $P_0 f=f$ for all $f\in\mathcal F$;
    \item $P_{t+s}=P_tP_s$;
    \item $P_t f\ge 0$ whenever $f\ge 0$;
    \item $P_t\mathbf 1=\mathbf 1$.
\end{enumerate}
\end{definition}
The Markov process can often be characterized through its infinitesimal generator, defined below.

\begin{definition}[Infinitesimal generator of a Markov process]
Let $(P_t)_{t\ge 0}$ be a Markov semigroup acting on bounded measurable functions
$f:E\to\mathbb R$. Its infinitesimal generator $\mathcal L$ is defined pointwise by
\begin{align}
        \mathcal L f(x)
    :=
    \lim_{t\downarrow 0}
    \frac{P_t f(x)-f(x)}{t},
\end{align}
for all functions $f$ such that the above limit exists for every $x\in E$.
\end{definition}

The same definition applies to time-inhomogeneous Markov processes.

\begin{definition}[Time-inhomogeneous Markov transition operators]
Let $(E,\mathcal B)$ be a measurable state space, and let
$(P_{s,t})_{0\le s\le t}$ be a family of linear operators acting on a space
$\mathcal F$ of bounded measurable functions $f:E\to\mathbb R$. We say that
$(P_{s,t})_{0\le s\le t}$ is a time-inhomogeneous Markov transition family if,
for all $0\le s\le t\le u$,
\begin{enumerate}
    \item $P_{s,s}f=f$ for all $f\in\mathcal F$;
    \item $P_{s,u}=P_{s,t}P_{t,u}$;
    \item $P_{s,t}f\ge 0$ whenever $f\ge 0$;
    \item $P_{s,t}\mathbf 1=\mathbf 1$.
\end{enumerate}
\end{definition}

For the Markov process $(X_t)_{t\in[0,1]}$, the transition operator is interpreted as
\begin{align}
     P_{s,t}f(x)
    =
    \mathbb E[f(X_t)\mid X_s=x],
    \qquad 0\le s\le t.
\end{align}

\begin{definition}[Time-dependent infinitesimal generator]
Let $(P_{s,t})_{0\le s\le t}$ be a time-inhomogeneous Markov transition family.
Its time-dependent infinitesimal generator $\mathcal L_t$ is defined pointwise by
\begin{align}
       \mathcal L_t f(x)
    :=
    \lim_{\delta\downarrow 0}
    \frac{P_{t,t+\delta}f(x)-f(x)}{\delta},
\end{align}
 
for all functions $f$ such that the above limit exists.
\end{definition}

The generator $\mathcal L$ describes how test functions, or observables, change
under the Markov transition. By duality, the same Markov evolution induces an
evolution on probability measures or densities. This dual evolution is described
by the adjoint operator of the generator.

\begin{definition}[Adjoint of a Markov generator]
Let $\mathcal L$ be the infinitesimal generator of a Markov process on a state
space $(E,\mathcal B)$, acting on a class of test functions. Let $\mu$ be a
reference measure on $E$. The adjoint operator $\mathcal L^\dagger$ is defined
through the duality relation
\begin{align}
    \int_E (\mathcal L f)(x)\, g(x)\,\mu(\mathrm dx)
    =
    \int_E f(x)\,(\mathcal L^\dagger g)(x)\,\mu(\mathrm dx),
\end{align}
for suitable functions $f$ and $g$ for which the above expression is
well-defined.

Equivalently, if $\langle f,g\rangle_\mu := \int_E f(x)g(x)\,\mu(\mathrm dx)$,
then $\mathcal L^\dagger$ is characterized by
\begin{align}
    \langle \mathcal L f, g\rangle_\mu
    =
    \langle f,\mathcal L^\dagger g\rangle_\mu .
\end{align}

If $\rho_t$ denotes the density of the law of $X_t$ with respect to $\mu$, then we have
\begin{align}
    \partial_t \rho_t
    =
    \mathcal L^\dagger \rho_t, 
\end{align}
which is known as the forward Kolmogorov equation. 
For diffusion processes, this
equation is also commonly called the Fokker-Planck equation.

For a time-inhomogeneous Markov process with time-dependent generator
$\mathcal L_t$, we simply have
\begin{align}
    \partial_t \rho_t
    =
    \mathcal L_t^\dagger \rho_t .
\end{align}
\end{definition}

\begin{remark}[Change of reference measure]\label{eq:change_of_reference}
The adjoint operator depends on the choice of reference measure. To make this
dependence explicit, write $\mathcal L^{\dagger,\mu}$ for the adjoint with
respect to $\mu$. Suppose that $\mu$ is absolutely continuous with respect to the
Lebesgue measure and write
\begin{align}
    \mu(\mathrm dx) = m(x)\,\mathrm dx .
\end{align}
Then the adjoint with respect to $\mu$ and the adjoint with respect to the
Lebesgue measure are formally related by
\begin{align}
    \mathcal L^{\dagger,m} g
    =
    \frac{1}{m}\,
    \mathcal L^{\dagger,\mathrm{Leb}}(mg),
\end{align}
whenever the expression is well-defined. In particular,
$\mathcal L^{\dagger,\mu}g$ should be interpreted as the infinitesimal change of
the density $g$ with respect to the reference measure $\mu$, not as a normalized
density itself.

In this work, unless otherwise stated, we take $\mu$ to be the Lebesgue measure
and write $\mathcal L^\dagger$ for $\mathcal L^{\dagger,\mathrm{Leb}}$.
\end{remark}

The generalized Doob's $h$-transform \citep{chetrite2011two,chetrite2013nonequilibrium,chetrite2015nonequilibrium} allows us to derive the path weight for the tilted path and its generator, which we will use in the following to derive the generalized EJE.
\begin{lemma}[Generalized Doob's $h$-transform]
\label{thm:conservative_doob_h_transform}
Let $(X_t)_{t\in[0,1]}$ be a time-inhomogeneous Markov process on a state space
$(E,\mathcal B)$ with path measure $P$ and time-dependent generator
$(\mathcal L_t)_{t\in[0,1]}$. Let $h_t:E\to(0,\infty)$ be strictly positive and
sufficiently regular. Define
\begin{align}
    \mathcal L_t^h f
    :=
    h_t^{-1}\mathcal L_t(h_t f)
    -
    h_t^{-1}f\,\mathcal L_t h_t .
    \label{eq:conservative_h_transform_generator}
\end{align}
Assume that $(\mathcal L_t^h)_{t\in[0,1]}$ generates a time-inhomogeneous Markov
The path measure $P^h$ is absolutely continuous with respect
to $P$, with Radon--Nikodym derivative
\begin{align}
    \frac{\mathrm d  P^h}{\mathrm d P}
    (x_{(0,1]}|x_0)
    =
    \frac{h_1(x_1)}{h_0(x_0)}
    \exp\left(
        -\int_0^1
        \left[
            h_t^{-1}\partial_t h_t
            +
            h_t^{-1}\mathcal L_t h_t
        \right](x_t)
        \,\mathrm dt
    \right).
    \label{eq:conservative_h_transform_rnd}
\end{align}
The adjoint generator of the $h$-transformed process is
\begin{align}
    (\mathcal L_t^h)^\dagger g
    =
    h_t\,\mathcal L_t^\dagger\!\left({h_t}^{-1} {g}\right)
    -
   {h_t}^{-1}  {\mathcal L_t h_t} \,g .
    \label{eq:h_transform_adjoint_generator}
\end{align}
\end{lemma}

\section{Theorems and Proofs}
\label{appendix:theorem}

\subsection{Infinitesimal  reversal and its connection to adjoint}
\begin{tcolorbox}[
    colback=gray!8,
    colframe=gray!40,
    boxrule=0.5pt,
    arc=4pt,
    left=1pt,
    right=1pt,
    top=2pt,
    bottom=2pt,
    breakable
]
\begin{proposition}[Infinitesimal $\pi_t$-reversal]
\label{prop:pi_reversal}
Let $(\mathcal L_t)_{t\in[0,1]}$ be a time-inhomogeneous Markov generator, and let $\mathcal L_t^\dagger$ denote its adjoint with
respect to a reference measure $\mu$. Let $\pi_t$ be a strictly positive
density with respect to $\mu$. The infinitesimal reverse generator with
respect to $\pi_t$ is
\begin{align}
    \bwd{\mathcal L}_t^\pi f
    &=
    \frac{1}{\pi_t}\mathcal L_t^\dagger(\pi_t f)
    -
    \frac{\mathcal L_t^\dagger \pi_t}{\pi_t}f .
    \label{eq:pi_reversal_generator}
\end{align}
\end{proposition}
\end{tcolorbox}

Before proving the formula, we clarify what ``reversal with respect to $\pi_t$''
means. Fix a time $t$ and a positive density $\pi_t$. We initialize the forward
transition from $X_t\sim \pi_t$, evolve one infinitesimal step using
$P_{t,t+\delta}$, and then reverse this one-step joint law by Bayes'
rule. That is, we define the reverse kernel
$\overleftarrow P^{\pi_t}_{t+\delta,t}$ through
\begin{align}
    \pi_t(x) \,P_{t,t+\delta}(x,\mathrm dy)
\mathrm dx      =
    (\pi_t P_{t,t+\delta})(y) \mathrm dy
    \bwd {P}^{\pi_t}_{t+\delta,t}(y,\mathrm dx).
\end{align}
The infinitesimal $\pi_t$-reversal generator is the generator of this reverse
kernel, defined by the expansion
\begin{align}
    \bwd{P}^{\pi_t}_{t+\delta,t}f(y)
    =
    f(y)
    +
    \delta\,\bwd{\mathcal L}_t^\pi f(y)
    +
    o(\delta),
    \qquad \delta\downarrow 0.
\end{align}
When $\pi_t$ is the actual marginal density of the forward process, this is the
``time-reversal" generator. 
For a general choice of $\pi_t$, it
is a local ``reversal" rather than the time-reversal of the full
process.

\begin{proof}
Let \(P_{t,t+\delta}\) denote the forward transition operator generated by
\(\mathcal L_t\). For small \(\delta>0\), we have the infinitesimal expansion
\begin{align}
    P_{t,t+\delta}g
    =
    g
    +
    \delta\,\mathcal L_t g
    +
    o(\delta),
    \qquad \delta\downarrow 0 .
    \label{eq:forward_semigroup_expansion}
\end{align}
Moreover, in the weak sense,
\begin{align}
    \pi_t P_{t,t+\delta}
    =
    \pi_t
    +
    \delta\,\mathcal L_t^\dagger \pi_t
    +
    o(\delta).
    \label{eq:forward_density_expansion}
\end{align}

By definition, the \(\pi_t\)-reversal kernel
\(\bwd P^{\pi_t}_{t+\delta,t}\) is the reverse kernel of the one-step joint law
initialized from \(X_t\sim \pi_t\):
\begin{align}
    \pi_t(x) P_{t,t+\delta}(x,\mathrm dy)\,\mathrm dx
    =
    (\pi_t P_{t,t+\delta})(y)\,\mathrm dy\,
    \bwd P^{\pi_t}_{t+\delta,t}(y,\mathrm dx).
\end{align}
Its infinitesimal generator is defined by
\begin{align}
    \bwd P^{\pi_t}_{t+\delta,t} f
    =
    f
    +
    \delta\,\bwd{\mathcal L}_t^\pi f
    +
    o(\delta),
    \qquad \delta\downarrow 0 .
    \label{eq:backward_semigroup_expansion}
\end{align}

We now derive the integral characterization of
\(\bwd{\mathcal L}_t^\pi\). Let \(f\) and \(g\) be test functions. Using the
forward representation of the one-step joint law, we have
\begin{align}
    \mathbb E[f(X_t)g(X_{t+\delta})]
    &=
    \int_E
    f(x)\,P_{t,t+\delta}g(x)\,\pi_t(x) \mathrm d x.
\end{align}
By \eqref{eq:forward_semigroup_expansion},
\begin{align}
    \mathbb E[f(X_t)g(X_{t+\delta})]
    &=
    \int_E
    f
    \left(g+\delta\mathcal L_t g+o(\delta)\right)
    \pi_t\mathrm dx
    \\
    &=
    \int_E fg\,\pi_t \mathrm d x
    +
    \delta\int_E f\,\mathcal L_t g\,\pi_t \mathrm d x
    +
    o(\delta).
    \label{eq:forward_two_time_expansion}
\end{align}

On the other hand, using the reversed representation of the same one-step joint
law, we obtain
\begin{align}
    \mathbb E[f(X_t)g(X_{t+\delta})]
    &=
    \int_E
    g(y)\,
    \bwd P^{\pi_t}_{t+\delta,t}f(y)\,
    (\pi_t P_{t,t+\delta})(y) \mathrm d y.
\end{align}
Using \eqref{eq:backward_semigroup_expansion} and
\eqref{eq:forward_density_expansion}, this becomes
\begin{align}
    \mathbb E[f(X_t)g(X_{t+\delta})]
    &=
    \int_E
    g
    \left(
        f+\delta\bwd{\mathcal L}_t^\pi f+o(\delta)
    \right)
    \left(
        \pi_t+\delta\mathcal L_t^\dagger\pi_t+o(\delta)
    \right)
    \mathrm d x
    \\
    &=
    \int_E fg\,\pi_t \mathrm d x
    +
    \delta\int_E g\,\bwd{\mathcal L}_t^\pi f\,\pi_t \mathrm d x
    +
    \delta\int_E fg\,\mathcal L_t^\dagger\pi_t \mathrm dx
    +
    o(\delta).
    \label{eq:backward_two_time_expansion}
\end{align}

Since \eqref{eq:forward_two_time_expansion} and
\eqref{eq:backward_two_time_expansion} are expansions of the same quantity, we
obtain
\begin{align}
    \int_E g\,\bwd{\mathcal L}_t^\pi f\,\pi_t\mathrm d  x
    =
    \int_E f\,\mathcal L_t g\,\pi_t \mathrm d  x
    -
    \int_E fg\,\mathcal L_t^\dagger\pi_t \mathrm d x.
    \label{eq:pi_reverse_integral_form}
\end{align}
This is the integral characterization of the infinitesimal
\(\pi_t\)-reversal.

It remains to identify the operator explicitly. By adjointness with respect to
Lesbague measure,
\begin{align}
    \int_E f\,\mathcal L_t g\,\pi_t \mathrm d  x
    =
    \int_E \mathcal L_t g\,(\pi_t f) \mathrm d x
    =
    \int_E g\,\mathcal L_t^\dagger(\pi_t f) \mathrm d  x .
\end{align}
Substituting this into \eqref{eq:pi_reverse_integral_form}, we get
\begin{align}
    \int_E g\,\bwd{\mathcal L}_t^\pi f\,\pi_t\mathrm  dx
    =
    \int_E g
    \left[
        \mathcal L_t^\dagger(\pi_t f)
        -
        f\,\mathcal L_t^\dagger\pi_t
    \right] \mathrm  dx.
\end{align}
Equivalently,
\begin{align}
    \int_E g\,\bwd{\mathcal L}_t^\pi f\,\pi_t\mathrm  dx
    =
    \int_E g
    \left[
        \frac{1}{\pi_t}\mathcal L_t^\dagger(\pi_t f)
        -
        \frac{\mathcal L_t^\dagger\pi_t}{\pi_t}f
    \right]
    \pi_t \mathrm d x .
\end{align}
Since this holds for all test functions \(g\), we identify
\begin{align}
    \bwd{\mathcal L}_t^\pi f
    =
    \frac{1}{\pi_t}\mathcal L_t^\dagger(\pi_t f)
    -
    \frac{\mathcal L_t^\dagger\pi_t}{\pi_t}f .
\end{align}
This proves the formula.
\end{proof}

\begin{remark}
We note that this infinitesimal reversal has a very similar form to the adjoint.
They only differ in an extra term.
In fact, this is because they have a deeper connection:
The infinitesimal reversal is simply a normalized version of $\pi$-weighted adjoint.
The first term in \Cref{eq:pi_reversal_generator}, i.e., 
\begin{align}
    \frac{1}{\pi_t}\mathcal L_t^\dagger(\pi_t f)
\end{align}
comes directly from the $\pi_t$-adjoint of the forward transition. 
It is
not yet a Markov generator because it does not necessarily send $\mathbf 1$ to
zero. 
The second term is exactly the normalization correction coming from
importance-normalizing the reverse kernel:
\begin{align}
    \frac{1}{\pi_t}\mathcal L_t^\dagger(\pi_t ) f =     \frac{1}{\pi_t}\mathcal L_t^\dagger(\pi_t \mathbf{1}) f
\end{align}
\end{remark}

We can also obtain a derivation using this normalization perspective to prove Proposition~\ref{prop:pi_reversal}:
\begin{proof}
Fix a small $\delta>0$, we start from:
\begin{align}
    \pi_t(x)\,\mathrm dx\,P_{t,t+\delta}(x,\mathrm dy)
    =
    (\pi_t P_{t,t+\delta})(y)\,\mathrm dy\,
    \bwd{P}^{\pi_t}_{t+\delta,t}(y,\mathrm dx).
\end{align}
Therefore, for test functions $f$ and $g$,
\begin{align}
    \int_E f(x)P_{t,t+\delta}g(x)\,\pi_t(x)\,\mathrm dx
    =
    \int_E g(y)\,
    \bwd{P}^{\pi_t}_{t+\delta,t}f(y)\,
    (\pi_t P_{t,t+\delta})(y)\,\mathrm dy .
\end{align}
Applying importance sampling to the RHS
  with proposal $\pi_t(y)\,\mathrm dy$, we obtain
\begin{align}\label{eq:is_qqqq}
    \int_E f(x)P_{t,t+\delta}g(x)\,\pi_t(x)\,\mathrm dx
    =
    \int_E g(y)\,Q_{t,\delta}^{\pi_t}f(y)\,
    \pi_t(y)\,\mathrm dy,
\end{align}
where
\begin{align}
    Q_{t,\delta}^{\pi_t}f(y)
    :=
    \bwd{P}^{\pi_t}_{t+\delta,t}f(y)
    \frac{(\pi_t P_{t,t+\delta})(y)}{\pi_t(y)} .
\end{align}

From \Cref{eq:is_qqqq}, we can see that $Q_{t,\delta}^{\pi_t}$ is the adjoint
of the forward transition with respect to the measure $\pi_t(x)\,\mathrm dx$.
This becomes explicit by using the infinitesimal expansions
\begin{align}
    P_{t,t+\delta}g
    =
    g+\delta\,\mathcal L_t g+o(\delta),
\end{align}
and
\begin{align}
    Q_{t,\delta}^{\pi_t}f
    =
    f
    +
    \delta\,\mathcal L_t^Q f
    +
    o(\delta).
\end{align}
Substituting these expansions into \Cref{eq:is_qqqq} gives
\begin{align}
    \int_E f(x)\mathcal L_t g(x)\pi_t(x)\,\mathrm dx
    =
    \int_E g(x)\mathcal L_t^Q f(x)\pi_t(x)\,\mathrm dx .
\end{align}
Equivalently, 
\begin{align}
    \langle f,\mathcal L_t g\rangle_{\pi_t}
    =
    \langle \mathcal L_t^Q f,g\rangle_{\pi_t}.
\end{align}
Thus $\mathcal L_t^Q$ is the adjoint of $\mathcal L_t$ with respect to the
reference measure $\pi_t(x)\,\mathrm dx$. Using the change-of-reference formula
in \Cref{eq:change_of_reference}, we obtain
\begin{align}
    \mathcal L_t^Q f
    =
    \mathcal L_t^{\dagger,\pi_t} f
    =
    \frac{1}{\pi_t}\mathcal L_t^\dagger(\pi_t f),
\end{align}
where $\mathcal L_t^\dagger$ denotes the adjoint with respect to Lebesgue
measure.

Recall that
\begin{align}
    Q_{t,\delta}^{\pi_t}f(y)
    :=
    \bwd{P}^{\pi_t}_{t+\delta,t}f(y)
    \frac{(\pi_t P_{t,t+\delta})(y)}{\pi_t(y)} ,
\end{align}
i.e., $Q_{t,\delta}^{\pi_t}$ is an importance-weighted version of the
reverse kernel. 
Since the backward kernel is normalized, i.e., $\bwd{P}^{\pi_t}_{t+\delta,t}\mathbf{1} = \mathbf{1}$, the weight
\begin{align}
    \frac{(\pi_t P_{t,t+\delta})(y)}{\pi_t(y)} =  Q_{t,\delta}^{\pi_t}\mathbf 1(y)
\end{align}
Hence, the actual reverse kernel is recovered by dividing out this normalizing
factor:
\begin{align}
    \bwd{P}^{\pi_t}_{t+\delta,t}f
    =
    \frac{Q_{t,\delta}^{\pi_t}f}
    {Q_{t,\delta}^{\pi_t}\mathbf 1}.
\end{align}

Therefore,
\begin{align}
    \bwd{P}^{\pi_t}_{t+\delta,t}f
    &=
    \frac{
        f+\delta\,\mathcal L_t^Q f+o(\delta)
    }{
        1+\delta\,\mathcal L_t^Q\mathbf 1+o(\delta)
    } \\
    &=
    f
    +
    \delta
    \left(
        \mathcal L_t^Q f
        -
        f\,\mathcal L_t^Q\mathbf 1
    \right)
    +
    o(\delta).
\end{align}
The second equation is due to the Taylor expension of $1/(1+x)$ at $x=0$:
\begin{align}
    \frac{1}{1+\delta a+o(\delta)}
=
1-\delta a+o(\delta)
\end{align}
and then drop higher-order terms in $\delta$.
Hence the infinitesimal generator of the normalized reverse kernel is
\begin{align}
    \bwd{\mathcal L}_t^\pi f
    &=
    \mathcal L_t^Q f
    -
    f\,\mathcal L_t^Q\mathbf 1 \\
    &=
    \frac{1}{\pi_t}\mathcal L_t^\dagger(\pi_t f)
    -
    \frac{\mathcal L_t^\dagger\pi_t}{\pi_t}f .
\end{align}
\end{proof}

\subsection{Infinitesimal reversal and its connection to Generalized Doob's $h$-transform}

The infinitesimal reversal has a deep connection to the generalized Doob's $h$-transform, 
as we stated in  \Cref{sec:reversal_doob}.
Here, we state a more detailed version.

\begin{tcolorbox}[
    colback=gray!8,
    colframe=gray!40,
    boxrule=0.5pt,
    arc=4pt,
    left=1pt,
    right=1pt,
    top=2pt,
    bottom=2pt,
    breakable
]
\begin{theorem}[Double infinitesimal reversal and generalized Doob's $h$-transform]
\label{thm:dual_reversal_doob}
Let $(\mathcal L_t)_{t\in[0,1]}$ be a family of time-inhomogeneous Markov
generators, and let $p_t$ and $q_t$ be strictly positive densities. 
Apply the infinitesimal reversal formula in \Cref{eq:pi_reversal_generator}:
first define
the infinitesimal reversal of $\mathcal L_t$ with respect to $q_t$ by
\begin{align}
    \widetilde{\mathcal L}_t f
    := \bwd{\mathcal L}_t^{q_t} f = 
    \frac{1}{q_t}\mathcal L_t^\dagger(q_t f)
    -
    \frac{\mathcal L_t^\dagger q_t}{q_t}f .
    \label{eq:first_q_reversal}
\end{align}
Then define the infinitesimal reversal of $\widetilde{\mathcal L}_t$ with
respect to $p_t$ by
\begin{align}
    \widehat{\mathcal L}_t f
    :=\bwd{\widetilde{\mathcal L}_t}^{p_t} f = 
    \frac{1}{p_t}\widetilde{\mathcal L}_t^\dagger(p_t f)
    -
    \frac{\widetilde{\mathcal L}_t^\dagger p_t}{p_t}f .
    \label{eq:second_p_reversal}
\end{align}
Then $\widehat{\mathcal L}_t$ is the generator of the Doob's $h$-transform of $\mathcal L_t$ with
\begin{align}
    h_t
    =
    \frac{p_t}{q_t}.
\end{align}
That is,
\begin{align}
    \widehat{\mathcal L}_t f
    =
    \mathcal L_t^h f
     =
    h_t^{-1}\mathcal L_t(h_t f)
    -
    h_t^{-1}f\,\mathcal L_t h_t .
    \label{eq:dual_reversal_doob_result}
\end{align}
\end{theorem}
\end{tcolorbox}

A corollary of this, which provides a different perspective to understand this result, is stated below:
\begin{tcolorbox}[
    colback=gray!8,
    colframe=gray!40,
    boxrule=0.5pt,
    arc=4pt,
    left=1pt,
    right=1pt,
    top=2pt,
    bottom=2pt,
    breakable
]
\begin{corollary}[Reversal of an $h$-transform]
\label{cor:reversal_h_transform}
Let $(\mathcal L_t)_{t\in[0,1]}$ be a family of time-inhomogeneous Markov
generators, and let $h_t:E\to(0,\infty)$ be strictly positive. Define the generalized Doob's $h$-transform generator by
\begin{align}
    \mathcal L_t^h f
    :=
    h_t^{-1}\mathcal L_t(h_t f)
    -
    h_t^{-1}f\,\mathcal L_t h_t .
\end{align}
Let $p_t$ be a strictly positive density, and define
\begin{align}
    q_t
    =
    \frac{p_t}{h_t}.
\end{align}
Then the infinitesimal reversal of $\mathcal L_t^h$ with respect to
$p_t$ is equal to the infinitesimal reversal of $\mathcal L_t$ with
respect to $q_t$:
\begin{align}
    \bwd{(\mathcal L_t^h)}^{p_t} f
    =
    \bwd{\mathcal L}_t^{q_t} f .
\end{align}
\end{corollary}
\end{tcolorbox}

We first prove \Cref{thm:dual_reversal_doob}.
\begin{proof}
The proof follows directly from algebra.
Fix $t\in[0,1]$ and suppress the subscript $t$ to simplify notation. Write
\begin{align}
    \widetilde{\mathcal L} f
    =
    \frac{1}{q}\mathcal L^\dagger(qf)
    -
    \frac{\mathcal L^\dagger q}{q}f .
\end{align}
We first compute the adjoint of $\widetilde{\mathcal L}$. For suitable test
functions $f$ and $g$,
\begin{align}
    \int_E (\widetilde{\mathcal L}f)(x)g(x)\,\mathrm dx
    &=
    \int_E
    \frac{1}{q(x)}
    \mathcal L^\dagger(qf)(x)g(x)\,\mathrm dx
    -
    \int_E
    \frac{\mathcal L^\dagger q(x)}{q(x)}
    f(x)g(x)\,\mathrm dx \\
    &=
    \int_E
    f(x)
    \left[
        q\,\mathcal L\!\left(\frac{g}{q}\right)
        -
        \frac{\mathcal L^\dagger q}{q}g
    \right](x)
    \,\mathrm dx .
\end{align}
Therefore,
\begin{align}
    \widetilde{\mathcal L}^\dagger g
    =
    q\,\mathcal L\!\left(\frac{g}{q}\right)
    -
    \frac{\mathcal L^\dagger q}{q}g .
    \label{eq:tilde_adjoint}
\end{align}
Now apply the second reversal with respect to $p$:
\begin{align}
    \widehat{\mathcal L}f
    &=
    \frac{1}{p}\widetilde{\mathcal L}^\dagger(pf)
    -
    \frac{\widetilde{\mathcal L}^\dagger p}{p}f .
\end{align}
Using \Cref{eq:tilde_adjoint}, we have
\begin{align}
    \frac{1}{p}\widetilde{\mathcal L}^\dagger(pf)
    &=
    \frac{1}{p}
    \left[
        q\,\mathcal L\!\left(\frac{pf}{q}\right)
        -
        \frac{\mathcal L^\dagger q}{q}pf
    \right] \\
    &=
    \frac{q}{p}
    \mathcal L\!\left(\frac{p}{q}f\right)
    -
    \frac{\mathcal L^\dagger q}{q}f ,
\end{align}
and
\begin{align}
    \frac{\widetilde{\mathcal L}^\dagger p}{p}f
    &=
    \frac{1}{p}
    \left[
        q\,\mathcal L\!\left(\frac{p}{q}\right)
        -
        \frac{\mathcal L^\dagger q}{q}p
    \right]f \\
    &=
    \frac{q}{p}
    \mathcal L\!\left(\frac{p}{q}\right)f
    -
    \frac{\mathcal L^\dagger q}{q}f .
\end{align}
Therefore, we have
\begin{align}
    \widehat{\mathcal L}f
    &=
    \frac{q}{p}
    \mathcal L\!\left(\frac{p}{q}f\right)
    -
    \frac{\mathcal L^\dagger q}{q}f
    -
    \left[
        \frac{q}{p}
        \mathcal L\!\left(\frac{p}{q}\right)f
        -
        \frac{\mathcal L^\dagger q}{q}f
    \right] \\
    &=
    \frac{q}{p}
    \mathcal L\!\left(\frac{p}{q}f\right)
    -
    \frac{q}{p}
    f\,\mathcal L\!\left(\frac{p}{q}\right).
\end{align}
Finally, using 
\begin{align}
    h=\frac{p}{q}
\end{align}
we obtain
\begin{align}
    \widehat{\mathcal L}_t f
    =
    h_t^{-1}\mathcal L_t(h_t f)
    -
    h_t^{-1}f\,\mathcal L_t h_t
    =
    \mathcal L_t^h f .
\end{align}
\end{proof}

This provides us with a very powerful tool to analyze infinitesimal reversal and derive path weight using generalized Doob's $h$ weight in \Cref{thm:conservative_doob_h_transform}.
We also have the following corollary.

We then derive the Corollary:
\begin{proof}
The proof is twofold. First, applying infinitesimal reversal twice with respect
to the same density gives back the original generator. This follows from
\Cref{thm:dual_reversal_doob} by taking $p_t=q_t$, so that
\begin{align}
    h_t
    =
    \frac{p_t}{q_t}
    =
    1.
\end{align}
By \Cref{thm:dual_reversal_doob}, applying reversal first with respect to $q_t$
and then with respect to $p_t$ gives the $h$-transformed generator
$\mathcal L_t^h$. 
Therefore, reversing $\mathcal L_t^h$ once more with respect
to $p_t$ cancels the second reversal and leaves the reversal of
$\mathcal L_t$ with respect to $q_t$. Hence
\begin{align}
    \bwd{(\mathcal L_t^h)}^{p_t} f
    =
    \bwd{\mathcal L}_t^{q_t} f .
\end{align}
\end{proof}

Another corollary  is stated below:
\begin{tcolorbox}[
    colback=gray!8,
    colframe=gray!40,
    boxrule=0.5pt,
    arc=4pt,
    left=1pt,
    right=1pt,
    top=2pt,
    bottom=2pt,
    breakable
]
\begin{corollary}[Doob's relation between two reversals]
\label{cor:relation_2_reversal}
Let $(\mathcal L_t)_{t\in[0,1]}$ be a family of time-inhomogeneous Markov
generators, and  $p_t$ and $q_t$  be a strictly positive density.
Then the infinitesimal reversal of $\mathcal L_t$ with respect to
$p_t$ is equal to the infinitesimal reversal of $\mathcal L_t$ with
respect to $q_t$, followed by a Doob's $h$ transform, i.e., 
\begin{align}\label{eq:_}
    \bwd{\mathcal L_t}^{p_t} f
    =
   ( \bwd{\mathcal L}_t^{q_t} )^h f .
\end{align}
with 
\begin{align}
    h_t = p_t / q_t
\end{align}
\end{corollary}
\end{tcolorbox}
\begin{proof}
To prove this, we can note that when $h=p_t/q_t$, we can view it as applying time reversal with $q_t$ first, and then applying time reversal with $p_t$.
Therefore, the RHS for \Cref{eq:_} is applying time reversal with $q_t$, $q_t$ and $p_t$ sequentially.
Applying $q_t$ twice yields the identity map, and hence it is equivalent to applying the reversal with $p_t$.
\end{proof}
Alternatively, we can also directly prove it with algebra:
\begin{proof}
   \begin{align}
    \bwd{\mathcal{L}}^{p_t}f = p_t^{-1}\mathcal{L}^\dagger(p_tf) - p_t^{-1}\mathcal{L}^\dagger(p_t ) f, \quad     \bwd{\mathcal{L}}^{q_t}f = q_t^{-1}\mathcal{L}^\dagger(q_t f) -q_t^{-1}\mathcal{L}^\dagger(q_t) f
\end{align}
Then, 
\begin{align}
     (\bwd{\mathcal{L}}^{q_t})^h  f&= h^{-1}  \bwd{\mathcal{L}}^{q_t} (hf) -  h^{-1}  \bwd{\mathcal{L}}^{q_t} (h ) f \\
     &= h^{-1} q_t^{-1}\mathcal{L}^\dagger(q_t h f) - h^{-1}  q_t^{-1}\mathcal{L}^\dagger(q_t h ) f -  h^{-1} q_t^{-1}\mathcal{L}^\dagger(q_t h )f + h^{-1}  q_t^{-1}\mathcal{L}^\dagger(q_t h ) f\\
     &= h^{-1} q_t^{-1}\mathcal{L}^\dagger(q_t h f) -  h^{-1} q_t^{-1}\mathcal{L}^\dagger(q_t h )f 
\end{align}
when we choose $h_t = p_t/q_t$, we can see this is exactly $  \bwd{\mathcal{L}}^{q_t}f = q_t^{-1}\mathcal{L}^\dagger(q_t f) -q_t^{-1}\mathcal{L}^\dagger(q_t) f$.
\end{proof}

\subsection{Marginal law and generalized Doob's $h$-transform}
The conclusions above are algebraic: so far, we have not imposed that either
$p_t$ or $q_t$ is the actual marginal law of a Markov process. In practice,
however, at least one of these density-like objects will be pinned down. For example, suppose pin one of the densities to be the marginal law.
The following result gives the necessary and sufficient condition for this to
hold.

\begin{tcolorbox}[
    colback=gray!8,
    colframe=gray!40,
    boxrule=0.5pt,
    arc=4pt,
    left=1pt,
    right=1pt,
    top=2pt,
    bottom=2pt,
    breakable
]
\begin{theorem}[Marginal prescription for generalized Doob's $h$-transform]\label{thm:prescribing_h_transform_marginal}
Let $(\mathcal L_t)_{t\in[0,1]}$ be a family of time-inhomogeneous Markov
generators, with adjoints $(\mathcal L_t^\dagger)_{t\in[0,1]}$. Let
$h_t:E\to(0,\infty)$ be strictly positive and sufficiently regular, and define
the Doob's $h$-transform generator by
\begin{align}
    \mathcal L_t^h f
    =
    h_t^{-1}\mathcal L_t(h_t f)
    -
    h_t^{-1}f\,\mathcal L_t h_t .
\end{align}
Let $q_t$ be a strictly positive density-like function and define
\begin{align}
    \rho_t^h =
    h_t q_t .
\end{align}
Then $\rho_t^h$ solves the forward equation of the $h$-transformed process,
\begin{align}
    \partial_t \rho_t^h
    =
    (\mathcal L_t^h)^\dagger \rho_t^h ,
\end{align}
\textbf{if and only if}
\begin{align}
    \frac{(\partial_t+\mathcal L_t)h_t}{h_t}
    =
    \frac{\mathcal L_t^\dagger q_t-\partial_t q_t}{q_t}.
    \label{eq:h_prescribed_marginal_condition}
\end{align}
\end{theorem}
\end{tcolorbox}

\begin{proof}
    Let's prove it in both directions.

\paragraph{(1)}
Assume that $\rho_t^h=h_tq_t$ solves the forward equation
\begin{align}
    \partial_t\rho_t^h
    =
    (\mathcal L_t^h)^\dagger\rho_t^h .
\end{align}
Using the adjoint formula for the $h$-transformed generator,
\begin{align}
    (\mathcal L_t^h)^\dagger g
    =
    h_t\,\mathcal L_t^\dagger(h_t^{-1}g)
    -
    h_t^{-1}(\mathcal L_t h_t)g ,
\end{align}
we obtain
\begin{align}
    (\mathcal L_t^h)^\dagger\rho_t^h
    &=
    h_t\,\mathcal L_t^\dagger q_t
    -
    q_t\,\mathcal L_t h_t .
\end{align}
On the other hand,
\begin{align}
    \partial_t\rho_t^h
    =
    \partial_t(h_tq_t)
    =
    q_t\,\partial_t h_t
    +
    h_t\,\partial_t q_t .
\end{align}
Therefore,
\begin{align}
    q_t\,\partial_t h_t
    +
    h_t\,\partial_t q_t
    =
    h_t\,\mathcal L_t^\dagger q_t
    -
    q_t\,\mathcal L_t h_t .
\end{align}
Therefore, 
\begin{align}
    q_t(\partial_t+\mathcal L_t)h_t
    =
    h_t(\mathcal L_t^\dagger q_t-\partial_t q_t).
\end{align}
Therefore, 
\begin{align}
    \frac{(\partial_t+\mathcal L_t)h_t}{h_t}
    =
    \frac{\mathcal L_t^\dagger q_t-\partial_t q_t}{q_t}.
\end{align}
\paragraph{(2)}
Now, let's start from:
\begin{align}
    \frac{(\partial_t+\mathcal L_t)h_t}{h_t}
    =
    \frac{\mathcal L_t^\dagger q_t-\partial_t q_t}{q_t}.
\end{align}
Multiplying both sides by $h_tq_t$  
\begin{align}
    q_t\,\partial_t h_t
    +
    q_t\,\mathcal L_t h_t
    =
    h_t\,\mathcal L_t^\dagger q_t
    -
    h_t\,\partial_t q_t .
\end{align}
Hence, 
\begin{align}
    q_t\,\partial_t h_t
    +
    h_t\,\partial_t q_t
    =
    h_t\,\mathcal L_t^\dagger q_t
    -
    q_t\,\mathcal L_t h_t .
\end{align}
LHS is  
\begin{align}
    q_t\,\partial_t h_t
    +
    h_t\,\partial_t q_t
    =
    \partial_t(h_tq_t)
    =
    \partial_t\rho_t^h .
\end{align}
Using $\rho_t^h=h_tq_t$, RHS is
\begin{align}
    h_t\,\mathcal L_t^\dagger q_t
    -
    q_t\,\mathcal L_t h_t
    &=
    h_t\,\mathcal L_t^\dagger(h_t^{-1}\rho_t^h)
    -
    h_t^{-1}(\mathcal L_t h_t)\rho_t^h \\
    &=
    (\mathcal L_t^h)^\dagger\rho_t^h .
\end{align}
Therefore,
\begin{align}
    \partial_t\rho_t^h
    =
    (\mathcal L_t^h)^\dagger\rho_t^h .
\end{align}
\end{proof}

\begin{remark}[Marginal prescription for generalized Doob's $h$-transform]
By \Cref{thm:prescribing_h_transform_marginal}, we note two cases when $\rho_t^h$ solves the forward equation of $h$-transformed process: 
\begin{enumerate}[leftmargin=*]
\item For a general density function $q_t$, the generalized Doob's $h$-weight becomes $\frac{\mathcal{L}_t^\dagger q_t}{q_t} - \partial_t \log q_t$.
\item When $q_t$ further solves the forward equation of the base process, i.e. $\mathcal{L}_t^\dagger q_t = \partial_tq_t$,   the generalized Doob's $h$-weight is zero.
\end{enumerate}
\end{remark}

\subsection{Generalized EJE, Crooks and Score Matching Objectives as Corollary of \Cref{thm:prescribing_h_transform_marginal} and Generalized Doob's $h$-transform}

\Cref{thm:prescribing_h_transform_marginal} plays a central role in our conclusion.
In this section, we show that the generalized EJE, Crooks can be directly derived from \Cref{thm:prescribing_h_transform_marginal} together with generalized Doob's $h$-transform.
\Cref{thm:prescribing_h_transform_marginal} also provides a new perspective for interpreting score matching objectives.

\paragraph{Generalized Escorted Jarzynski and Crooks fluctuation theorem}
We first prove Proposition~\ref{prop:crooks}, and Proposition~\ref{prop:jarzynski} follows by directly taking expectation.

\begin{proposition}[Generalized Crooks fluctuation theorem, restated]
\label{prop:crooks_restatement}
Let $(\mathcal L_t^F)_{t\in[0,1]}$ be the forward time-inhomogeneous Markov
generator, and let $(\mathcal L_t^F)^\dagger$ denote its adjoint. 
Let $\gamma_t$ be a strictly positive
density-like function. 
Define the backward generator by the infinitesimal
$\gamma_t$-reversal of $\mathcal L_t^F$:
\begin{align}
    \mathcal L_t^B f
    :=
    \frac{1}{\gamma_t}(\mathcal L_t^F)^\dagger(\gamma_t f)
    -
    \frac{1}{\gamma_t}(\mathcal L_t^F)^\dagger(\gamma_t)\,f .
    \label{eq:bwd_crooks_cont_restate}
\end{align}

Let $\bwd Q$ be the path law of the backward process whose marginal density
$q_t$ evolves backward in time according to
\begin{align}
    \partial_t q_t
    =
    -(\mathcal L_t^B)^\dagger q_t,
    \qquad
    q_1= \pi_1 :=\pi_B,
    \qquad
    \pi_B \propto \exp(-U_B).
\end{align}
Let $\fwd P$ be the path law of the backward process whose marginal density
$p_t$ evolves backward in time according to
\begin{align} 
    \partial_t p_t = (\mathcal L_t^F)^\dagger p_t,
    \qquad
    p_0 = \pi_0 := \pi_A,
    \qquad
    \pi_A \propto e^{-U_A},
\end{align}
Then the Radon-Nikodym derivative between the backward path law $\bwd Q$ and
the forward path law $\fwd P$  is
\begin{align}
    \frac{\mathrm d\bwd Q}{\mathrm d\fwd P}(X_{[0,1]})
    =
    \exp\bigl(-\widetilde{W}(X_{[0,1]})+\Delta F\bigr).
    \label{eq:crooks_eq_restate}
\end{align}
where
\begin{align}
     \widetilde{W}(X_{[0,1]})
      =
     {\int_0^1 \partial_t U_t(X_t)\,\mathrm{d}t}
      +
\int_0^1 \frac{((\mathcal L_t^F)^\dagger \gamma_t)(X_t)}{\gamma_t(X_t)}\,\mathrm{d}t.
\end{align}
\end{proposition}

\begin{proof}
    We start from 
\begin{align}
     \mathcal L_t^B f
      =
      \frac{1}{\gamma_t}(\mathcal L_t^F)^\dagger(\gamma_t f)
      -
      \frac{1}{\gamma_t}(\mathcal L_t^F)^\dagger(\gamma_t)\,f,
      \qquad
      (\mathcal L_t^B)^\dagger p
      =
      \gamma_t\,\mathcal L_t^F\!\left(\frac{p}{\gamma_t}\right)
      -
      \frac{p}{\gamma_t}\,(\mathcal L_t^F)^\dagger \gamma_t
\end{align}
We assume its marginal density is $q_t$:
\begin{align}
    \partial_t q_t = -(\mathcal L_t^B)^\dagger q_t,
    \qquad
    q_1 = \pi_1 := \pi_B,
    \qquad
    \pi_B \propto e^{-U_B}.
\end{align}
From \Cref{thm:prescribing_h_transform_marginal}, the time-reversal of this backward process (denoted as \(\fwd{Q}\)) has the same generator as Doob's $h$-transform of $\mathcal L_t^F$, with $h_t=  q_t / \pi_t$.
By generalized Doob's $h$ in \Cref{thm:conservative_doob_h_transform}, we have
\begin{align}
    \frac{\mathrm{d}\fwd{Q}}{\mathrm{d}\fwd{P}}(X_{(0, 1]}|X_0) = \frac{h_1}{h_0}\exp(-\int \left[
            h_t^{-1}\partial_t h_t
            +
            h_t^{-1}\mathcal L_t h_t
        \right] \mathrm{d} t)
\end{align}
Also, $h_1=q_1/\pi_1=1$ and $p_0 = \pi_0$, hence 
\begin{align}
    \frac{\mathrm{d}\bwd{Q}}{\mathrm{d}\fwd{P}}(X_{[0, 1]} ) &=   \frac{\mathrm{d}\fwd{Q}}{\mathrm{d}\bwd{Q}}(X_{[0, 1]} )   \frac{\mathrm{d}\bwd{Q}}{\mathrm{d}\fwd{P}}(X_{(0, 1]}|X_0 ) \frac{q_0(X_0)}{p_0(X_0)} \\
    &=  \frac{\mathrm{d}\fwd{Q}}{\mathrm{d}\fwd{P}}(X_{(0, 1]}|X_0 ) h_0(X_0)\\
    &= \exp(-\int \left[
            h_t^{-1}\partial_t h_t
            +
            h_t^{-1}\mathcal L_t h_t
        \right] \mathrm{d} t)
\end{align}
By \Cref{thm:prescribing_h_transform_marginal}, we have
\begin{align}
            h_t^{-1}\partial_t h_t
            +
            h_t^{-1}\mathcal L_t h_t &= \pi_t^{-1}\mathcal L_t^\dagger \pi_t - \pi_t^{-1}\partial_t \pi_t
\end{align}
which proves the Crooks fluctuation theorem.
\end{proof}

\paragraph{Score matching objectives}
\Cref{thm:prescribing_h_transform_marginal} also connects with score-matching objectives.
Different from the Jarzynskis discussed above, where we fix the path $\pi_t$, we aim to learn a path of marginal densities  $\varphi_t$ (in the form of score or concrete score) to align with the marginal law of some pre-defined forward process.
Assume the diffusion model has a forward generator $\mathcal{L}$, corresponding to a known forward corruption process with marginal $p_t$, and we aim to learn $\varphi_t=p_t$.

\paragraph{Explicit score matching}
Let's consider the generator obtained by taking infinitesimal reversal w.r.t $p_t$ and $\varphi_t$:
\begin{align}
    \bwd{\mathcal{L}}^{p_t}f = p_t^{-1}\mathcal{L}^\dagger(p_tf) - p_t^{-1}\mathcal{L}^\dagger(p_t ) f, \quad     \bwd{\mathcal{L}}^{\varphi_t}f = \varphi_t^{-1}\mathcal{L}^\dagger(\varphi_t f) - \varphi_t^{-1}\mathcal{L}^\dagger(\varphi_t) f
\end{align}
Using results in Proposition~\ref{cor:relation_2_reversal}, we have that $\bwd{\mathcal{L}}^{p_t}$ is the Doob's-$h=p_t/\varphi_t$ transform of 
$\bwd{\mathcal{L}}^{\varphi_t}$.

Additionally, as we require $p_t$ to be the true marginal, using \Cref{thm:prescribing_h_transform_marginal}, we have
\begin{align}
    \frac{(-\partial_t+ \bwd{\mathcal{L}}^{\varphi_t})h_t}{h_t}
    =
    \frac{\mathcal (\bwd{\mathcal{L}}^{\varphi_t})^\dagger\varphi_t + \partial_t \varphi_t}{\varphi_t}.
\end{align}
Note the sign flip as the process is running backward in time.
If we want $\varphi_t$ to be the marginal, i.e., $(\bwd{\mathcal{L}}^{\varphi_t})^\dagger\varphi_t + \partial_t \varphi_t=0$, it suffices for us to impose $\frac{(-\partial_t+ \bwd{\mathcal{L}}^{\varphi_t})h_t}{h_t}=0$. 
\par
Taking expectations over time and state space:
\begin{align}
    &\int_0^1\E_{p_t} \left[\frac{(\partial_t - \bwd{\mathcal{L}}^{\varphi_t})h_t}{h_t} \right]\mathrm{d} t, \quad 
\end{align}
Note that
\begin{align}
   \int p_t(x) h_t^{-1}  \bwd{\mathcal{L}}^{\varphi_t} (h_t) \mathrm d x &=  \int p_t h_t^{-1}  \varphi_t^{-1}\mathcal{L}^\dagger(\varphi_t h_t)\mathrm d x -  \int p_t h_t^{-1}  \varphi_t^{-1}\mathcal{L}^\dagger(\varphi_t) h_t\mathrm d x\\
   &=\int p_t h_t^{-1}  \varphi_t^{-1} \partial_t p_t\mathrm d x -  \int p_t h_t^{-1}  \varphi_t^{-1}\mathcal{L}^\dagger(\varphi_t) h_t\mathrm d x\\
   &=\cancel{\int  \partial_t p_t\mathrm d x }-  \int p_t h_t^{-1}  \varphi_t^{-1}\mathcal{L}^\dagger(\varphi_t) h_t\mathrm d x\\
   &=-  \int h_t\mathcal{L}^\dagger(\varphi_t)  \mathrm d x\\
   &=-  \int \varphi_t \mathcal{L} (h_t)  \mathrm d x\\
    &=-  \int p_t h_t^{-1} \mathcal{L} (h_t)  \mathrm d x
\end{align}
Hence, 
\begin{align}
    &\int_0^1\E_{p_t} \left[\frac{(\partial_t - \bwd{\mathcal{L}}^{\varphi_t})h_t}{h_t} \right]\mathrm{d} t\\
    =&\int_0^1\E_{p_t} \left[\frac{(\partial_t +  {\mathcal{L}_t})h_t}{h_t} \right]\mathrm{d} t\\
   =&\int_0^1\E_{p_t} \left[ \partial_t \log h_t  +  h_t ^{-1}{\mathcal{L}_t} h_t \right]\mathrm{d} t
\end{align} 
With Dynkin's formula, we have $\mathrm{d} \log h_t = \partial_t \log h_t \mathrm{d} t  + \mathcal{L}_t \log h_t \mathrm{d} t  + \mathrm{d} M$, and hence
\begin{align}
   &\int_0^1\E_{p_t} \left[ \partial_t \log h_t  +  h_t ^{-1}{\mathcal{L}_t} h_t \right]\mathrm{d} t\\ =&    \int_0^1\E_{p_t} \left[ \mathrm{d} \log h_t  +  h_t ^{-1}{\mathcal{L}_t} h_t \mathrm{d} t  - \mathcal{L}_t \log h_t \mathrm{d} t  \right]\\
   =&\underbrace{\int_0^1\E_{p_t} \left[ h_t ^{-1}{\mathcal{L}_t} h_t \mathrm{d} t  - \mathcal{L}_t \log h_t \mathrm{d} t  \right]}_{\text{ESM}} + \E_{p_1}[\log h_1] - \E_{p_0}[\log h_0] 
\end{align}

\paragraph{Denoising score matching}
Consider the generator obtained by taking infinitesimal reversal w.r.t $p_t(x_t|x_0)$ and $\varphi_t$:
\begin{align}
    &\bwd{\mathcal{L}}^{p_t(\cdot|x_0)}f = p_t(\cdot|x_0)^{-1}\mathcal{L}^\dagger(p_t(\cdot|x_0)f) - p_t(\cdot|x_0)^{-1}\mathcal{L}^\dagger(p_t(\cdot|x_0) ) f, \\    &\bwd{\mathcal{L}}^{\varphi_t}f = \varphi_t^{-1}\mathcal{L}^\dagger(\varphi_t f) - \varphi_t^{-1}\mathcal{L}^\dagger(\varphi_t) f
\end{align}
Using results in Proposition~\ref{cor:relation_2_reversal}, we have that $\bwd{\mathcal{L}}^{p_t(\cdot|x_0)}$ is the Doob's-$h=p_t(\cdot|x_0)/\varphi_t$ transform of 
$\bwd{\mathcal{L}}^{\varphi_t}$.
As $p_t(\cdot|x_0)$ is the true marginal for the backward process, using \Cref{thm:prescribing_h_transform_marginal}, we have
\begin{align}
    \frac{(-\partial_t+ \bwd{\mathcal{L}}^{\varphi_t})h_t}{h_t}
    =
    \frac{\mathcal (\bwd{\mathcal{L}}^{\varphi_t})^\dagger\varphi_t + \partial_t \varphi_t}{\varphi_t}.
\end{align}
If we want $\varphi_t$ to be the marginal, i.e., $(\bwd{\mathcal{L}}^{\varphi_t})^\dagger\varphi_t + \partial_t \varphi_t=0$, it suffices for us to impose $\frac{(-\partial_t+ \bwd{\mathcal{L}}^{\varphi_t})h_t}{h_t}=0$. 
Taking expectations over time and state space:
\begin{align}
    &\int_0^1\E_{x_0, x_t} \left[\frac{(\partial_t - \bwd{\mathcal{L}}^{\varphi_t})h_t}{h_t} \right]\mathrm{d} t, \quad 
\end{align}
Similar to the ESM case, we have
\begin{align}
  & \int p_0(x)p_t(x_t|x_0) h_t^{-1}  \bwd{\mathcal{L}}^{\varphi_t} (h_t) \mathrm d x \\=  &\int p_0(x)p_t(x_t|x_0) h_t^{-1}  \varphi_t^{-1}\mathcal{L}^\dagger(\varphi_t h_t)\mathrm d x -  \int p_0(x)p_t(x_t|x_0) h_t^{-1}  \varphi_t^{-1}\mathcal{L}^\dagger(\varphi_t) h_t\mathrm d x\\
    &=-  \int p_{0}p_{t}(x_t|x_0) h_t^{-1} \mathcal{L} (h_t)  \mathrm d x
\end{align}
We then apply Dynkin's formula to obtain the DSM formulation with boundary conditions.

\subsection{Proof of \Cref{thm:group_iso}}\label{app:group_iso_proof}

\begin{proof}
The algebraic identities
\begin{align}
    H_hH_{h'} &= H_{hh'}, &
    H_hR_q &= R_{hq}, &
    R_pH_h &= R_{p/h}, &
    R_pR_q &= H_{p/q}
\end{align}
show that $\mathcal G$ is closed under composition.
Associativity follows from associativity of composition.
The identity element is $H_1=I$, and the inverses are
\begin{align}
    H_h^{-1}=H_{1/h},
    \qquad
    R_p^{-1}=R_p.
\end{align}
Thus $\mathcal G$ is a group.

Now define
\begin{align}
    \Phi:\mathcal G\to A\rtimes \mathbb Z_2
\end{align}
by
\begin{align}
    \Phi(H_h)=([h],0),
    \qquad
    \Phi(R_p)=([p],1),
\end{align}
where $[h]$ denotes the equivalence class of $h$ modulo positive constants.
The multiplication in $A\rtimes \mathbb Z_2$ is
\begin{align}
    (a,z_1)(b,z_2)
    =
    \bigl(a b^{(-1)^{z_1}},z_1+z_2 \bmod 2\bigr).
\end{align}
Using this rule, the four products above are mapped respectively to
\begin{align}
    ([h],0)([h'],0) &= ([hh'],0),\\
    ([h],0)([q],1) &= ([hq],1),\\
    ([p],1)([h],0) &= ([p/h],1),\\
    ([p],1)([q],1) &= ([p/q],0).
\end{align}
Hence $\Phi$ preserves multiplication.
It is also bijective, since every element of $A\rtimes\mathbb Z_2$ is uniquely of the form $([h],0)$ or $([p],1)$.
Therefore,
\begin{align}
    \mathcal G \cong A\rtimes \mathbb Z_2
    =
    \operatorname{Dih}(A).
\end{align}
\end{proof}

\subsection{Discrete-time Generalized EJE and Crooks}
\begin{proposition}[Discrete-time Jarzynski equality and Crooks relation, restate]
\label{prop:discretetime_restate}
Let $\{U_n\}_{n=0}^N$ be an energy path with $U_0=U_A$ and $U_N=U_B$, and define
\begin{align}
    \gamma_n(x)
    :=
    e^{-U_n(x)},
    \qquad
    Z_n
    :=
    \int \gamma_n(x)\,\mathrm dx,
    \qquad
    \pi_n(x)
    :=
    \frac{\gamma_n(x)}{Z_n}.
\end{align}
Let $\{P_{F,n}\}_{n=1}^N$ be forward transition kernels, and let $\fwd P$ be
the forward path law
\begin{align}
    X_0\sim \pi_A\propto e^{-U_A},
    \qquad
    X_n\mid X_{n-1}\sim P_{F,n}(\cdot\mid X_{n-1}),
    \qquad n=1,\dots,N,
\end{align}
 Define the pushforward density
\begin{align}
    (P_{F,n}\#\pi_{n-1})(x)
    :=
    \int P_{F,n}(x\mid x')\,\pi_{n-1}(x')\,\mathrm dx'.
\end{align}
Then the discrete-time generalized work is
\begin{align}
    W(X_{0:N})
   =
    \sum_{n=1}^N
    \left[
        U_n(X_n)-U_{n-1}(X_n)
        +
        \log
        \frac{(P_{F,n}\#\pi_{n-1})(X_n)}
        {\pi_{n-1}(X_n)}
    \right],
    \label{eq:general_work_disc_restate}
\end{align}
 The discrete-time
Jarzynski equality is
\begin{align}
    \Delta F
    =
    -\log \frac{Z_B}{Z_A}
    =
    -\log
    \mathbb E_{\fwd P}
    \left[
        \exp\bigl(-W(X_{0:N})\bigr)
    \right].
    \label{eq:jarzynski_discrete_time_restate}
\end{align}

Moreover, assume that there exist backward kernels
$\{P_{B,n}\}_{n=1}^N$ satisfying, for each $n=1,\dots,N$,
\begin{align}
    \pi_{n-1}(x_{n-1})\,P_{F,n}(x_n\mid x_{n-1})
    =
    (P_{F,n}\#\pi_{n-1})(x_n)\,
    P_{B,n}(x_{n-1}\mid x_n).
    \label{eq:discrete_requirement_restate}
\end{align}
Let $\bwd Q$ be the corresponding backward path law
\begin{align}
    X_N\sim \pi_B\propto e^{-U_B},
    \qquad
    X_{n-1}\mid X_n\sim P_{B,n}(\cdot\mid X_n),
    \qquad n=N,\dots,1,
\end{align}
Then the discrete-time Crooks relation is
\begin{align}
    \frac{\mathrm d\bwd Q}{\mathrm d\fwd P}(X_{0:N})
    &=
    \frac{
        \pi_B(X_N)\prod_{n=1}^N P_{B,n}(X_{n-1}\mid X_n)
    }{
        \pi_A(X_0)\prod_{n=1}^N P_{F,n}(X_n\mid X_{n-1})
    } \\
    &=
    \exp\bigl(-W(X_{0:N})+\Delta F\bigr).
    \label{eq:crooks_discrete_time}
\end{align}
\end{proposition}
\begin{proof}
    \begin{align} &\frac{
        \pi_B(X_N)\prod_{n=1}^N P_{B,n}(X_{n-1}\mid X_n)
    }{
        \pi_A(X_0)\prod_{n=1}^N P_{F,n}(X_n\mid X_{n-1})
    }\\
    =
        &\frac{ P_{B, 1}(x_0|x_{1})\cdots P_{B, N}(x_{N-1}|x_N)\pi_N(x_N)}{\pi_0(x_0) P_{F, 1}(x_1|x_{0})\cdots P_{F, N}(x_N|x_{N-1})}\\
        =&\frac{1}{\pi_0(x_0)}\frac{P_{B, 1}(x_0|x_{1})}{ P_{F, 1}(x_1|x_{0})} \cdots \frac{P_{B, N}(x_{N-1}|x_N)}{ P_{F, N}(x_N|x_{N-1})}\pi_N(x_N)\\
        =& \frac{1}{\pi_0(x_0)}
        \frac{P_{B, 1}(x_0|x_{1})}{ P_{F, 1}(x_1|x_{0})} \frac{\pi_0(x_0) P_{F, 1}(x_1|x_{0})} { P_{F, 1}\#\pi_0 (x_1) P_{B, 1}(x_0|x_{1})}
        \cdots \nonumber\\&  \quad \quad  \quad \quad \quad \quad \quad \quad  \cdots 
        \frac{P_{B, N}(x_{N-1}|x_N)}{ P_{F, N}(x_N|x_{N-1})} \frac{\pi_{N-1}(x_{N-1}) P_{F, N}(x_N|x_{N-1})}  { P_{F, N}\#\pi_{N-1} (x_N) P_{B, N}(x_{N-1}|x_{N})}
        \pi_N(x_N)\\
        =&\frac{1}{\pi_0(x_0)}
        \frac{\pi_0(x_0)} { P_{F, 1}\#\pi_0 (x_1)}  \frac{\pi_1(x_1)} { P_{F, 2}\#\pi_1 (x_2)}
        \cdots 
       \frac{\pi_{N-1}(x_{N-1}) }  { P_{F, N}\#\pi_{N-1} (x_N) }
        \pi_N(x_N)\\
         =&\frac{1}{\pi_0(x_0)}
        \frac{\pi_0(x_0)\pi_0(x_1)} { P_{F, 1}\#\pi_0 (x_1)\pi_0(x_1)}  
        \frac{\pi_1(x_1) \pi_1(x_2)} { P_{F, 2}\#\pi_1 (x_2)\pi_1(x_2)}
        \cdots 
       \frac{\pi_{N-1}(x_{N-1}) \pi_{N-1}(x_N)}  { P_{F, N}\#\pi_{N-1} (x_N) \pi_{N-1}(x_N)}
        \pi_N(x_N)\\
        =& \frac{ \pi_0(x_1)} { P_{F, 1}\#\pi_0 (x_1) } 
         \frac{\pi_1(x_1) } { \pi_0(x_1)}
        \frac{\pi_1(x_2)} { P_{F, 2}\#\pi_1 (x_2)}\frac{\pi_2(x_2)}{\pi_1(x_2)}
        \cdots 
       \frac{\pi_{N-1}(x_N)}  { P_{F, N}\#\pi_{N-1} (x_N)}
     \frac{ \pi_N(x_N)}{ \pi_{N-1}(x_N)}
    \end{align}
  Therefore, taking log-space and canceling normalization factors for $\pi$-s, we have
\begin{align}
    W = (U_n(x_n) - U_{n-1}(x_n) ) + \log\frac{P_{F, n}\#\pi_{n-1} (x_n)}{\pi_{n-1} (x_n)} 
\end{align}
Then, taking expectation over the work, we obtain the discrete-time Jarzynski.
\end{proof}

\subsection{Proof of Remark~\ref{remark:variance_bound} (Variance-dissipation Relation)}

\begin{remark}[Variance-dissipation Relation, restate]
    Denote the  free energy estimate  as  \(\widehat {Z} := \widehat {\exp(-\Delta F)} = \frac1N \sum_{i=1}^N \exp({-\widetilde{W}(X^{(i)}_{[0,1]})})\) with \(X^{(i)}_{[0,1]} \sim \fwd{P}\). Then
    \begin{align}
        \Var(\widehat Z)=
\mathcal{O}\left(\tfrac1N
\chi^2(\bwd{Q}\,\|\,\fwd{P})  
\right)\geq \mathcal{O}\left(\tfrac1ND_{\mathrm{KL}}(\bwd{Q}\,\|\,\fwd{P}) \right)
    \end{align}
\end{remark}
\begin{proof}
    \begin{align}
        \Var(\widehat Z) &= \E\left[( \frac1N \sum_{i=1}^N \exp(-W^{(i)}))^2 \right] + \mathrm{const.}\\
        &=\E\left[( \frac1N \sum_{i=1}^N \frac{\mathrm{d}\bwd{Q}}{\mathrm{d}\fwd{P}})^2 \right]+ \mathrm{const.}\\
       &= \frac1N\Var( \frac{\mathrm{d}\bwd{Q}}{\mathrm{d}\fwd{P}})+ \mathrm{const.}\\
       &= \frac1N\int \left(\frac{\mathrm{d}\bwd{Q}}{\mathrm{d}\fwd{P}}-1\right)^2\mathrm{d}\fwd{P} + \mathrm{const.}\\
       &=\mathcal{O}\left(\frac1N
\chi^2(\bwd{Q}\,\|\,\fwd{P})  
\right)
    \end{align}
    
We then prove the lower bound of this.
Let
\begin{align}
    r
    =
    \frac{\mathrm d\bwd Q}{\mathrm d\fwd P}.
\end{align}
and we have 
\begin{align}
    D_{\mathrm{KL}}(\bwd Q\|\fwd P)
    &=
    \int r\log r\,\mathrm d\fwd P,\\
    \chi^2(\bwd Q\|\fwd P)
    &=
    \int (r-1)^2\,\mathrm d\fwd P
    =
    \int r^2\,\mathrm d\fwd P - 1.
\end{align}
We note that 
\begin{align}
    \log u \le u-1,
    \qquad u>0.
\end{align}
and 
\begin{align}
    u\log u
    \le
    u(u-1)
    =
    u^2-u .
\end{align}
Set $u$ to $r$ and take expectation, 
\begin{align}
    D_{\mathrm{KL}}(\bwd Q\|\fwd P)
    &=
    \int r\log r\,\mathrm d\fwd P \\
    &\le
    \int (r^2-r)\,\mathrm d\fwd P \\
    &=
    \int r^2\,\mathrm d\fwd P - \int r\,\mathrm d\fwd P \\
    &=
    \int r^2\,\mathrm d\fwd P - 1 \\
    &=
    \chi^2(\bwd Q\|\fwd P),
\end{align}

\end{proof}

\subsection{Proof of Proposition~\ref{prop:marginal_reversal} (marginalization maintains time-reversal)}
\label{app:proof_marginal_time_rev}

\begin{proposition}[Marginalization maintains time-reversal, restate]
Let \(p_t(x_t|x_0,x_1)\) be a conditional interpolant, and let
\(\mathcal{L}_{F,t}^{x_0,x_1}\) and \(\mathcal{L}_{B,t}^{x_0,x_1}\)
be the corresponding forward and backward conditional generators. Assume that, for every \((x_0,x_1)\), they are time reversals of each other, i.e.
\begin{align}
    \resizebox{0.93\linewidth}{!}{$
    \mathcal{L}_{B,t}^{x_0,x_1} f
    =
    \frac{1}{p_t(x_t|x_0,x_1)}
    (\mathcal{L}_{F,t}^{x_0,x_1})^\dagger
    \!\left(
        p_t(x_t|x_0,x_1) f
    \right)
    -
    \frac{1}{p_t(x_t| x_0,x_1)}
    (\mathcal{L}_{F,t}^{x_0,x_1})^\dagger
    p_t(x_t|x_0,x_1)\, f
    $}
\end{align}
Then \(\mathcal{L}_{F,t}\) and \(\mathcal{L}_{B,t}\) defined by \Cref{eq:marginal_generator} are also time-reversal with respect to \(p_t(x_t)\), namely
\begin{align}\label{eq:time-reversal_restate}
    \mathcal{L}_{B,t} f
    =
    \frac{1}{p_t}
    \mathcal{L}_{F,t}^\dagger (p_t f)
    -
    \frac{1}{p_t}
    \mathcal{L}_{F,t}^\dagger p_t \, f .
\end{align}
\end{proposition}

\begin{proof}
Write
\begin{align}
    z := (x_0,x_1),
    \qquad
    p_t^z(x) := p_t(x\mid x_0,x_1),
    \qquad
    \pi(z) := \pi(x_0,x_1).
\end{align}
The marginal density is
\begin{align}
    p_t(x)
    =
    \int p_t^z(x)\,\pi(z)\,\mathrm dz,
\end{align}
and the posterior over endpoints given the current state is
\begin{align}
    \pi_t(z\mid x)
    =
    \frac{\pi(z)p_t^z(x)}{p_t(x)}.
\end{align}
By the definition of the marginal generators in \Cref{eq:marginal_generator},
\begin{align}
    \mathcal L_{F,t}f(x)
    =
    \int
    \mathcal L_{F,t}^z f(x)\,
    \pi_t(z\mid x)\,\mathrm dz,
    \qquad
    \mathcal L_{B,t}f(x)
    =
    \int
    \mathcal L_{B,t}^z f(x)\,
    \pi_t(z\mid x)\,\mathrm dz .
\end{align}

We first compute the adjoint of the marginal forward generator. For any test
function $f$ and density $r$, we have
\begin{align}
    \int r(x)\,\mathcal L_{F,t}f(x)\,\mathrm dx
    &=
    \int r(x)
    \left[
        \int
        \mathcal L_{F,t}^z f(x)\,
        \pi_t(z\mid x)\,\mathrm dz
    \right]
    \mathrm dx \\
    &=
    \int
    \int
    r(x)\pi_t(z\mid x)\,
    \mathcal L_{F,t}^z f(x)
    \,\mathrm dx\,\mathrm dz \\
    &=
    \int
    \int
    f(x)\,
    (\mathcal L_{F,t}^z)^\dagger
    \!\left(
        r(\cdot)\pi_t(z\mid \cdot)
    \right)(x)
    \,\mathrm dx\,\mathrm dz \\
    &=
    \int
    f(x)
    \left[
        \int
        (\mathcal L_{F,t}^z)^\dagger
        \!\left(
            r(\cdot)\pi_t(z\mid \cdot)
        \right)(x)
        \,\mathrm dz
    \right]
    \mathrm dx .
\end{align}
Therefore,
\begin{align}
    \mathcal L_{F,t}^\dagger r(x)
    =
    \int
    (\mathcal L_{F,t}^z)^\dagger
    \!\left(
        r(\cdot)\pi_t(z\mid \cdot)
    \right)(x)
    \,\mathrm dz .
    \label{eq:marginal_adjoint_formula}
\end{align}

Now take $r=p_t f$. Since
\begin{align}
    p_t(x)\pi_t(z\mid x)
    =
    \pi(z)p_t^z(x),
\end{align}
we get
\begin{align}
    \mathcal L_{F,t}^\dagger(p_t f)(x)
    =
    \int
    \pi(z)
    (\mathcal L_{F,t}^z)^\dagger
    \!\left(
        p_t^z f
    \right)(x)
    \,\mathrm dz .
    \label{eq:marginal_adjoint_ptf}
\end{align}
We can move $\pi(z)$ outside the integral as the generator operates on $x$.
Similarly,
\begin{align}
    \mathcal L_{F,t}^\dagger p_t(x)
    =
    \int
    \pi(z)
    (\mathcal L_{F,t}^z)^\dagger
    p_t^z(x)
    \,\mathrm dz .
    \label{eq:marginal_adjoint_pt}
\end{align}
Therefore, the time-reversal of the marginal forward generator is
\begin{align}
    &\frac{1}{p_t(x)}
    \mathcal L_{F,t}^\dagger(p_t f)(x)
    -
    \frac{1}{p_t(x)}
    \mathcal L_{F,t}^\dagger p_t(x)\,f(x)
   \\ =
    &\frac{1}{p_t(x)}
    \int
    \pi(z)
    \left[
        (\mathcal L_{F,t}^z)^\dagger(p_t^z f)(x)
        -
        (\mathcal L_{F,t}^z)^\dagger p_t^z(x)\,f(x)
    \right]
    \mathrm dz .
    \label{eq:time_reversal_marginal_rhs}
\end{align}

On the other hand, we assume
\begin{align}
    \mathcal L_{B,t}^z f(x)
    =
    \frac{1}{p_t^z(x)}
    (\mathcal L_{F,t}^z)^\dagger(p_t^z f)(x)
    -
    \frac{1}{p_t^z(x)}
    (\mathcal L_{F,t}^z)^\dagger p_t^z(x)\,f(x).
\end{align}
Thus the marginal backward generator satisfies
\begin{align}
    \mathcal L_{B,t}f(x)
    &=
    \int
    \mathcal L_{B,t}^z f(x)\,
    \pi_t(z\mid x)\,\mathrm dz \\
    &=
    \int
    \left[
        \frac{1}{p_t^z(x)}
        (\mathcal L_{F,t}^z)^\dagger(p_t^z f)(x)
        -
        \frac{1}{p_t^z(x)}
        (\mathcal L_{F,t}^z)^\dagger p_t^z(x)\,f(x)
    \right]
    \frac{\pi(z)p_t^z(x)}{p_t(x)}
    \,\mathrm dz \\
    &=
    \frac{1}{p_t(x)}
    \int
    \pi(z)
    \left[
        (\mathcal L_{F,t}^z)^\dagger(p_t^z f)(x)
        -
        (\mathcal L_{F,t}^z)^\dagger p_t^z(x)\,f(x)
    \right]
    \mathrm dz .
\end{align}
Comparing this with \Cref{eq:time_reversal_marginal_rhs}, we obtain
\begin{align}
    \mathcal L_{B,t} f(x)
    =
    \frac{1}{p_t(x)}
    \mathcal L_{F,t}^\dagger(p_t f)(x)
    -
    \frac{1}{p_t(x)}
    \mathcal L_{F,t}^\dagger p_t(x)\,f(x).
\end{align}
\end{proof}

\subsection{Connection between Generator Matching and Denoising Score Matching}\label{app:GM_DSM}
We consider two approaches in the main text:
we either learn a backward process directly, or learn the marginal to reverse the forward process. 
We will show that learning the backward process with generator matching is the same as learning the marginal.

For simplicity, we consider the case of a denoising Markov model.
For the bridge model we discussed above, we could simply choose $x_0 = I_t$, and define a new axis of time $\tau$ from $I$ to $x$. 
Assume we have a forward generator $\mathcal{L}$, corresponding to a known forward process.
Assume the law of forward process, conditional on the initial sample $x_0$, is $p_t^{x_0}$.
The backward conditional generator is given by
\begin{align}
   \tilde{\mathcal{L}}_t^{x_0} f = (p_t^{x_0})^{-1} 
(\mathcal{L}_t)^\dagger (p_t^{x_0} f) - (p_t^{x_0})^{-1}  (\mathcal{L}_t)^\dagger  (p_t^{x_0})f,
\end{align}
We define another backward process as 
\begin{align}
    \tilde{\mathcal{L}}_t f = \varphi_t^{-1} 
(\mathcal{L}_t)^\dagger (\varphi_t f) - \varphi_t^{-1}  (\mathcal{L}_t)^\dagger  (\varphi_t)f, 
\end{align}

Recall the DSM objective:
\begin{align}
    \mathcal{J}_{\text{DSM}}(\varphi_t) =\mathbb{E}_{x_0, x_t} \left[
    \frac{\mathcal{L}_t( p(x_t|x_0) / \varphi_t(x_t) )}{p(x_t|x_0) / \varphi_t(x_t)} - \mathcal{L}_t\log (
p(x_t|x_0) / \varphi_t(x_t)
    )
    \right] 
\end{align}
We now assume an additive perturbation in log-space of $\varphi$:
\begin{align}
    \varphi'_t = \varphi_t \exp(-\epsilon \phi)
\end{align}
Then, we have
\begin{align}
     \mathcal{J}_{\text{DSM}}(\varphi'_t) =\mathbb{E}_{x_0| x_t} \left[
    \frac{\mathcal{L}_t( \exp(\epsilon \phi(x_t) ) p(x_t|x_0) / \varphi_t(x_t) )}{\exp(\epsilon \phi(x_t) ) p(x_t|x_0) / \varphi_t(x_t)} - \mathcal{L}_t\log (
\exp(\epsilon \phi(x_t) ) p(x_t|x_0) / \varphi_t(x_t)
    )
    \right] 
\end{align}
Take the gradient w.r.t. $\epsilon$, we have 
\begin{align}
\frac{\partial}{\partial \epsilon}\mathcal{J}_{\text{DSM}}(\varphi'_t)
\\=
\mathbb{E}_{x_0, x_t} \Bigg[
&
\frac{
\mathcal{L}_t\!\left(
\phi(x_t)\exp(\epsilon \phi(x_t))\, p(x_t|x_0) / \varphi_t(x_t)
\right)
}{
\exp(\epsilon \phi(x_t))\, p(x_t|x_0) / \varphi_t(x_t)
}
\\
&\quad -
\frac{
\mathcal{L}_t\!\left(
\exp(\epsilon \phi(x_t))\, p(x_t|x_0) / \varphi_t(x_t)
\right)
}{
\left(\exp(\epsilon \phi(x_t))\, p(x_t|x_0) / \varphi_t(x_t)\right)^2
}
\phi(x_t)\exp(\epsilon \phi(x_t))\, p(x_t|x_0) / \varphi_t(x_t)
\\
&\quad -
\mathcal{L}_t \phi(x_t)
\Bigg].
\end{align}
Hence,
\begin{align}
\frac{\partial}{\partial \epsilon}\mathcal{J}_{\text{DSM}}(\varphi'_t)
=
\mathbb{E}_{x_0, x_t} \Bigg[
&
\frac{
\mathcal{L}_t\!\left(
\phi(x_t)\exp(\epsilon \phi(x_t))\, p(x_t|x_0) / \varphi_t(x_t)
\right)
}{
\exp(\epsilon \phi(x_t))\, p(x_t|x_0) / \varphi_t(x_t)
}
\\
&\quad -
\frac{
\mathcal{L}_t\!\left(
\exp(\epsilon \phi(x_t))\, p(x_t|x_0) / \varphi_t(x_t)
\right)
}{
\exp(\epsilon \phi(x_t))\, p(x_t|x_0) / \varphi_t(x_t)
}\,\phi(x_t)
-
\mathcal{L}_t \phi(x_t)
\Bigg].
\end{align}
Therefore, the G\^ateaux derivative of DSM objective is given by
\begin{align}
    &\left.\frac{d}{d\epsilon}\mathcal J_{\mathrm{DSM}}(\varphi_t')\right|_{\epsilon=0} \\= &\mathbb{E}_{x_0, x_t} \Bigg[
\frac{
\mathcal{L}_t\!\left(
\phi(x_t)  p(x_t|x_0) / \varphi_t(x_t)
\right)
}{
 p(x_t|x_0) / \varphi_t(x_t)
}-
\frac{
\mathcal{L}_t\!\left(
 p(x_t|x_0) / \varphi_t(x_t)
\right)
}{
 p(x_t|x_0) / \varphi_t(x_t)
}\,\phi(x_t)
-
\mathcal{L}_t \phi(x_t)
\Bigg]\\
=& \mathbb{E}_{x_0} \Bigg[ \int
\frac{
\mathcal{L}_t\!\left(
\phi(x_t)  p(x_t|x_0) / \varphi_t(x_t)
\right)
}{
 1 / \varphi_t(x_t)
}-
\frac{
\mathcal{L}_t\!\left(
 p(x_t|x_0) / \varphi_t(x_t)
\right)
}{
 1 / \varphi_t(x_t)
}\,\phi(x_t)
-
p(x_t|x_0)  \mathcal{L}_t \phi(x_t)\d x_t
\Bigg]\\
=& \mathbb{E}_{x_0} \Bigg[ \int
\frac{
\mathcal{L}_t\!\left(
\phi(x_t)  p(x_t|x_0) / \varphi_t(x_t)
\right)
}{
 1 / \varphi_t(x_t)
} -
\frac{
\mathcal{L}_t\!\left(
 p(x_t|x_0) / \varphi_t(x_t)
\right)
}{
 1 / \varphi_t(x_t)
}\,\phi(x_t)
\\ &-
p(x_t|x_0)  \mathcal{L}_t \phi(x_t) + p(x_t|x_0) \phi(x_t)  \mathcal{L}_t 1 \d x_t
\Bigg]\\
=&\mathbb{E}_{x_0} \Bigg[ \int
\phi(x_t)  p(x_t|x_0) / \varphi_t(x_t)\mathcal{L}^\dagger_t\!\left(
\varphi_t(x_t)
\right) -
  p(x_t|x_0) / \varphi_t(x_t) \mathcal{L}_t^\dagger \left(
\varphi_t(x_t)\phi(x_t)
\right)
\\ &-
  \phi(x_t)  \mathcal{L}^\dagger_t p(x_t|x_0) +  \mathcal{L}^\dagger_t (p(x_t|x_0) \phi(x_t) )  \d x_t
\Bigg]\\
=& \mathbb{E}_{x_0, x_t} \Bigg[  
\phi(x_t)   / \varphi_t(x_t)\mathcal{L}^\dagger_t\!\left(
\varphi_t(x_t)
\right) -
1 / \varphi_t(x_t) \mathcal{L}_t^\dagger \left(
\varphi_t(x_t)\phi(x_t)
\right)
\\ &-
  \phi(x_t) /p(x_t|x_0) \mathcal{L}^\dagger_t p(x_t|x_0) +  1/p(x_t|x_0)\mathcal{L}^\dagger_t (p(x_t|x_0) \phi(x_t) )  
\Bigg]\\
=& \mathbb{E}_{x_0, x_t} [
-\tilde{\mathcal{L}}_t \phi + \tilde{\mathcal{L}}_t^{x_0} \phi
]
\end{align}
Therefore, optimizing the DSM objective is weakly equivalent to matching the generators.

\section{Concrete Instantiations and Examples of the Generalized EJE, Crooks and Learning Methods}
\label{app:everything_here}
In the main text, we state the generalized EJE and Crooks fluctuation theorem in an abstract Markov process form, and discuss learning strategies at the same level of generality.
Here, we instantiate these results in several representative settings.
We consider continuous Euclidean dynamics based on deterministic flows, overdamped SDEs, and momentum-augmented SDEs, as well as finite-state Continuous-Time Markov Chains (CTMC) dynamics.

\subsection{Instantiate the Generalized EJE and Crooks}

\subsubsection{Deterministic Flow in Continuous Euclidean Space}

We first instantiate Propositions~\ref{prop:jarzynski}~and~\ref{prop:crooks} for flow-based transport in Euclidean space.
Let \(X_t\in\mathbb R^d\) follow the time-dependent ODE
\begin{align}
    \frac{\mathrm d X_t}{\mathrm dt}
    =
    v_t(X_t),
    \qquad 
    X_0\sim \pi_A .
\end{align}
The corresponding generator is
\begin{align}
    \mathcal L_t^F f
    =
    v_t \cdot \nabla f,
    \qquad 
    (\mathcal L_t^F)^\dagger p
    =
    -\nabla\cdot(v_t p).
\end{align}
Therefore,
\begin{align}
    \frac{((\mathcal L_t^F)^\dagger \gamma_t)(x)}{\gamma_t(x)}
    &=
    -\frac{\nabla\cdot(v_t(x)\gamma_t(x))}{\gamma_t(x)}
\\
    &=
    -\nabla\cdot v_t(x)
    -
    v_t(x)\cdot \nabla \log \gamma_t(x)
\\
    &=
    -\nabla\cdot v_t(x)
    +
    v_t(x)\cdot \nabla U_t(x).
\end{align}
Plugging this into \Cref{eq:general_work_cont}, the generalized work becomes
\begin{align}
    \widetilde W(X_{[0,1]})
    =
    \int_0^1
    \left[
        \partial_t U_t(X_t)
        +
        v_t(X_t)\cdot \nabla U_t(X_t)
        -
        \nabla\cdot v_t(X_t)
    \right]\mathrm dt .
\end{align}
By the definition of total derivative, we have
\begin{align}
    \frac{\mathrm d}{\mathrm dt}U_t(X_t)
    =
    \partial_t U_t(X_t)
    +
    v_t(X_t)\cdot\nabla U_t(X_t),
\end{align}
and hence we have
\begin{align}\label{eq:work_flow}
    \widetilde W(X_{[0,1]})
    =
    U_1(X_1)-U_0(X_0)
    -
    \int_0^1 \nabla\cdot v_t(X_t)\,\mathrm dt .
\end{align}
Equivalently, if \(T\) denotes the deterministic map from \(X_0\) to \(X_1\), then
\begin{align}
    \widetilde W(X_{[0,1]})
    =
    U_1(T(X_0))-U_0(X_0)
    -
    \log\left|
        \det \nabla T(X_0)
    \right|.
\end{align}
This recovers the standard formulation of target FEP \citep{wirnsberger2020targeted,zhao2023bounding,deepbar}.

The corresponding backward generator from \Cref{eq:bwd_crooks_cont} can be derived similarly:
\begin{align}
    (\mathcal L_t^F)^\dagger(\gamma_t f)
    =
    -\nabla\cdot(v_t\gamma_t f),
    \qquad
    (\mathcal L_t^F)^\dagger\gamma_t
    =
    -\nabla\cdot(v_t\gamma_t).
\end{align}
Hence
\begin{align}
    \mathcal L_t^B f
    &=
    -\frac{1}{\gamma_t}\nabla\cdot(v_t\gamma_t f)
    +
    \frac{f}{\gamma_t}\nabla\cdot(v_t\gamma_t)
\\
    &=
    -\frac{1}{\gamma_t}
    \left[
        f\nabla\cdot(v_t\gamma_t)
        +
        \gamma_t v_t\cdot \nabla f
    \right]
    +
    \frac{f}{\gamma_t}\nabla\cdot(v_t\gamma_t)
\\
    &=
    -v_t\cdot \nabla f .
\end{align}
Therefore,
\begin{align}
    (\mathcal L_t^B)^\dagger q
    =
    \nabla\cdot(v_t q).
\end{align}
The backward marginal equation in Proposition~\ref{prop:crooks} is then
\begin{align}
    \partial_t q_t
    =
    -(\mathcal L_t^B)^\dagger q_t
    =
    -\nabla\cdot(v_t q_t),
    \qquad 
    q_1=\pi_B .
\end{align}
Equivalently, the backward path is the same as the forward path:
\begin{align}
    \frac{\mathrm d X_t}{\mathrm dt}
    =
    v_t(X_t),
    \qquad 
    X_1\sim \pi_B,
\end{align}
but running from \(t=1\) to \(t=0\).

However, in the deterministic-flow setting, the forward and backward transition kernels are Dirac measures. 
Therefore, the path-space RND cannot be computed through a ratio of transition densities, and we can only calculate the work with \Cref{eq:work_flow}.
Therefore, this strategy is typically computationally expensive due to the divergence term in \Cref{eq:work_flow}.

\subsubsection{Diffusion Process in Continuous Euclidean Space}

We next instantiate Propositions~\ref{prop:jarzynski}~and~\ref{prop:crooks} for diffusion-based transport, which recover the stochastic interpolant \citep{albergo2025stochastic}.
Precisely, we use the following process:
\begin{align}
    \mathrm d X_t
    =
    b_t(X_t)\,\mathrm dt
    +
    \sqrt{2}\sigma_t\,\fwd{\mathrm d W_t},
    \qquad
    X_0\sim \pi_A .
\end{align}
The corresponding generator and adjoint are
\begin{align}
    \mathcal L_t^F f
    =
    b_t\cdot \nabla f
    +
    \sigma_t^2 \Delta f,
    \qquad
    (\mathcal L_t^F)^\dagger p
    =
    -\nabla\cdot(b_t p)
    +
    \sigma_t^2 \Delta p .
\end{align}
Plugging this generator into \Cref{eq:bwd_crooks_cont}, we obtain
\begin{align}
    \mathcal L_t^B f
    &=
    \frac{1}{\gamma_t}
    \left[
        -\nabla\cdot(b_t\gamma_t f)
        +
        \sigma_t^2 \Delta(\gamma_t f)
    \right]
    -
    \frac{f}{\gamma_t}
    \left[
        -\nabla\cdot(b_t\gamma_t)
        +
        \sigma_t^2 \Delta\gamma_t
    \right]
\\
    &=
    -b_t\cdot\nabla f
    +
    \sigma_t^2 \Delta f
    +
    2\sigma_t^2 \nabla\log\gamma_t\cdot\nabla f
\\
    &=
    \left(
        -b_t
        -
        2\sigma_t^2\nabla U_t
    \right)\cdot\nabla f
    +
    \sigma_t^2\Delta f .
\end{align}
Thus, if the forward drift is \(b_t\), the backward generator has drift
\begin{align}
    -b_t - 2\sigma_t^2\nabla U_t .
\end{align}
i.e., the diffusion running backward in time is 
\begin{align}
    \mathrm d X_t
    =
   ( b_t(X_t) + 2\sigma_t^2\nabla U_t )\,\mathrm dt
    +
    \sqrt{2}\sigma_t\,\bwd{\mathrm d W_t},
    \qquad
    X_1\sim \pi_B .
\end{align}

A more symmetric parametrization is to decompose \(b_t\):
\begin{align}
    \mathrm d X_t
    =
    \left(
        v_t(X_t)
        -
        \sigma_t^2\nabla U_t(X_t)
    \right)\mathrm dt
    +
    \sqrt{2}\sigma_t\,\fwd{\mathrm d W_t},
    \qquad
    X_0\sim \pi_A .
\end{align}
By the calculation above, the corresponding backward dynamics is 
\begin{align}
    \mathrm d X_t
    =
    \left(
        v_t(X_t)
        +
        \sigma_t^2\nabla U_t(X_t)
    \right)\mathrm dt
    +
    \sqrt{2}\sigma_t\,\bwd{\mathrm d  W_t},
    \qquad
    X_1\sim \pi_B ,
\end{align}

We now derive the generalized work for this symmetric choice.
The forward generator is
\begin{align}
    \mathcal L_t^F f
    =
    \left(
        v_t-\sigma_t^2\nabla U_t
    \right)\cdot\nabla f
    +
    \sigma_t^2\Delta f .
\end{align}
and hence
\begin{align}
   &\quad  \frac{((\mathcal L_t^F)^\dagger \gamma_t)(x)}{\gamma_t(x)}
   \\ &=
    -\nabla\cdot
    \left(
        v_t(x)-\sigma_t^2\nabla U_t(x)
    \right)+
    \left(
        v_t(x)-\sigma_t^2\nabla U_t(x)
    \right)
    \cdot\nabla U_t(x)+
    \sigma_t^2
    \left(
        \|\nabla U_t(x)\|^2
        -
        \Delta U_t(x)
    \right)
\\
    &=
    -\nabla\cdot v_t(x)
    +
    v_t(x)\cdot\nabla U_t(x).
\end{align}
and hence the generalized work simplifies to
\begin{align}
    \widetilde W(X_{[0,1]})
    =
    \int_0^1
    \left[
        \partial_t U_t(X_t)
        +
       v_t(X_t)\cdot\nabla U_t(X_t)
        -
        \nabla\cdot v_t(X_t)
    \right]\mathrm dt .
\end{align}
Thus, under the symmetric diffusion parametrization, we recover the standard formula of escorted Jarzynski \citep{vaikuntanathan2008escorted}.

\subsubsection{Momentum-Augmented Transport in Continuous Space}\label{sec:underdamp_form}

We next instantiate Propositions~\ref{prop:jarzynski}~and~\ref{prop:crooks} for momentum-augmented transport.
Let \(Z_t=(X_t,Y_t)\in\mathbb R^d\times\mathbb R^d\), and define an augmented energy $H_t(x,y)$ on the joint space, as well as the unnormalized density
\begin{align}
    \Gamma_t(x,y) = e^{-H_t(x,y)}, \quad     \Pi_t(x,y) =\frac{ e^{-H_t(x,y)} }{\int e^{-H_t(x,y)} \mathrm{d}xy}
\end{align}
For two endpoints, we can set $Y_0 \perp X_0, Y_0 \sim \mathcal{N}$ and $Y_0 \perp X_1, Y_1 \sim \mathcal{N}$, so that the free energy between $H_0$ and $H_1$ remains the same as the free energy between $U_0$ and $U_1$.

Consider the forward momentum-augmented diffusion
\begin{align}
    \mathrm d X_t
    &=
    Y_t\,\mathrm dt,
    \\
    \mathrm d Y_t
    &=
    B_t(X_t,Y_t)\,\mathrm dt
    +
    \sqrt{2}\sigma_t\,\fwd{\mathrm d W_t},
    \qquad
    (X_0,Y_0)\sim \Pi_A .
\end{align}
The corresponding generator and adjoint are
\begin{align}
    \mathcal L_t^F f
    &=
    y\cdot\nabla_x f
    +
    B_t\cdot\nabla_y f
    +
    \sigma_t^2\Delta_y f,
    \\
    (\mathcal L_t^F)^\dagger p
    &=
    -\nabla_x\cdot(yp)
    -
    \nabla_y\cdot(B_t p)
    +
    \sigma_t^2\Delta_y p .
\end{align}
Plugging this generator into \Cref{eq:bwd_crooks_cont}, we obtain
\begin{align}
    \mathcal L_t^B f
    &=
    \frac{1}{\Gamma_t}
    \left[
        -\nabla_x\cdot(y\Gamma_t f)
        -
        \nabla_y\cdot(B_t\Gamma_t f)
        +
        \sigma_t^2\Delta_y(\Gamma_t f)
    \right]
    \notag\\
    &\quad
    -
    \frac{f}{\Gamma_t}
    \left[
        -\nabla_x\cdot(y\Gamma_t)
        -
        \nabla_y\cdot(B_t\Gamma_t)
        +
        \sigma_t^2\Delta_y\Gamma_t
    \right]
    \\
    &=
    -y\cdot\nabla_x f
    -
    B_t\cdot\nabla_y f
    +
    \sigma_t^2\Delta_y f
    +
    2\sigma_t^2\nabla_y\log\Gamma_t\cdot\nabla_y f
    \\
    &=
    -y\cdot\nabla_x f
    +
    \left(
        -B_t
        -
        2\sigma_t^2\nabla_y H_t
    \right)\cdot\nabla_y f
    +
    \sigma_t^2\Delta_y f .
\end{align}
where $\nabla_yH_t =- \nabla_y \log \Gamma_t$.
Therefore, the diffusion running backward in time is
\begin{align}
    \mathrm d X_t
    &=
    Y_t\, \bwd{\mathrm d t},
    \\
    \mathrm d Y_t
    &=
    \left(
        B_t(X_t,Y_t)
        +
        2\sigma_t^2\nabla_y H_t(X_t,Y_t)
    \right)\mathrm dt
    +
    \sqrt{2}\sigma_t\,\bwd{\mathrm d W_t},
    \qquad
    (X_1,Y_1)\sim \Pi_B .
\end{align}

Similar to the diffusion case, a more symmetric parametrization is to decompose the momentum drift:
\begin{align}
    B_t(x,y)
    =
    v_t(x,y)
    -
    \sigma_t^2\nabla_y H_t(x,y).
\end{align}
Then the forward dynamics becomes
\begin{align}
    \mathrm d X_t
    &=
    Y_t\,\mathrm dt,
    \\
    \mathrm d Y_t
    &=
    \left(
        v_t(X_t,Y_t)
        -
        \sigma_t^2\nabla_y H_t(X_t,Y_t)
    \right)\mathrm dt
    +
    \sqrt{2}\sigma_t\,\fwd{\mathrm d W_t},
    \qquad
    (X_0,Y_0)\sim \Pi_A .
\end{align}
By the calculation above, the corresponding backward dynamics is
\begin{align}
    \mathrm d X_t
    &=
    Y_t\,\bwd{\mathrm dt},
    \\
    \mathrm d Y_t
    &=
    \left(
        v_t(X_t,Y_t)
        +
        \sigma_t^2\nabla_y H_t(X_t,Y_t)
    \right)\mathrm dt
    +
    \sqrt{2}\sigma_t\,\bwd{\mathrm d W_t},
    \qquad
    (X_1,Y_1)\sim \Pi_B .
\end{align}

We now derive the generalized work for this symmetric choice.
The forward generator is
\begin{align}
    \mathcal L_t^F f
    &=
    y\cdot\nabla_x f
    +
    \left(
        v_t-\sigma_t^2\nabla_y H_t
    \right)\cdot\nabla_y f
    +
    \sigma_t^2\Delta_y f .\\
     (\mathcal L_t^F)^\dagger p
    &=
    -\nabla_x\cdot(yp)
    -
    \nabla_y\cdot( \left(
        v_t-\sigma_t^2\nabla_y H_t
    \right) p)
    +
    \sigma_t^2\Delta_y p .
\end{align}
Therefore, using 
 \(\nabla_x\cdot y=0\), we have
\begin{align}
    &\quad
    \frac{((\mathcal L_t^F)^\dagger\Gamma_t)(x,y)}{\Gamma_t(x,y)}
    \\
    &=
    -\nabla_x\cdot y
    +
    y\cdot\nabla_x H_t(x,y)
    -
    \nabla_y\cdot
    \left(
        v_t(x,y)-\sigma_t^2\nabla_y H_t(x,y)
    \right)
    \notag\\
    &\quad
    +
    \left(
        v_t(x,y)-\sigma_t^2\nabla_y H_t(x,y)
    \right)\cdot\nabla_y H_t(x,y)
    +
    \sigma_t^2
    \left(
        \|\nabla_y H_t(x,y)\|^2
        -
        \Delta_y H_t(x,y)
    \right)
    \\
    &=
    y\cdot\nabla_x H_t(x,y)
    -
    \nabla_y\cdot v_t(x,y)
    +
    v_t(x,y)\cdot\nabla_y H_t(x,y),
\end{align}
Hence the generalized work simplifies to
\begin{align}
    \widetilde W(Z_{[0,1]})
    =
    \int_0^1
    \left[
        \partial_t H_t(Z_t)
        +
        Y_t\cdot\nabla_x H_t(Z_t)
        +
        v_t(Z_t)\cdot\nabla_y H_t(Z_t)
        -
        \nabla_y\cdot v_t(Z_t)
    \right]\mathrm dt .
\end{align}

\paragraph{Integrator of Momentum-Augmented Transport}

Unlike SDE without momentum, we now have two parts: one is the ODE for sample position, and the other is the SDE for momentum.
We now describe how to discretize these dynamics and how to calculate the forward-backward path RND.
Following the notation of \citet{blessing2025underdamped}, the dynamics can be split into only two parts:
\begin{align}
    \begin{bmatrix}
        \mathrm d X_t\\
        \mathrm d Y_t
    \end{bmatrix}
    =
    \underbrace{
    \begin{bmatrix}
        Y_t\\
        0
    \end{bmatrix}
    \mathrm dt}_{A}
    +
    \underbrace{
    \begin{bmatrix}
        0\\
        \left(
            v_t(X_t,Y_t)
            -
            \sigma_t^2\nabla_y H_t(X_t,Y_t)
        \right)\mathrm dt
        +
        \sqrt{2}\sigma_t\,\fwd{\mathrm d W_t}
    \end{bmatrix}}_{O}.
\end{align}
\(A\) is the deterministic position update, and \(O\) is the stochastic momentum update.
For compactness, define
\begin{align}
    G_t^F(x,y)
    :=
    v_t(x,y)
    -
    \sigma_t^2\nabla_y H_t(x,y),
\end{align}
Similarly, from the backward dynamics, we can also split it into O and A steps:
\begin{align}
    \begin{bmatrix}
        \mathrm d X_t\\
        \mathrm d Y_t
    \end{bmatrix}
    =
    \underbrace{
    \begin{bmatrix}
        Y_t\\
        0
    \end{bmatrix}
    \bwd{\mathrm dt}}_{A}
    +
    \underbrace{
    \begin{bmatrix}
        0\\
        \left(
            v_t(X_t,Y_t)
            +
            \sigma_t^2\nabla_y H_t(X_t,Y_t)
        \right)\mathrm dt
        +
        \sqrt{2}\sigma_t\,\bwd{\mathrm d W_t}
    \end{bmatrix}}_{O}.
\end{align}
and we also write
\begin{align}
    G_t^B(x,y)
    :=
    v_t(x,y)
    +
    \sigma_t^2\nabla_y H_t(x,y).
\end{align}

Now, let's assume we discretize the time horizon with $N$ boundaries:
\begin{align}
    0=t_0<t_1<\cdots<t_N=1,
    \qquad
    \delta_t=t_{i+1}-t_i,
    \qquad
    \sigma_i=\sigma_{t_i}.
\end{align}
A forward transition goes from \(Z_i=(x_i,y_i)\) to \(Z_{i+1}=(x_{i+1},y_{i+1})\).
A backward transition goes from \(Z_{i+1}\) back to \(Z_i\).
Instead of directly doing the transition for the entire $Z$ space, we split the transition into two parts, one for A and one for O.
Depending on the order of the A and O steps, we have the AO or OA integrators:

\textsc{(1) AO integrator.}  The AO integrator first applies \(A\), then \(O\).

Starting from \((x_i,y_i)\), the \(A\)-step is deterministic:
\begin{align}
    x_i'
    =
    x_i+\delta_i y_i,
    \qquad
    y_i'=y_i .
\end{align}
Then the \(O\)-step updates only the momentum:
\begin{align}
    y_{i+1}
    &\sim
    \mathcal N
    \left(
        y_i'
        +
       \delta_i G_{t_i}^F(x_i',y_i'),
        2\sigma_i^2\delta_i I
    \right),
    \\
    x_{i+1}
    &=
    x_i' .
\end{align}
Therefore,  the forward AO transition kernel is
\begin{align}
    \fwd p_{i+1|i}^{\mathrm{AO}}(Z_{i+1}|Z_i)
    =
    \bm{\delta}_{x_i+\delta_i y_i}(x_{i+1})
    \,
    \mathcal N
    \left(
        y_{i+1}
        \,\middle|\,
       y_i'
        +
       \delta_i G_{t_i}^F(x_i',y_i'),
        2\sigma_i^2\delta_iI
    \right).
\end{align}
where $\bm{\delta}$ is the Delta-measure.

\textsc{(2) OA integrator.} The OA integrator first applies \(O\), then \(A\).

Starting from \((x_i,y_i)\), the \(O\)-step updates only momentum:
\begin{align}
    y_i'
    &\sim
    \mathcal N
    \left(
        y_i
        +
        \delta_iG_{t_i}^F(x_i,y_i),
        2\sigma_i^2 \delta_iI
    \right),
    \\
    x_i'
    &=
    x_i .
\end{align}
Then the \(A\)-step updates only position:
\begin{align}
    x_{i+1}
    =
    x_i'+ \delta_i y_i',
    \qquad
    y_{i+1}=y_i' .
\end{align}
Therefore, the OA transition kernel is 
\begin{align}
    \fwd p_{i+1|i}^{\mathrm{OA}}(Z_{i+1}|Z_i)
    =
    \bm{\delta}_{x_i+ \delta_i y_{i+1}}(x_{i+1})
    \,
    \mathcal N
    \left(
       y_i
        +
       \delta_iG_{t_i}^F(x_i,y_i),
        2\sigma_i^2 \delta_iI
    \right).
\end{align}

\textsc{(3) Backward Integrator of AO.}
The reverse of a forward AO step is a backward OA step.
Starting from \(Z_{i+1}=(x_{i+1},y_{i+1})\), we first apply the backward \(O\)-step at fixed \(x_{i+1}\):
\begin{align}
    y_i
    \sim
    \mathcal N
    \left(
        y_{i+1}
        -
        \delta_i G_{t_{i+1}}^B(x_{i+1},y_{i+1}),
        2\sigma_{i+1}^2\delta_i I
    \right).
\end{align}
We then invert the deterministic \(A\)-step:
\begin{align}
    x_i
    =
    x_{i+1}-\delta_i y_i .
\end{align}
Therefore, the backward transition kernel is
\begin{align}
    \bwd p_{i|i+1}^{\mathrm{OA}}(Z_i|Z_{i+1})
    &=
    \bm{\delta}_{x_{i+1}-\delta_i y_i}(x_i)
    \,
    \mathcal N
    \left(
        y_i
        \,\middle|\,
        y_{i+1}
        -
        \delta_i G_{t_{i+1}}^B(x_{i+1},y_{i+1}),
        2\sigma_{i+1}^2\delta_i I
    \right).
\end{align}

\textsc{(4) Backward Integrator of OA.}
The reverse of a forward OA step is a backward AO step.
Starting from \(Z_{i+1}=(x_{i+1},y_{i+1})\), we first invert the deterministic \(A\)-step:
\begin{align}
    x_i
    =
    x_{i+1}-\delta_i y_{i+1}.
\end{align}
We then apply the backward \(O\)-step at fixed \(x_i\):
\begin{align}
    y_i
    \sim
    \mathcal N
    \left(
        y_{i+1}
        -
        \delta_i G_{t_{i+1}}^B(x_i,y_{i+1}),
        2\sigma_{i+1}^2\delta_i I
    \right).
\end{align}
Therefore, the backward transition kernel is
\begin{align}
    \bwd p_{i|i+1}^{\mathrm{AO}}(Z_i|Z_{i+1})
    &=
    \bm{\delta}_{x_{i+1}-\delta_i y_{i+1}}(x_i)
    \,
    \mathcal N
    \left(
        y_i
        \,\middle|\,
        y_{i+1}
        -
        \delta_i G_{t_{i+1}}^B(x_i,y_{i+1}),
        2\sigma_{i+1}^2\delta_i I
    \right).
\end{align}

\textsc{Density ratio between AO and OA.} We now consider the density ratio between a forward path discretized by OA and a backward path discretized by AO.
The key point is that the deterministic \(A\)-step is volume-preserving:
\begin{align}
    (x,y) \mapsto (x+\delta_i y,y)
\end{align}
has Jacobian determinant equal to one.
Therefore, the \(A\)-step contributes no Jacobian correction to the transition-density ratio.
All nontrivial density contributions come from the Gaussian \(O\)-steps.

For one forward OA step and its backward AO reverse, we have
\begin{align}
    \log
    \frac{
        \fwd p_{i+1|i}^{\mathrm{OA}}(Z_{i+1}|Z_i)
    }{
        \bwd p_{i|i+1}^{\mathrm{AO}}(Z_i|Z_{i+1})
    }
    &=
    \log
    \mathcal N
    \left(
        y_{i+1}
        \,\middle|\,
        \mu_i^{F,\mathrm{OA}},
        2\sigma_i^2\delta_i I
    \right)
    -
    \log
    \mathcal N
    \left(
        y_i
        \,\middle|\,
        \mu_i^{B,\mathrm{AO}},
        2\sigma_{i+1}^2\delta_i I
    \right),
\end{align}
where
\begin{align}
    \mu_i^{F,\mathrm{OA}}&=
    y_i
    +
    \delta_i G_{t_i}^F(x_i,y_i)
    \\
    \mu_i^{B,\mathrm{AO}}
    &=
    y_{i+1}
    -
    \delta_i G_{t_{i+1}}^B(x_i,y_{i+1})
\end{align}

The discrete path ratio for work is then obtained by summing only the Gaussian log-density differences, together with the endpoint density ratio:
\begin{multline}\label{eq:RND_AOOA}
      \log
    \frac{\mathrm d\fwd P}{\mathrm d\bwd Q}(Z_{0:N})
   \\ =
    \log\frac{\Pi_A(Z_0)}{\Pi_B(Z_N)}
    +
    \sum_{i=0}^{N-1}
    \left[
        \log  \mathcal N
    \left(
        y_{i+1}
        \,\middle|\,
        \mu_i^{F,\mathrm{OA}},
        2\sigma_i^2\delta_i I
    \right)
        -
        \log \mathcal N
    \left(
        y_i
        \,\middle|\,
        \mu_i^{B,\mathrm{AO}},
        2\sigma_{i+1}^2\delta_i I
    \right)
    \right].
\end{multline}

Momentum, or generally, auxiliary variables have been introduced to build generative models with the goal to introduce more flexible parameterization and faster inference \citep{dockhorn2021score,du2023flexible,singhal2023diffuse,blessing2025underdamped,chen2023generative}. They have also been baked into bridge models for smooth trajectory inference \citep{chen2023deep,theodoropoulos2025momentum}.

\subsubsection{Continuous-Time Markov Chains on Finite State Spaces}

We finally instantiate Propositions~\ref{prop:jarzynski}~and~\ref{prop:crooks} for finite discrete spaces.
Let's define the transition rate matrix \(R_t^F(i,j)\ge 0\), for \(i\neq j\).
We also define the diagonal entries by
\begin{align}
    R_t^F(i,i)
    =
    -\sum_{j\neq i}R_t^F(i,j).
\end{align}
Then the corresponding generator can be written compactly as
\begin{align}
    \mathcal L_t^F f(i)
    =
    \sum_{j\in\mathcal X}
    R_t^F(i,j)f(j).
\end{align}
The adjoint with respect to counting measure is
\begin{align}
    (\mathcal L_t^F)^\dagger p(i)
    =
    \sum_{j\in\mathcal X}
    R_t^F(j,i)p(j).
\end{align}
Therefore,
\begin{align}
    \frac{((\mathcal L_t^F)^\dagger\gamma_t)(i)}{\gamma_t(i)}
    &=
    \sum_{j\in\mathcal X}
    R_t^F(j,i)
    \frac{\gamma_t(j)}{\gamma_t(i)}
\\
    &=
    \sum_{j\in\mathcal X}
    R_t^F(j,i)
    \exp\left(
        U_t(i)-U_t(j)
    \right).
\end{align}
and the term
$  \exp\left(
        U_t(i)-U_t(j)
    \right)$ are known as the concrete score or discrete score \citep{lou2023discrete}.
Plugging this into \Cref{eq:general_work_cont}, the generalized work becomes
\begin{align}
    \widetilde W(X_{[0,1]})
    =
    \int_0^1
    \left[
        \partial_t U_t(X_t)
        +
        \sum_{j\in\mathcal X}
        R_t^F(j,X_t)
        \exp\left(
            U_t(X_t)-U_t(j)
        \right)
    \right]\mathrm dt .
\end{align}

The corresponding backward generator from \Cref{eq:bwd_crooks_cont} can be derived similarly:
\begin{align}
    (\mathcal L_t^F)^\dagger(\gamma_t f)(i)
    =
    \sum_{j\in\mathcal X}
    R_t^F(j,i)\gamma_t(j)f(j),
    \qquad
    (\mathcal L_t^F)^\dagger\gamma_t(i)
    =
    \sum_{j\in\mathcal X}
    R_t^F(j,i)\gamma_t(j).
\end{align}
Hence
\begin{align}
    \mathcal L_t^B f(i)
    &=
    \frac{1}{\gamma_t(i)}
    \sum_{j\in\mathcal X}
    R_t^F(j,i)\gamma_t(j)f(j)
    -
    \frac{f(i)}{\gamma_t(i)}
    \sum_{j\in\mathcal X}
    R_t^F(j,i)\gamma_t(j)
\\
    &=
    \sum_{j\in\mathcal X}
    R_t^F(j,i)
    \frac{\gamma_t(j)}{\gamma_t(i)}
    \left[
        f(j)-f(i)
    \right]
\\
    &=
    \sum_{j\in\mathcal X}
    R_t^F(j,i)
    \exp\left(
        U_t(i)-U_t(j)
    \right)
    \left[
        f(j)-f(i)
    \right].
\end{align}
Therefore, the backward process 's rate matrix is given by  
\begin{align}
    R_t^B(i,j)
    =
    R_t^F(j,i)
    \frac{\gamma_t(j)}{\gamma_t(i)}
    =
    R_t^F(j,i)
    \exp\left(
        U_t(i)-U_t(j)
    \right), i\neq j
\end{align}
and
\begin{align}
    R_t^B(i,i)
    =
    -\sum_{j\neq i}R_t^B(i,j).
\end{align}
The backward marginal equation in Proposition~\ref{prop:crooks} is then
\begin{align}
    \partial_t q_t
    =
    -(\mathcal L_t^B)^\dagger q_t,
    \qquad
    q_1=\pi_B .
\end{align}

\subsection{Exemplify Learning Methods for the Transport}

In this section, we instantiate the learning strategy for several concrete cases. 
For completeness, we also include the diffusion-process setting. 
However, since this case has already been studied in \citepalias{he2025feat}, we do not repeat the corresponding experiments in this work.

We also emphasize that the learning design is not unique. 
For example, we can design many different interpolant paths bridging between two distributions.
Moreover, even for a fixed interpolant, the generator realizing the same marginal path is generally non-unique. 
Our goal here is therefore not to provide a systematic study of all possible choices, but rather to illustrate the general recipe through representative examples. 
Regardless of the specific design, the learning principle discussed in \Cref{sec:learn} remains unchanged and applies generally.

\subsubsection{Diffusion Process in Continuous Euclidean Space}
\label{appendix:continuous_euclidean_transport}

We first consider the diffusion process setting in continuous Euclidean space.
Following \citet{albergo2025stochastic}, we define a stochastic interpolant between endpoint samples
\((x_0,x_1)\sim \pi(x_0,x_1)\), with \(x_0\sim \pi_A\) and \(x_1\sim \pi_B\):
\begin{align}
    x_t
    =
    \alpha_t x_0+\beta_t x_1+\gamma_t\epsilon,
    \qquad
    \epsilon\sim\mathcal N(0,I),
\end{align}
where \(\alpha_0=1,\beta_0=0,\gamma_0=0\) and
\(\alpha_1=0,\beta_1=1,\gamma_1=0\).
For fixed \((x_0,x_1)\), this gives the conditional density path
\begin{align}
    p_t(x|x_0,x_1)
    =
    \mathcal N
    \left(
        x
        \,\middle|\,
        I_t(x_0,x_1),\gamma_t^2 I
    \right),
    \qquad
    I_t(x_0,x_1)=\alpha_t x_0+\beta_t x_1 .
\end{align}
The corresponding conditional velocity and conditional score realizing this path are
\begin{align}\label{eq:si}
    v_t^{x_0,x_1}(x)
    &=
    \dot\alpha_t x_0+\dot\beta_t x_1
    +
    \frac{\dot\gamma_t}{\gamma_t}
    \left(
        x-I_t(x_0,x_1)
    \right),
    \\
    s_t^{x_0,x_1}(x)
    &:=
    \nabla_x\log p_t(x|x_0,x_1)
    =
    -\frac{x-I_t(x_0,x_1)}{\gamma_t^2}.
\end{align}

Now, we will use the two strategies discussed in \Cref{sec:learn} (step 2(A) and step 2(B) ) to obtain the learning objective.
\paragraph{Strategy 1: directly learn the forward and backward generators.}

For each fixed pair \((x_0,x_1)\), the conditional forward and backward generators are given by 
\begin{align}
    \mathcal L_{F,t}^{x_0,x_1} f
    &=
    \left(
        v_t^{x_0,x_1}
        +
        \sigma_t^2 s_t^{x_0,x_1}
    \right)\cdot\nabla f
    +
    \sigma_t^2\Delta f,
    \\
    \mathcal L_{B,t}^{x_0,x_1} f
    &=
    \left(
        -v_t^{x_0,x_1}
        +
        \sigma_t^2 s_t^{x_0,x_1}
    \right)\cdot\nabla f
    +
    \sigma_t^2\Delta f .
\end{align}
By construction, these two conditional generators are time reversals with respect to
\(p_t(x|x_0,x_1)\).
Therefore, by Proposition~\ref{prop:marginal_reversal}, their marginalization will also be time reversals.

Therefore, we may parameterize the forward and backward drifts
\(b_{\psi,t}^F(x)\) and \(b_{\phi,t}^B(x)\). 
The marginal generators are then given by
\begin{align}
    \mathcal L_{F,t} f
    &=
    b_{\psi,t}^F \cdot \nabla f
    +
    \sigma_t^2\Delta f,
    \\
    \mathcal L_{B,t} f
    &=
    b_{\phi,t}^B \cdot \nabla f
    +
    \sigma_t^2\Delta f .
\end{align}
Matching them to the conditional generators, we obtain the objectives
\begin{align}\label{eq:fwd_sde_1}
    \mathcal J_F(\psi)
    &=
    \mathbb E_{t,x_0,x_1,\epsilon}
    \left[
        \left\|
        b_{\psi,t}^F(x_t)
        -
        \left(
            v_t^{x_0,x_1}(x_t)
            +
            \sigma_t^2 s_t^{x_0,x_1}(x_t)
        \right)
        \right\|^2
    \right],
    \\
    \mathcal J_B(\phi)
    &=
    \mathbb E_{t,x_0,x_1,\epsilon}
    \left[
        \left\|
        b_{\phi,t}^B(x_t)
        -
        \left(
            -v_t^{x_0,x_1}(x_t)
            +
            \sigma_t^2 s_t^{x_0,x_1}(x_t)
        \right)
        \right\|^2
    \right].
\end{align}

In practice, we can also use the symmetric diffusion parametrization, i.e.,
parameterize the forward and backward drifts through a shared vector field
\(v_{\psi,t}\) and a shared score \(s_{\theta,t}\). 
Then the forward and backward generators for the marginal path \(p_t\) are
\begin{align}
    \mathcal L_t^F f
    &=
    \left(
        v_{\psi,t}
        +
        \sigma_t^2 s_{\theta,t}
    \right)\cdot\nabla f
    +
    \sigma_t^2\Delta f,
    \\
    \mathcal L_t^B f
    &=
    \left(
        -v_{\psi,t}
        +
        \sigma_t^2 s_{\theta,t}
    \right)\cdot\nabla f
    +
    \sigma_t^2\Delta f .
\end{align}
Matching these two generators to the conditional forward and backward generators gives
\begin{align}
    \mathcal J_{\mathrm{sym}}(\psi,\theta)
    &=
    \mathbb E_{t,x_0,x_1,\epsilon}
    \Bigg[
        \left\|
        v_{\psi,t}(x_t)
        +
        \sigma_t^2 s_{\theta,t}(x_t)
        -
        \left(
            v_t^{x_0,x_1}(x_t)
            +
            \sigma_t^2 s_t^{x_0,x_1}(x_t)
        \right)
        \right\|^2
        \notag\\
    &\qquad\qquad\qquad
        +
        \left\|
        -v_{\psi,t}(x_t)
        +
        \sigma_t^2 s_{\theta,t}(x_t)
        -
        \left(
            -v_t^{x_0,x_1}(x_t)
            +
            \sigma_t^2 s_t^{x_0,x_1}(x_t)
        \right)
        \right\|^2
    \Bigg].
\end{align}
Equivalently, this can be written as directly matching the marginal vector field and score:
\begin{align}
    \mathcal J_v(\psi)
    &=
    \mathbb E_{t,x_0,x_1,\epsilon}
    \left[
        \left\|
        v_{\psi,t}(x_t)
        -
        v_t^{x_0,x_1}(x_t)
        \right\|^2
    \right],
    \\
    \mathcal J_s(\theta)
    &=
    \mathbb E_{t,x_0,x_1,\epsilon}
    \left[
        \left\|
        s_{\theta,t}(x_t)
        -
        s_t^{x_0,x_1}(x_t)
        \right\|^2
    \right],
\end{align}
which recovers the objective by \citet{albergo2025stochastic} and used in \citepalias{he2025feat}.
In practice, we may also reweight the objective with time-dependent weights to balance the learning at different times.

\paragraph{Strategy 2: learn one direction and the marginal score.}

A second strategy is to learn one direction together with the marginal statistic needed to reverse it.

Recall that the Gaussian interpolant can be viewed as first defining the uncorrupted interpolating state
\begin{align}
    I_t
    =
    \alpha_t x_0+\beta_t x_1,
\end{align}
and then applying a Gaussian corruption
\begin{align}
    x_t
    =
    I_t+\gamma_t\epsilon,
    \qquad
    \epsilon\sim\mathcal N(0,I).
\end{align}
We can view this corruption as a new stochastic process, with a time-independent corruption generator
\begin{align}
    \mathcal A =\frac12\Delta .
\end{align}
For each time $t$, we will calculate the marginal of this new corruption process with
\begin{align}
    q_{\tau}(x|I_t)
    =
    \mathcal N(x|I_t,\gamma_t^2  I).
\end{align}
Therefore, the marginal interpolant can be written as
\begin{align}
    p_t(x)
    =
    \int \mathcal{N}(x|I_t,\gamma_t^2  I)\,p_t^I(I)\,\mathrm d I .
\end{align}

Applying the general DSM objective from \Cref{sec:learn}, and using
\(\mathcal A =\frac12\Delta\), we obtain
\begin{align}
    \frac{
        \mathcal A 
        \left(
            q_\tau(x_t|I_t)/\varphi_t(x_t)
        \right)
    }{
        q_\tau(x_t|I_t)/\varphi_t(x_t)
    }
    -
    \mathcal A 
    \log
    \left(
        q_\tau(x_t|I_t)/\varphi_t(x_t)
    \right)
    =
    \frac12
    \left\|
        \nabla_x\log q_\tau(x_t|I_t)
        -
        \nabla_x\log\varphi_t(x_t)
    \right\|^2 .
\end{align}
Thus, if we parameterize
\begin{align}
    s_{\theta,t}(x)
    \approx
    \nabla_x\log \varphi_t(x),
\end{align}
Then the DSM objective reduces to the usual denoising score-matching objective
\begin{align}
    \mathcal J_s(\theta)
    =
    \mathbb E_{t,x_0,x_1,\epsilon}
    \left[
        \left\|
            s_{\theta,t}(x_t)
            -
            \nabla_x\log \mathcal{N}(x|I_t,\gamma_t^2  I)
        \right\|^2
    \right].
\end{align}
In practice, we may also reweight the objective with time-dependent weights to balance the learning at different times. 
At the optimum, this gives the marginal score
\begin{align}
    s_{\theta,t}(x)
    =
    \nabla_x\log p_t(x).
\end{align}
With the score being learned, we then only need to learn one direction, either the forward drift or the backward drift. 
Similar to the first strategy, we can either parameterize the entire forward drift \(b_t^F\) and obtain the backward drift as
\begin{align}
    b_t^B(x)
    =
    - b_t^F(x)
    +
    2\sigma_t^2 s_{\theta,t}(x),
\end{align}
or parameterize a vector field \(v_t\) and recover the forward and backward drifts by
\begin{align}
    b_t^F(x)
    &=
    v_t(x)
    +
    \sigma_t^2 s_{\theta,t}(x),
    \\
    b_t^B(x)
    &=
    -v_t(x)
    +
    \sigma_t^2 s_{\theta,t}(x),
\end{align}
where \(s_{\theta,t}\) denotes the learned marginal score.

\par
Additionally,  note that the vector field and score do not depend on $\sigma_t$, and the above discussion holds for any $\sigma_t>0$.
Therefore, after learning $v_t$ and $s_{\theta, t}$, we can tune the value of $\sigma_t$, which can yield different performance.

\subsubsection{Momentum-Augmented Transport in Continuous Space}\label{app:learn_underdamp}

We now explain the training strategy for momentum-augmented transport.
Recall from \Cref{sec:underdamp_form}:
Let \(Z_t=(X_t,Y_t)\in\mathbb R^d\times\mathbb R^d\), and define an augmented energy $H_t(x,y)$ on the joint space, as well as the unnormalized density
\begin{align}
    \Gamma_t(x,y) = e^{-H_t(x,y)}, \quad     \Pi_t(x,y) =\frac{ e^{-H_t(x,y)} }{\int e^{-H_t(x,y)} \mathrm{d}xy}
\end{align}
Also recall, that under the symmetric parametrization, the forward and backward dynamics is
\begin{align}
    \mathrm d X_t
    &=
    Y_t\,\mathrm dt,
    \\
    \mathrm d Y_t
    &=
    \left(
        v_t(X_t,Y_t)
        -
        \sigma_t^2\nabla_y H_t(X_t,Y_t)
    \right)\mathrm dt
    +
    \sqrt{2}\sigma_t\,\fwd{\mathrm d W_t},
    \qquad
    (X_0,Y_0)\sim \Pi_A .
\end{align}
and
\begin{align}
    \mathrm d X_t
    &=
    Y_t\,\bwd{\mathrm dt},
    \\
    \mathrm d Y_t
    &=
    \left(
        v_t(X_t,Y_t)
        +
        \sigma_t^2\nabla_y H_t(X_t,Y_t)
    \right)\mathrm dt
    +
    \sqrt{2}\sigma_t\,\bwd{\mathrm d W_t},
    \qquad
    (X_1,Y_1)\sim \Pi_B .
\end{align}
Similar to the overdamped  diffusion process case, we learn the score 
$    s_t^y(x,y)
=  \nabla_y \log \Gamma_t(x,y) =
    -\nabla_y H_t(x,y),$
and the vector field $v_t$ to obtain the forward and backward processes.
The difference is that the vector field is now an acceleration field in the momentum equation, and the required score is only the momentum score
\(\nabla_y\log \Gamma_t(x,y)\), rather than the full phase-space score.

\paragraph{Define the Stochastic Interpolant with Momentum}
To achieve this, we still consider constructing an interpolant between samples $x_0$ and $x_1$.
Different from the standard stochastic interpolant, we augment the space with two endpoint momenta:
\begin{align}
    x_0\sim \pi_A,
    \qquad
    x_1\sim \pi_B,
    \qquad
    y_0\sim \mathcal N(0,I),
    \qquad
    y_1\sim \mathcal N(0,I),
\end{align}
with \(y_0,y_1\) independent of the positions.
This gives the endpoint laws as
\begin{align}\label{eq:msi_law}
    \Gamma_A(x,y)=\pi_A(x)\mathcal N(y|0,I),
    \qquad
    \Gamma_B(x,y)=\pi_B(x)\mathcal N(y|0,I).
\end{align}
Different from the standard stochastic interpolant, we need an interpolant for both $x_t$ and $y_t$, so that $\dot x_t = y_t$.
This constrains our design space to a large extent.

To this end, we define a cubic Hermite interpolant.
Let
\begin{align}
    h_{00}(t)=2t^3-3t^2+1,\quad 
    h_{10}(t)=t^3-2t^2+t,\quad 
    h_{01}(t)=-2t^3+3t^2,\quad 
    h_{11}(t)=t^3-t^2 .
\end{align}
We then define
\begin{align}
    x_t
    =
    h_{00}(t)x_0
    +
    h_{10}(t)y_0
    +
    h_{01}(t)x_1
    +
    h_{11}(t)y_1 , \quad  y_t=\dot x_t.
\end{align}
Additionally, to learn the vector field for $y$, we also need to calculate the time derivative for $y$, leading to the acceleration:
\begin{align}
    a_t=\dot y_t=\ddot x_t .
\end{align}
More explicitly, we have
\begin{align}
    y_t
    &=
    (6t^2-6t)x_0
    +
    (3t^2-4t+1)y_0
    +
    (-6t^2+6t)x_1
    +
    (3t^2-2t)y_1,
    \\
    a_t
    &=
    (12t-6)x_0
    +
    (6t-4)y_0
    +
    (-12t+6)x_1
    +
    (6t-2)y_1 .
\end{align}
One can verify that:
\begin{align}
    x_{t=0}=x_0,\qquad y_{t=0}=y_0,
    \qquad
    x_{t=1}=x_1,\qquad y_{t=1}=y_1,
\end{align}
which ensures that, at both end, the momentum and position are independent, following the law defined in  \Cref{eq:msi_law}.

\paragraph{Learning Momentum-Augmented Transport with Stochastic Interpolant}
We now derive how to learn the two quantities needed for the momentum-augmented dynamics:
the acceleration field \(v_t\) and the momentum score \(s_t^y\).
For simplicity, we use strategy 1, where we first write down the generator condition on the samples and marginalize them.

More precisely, for a fixed endpoint pair \((x_0,x_1)\), let 
\(\rho_t(z|x_0,x_1)\) denote the conditional law of \(Z_t=(X_t,Y_t)\), where the randomness comes from the endpoint momenta \(y_0,y_1\).
Then, we write down the forward and backward dynamics (and their generators) that realize this path law and are time reversals of each other.

We define the conditional acceleration field and the conditional momentum score as follows:
\begin{align}
    v_t^{x_0,x_1}(z)
   & = a(x_0, x_1, z),
    \\
    s_t^{y,x_0,x_1}(z)
    &=
    \nabla_y \log \rho_t(z|x_0,x_1).
\end{align}
where $a(x_0, x_1, z)$ denote the function to solve the acceleration from $x_0, x_1$ and $z=(x_t, y_t)$.
This is possible: for fixed \((x_0,x_1)\), the map from \((y_0,y_1)\) to \((x_t,y_t)\) is linear and invertible.
After solving \((y_0,y_1)\), we can obtain $a =  
    \ddot h_{00}(t)x_0
    +
    \ddot h_{10}(t)y_0
    +
    \ddot h_{01}(t)x_1
    +
    \ddot h_{11}(t)y_1 $.
This is similar to the previous stochastic case in \Cref{eq:si}.

For simplicity, we will later write \(y_0,y_1\) solved from \(z=(x_t,y_t)\) as \(y_0(z),y_1(z)\).
With this notation, we have
\begin{align}
      v_t^{x_0,x_1}(z)
   & = \ddot h_{00}(t)x_0
    +
    \ddot h_{10}(t)y_0(z)
    +
    \ddot h_{01}(t)x_1
    +
    \ddot h_{11}(t)y_1(z) 
    \\
    s_t^{y,x_0,x_1}(z)
    &=
    \nabla_y \log \rho_t(z|x_0,x_1)\\
   &=  \nabla_y \log 
   \mathcal{N}\left(
y_0(z)
    \right) + \nabla_y \log 
   \mathcal{N}\left(
y_1(z)
    \right)\\
    &= \nabla_y \left( -\frac12\|y_0(z)\|^2
    -
    \frac12\|y_1(z)\|^2\right)\\
    &= 
    -
        \nabla_y y_0(z)
\cdot y_0(z)
    -
    \nabla_y y_1(z)
\cdot y_1(z).
\end{align}

For \(z=(x,y)\), the endpoint momenta \(y_0,y_1\) can be solved from
\begin{align}
    \begin{bmatrix}
        x-h_{00}(t)x_0-h_{01}(t)x_1\\
        y-\dot h_{00}(t)x_0-\dot h_{01}(t)x_1
    \end{bmatrix}
    =
    \begin{bmatrix}
        h_{10}(t) & h_{11}(t)\\
        \dot h_{10}(t) & \dot h_{11}(t)
    \end{bmatrix}
    \begin{bmatrix}
        y_0\\
        y_1
    \end{bmatrix}.
\end{align}
Since
\begin{align}
    h_{10}(t)\dot h_{11}(t)-h_{11}(t)\dot h_{10}(t)
    =
    -t^2(1-t)^2,
\end{align}
we obtain
\begin{align}
    y_0(z)
    &=
    \frac{
        \dot h_{11}(t)
        \left[
            x-h_{00}(t)x_0-h_{01}(t)x_1
        \right]
        -
        h_{11}(t)
        \left[
            y-\dot h_{00}(t)x_0-\dot h_{01}(t)x_1
        \right]
    }{
        -t^2(1-t)^2
    },
    \\
    y_1(z)
    &=
    \frac{
        -\dot h_{10}(t)
        \left[
            x-h_{00}(t)x_0-h_{01}(t)x_1
        \right]
        +
        h_{10}(t)
        \left[
            y-\dot h_{00}(t)x_0-\dot h_{01}(t)x_1
        \right]
    }{
        -t^2(1-t)^2
    }.
\end{align}
and hence
\begin{align}
    \nabla_y y_0(z) = \frac{h_{11}(t)}{t^2(1-t)^2}I = \frac{1}{t-1}
\end{align}
\begin{align}
    \nabla_y y_1(z) = -\frac{h_{10}(t)}{t^2(1-t)^2}I = -\frac{1}{t}
\end{align}
Hence
\begin{align}
     s_t^{y,x_0,x_1}(z)
   = 
    -
        \nabla_y y_0(z)
\cdot y_0(z)
    -
    \nabla_y y_1(z)
\cdot y_1(z) = \frac{1}{1-t} y_0(z) + \frac{1}{t} y_1(z) 
\end{align}

Similarly, since \(y_0,y_1\sim\mathcal N(0,I)\), the conditional density
\(\rho_t(z|x_0,x_1)\) is obtained by change of variables from the Gaussian density of
\((y_0,y_1)\). Differentiating this density with respect to \(y\), we obtain the conditional
momentum score
\begin{align}\label{eq:cond_score_hermite}
    s_t^{y,x_0,x_1}(z)
    &:=
    \nabla_y\log \rho_t(z|x_0,x_1)
    \\
    &=
    \frac{y_0(z)}{1-t}
    +
    \frac{y_1(z)}{t}.
\end{align}

The conditional forward dynamics is defined by
\begin{align}
    \mathrm d X_t
    &=
    Y_t\,\mathrm dt,
    \\
    \mathrm d Y_t
    &=
    \left(
        v_t^{x_0,x_1}(X_t,Y_t)
        +
        \sigma_t^2 s_t^{y,x_0,x_1}(X_t,Y_t)
    \right)\mathrm dt
    +
    \sqrt{2}\sigma_t\,\fwd{\mathrm d W_t}.
\end{align}
The corresponding conditional generator is
\begin{align}
    \mathcal L_{F,t}^{x_0,x_1} f
    =
    y\cdot\nabla_x f
    +
    \left(
        v_t^{x_0,x_1}
        +
        \sigma_t^2 s_t^{y,x_0,x_1}
    \right)\cdot\nabla_y f
    +
    \sigma_t^2\Delta_y f .
\end{align}
Also, the conditional backward dynamics, which is the time-reversal of the forward process, is
\begin{align}
    \mathrm d X_t
    &=
    Y_t\,\bwd{\mathrm dt},
    \\
    \mathrm d Y_t
    &=
    \left(
        v_t^{x_0,x_1}(X_t,Y_t)
        -
        \sigma_t^2 s_t^{y,x_0,x_1}(X_t,Y_t)
    \right)\mathrm dt
    +
    \sqrt{2}\sigma_t\,\bwd{\mathrm d W_t}.
\end{align}
with generator
\begin{align}
    \mathcal L_{B,t}^{x_0,x_1} f
    =
  -  y\cdot\nabla_x f
    +
    \left(
       - v_t^{x_0,x_1}
        +
        \sigma_t^2 s_t^{y,x_0,x_1}
    \right)\cdot\nabla_y f
    +
    \sigma_t^2\Delta_y f .
\end{align}
Similar to the stochastic interpolant case, we learn to marginalize these generators.
We parameterize the acceleration field \(v_{\psi,t}\) and the momentum score \(s_{\theta,t}^y\).
Then the learned generators are
\begin{align}
    \mathcal L_{\psi,\theta,t}^{F} f
    &=
    y\cdot\nabla_x f
    +
    \left(
        v_{\psi,t}
        +
        \sigma_t^2 s_{\theta,t}^y
    \right)\cdot\nabla_y f
    +
    \sigma_t^2\Delta_y f,
    \\
    \mathcal L_{\psi,\theta,t}^{B} f
    &=
    -y\cdot\nabla_x f
    +
    \left(
        -v_{\psi,t}
        +
        \sigma_t^2 s_{\theta,t}^y
    \right)\cdot\nabla_y f
    +
    \sigma_t^2\Delta_y f .
\end{align}
Matching these to the conditional forward and backward generators gives
\begin{align}
    &\mathcal J_{\mathrm{sym}}(\psi,\theta)
   \\ =
  &  \mathbb E_{t,x_0,x_1,y_0,y_1}
    \Bigg[
        \left\|
            v_{\psi,t}(x_t,y_t)
            +
            \sigma_t^2 s_{\theta,t}^y(x_t,y_t)
            -
            \left(
                 v_t^{x_0,x_1} (x_t,y_t)
        +
        \sigma_t^2 s_t^{y,x_0,x_1} (x_t,y_t)
            \right)
        \right\|^2
        \notag\\
    &\qquad 
        +
        \left\|
            -v_{\psi,t}(x_t,y_t)
            +
            \sigma_t^2 s_{\theta,t}^y(x_t,y_t)
            -
             \left(
                - v_t^{x_0,x_1} (x_t,y_t)
        +
        \sigma_t^2 s_t^{y,x_0,x_1} (x_t,y_t)
            \right)
        \right\|^2
    \Bigg].
\end{align}
Equivalently:
\begin{align}\label{loss_underdamp1}
    \mathcal J_v(\psi)
    &=
    \mathbb E_{t,x_0,x_1,y_0,y_1}
    \left[
        \left\|
            v_{\psi,t}(x_t,y_t)- v_t^{x_0,x_1} (x_t,y_t)
        \right\|^2
    \right],
    \\
    \mathcal J_s(\theta)
    &=
    \mathbb E_{t,x_0,x_1,y_0,y_1}
    \left[
        \left\|
            s_{\theta,t}^y(x_t,y_t)- s_t^{y,x_0,x_1} (x_t,y_t)
        \right\|^2
    \right].
\end{align}
Since the conditional score target blows up near \(t=0\) and \(t=1\) (see \Cref{eq:cond_score_hermite}), we scale the target and the score network both by $t(1-t)$:
\begin{align}\label{loss_underdamp2}
    \mathcal J_s(\theta)
    =
    \mathbb E_{t,x_0,x_1,y_0,y_1}
    \left[
        \left\|
            t(1-t)s_{\theta,t}^y(x_t,y_t)
            -
            \left(
                t y_0+(1-t)y_1
            \right)
        \right\|^2
    \right].
\end{align}

Finally, for free energy estimation, since $N( 0,I)$ are normalized, we directly estimate the free energy between the augmented distributions:
\begin{align}
    \Pi_A(x,y)=\pi_A(x)\mathcal N(y|0,I),
    \qquad
    \Pi_B(x,y)=\pi_B(x)\mathcal N(y|0,I).
\end{align}
The work is calculated by the forward-backward RND between OA and AO discretized path, as we discussed in \Cref{eq:RND_AOOA}.

Also, similar to the diffusion case without momentum, the vector field and score do not depend on $\sigma_t$, and the above discussion holds for any $\sigma_t>0$.
Therefore, we can also tune different $\sigma_t$ to improve performance further.
This is studied in \Cref{fig:underdamped_vis_40D}.

\subsubsection{Continuous-Time Markov Chains on Finite State Spaces} \label{sec:appendix_ctmc_feat_theory}

We now discuss how to learn transport in finite discrete spaces.
As in the previous examples, we first define a conditional interpolant between two endpoint samples
\((x_0,x_1)\sim \pi(x_0,x_1)\), and then construct conditional generators whose marginalization gives a generator for the marginal path.

\paragraph{Interpolant Design}
Different from the continuous cases, we design the interpolant by a mixture, as defined by \Cref{eq:int2}.
We restate here for easier reference:
\begin{align}
    &x_t \sim  \alpha_t \delta_{x_0}(x_t) + \beta_t \delta_{x_1}(x_t) ,
\end{align}
Another opinion is to add ``noise", similar to the continuous case.
The ``noise" can be a mask token, or a uniform distribution over the entire vocabulary.
Here, we consider mask token for simplicity:
\begin{align}
  x_t \sim  \alpha_t \delta_{x_0}(x_t) + \beta_t \delta_{x_1}(x_t) + \gamma_t \delta_{M}(x_t) 
\end{align}
Interestingly, in CTMC, both choices are valid and allow us to calculate the forward backward path RND.
This is different from the continuous case, where we need to add noise to the path, otherwise, we only learn a flow, which involves Delta measure over the path and cannot be handled properly with a transition kernel.
In CTMC, even the interpolant without adding mask still involves the Categorical distribution as the transition kernel, and hence the path always has stochasticity.

We now discuss three design choices for the dynamics realizing the interpolants above.
The first design is based on the path without adding mask, while the second and third designs involve mask.

For easier reference, we restate the rate matrix and generator definitation:
throughout this section, we use the convention
\begin{align}
    \mathcal L_t f(x)
    =
    \sum_{y\in\mathcal X}
    R_t(x,y)f(y),
\end{align}
where \(R_t(x,y)\) denotes the rate of jumping from \(x\) to \(y\), and
\begin{align}
    R_t(x,x)
    =
    -\sum_{y\neq x}R_t(x,y).
\end{align}

\paragraph{``Latent" Generator and Marginalization}
Before discussing the individual cases, we first state the following proposition that allows us to marginalize a conditional CTMC defined on a ``latent" state space.

In particular, the mixture interpolant induces a ``latent" CTMC on a finite state space \(\mathcal S\). 
This ``latent" CTMC is not defined directly on the original data space $\mathcal X$. 
Rather, each state \(s\in\mathcal S\) specifies the source (e.g., $x_0$ or $x_1$ or $M$) of the current observation \(x_t\): 
for example, whether \(x_t\) is copied from \(x_0\), copied from \(x_1\), or replaced by a mask state. The observed process is therefore a projection of this latent CTMC onto the original space. 
This latent generator allows us to simplify the conditional generator.
The following proposition shows that, by averaging the conditional generator of the latent CTMC with respect to the posterior distribution over latent states, one obtains a marginalized generator acting directly on the observed space.

\begin{tcolorbox}[
    colback=gray!8,
    colframe=gray!40,
    boxrule=0.5pt,
    arc=4pt,
    left=1pt,
    right=1pt,
    top=2pt,
    bottom=2pt,
    breakable
]
\begin{proposition}[Marginalization of latent conditional CTMC]\label{prop:latent_ctmc_marginal}
We define a latent CTMC $S_t\in \mathcal S$ with generator
\begin{align}
    \mathcal A_t  g(s)
    =
    \sum_{s'\in\mathcal S}
    A_t (s,s')g(s'),
\end{align}
where
\begin{align}
    A_t (s,s)
    =
    -\sum_{s'\neq s}A_t (s,s').
\end{align}
Let the observed interpolant be obtained by a deterministic decoding map
\begin{align}
    X_t=\phi(x_0,x_1,S_t).
\end{align}

Define the marginalized generator as
\begin{align}\label{eq:ctmc_marginal_generator}
    \mathcal L_t f(x)
    =
    \mathbb E_{s_t, x_0, x_1|X_t=x}
    \left[
        \mathcal A_t 
        \left(
            f\circ \phi(x_0,x_1,\cdot)
        \right)(s_t)
    \right].
\end{align}
and we have
\begin{align}
    \partial_t p_t
    =
    \mathcal L_t^\dagger p_t, \quad p_t = \sum_{x_0, x_1} \pi(x_0, x_1)p_t(x_t|x_0, x_1) 
\end{align}
Equivalently, we have the rate matrix form:
\begin{align}\label{eq:ctmc_marginal_rate}
    R_t(x,y)
    =
    \mathbb  E_{s_t, x_0, x_1|X_t=x}
    \left[
        \sum_{s'\in\mathcal S}
        \bm 1\{\phi(x_0,x_1,s')=y\}
        A_t (s_t,s')
    \right].
\end{align}
and
\begin{align}
    \partial_t p_t(y)
    =
   \sum_{x\in \mathcal{X}}  R_t(x,y) p_t(x).
\end{align}
\end{proposition}
\end{tcolorbox}
\begin{proof}
\begin{align}
      \sum_{x\in \mathcal{X}}  R_t(x,y) p_t(x) &=   \sum_{x\in \mathcal{X}} \mathbb  E_{s_t, x_0, x_1|X_t=x}
    \left[
        \sum_{s'\in\mathcal S}
        \bm 1\{\phi(x_0,x_1,s')=y\}
        A_t (s_t,s') 
    \right]p_t(x)\\
    &= \mathbb  E_{s_t, x_0, x_1, X_t=x}
    \left[
        \sum_{s'\in\mathcal S}
        \bm 1\{\phi(x_0,x_1,s')=y\}
        A_t (s_t,s') 
    \right] \\
    &= \mathbb  E_{s_t, x_0, x_1 }
    \left[
        \sum_{s'\in\mathcal S}
        \bm 1\{\phi(x_0,x_1,s')=y\}
        A_t (s_t,s') 
    \right] \\
    &=  \sum_{s'\in\mathcal S}\mathbb  E_{s_t, x_0, x_1 }
    \left[
        \bm 1\{\phi(x_0,x_1,s')=y\}
        A_t (s_t,s') 
    \right] \\
     &=  \sum_{s' }\sum_{s_t} \mathbb  E_{x_0, x_1 |s_t }
    \left[p_t(s_t) 
        \bm 1\{\phi(x_0,x_1,s')=y\}
        A_t (s_t,s') 
    \right] \\
     &\overset{\text{$x_0, x_1\perp s_t $}}{=}  \sum_{s' }\sum_{s_t} \mathbb  E_{x_0, x_1  }
    \left[p_t(s_t) 
        \bm 1\{\phi(x_0,x_1,s')=y\}
        A_t (s_t,s') 
    \right] \\
    &= \sum_{s' } \mathbb  E_{x_0, x_1  }
    \left[\sum_{s_t} p_t(s_t) 
        \bm 1\{\phi(x_0,x_1,s')=y\}
        A_t (s_t,s') 
    \right] \\
     &= \sum_{s' } \mathbb  E_{x_0, x_1  }
    \left[\sum_{s_t} p_t(s_t)  A_t (s_t,s') 
        \bm 1\{\phi(x_0,x_1,s')=y\}
    \right] \\
    &= \sum_{s' } \mathbb  E_{x_0, x_1  }
    \left[\partial_t p_t(s')
        \bm 1\{\phi(x_0,x_1,s')=y\}
    \right]\\
    &= \partial_t  \sum_{s' } \mathbb  E_{x_0, x_1  }
    \left[p_t(s')
        \bm 1\{\phi(x_0,x_1,s')=y\}
    \right] \\
    &= \partial_t p_t(y)
\end{align}
\end{proof}

We also have a similar result as Proposition~\ref{prop:marginal_reversal} for the latent CTMC case.

\begin{tcolorbox}[
    colback=gray!8,
    colframe=gray!40,
    boxrule=0.5pt,
    arc=4pt,
    left=1pt,
    right=1pt,
    top=2pt,
    bottom=2pt,
    breakable
]
\begin{proposition}[Marginalization in latent space maintains time-reversal]
\label{prop:latent_ctmc_marginal_reversal}
Let \(S_t\in\mathcal S\) be a latent CTMC with marginal law
\(p_t(s)\). 
Suppose \(A_t \) and \(B_t \) are time reversals with respect to
\(p_t(s )\), i.e. for \(s\neq s'\),
\begin{align}
    B_t (s,s')
    =
    A_t (s',s)
    \frac{
        p_t(s' )
    }{
        p_t(s )
    } .
\end{align}
Let the observed interpolant be obtained by a deterministic decoding map
\begin{align}
    X_t=\phi(x_0,x_1,S_t),
\end{align}
and denote the marginal law of \(X_t\) by \(p_t(x)\).

Define the marginalized forward and backward rate matrices by
\begin{align}
    R_t^F(x,y)
    &=
    \mathbb E_{s_t, x_0, x_1|X_t=x}
    \left[
        \sum_{s'\in\mathcal S}
        \bm 1\{\phi(x_0,x_1,s')=y\}
        A_t (s_t,s')
    \right],
    \\
    R_t^B(x,y)
    &=
    \mathbb E_{s_t, x_0, x_1|X_t=x}
    \left[
        \sum_{s'\in\mathcal S}
        \bm 1\{\phi(x_0,x_1,s')=y\}
        B_t (s_t,s')
    \right].
\end{align}
Then \(R_t^F\) and \(R_t^B\) are time reversals with respect to \(p_t(x)\). That is, for \(x\neq y\),
\begin{align}
    R_t^B(x,y)
    =
    R_t^F(y,x)
    \frac{p_t(y)}{p_t(x)}.
\end{align}
\end{proposition}
\end{tcolorbox}
\begin{proof}
\begin{align}
    R_t^F(y,x)p_t(y)
    &=
    \mathbb E_{s_t,x_0,x_1|X_t=y}
    \left[
        \sum_{s'\in\mathcal S}
        \bm 1\{\phi(x_0,x_1,s')=x\}
        A_t (s_t,s')
    \right]p_t(y)
    \\
    &=
    \sum_{s'\in\mathcal S}
    \mathbb E_{x_0,x_1}
    \left[
        \sum_{s_t\in\mathcal S}
        p_t(s_t)
        \bm 1\{\phi(x_0,x_1,s_t)=y\}
        \bm 1\{\phi(x_0,x_1,s')=x\}
        A_t (s_t,s')
    \right].
\end{align}

For backward kernel:
\begin{align}
    R_t^B(x,y)p_t(x)
    &=
    \mathbb E_{s_t,x_0,x_1|X_t=x}
    \left[
        \sum_{s'\in\mathcal S}
        \bm 1\{\phi(x_0,x_1,s')=y\}
        B_t (s_t,s')
    \right]p_t(x)
    \\
    &=
    \sum_{s'\in\mathcal S}
    \mathbb E_{x_0,x_1}
    \left[
        \sum_{s_t\in\mathcal S}
        p_t(s_t)
        \bm 1\{\phi(x_0,x_1,s_t)=x\}
        \bm 1\{\phi(x_0,x_1,s')=y\}
        B_t (s_t,s')
    \right]\\
     &=
    \sum_{s_t\in\mathcal S}
    \mathbb E_{x_0,x_1}
    \left[
        \sum_{s'\in\mathcal S}
        p_t(s')
        \bm 1\{\phi(x_0,x_1,s')=x\}
        \bm 1\{\phi(x_0,x_1,s_t)=y\}
        B_t (s,s_t)
    \right].
\end{align}
The last equality is just relabeling $s$ and $s'$.

Since \(A_t \) and \(B_t \) are time reversals :
\begin{align}
    p_t(s)A_t (s,s')
    =
    p_t(s')B_t (s',s).
\end{align}
\begin{align}
     R_t^B(x,y)p_t(x) =  R_t^F(y,x)p_t(y)
\end{align}
\end{proof}

With these two propositions, we can simply design the generator in this latent space and then marginalize it to obtain the generator in the original data space.

We now consider three cases, from simple to slightly more complicated designs.

\paragraph{Case 1: no mask state and one-way jumps from \(x_0\) to \(x_1\).}\label{par:ctmc_case1}

We first consider the simplest interpolant, with no mask state.
\begin{align}
     &x_t \sim  \alpha_t \delta_{x_0}(x_t) + \beta_t \delta_{x_1}(x_t). 
\end{align}
We choose \(\alpha_0=1,\beta_0=0\) and \(\alpha_1=0,\beta_1=1\), with \(\alpha_t\) decreasing and \(\beta_t\) increasing.

Let the latent state space be
\begin{align}
    \mathcal S=\{0,1\},
\end{align}
with 
\begin{align}
    \mathbb P(S_t=0)=\alpha_t,
    \qquad
    \mathbb P(S_t=1)=\beta_t,
    \qquad
    \alpha_t+\beta_t=1.
\end{align}
and the decoding map is 
\begin{align}
    \phi(x_0,x_1,0)=x_0,
    \qquad
    \phi(x_0,x_1,1)=x_1 .
\end{align}

A conditional forward latent generator in matrix form is as follows.
Under our convention, rows are source states and columns are target states,
\begin{align} \label{eq:appendix_case1_conditional_At}
    A_t
    =
    \begin{bmatrix}
        \dot \alpha_t/\alpha_t &  -\dot \alpha_t/\alpha_t\\
        0 & 0
    \end{bmatrix}.
\end{align}
We can verify the forward Kolmogorov equation in latent space $\mathcal{S}$:
\begin{align}
  \partial_t   \begin{bmatrix}
     \alpha_t\\\beta_t
   \end{bmatrix}= \begin{bmatrix}
        \dot \alpha_t/\alpha_t &  -\dot \alpha_t/\alpha_t\\
        0 & 0
    \end{bmatrix}^\top  \begin{bmatrix}
       \alpha_t\\\beta_t
   \end{bmatrix}
\end{align}
Similarly, we have the backward generator:
\begin{align} \label{eq:appendix_case1_conditional_Bt}
    B_t
    =
    \begin{bmatrix}
        0 & 0\\
                \dot \beta_t/\beta_t &  -\dot \beta_t/\beta_t\\
    \end{bmatrix}.
\end{align}

Using \Cref{eq:ctmc_marginal_generator,eq:ctmc_marginal_rate}, we can derive the marginal generator:
\begin{align}
      R_t^F(x,y)
    =&
    \mathbb  E_{s_t, x_0, x_1|X_t=x}
    \left[
        \sum_{s'\in\mathcal S}
        \bm 1\{\phi(x_0,x_1,s')=y\}
        A_t(s_t,s')
    \right]
\end{align}

For \(y\neq x\),   the only transition that has non-zero entry in $A$ is from \(x_0\)-state to the \(x_1\)-state, with rate
\begin{align}
    A_t(0,1)
    =
    -\frac{\dot\alpha_t}{\alpha_t}.
\end{align}
Therefore, for \(y\neq x\),
\begin{align}
    R_t^F(x,y)
    &=
    \mathbb E_{s_t,x_0,x_1 \mid X_t=x}
    \left[
        \mathbf{1}\{\phi(x_0,x_1,1)=y\}
        A_t(s_t,1)
    \right]
    \\
    &=
    \mathbb E_{s_t,x_0,x_1 \mid X_t=x}
    \left[
        \mathbf{1}\{x_1=y\}
        \mathbf{1}\{s_t=0\}
        \left(-\frac{\dot\alpha_t}{\alpha_t}\right)
    \right]
    \\
     &=
    \mathbb E_{s_t,x_0,x_1 \mid X_t=x}
    \left[
        \mathbf{1}\{x_1=y\}
        \mathbf{1}\{s_t=0\}
        \left(-\frac{\dot\alpha_t}{\alpha_t}\right)
    \right] \\ &\quad \quad \quad \quad \quad \quad  +  \underbrace{\mathbb E_{s_t,x_0,x_1 \mid X_t=x}
    \left[
        \mathbf{1}\{x_1=y\}
        \mathbf{1}\{s_t=1\}
        \left(-\frac{\dot\alpha_t}{\alpha_t}\right)
    \right]}_{=0 \text{ due to } y\neq x}
    \\
     &=
    \mathbb E_{s_t,x_0,x_1 \mid X_t=x}
    \left[
        \mathbf{1}\{x_1=y\}
        \left(-\frac{\dot\alpha_t}{\alpha_t}\right)
    \right]\\
    &=
    \mathbb E_{ x_1 \mid X_t=x}
    \left[
        \mathbf{1}\{x_1=y\}
        \left(-\frac{\dot\alpha_t}{\alpha_t}\right)
    \right]\\
    &=
    -\frac{\dot\alpha_t}{\alpha_t}
    \mathbb P(x_1=y \mid X_t=x)
    \\
    &=
    -\frac{\dot\alpha_t}{\alpha_t}
    p_{1|t}(y|x).
\end{align}

The diagonal entry is then fixed by the generator row-sum condition:
\begin{align}
    R_t^F(x,x)
    &=
    -\sum_{y\neq x} R_t^F(x,y)
    \\
    &=
    \frac{\dot\alpha_t}{\alpha_t}
    \sum_{y\neq x} p_{1|t}(y|x)
    \\
    &=
    \frac{\dot\alpha_t}{\alpha_t}
    \left(1-p_{1|t}(x|x)\right)
    \\
    &=
    -\frac{\dot\alpha_t}{\alpha_t}
    \left(p_{1|t}(x|x)-1\right).
\end{align}
Combining the off-diagonal and diagonal cases gives
the marginal forward generator is
\begin{align} \label{eq:appendix_case1_marginal_RF}
    R_t^F(x,y)
    =
    -\frac{\dot\alpha_t}{\alpha_t}
    \left[
        p_{1|t}(y|x)-\delta_x(y)
    \right],
\end{align}
Similarly, the marginal backward generator is
\begin{align} \label{eq:appendix_case1_marginal_RB}
    R_t^B(x,y)
    =
    \frac{\dot\beta_t}{\beta_t}
    \left[
        p_{0|t}(y|x)-\delta_x(y)
    \right],
\end{align}
In practice, this case suggests a simple parameterization. 
We can directly parameterize \(q_{\theta,t}^F(\cdot|x)\), \(q_{\phi,t}^B(\cdot|x)\)  for posterior endpoint distribution \(p_{1|t}(\cdot|x)\), \(p_{0|t}(\cdot|x)\).
Then learn the generator by Cross-entropy loss:
\begin{align}  \label{eq:appendix_case1_celoss_1}
    \mathcal J_F(\theta)
    &=
    \mathbb E_{t,x_0,x_1,s_t}
    \left[
        -\log q_{\theta,t}^F(x_1|x_t)
    \right],
    \\ \label{eq:appendix_case1_celoss_2}
    \mathcal J_B(\phi)
    &=
    \mathbb E_{t,x_0,x_1,s_t}
    \left[
        -\log q_{\phi,t}^B(x_0|x_t)
    \right].
\end{align}
Note that this case coincides with the discrete flow matching formulation of
\citet[DFM]{gat2024discrete}. However, DFM mainly considers one transport direction in its experiments, and does not make explicit that the two directional formulations are time reversals of one another.

DFM also proposes a more general construction for the forward and
backward processes, involving mixtures with additional components
\citep[Theorem~3, Eq.~21]{gat2024discrete}. 
However, the forward and backward processes designed by this construction are NOT time-reversal of each other in general.
This is not an issue for the application in generative modeling, as we only care about the endpoint marginal of the transport.
However, for our application, we require the forward and backward \emph{path} to be time-reversal in optimality to reduce the variance of the estimator.
Therefore, the form of DFM is not general enough for our usage.

Our design framework is hence more general and stronger: 
it gives a concrete recipe for designing forward and backward discrete dynamics that are explicit time reversals of one another. 
In the following case 2, we consider the second case where we introduce noise, and we can see that this is distinct from the form of DFM.

Before considering the second case, we also note that the objective by \Cref{eq:appendix_case1_celoss_1} is not the only way to learn the forward and backward rate matrices.
We can also directly use generator matching and the conclusion of Proposition~\ref{prop:latent_ctmc_marginal_reversal} allows us to directly marginalize the latent forward and backward generators.

Precisely, for a sampled tuple \((t,x_0,x_1,s_t)\), define the sample-level forward and backward rate targets
\begin{align}
    \widehat R_t^F(y|x_t,x_0,x_1,s_t)
    &:=
    \sum_{s'\in\mathcal S}
    \bm 1\{\phi(x_0,x_1,s')=y\}
    A_t(s_t,s'),
    \\
    \widehat R_t^B(y|x_t,x_0,x_1,s_t)
    &:=
    \sum_{s'\in\mathcal S}
    \bm 1\{\phi(x_0,x_1,s')=y\}
    B_t(s_t,s').
\end{align}
Then
\begin{align}
    R_t^F(x,y)
    =
    \mathbb E
    \left[
        \widehat R_t^F(y|x_t,x_0,x_1,s_t)
        \,\middle|\,
        X_t=x
    \right],
\end{align}
and similarly for \(R_t^B\).

The generator-matching loss following \citep{holderrieth2024generator} is then
\begin{align} \label{eq:case1_gm}
    \mathcal J_F(\theta)
    =
    \mathbb E_{t,x_0,x_1,s_t}
    \left[
        \sum_{y\neq x_t}
        \left(
            R_{\theta,t}^F(x_t,y)
            -
            \widehat R_t^F(y|x_t,x_0,x_1,s_t)
            \log R_{\theta,t}^F(x_t,y)
        \right)
    \right],
\end{align} 
and
\begin{align} \label{eq:case1_gm2}
    \mathcal J_B(\phi)
    =
    \mathbb E_{t,x_0,x_1,s_t}
    \left[
        \sum_{y\neq x_t}
        \left(
            R_{\phi,t}^B(x_t,y)
            -
            \widehat R_t^B(y|x_t,x_0,x_1,s_t)
            \log R_{\phi,t}^B(x_t,y)
        \right)
    \right].
\end{align}

\paragraph{Case 2: mask state with time-reversal latent dynamics.}\label{par:ctmc_case2}
We now consider the interpolant with a mask state:
\begin{align}
    p_t(x_t|x_0,x_1)
    =
    \alpha_t\delta_{x_0}(x_t)
    +
    \beta_t\delta_{x_1}(x_t)
    +
    \gamma_t\delta_M(x_t).
\end{align}
In our experiments and also in this discussion, we choose the schedule for simplicity.
\emph{We note that our discussion is not limited to this design choice; we choose it to avoid overloading of notations.}
\begin{align}
    \alpha_t=(1-a_t)^2,
    \qquad
    \beta_t=a_t^2,
    \qquad
    \gamma_t=2a_t(1-a_t),
\end{align}
where \(a_0=0\), \(a_1=1\), and \(\dot a_t>0\).

The latent state space is
\begin{align}
    \mathcal S=\{0,M,1\},
\end{align}
with decoding map
\begin{align}
    \phi(x_0,x_1,0)=x_0,
    \qquad
    \phi(x_0,x_1,M)=M,
    \qquad
    \phi(x_0,x_1,1)=x_1 .
\end{align}
The latent marginal is
\begin{align}
    p_t(S_t=0)=\alpha_t,
    \qquad
    p_t(S_t=M)=\gamma_t,
    \qquad
    p_t(S_t=1)=\beta_t .
\end{align}
We then design a conditional latent generator realizing this latent path.
The latent generator ($3 \times 3$ matrix) has $9$ elements, with 6 out of them to be free (diagonal elements are not free).
We then need the generator to satisfy the forward Kolmogorov Equation.
This Equation has 3 states, with 2 of them to be free (as they need to sum to one).
Therefore, we still have 4 degrees of freedom to choose.
For simplicity, we design the following  forward latent generator:
\begin{align} \label{eq:appendix_case2_conditional_At}
    A_t
    =
    \frac{2\dot a_t}{1-a_t}
    \begin{bmatrix}
        -1 & 1-a_t & a_t\\
        0 & -\frac{a_t}{2} & \frac{a_t}{2}\\
        0 & 0 & 0
    \end{bmatrix}.
\end{align}
Equivalently, the nonzero forward transitions are
\begin{align}
    &0 \to M:
    &&\text{with rate } 2\dot a_t,
    \\
    &0 \to 1:
    &&\text{with rate } \frac{2a_t\dot a_t}{1-a_t},
    \\
   & M \to 1:
    &&\text{with rate } \frac{a_t\dot a_t}{1-a_t}.
\end{align}
One can verify the forward Kolmogorov Equation:
\begin{align}
    \partial_t
    \begin{bmatrix}
        \alpha_t\\
        \gamma_t\\
        \beta_t
    \end{bmatrix}
    =
    A_t^\top
    \begin{bmatrix}
        \alpha_t\\
        \gamma_t\\
        \beta_t
    \end{bmatrix}.
\end{align}

The corresponding backward latent generator is
\begin{align} \label{eq:appendix_case2_conditional_Bt}
    B_t
    =
    \frac{2\dot a_t}{a_t}
    \begin{bmatrix}
        0 & 0 & 0\\
        \frac{1-a_t}{2} & -\frac{1-a_t}{2} & 0\\
        1-a_t & a_t & -1
    \end{bmatrix}.
\end{align}
Equivalently, the nonzero backward transitions are
\begin{align}
   & 1  \to M:
    &&\text{with rate } 2\dot a_t,
    \\
    &1 \to 0:
    &&\text{with rate } \frac{2(1-a_t)\dot a_t}{a_t},
    \\
    &M \to 0:
    &&\text{with rate } \frac{(1-a_t)\dot a_t}{a_t}.
\end{align}
By direct calculation, we can verify that for \(s\neq s'\), we have
\begin{align}
    B_t(s,s')
    =
    A_t(s',s)
    \frac{p_t(s')}{p_t(s)}.
\end{align}
i.e., the latent forward and backward generators are time reversals w.r.t the latent marginal \(p_t(s)\). 

Unfortunately, unlike case 1, we cannot the clean form for the forward and backward generator involving only \(p_{1|t}(\cdot|x)\) and  \(p_{0|t}(\cdot|x)\).
Therefore, we directly parameterize the marginal generators in the form of nonnegative forward and backward rates
\begin{align}
    R_{\theta,t}^F(x,y),
    \qquad
    R_{\phi,t}^B(x,y).
\end{align}
More precisely, given the current state \(x\), the network's output is one nonnegative vector over possible next states for each token position, corresponding to the off-diagonal entries of
\(R_{\theta,t}^F(x,\cdot)\) and \(R_{\phi,t}^B(x,\cdot)\).
The diagonal entries are then reconstructed by the CTMC convention
$
    R_{\theta,t}^F(x,x)
    =
    -\sum_{y\neq x}R_{\theta,t}^F(x,y),
    R_{\phi,t}^B(x,x)
    =
    -\sum_{y\neq x}R_{\phi,t}^B(x,y).$

By Proposition~\ref{prop:latent_ctmc_marginal_reversal}, we want to learn the generator as marginalizations of the latent forward and backward generators.
We train this marginalization by generator matching.
This yields the same loss as those we discuss at the end of case 1.
For easy reference, we restate them here.

For a sampled tuple \((t,x_0,x_1,s_t)\), define the sample-level forward and backward rate targets
\begin{align}
    \widehat R_t^F(y|x_t,x_0,x_1,s_t)
    &:=
    \sum_{s'\in\mathcal S}
    \bm 1\{\phi(x_0,x_1,s')=y\}
    A_t(s_t,s'),
    \\
    \widehat R_t^B(y|x_t,x_0,x_1,s_t)
    &:=
    \sum_{s'\in\mathcal S}
    \bm 1\{\phi(x_0,x_1,s')=y\}
    B_t(s_t,s').
\end{align}
Then
\begin{align}
    R_t^F(x,y)
    =
    \mathbb E
    \left[
        \widehat R_t^F(y|x_t,x_0,x_1,s_t)
        \,\middle|\,
        X_t=x
    \right],
\end{align}
and similarly for \(R_t^B\).

The generator-matching loss following \citep{holderrieth2024generator} is then
\begin{align} \label{eq:appendix_case2_GMF}
    \mathcal J_F(\theta)
    =
    \mathbb E_{t,x_0,x_1,s_t}
    \left[
        \sum_{y\neq x_t}
        \left(
            R_{\theta,t}^F(x_t,y)
            -
            \widehat R_t^F(y|x_t,x_0,x_1,s_t)
            \log R_{\theta,t}^F(x_t,y)
        \right)
    \right],
\end{align} 
and
\begin{align} \label{eq:appendix_case2_GMB}
    \mathcal J_B(\phi)
    =
    \mathbb E_{t,x_0,x_1,s_t}
    \left[
        \sum_{y\neq x_t}
        \left(
            R_{\phi,t}^B(x_t,y)
            -
            \widehat R_t^B(y|x_t,x_0,x_1,s_t)
            \log R_{\phi,t}^B(x_t,y)
        \right)
    \right].
\end{align}

After learning, during simulation, we discretize the CTMC with an Euler step, and given the current state \(x\) at time \(t\), the forward one-step transition kernel is
\begin{align}
    K_{\theta,t}^F(x,y)
    =
    \delta_x(y)
    +
    \Delta t\, R_{\theta,t}^F(x,y),
\end{align}

Similarly, the backward one-step transition kernel is
\begin{align}
    K_{\phi,t}^B(x,y)
    =
    \delta_x(y)
    +
    \Delta t\, R_{\phi,t}^B(x,y),
\end{align}
In practice, \(\Delta t\) is chosen small enough so that the diagonal probabilities are mostly nonnegative; otherwise, the transition probabilities are clipped and renormalized for numerical stability.
The forward-backward path RND can then be calculated using these categorical distribution probabilities.

\paragraph{Case 3: mask state with only the path \(x_0\to M\to x_1\).}
In the above case, we mentioned that there are many different generator designs to realize the same path (with 4 extra DOF). 
In this section, we discuss the simplest case, where we only allow the path \(x_0\to M\to x_1\) for the latent generator.
However, as we will see, this design is problematic, and we can only rescue it with a special design of the interpolant schedule.

We now still consider the following masked interpolant:
\begin{align}
    p_t(x_t|x_0,x_1)
    =
    \alpha_t\delta_{x_0}(x_t)
    +
    \beta_t\delta_{x_1}(x_t)
    +
    \gamma_t\delta_M(x_t).
\end{align}
We choose a latent forward generator that only enables the edges
\begin{align}
    0\to M,
    \qquad
    M\to 1 .
\end{align}
This generator matrix is
\begin{align}
    A_t
    =
    \begin{bmatrix}
        -r_{0\to M}(t) & r_{0\to M}(t) & 0\\
        0 & -r_{M\to 1}(t) & r_{M\to 1}(t)\\
        0 & 0 & 0
    \end{bmatrix}.
\end{align}
The forward Kolmogorov equation in latent space gives
\begin{align}
    \dot\alpha_t
    &=
    -r_{0\to M}(t)\alpha_t,
    \\
    \dot\beta_t
    &=
    r_{M\to 1}(t)\gamma_t .
\end{align}
Therefore,
\begin{align}
    r_{0\to M}(t)
    =
    -\frac{\dot\alpha_t}{\alpha_t},
    \qquad
    r_{M\to 1}(t)
    =
    \frac{\dot\beta_t}{\gamma_t}.
\end{align}
Thus,
\begin{align}
    A_t
    =
    \begin{bmatrix}
        \dfrac{\dot\alpha_t}{\alpha_t}
        &
       \tcboxmath[
    colback=red!30,
    colframe=red!0,
    boxrule=0pt,
    arc=1pt,
    left=1pt,
    right=1pt,
    top=1pt,
    bottom=1pt
  ]{ -\dfrac{\dot\alpha_t}{\alpha_t}}
        &
        0
        \\[8pt]
        0
        &
        -\dfrac{\dot\beta_t}{\gamma_t}
        &
        \dfrac{\dot\beta_t}{\gamma_t}
        \\[8pt]
        0 & 0 & 0
    \end{bmatrix}.
\end{align}
So far so good.
However, let's take a closer look at the highlighted number in the above matrix.
The denominator $\alpha_t = 0$, when $t=1$, ensuring that the marginal density at $t=1$ is $\pi_1=\pi_B$.
However, the numerator is typically non-zero unless $\alpha_t = 0$ when $t$ is larger than some threshold.

\emph{What does this mean?}
In the latent space, the interpretation is simple: any trajectory that remains in latent state \(0\) is forced to jump to the mask state as \(t\to 1\).
This is harmless at the latent level, because the probability of the latent state \(0\) also vanishes.

However, after projecting to the original data space, this behavior can become problematic.
The observed state \(x_t=x\neq M\) does not necessarily reveal whether the latent state is \(0\) or \(1\), especially when the same token can appear in both \(x_0\) and \(x_1\).
Therefore, near the endpoint, the observed generator may assign a very large masking rate to any sample that \emph{looks compatible} with having come from \(x_0\).
Therefore, this marginalized generator on the observed space assigns a large transition probability towards masks and will not be stabilized.
This makes the dynamics unstable.

A simple rescue is to design a cutoff, for example:
\begin{align}
    \alpha_t
    =
    \max\{0,1-2a_t\},
    \qquad
    \beta_t
    =
    \max\{0,2a_t-1\},
    \qquad
    \gamma_t
    =
    \min\{2a_t,2-2a_t\},
\end{align}
where \(a_0=0\), \(a_1=1\), and \(\dot a_t>0\). 
This stops the transition \(0\to M\) before its rate blows up.
In fact, this schedule also implies that \(\alpha_t\) and \(\beta_t\) are never nonzero at the same time.
Therefore, we can view this construction as concatenating two masked diffusion models \citep{lou2023discrete,shi2024simplified} back-to-back:
first \(x_0\to M\), and then \(M\to x_1\).

\section{Composition of Local Generators}
\label{appendix:composition}

In the previous section, we mainly discussed generators for a single modality, or even for a single coordinate in the CTMC case.
In practice, however, the systems of interest may contain multiple coordinates or coordinates in different modalities.
Without loss of generality, suppose the full state is 
\begin{align}
    Z=(X,Y)\in  \mathcal{Z} = X\times\mathcal Y.
\end{align}
where \(X\) and \(Y\) may represent different modalities or different parts of the system.
In this section, we show that it is sufficient to define a generator for each component, conditioned on the full state, and then combine these components to obtain the full generator on \(\mathcal Z\).

Let's define the \(X\)-component generator  by freezing \(Y_t=y\) and updating only \(X_t\):
\begin{align}
    \mathcal L_{X, t} f(x,y)
    :=
    \lim_{\delta\downarrow 0}
    \frac{
        \mathbb E\!\left[
            f(X_{t+\delta},y)
            \,\middle|\,
            X_t=x,Y_t=y
        \right]
        -
        f(x,y)
    }{\delta}.
\end{align}
Similarly, the \(Y\)-component generator is defined by freezing \(X_t=x\) and updating only \(Y_t\):
\begin{align}
    \mathcal L_{Y, t} f(x,y)
    :=
    \lim_{\delta\downarrow 0}
    \frac{
        \mathbb E\!\left[
            f(x,Y_{t+\delta})
            \,\middle|\,
            X_t=x,Y_t=y
        \right]
        -
        f(x,y)
    }{\delta}.
\end{align}
Then, we can define the full generator as the summation of these local generators.
\begin{align}
     \mathcal L_t f(x,y) =      \mathcal L_{X, t} f(x,y)  +  \mathcal L_{Y, t} f(x,y) 
\end{align}
This means that we can learn a generator for each component separately, while conditioning each generator on the full joint state.
After learning, we simply add these component generators to obtain the full generator and evolve the whole state simultaneously.
Importantly, this does not mean the two components are independent.
The generators \(\mathcal L_{X,t}\) and \(\mathcal L_{Y,t}\) may both depend on the full state \((x,y)\), so the dynamics can still couple the two through their rate matrices or drifts.

Additionally, we note that this summation is not the general form. 
In fact, the most general form is $ \mathcal L_t
    =
    \mathcal L_{X,t}
    +
    \mathcal L_{Y,t}
    +
    \mathcal L_{XY,t}$, where we need an extra interaction term $\mathcal L_{XY,t}$.
    Here, we make this modeling choice to simplify our parameterization and learning.

    \emph{Why can we make this design choice?}
    This is valid because our learned transport dynamics are determined by our designed conditional dynamics.
 If we design the conditional generator to be ``separable" across components, then its marginalization is also ``separable", conditional on the joint states, i.e.,
\begin{align}
    &\mathcal L_t^{x_0,x_1}
    =
    \mathcal L_{X,t}^{x_0,x_1}
    +
    \mathcal L_{Y,t}^{x_0,x_1}
    \quad
    \Longrightarrow
    \quad
    \mathcal L_t
    =
    \mathcal L_{X,t}
    +
    \mathcal L_{Y,t},\\
   \text{where}\quad   &\mathcal L_{X,t} f
    =
    \mathbb E_{x_0,x_1|z_t}
    \left[
        \mathcal L_{X,t}^{x_0,x_1} f
    \right],
    \mathcal L_{Y,t} f
    =
    \mathbb E_{x_0,x_1|z_t}
    \left[
        \mathcal L_{Y,t}^{x_0,x_1} f
    \right].
\end{align}
We also note that this component-wise ``conditional-independent" construction has also been used in graph generation \citep{jo2022score}, and is a common choice in discrete diffusion
\citep{lou2023discrete,shi2024simplified}.
In discrete diffusion, one typically learns a rate matrix (typically in the form of the concrete score, or posterior distribution) for each coordinate, which corresponds to a sparse joint rate matrix where only transitions between states that differ in one token are allowed.
This construction is also widely used for designing multimodal diffusion models\citep{holderrieth2024generator,rojas2025diffuse}.

During simulation, the component updates do not need to be applied sequentially.
Changes in different components can occur simultaneously, known as \(\tau\)-leaping \citep{gillespie2001approximate}.

\section{Simplified Forward and Backward Ratio for Mask Diffusion }\label{app:mask_dm_simple}

Assume we are not learning between two complicated targets.
Instead, we are learning from a fully masked target to an unmasked target.
We do not need to learn both the forward and backward generators.
This is because the masking process has a known generator, and we only need to learn its marginal score, or directly learn the backward.
This covers the case of the commonly used masked diffusion \citep{shi2024simplified}.

Interestingly, in the setup of mask diffusion following \citet{shi2024simplified}, we do not need to explicitly calculate the mask and unmask probabilities in the work.
Instead, we only need to calculate the unmask posterior.
\par
More precisely, consider the unnormalized importance weight (negative work) calculated by forward and backward kernel ratios:
\begin{align}
W(x_{0:N}) = \tilde{P}_0(x_0)\frac{P^F(x_1|x_0)\cdots P^F(x_N|x_{N-1})}{P^B(x_0|x_1)\cdots P^B(x_N-1|x_N) }\frac{1}{P_M(x_N)}
\end{align}
Following \citet{shi2024simplified} and defining $\alpha_n$ in the same way as \citet{shi2024simplified} (e.g., ($1-t$ or  $1-\cos(\frac{\pi}{2}(1-t))$), the transition kernel can have the following types:
\begin{enumerate}
    \item both $x_n$ and $x_{n+1}$ are masks:
    \begin{align}
        &P^F(x_{n+1}|x_n) 
        =1\\
       & P^B(x_{n}|x_{n+1}) = 1-\mathbb{P}\{
        \text{unmask}
        \}
        =  \frac{1-\alpha_n}{1-\alpha_{n+1}}
    \end{align}
    \item Neither $x_n$ or  $x_{n+1}$ is mask:
    \begin{align}
       & P^F(x_{n+1}|x_n) = 1-\mathbb{P}\{
        \text{mask}
        \}
        =\frac{\alpha_{n+1}}{\alpha_n}\\
        &P^B(x_{n}|x_{n+1}) = 1
    \end{align}
    \item 
    $x_n$ is not mask, while $x_{n+1}$ is mask:
    \begin{align}
       & P^F(x_{n+1}|x_n) = \mathbb{P}\{
        \text{mask}
        \}
        =\frac{\alpha_n-\alpha_{n+1}}{\alpha_n}\\
        &P^B(x_{n}|x_{n+1}) = P(x_0=x_n|x_{n+1}) \mathbb{P}\{\text{unmask}\} = P(x_0=x_n|x_{n+1}) \frac{\alpha_n-\alpha_{n+1}}{1-\alpha_{n+1}}
    \end{align}
\end{enumerate}
Assume for a trajectory $x_{0:N}$, the case 3 happen at step $m$ to $m+1$.
Then the importance weight is given by
\begin{multline}
     W(x_{0:N}) = \tilde{P}_0(x_0)
    \frac{\alpha_1/\alpha_0}{1}\frac{\alpha_2/\alpha_1}{1}\cdots \\\frac{\alpha_m/\alpha_{m-1}}{1}\frac{(\alpha_m-\alpha_{m+1})/\alpha_m}{P(x_0=x_n|x_{n+1}) (\alpha_m-\alpha_{m+1})/ (1-\alpha_{m+1})}\frac{1}{(1-\alpha_{m+1}) / (1-\alpha_{m+2})}\cdots \\ \frac{1}{(1-\alpha_{N-2}) / (1-\alpha_{N-1})}  \frac{1}{(1-\alpha_{N-1}) / (1-\alpha_{N})} \frac{1}{P_M(x_N)}
\end{multline}
Hence
\begin{multline}
     W(x_{0:N}) = \tilde{P}_0(x_0)
    \frac{\alpha_1}{\alpha_0}\frac{\alpha_2}{\alpha_1}\cdots \\\frac{\alpha_m}{\alpha_{m-1}}\frac{(\alpha_m-\alpha_{m+1})/\alpha_m}{P(x_0=x_n|x_{n+1}) (\alpha_m-\alpha_{m+1})/ (1-\alpha_{m+1})}\frac{1-\alpha_{m+2}}{1-\alpha_{m+1} }\cdots \\ \frac{1-\alpha_{N-1}}{1-\alpha_{N-2} }\frac{1-\alpha_{N}}{1-\alpha_{N-1} } \frac{1}{P_M(x_N)}
\end{multline}
After cancellation, we obtain:
\begin{align}
     W(x_{0:N}) = \tilde{P}_0(x_0)
    \frac{1-\alpha_{N}}{\alpha_0}  \frac{1}{P(x_0=x_n|x_{n+1})} \frac{1}{P_M(x_N)} = \tilde{P}_0(x_0)
  \frac{1}{P(x_0=x_n|x_{n+1})} \frac{1}{P_M(x_N)} 
\end{align}
Therefore, we do not need to track the probability of masking/unmasking when estimating free energy with standard mask diffusion.

In fact, one may realize that this simplification leads to a weight similar to the AR case, but along an arbitrary reveal order. 
Therefore, using a masked diffusion as a transport can be interpreted as a random-order autoregressive transport, also for the free energy estimation case. Conversely, the AR free energy formulation can be viewed as a special case of this construction, where the random reveal order happens to coincide with the fixed autoregressive order.

\section{A Note on Heavy-tailed Distributions}\label{app:heavy_tailed}
By Remark~\ref{remark:variance_bound}, if we learn a diffusion process with finite time and a bounded, Lipschitz drift function, then we may obtain an estimator with infinite variance if the distribution on one side is heavy-tailed:
\begin{remark}[Heavy-tailed endpoints can force infinite path divergence]  Consider the forward process
\begin{align}
    \mathrm d X_t
    =
    f_t(X_t)\,\mathrm dt
    +
    \mathrm d W_t,
    \qquad
    X_0\sim p_0,
\end{align}
and a backward path law $\bwd Q$ whose terminal marginal is \(q_1\). Assume that
\(f_t\) is bounded and Lipschitz on the finite time interval \([0,1]\). If \(p_0\)
has infinite second moment while \(q_1\) is sub-Gaussian, then
\begin{align}
    D_{\mathrm{KL}}(\fwd P\,\|\,\bwd Q)
=
    \infty.
\end{align}
\end{remark}

\begin{proof}
By the chain rule for KL divergence on path space,
\begin{align}
    D_{\mathrm{KL}}(\fwd P\,\|\,\bwd Q)
    =
    D_{\mathrm{KL}}(p_1\,\|\,q_1)
    +
    \mathbb E_{X_1\sim p_1}
    \left[
        D_{\mathrm{KL}}
        \bigl(
            \fwd P(\,\cdot\,|X_1)
            \,\|\,
            \bwd Q(\,\cdot\,|X_1)
        \bigr)
    \right].
\end{align}
Therefore,
\begin{align}
    D_{\mathrm{KL}}(\fwd P\,\|\,\bwd Q)
    \geq
    D_{\mathrm{KL}}(p_1\,\|\,q_1).
\end{align}

It remains to show that \(D_{\mathrm{KL}}(p_1\,\|\,q_1)=\infty\).
To simplify the proof, we first write down the corresponding ODEs of the forward  processes:
\begin{align}
    \mathrm{d} x_t = a_t(x_t)\mathrm{d} t , \quad  x_0 \sim p_0 
\end{align}
where $a_t = f - \frac12 \nabla \log p_t$.
As the ODEs have the same marginals as the SDEs, we will focus our discussion on the ODEs.
\par
We now consider $x_t $ and a fixed  $ y_t = \mathbb{E}_{p_t}[X_t]$ for the first ODE. We denote $x_0, y_0 $ as the results by solving the ODE backward from $t$ to $0$.
\begin{align}
    x_0 - y_0 = (x_t - y_t) - \int \left[a_s(x_s) - a_s(y_s)\right]\mathrm{d}t
\end{align}
From the Lipschitz assumption, we get
\begin{align}
  \|  x_0 - y_0\| &\leq  \|x_t - y_t\| + \|\int  \left[a_s(x_s) - a_s(y_s)\right]\mathrm{d}t\|\\
  &\leq \|x_t - y_t\| + L\|\int  \left[ x_s - y_s\right]\mathrm{d}t\|
  \\
  &\leq \|x_t - y_t\| + L\int  \| x_s - y_s\|\mathrm{d}t
\end{align}
From Grönwall’s inequality and take square on both sides, we get
\begin{align}
      \|  x_0 - y_0\|^2&\leq  \|x_t - y_t\|^2 e^{2Lt}
\end{align}
We now take expectation over $x_t\sim p_t$, and then we have
\begin{align}
    \mathbb{E} \|x_t - y_t\|^2 \geq e^{-2Lt}    \mathbb{E}_{x_0\sim p_0} [\|  x_0 - y_0\|^2] = \infty
\end{align}

Now use the sub-Gaussian assumption on \(q_1\). In particular, there exists
\(\lambda>0\) such that
\begin{align}
    \mathbb E_{y\sim q_1}
    \left[
        \exp\left(\lambda \|y\|^2\right)
    \right]
    <
    \infty.
\end{align}
 Assume \(p_1\ll q_1\), and we write
\begin{align}
    r(x)=p_1(x)/q_1(x)
\end{align}
We use the  inequality
\begin{align}
    uv
    \leq
    u\log u - u + e^v,
    \qquad
    u\geq 0,\ v\in\mathbb R .
\end{align}
Let \(u=r(x)\) and \(v=\lambda\|x\|^2\), we obtain
\begin{align}
    r(x)\lambda\|x\|^2
    \leq
    r(x)\log r(x)
    -
    r(x)
    +
    \exp\left(\lambda\|x\|^2\right).
\end{align}
Taking expectation under \(q_1\) gives
\begin{align}
    \lambda\mathbb E_{p_1}\|x\|^2
    &\leq
    \mathbb E_{q_1}[r(X)\log r(X)]
    -
    \mathbb E_{q_1}[r(X)]
    +
    \mathbb E_{q_1}
    \left[
        \exp\left(\lambda\|x\|^2\right)
    \right]
    \\
    &=
    D_{\mathrm{KL}}(p_1\|q_1)
    -
    1
    +
  \mathbb E_{q_1}
    \left[
        \exp\left(\lambda\|x\|^2\right)
    \right]  .
\end{align}  
LHS is unbounded, while $ -
    1
    +
  \mathbb E_{q_1}
    \left[
        \exp\left(\lambda\|x\|^2\right)
    \right] $ is bounded.
Therefore, 
\begin{align}
    D_{\mathrm{KL}}(p_1\|q_1)=\infty .
\end{align}
\end{proof}

This conclusion is a generalized version of the results discussed by \citet{ren2025unified}, where they show that the heavy-tailed property will be maintained with an affine drift, while we show that this is also a potential issue for a Lipschitz drift function.

\section{Additional Experimental Details}
\label{sec:hyperparameters}

\subsection{CTMC FEAT: Lattice Ising Models} \label{sec:appendix_disc_feat_ising}
\subsubsection{System}
We consider estimating the free energy differences between two Ising models on lattices of size of $15\times15$, $25\times25$, and $32\times32$.
Samples are generated via Gibbs sampling by running 10,000 parallel Markov chains for 1,000,000 steps, yielding 10,000 samples in total.
Of these, 2,000 are reserved for testing, and the remaining samples are used for training.

The lattice Ising model is defined as $p(x) \propto \exp(- \beta H(x))$, where $\beta$ denotes the inverse temperature of the system, and $H(x): \{-1, 1\}^{L \times L} \rightarrow \mathbb{R}$ is the Hamiltonian given by $H(x) = - \sum_{\langle i, j \rangle} x_i x_j$.
Here, $\langle i, j \rangle$ runs over all neighboring spin pairs on a 2D cyclic lattice.
The task of CTMC FEAT is then to estimate the free energy difference:
\begin{align}
    \Delta F = F_{\beta_a} - F_{\beta_b}, \quad F_{\beta} = - \frac{1}{\beta} \log Z, \quad Z = \sum_x \exp(-\beta H(x)), 
\end{align}
in which we consider two settings of $(\beta_a, \beta_b) = (0.2, 0.4)$ and $(\beta_a, \beta_b) = (0.2, 0.6)$.

\subsubsection{CTMC FEAT Variants}

Following \cref{sec:appendix_ctmc_feat_theory}, we propose three CTMC FEAT variants.

\begin{wraptable}{r}{0.57\linewidth}
    \vspace{-0.5em}
    \centering
    \setlength{\tabcolsep}{3pt}
    \caption{Hyperparameters of CTMC FEAT.}
    \label{tab:hyperparam_ctmc}
    \resizebox{\linewidth}{!}{%
        \begin{tabular}{@{}lccc@{}}
            \toprule
            \textbf{Hyperparameter} & \textbf{FM} & \textbf{GM-no mask} & \textbf{GM-w.\ mask} \\
            \midrule
            \multicolumn{4}{c}{\textbf{System}} \\
            \midrule
            Lattice sizes & \multicolumn{3}{c}{$15\!\times\!15$, $25\!\times\!25$, $32\!\times\!32$} \\
            Inverse temperatures & \multicolumn{3}{c}{$\beta = 0.2 \leftrightarrow 0.4$, $\beta = 0.2 \leftrightarrow 0.6$} \\
            \midrule
            \multicolumn{4}{c}{\textbf{Scheduler}} \\
            \midrule
            Discrete bridges & \multicolumn{2}{c}{\makecell[c]{
            $x_t \sim \alpha_t \delta_{x_0}(x_t)$
            $+\, \beta_t \delta_{x_1}(x_t)$
            }} & \makecell[c]{
            $x_t \sim \alpha_t \delta_{x_0}(x_t)$ \\
            $+\, \beta_t \delta_{x_1}(x_t) + \gamma_t \delta_M(x_t)$
            } \\ [0.8em]
            Schedulers & \multicolumn{2}{c}{$\alpha_t = 1-t, \beta_t = t$} & \makecell[c]{
            $\alpha_t = (1-a_t)^2, \beta_t=a_t^2$ \\
            $\gamma_t = 2a_t (1-a_t); a_t = t$
            } \\
            \midrule
            \multicolumn{4}{c}{\textbf{Model}} \\
            \midrule
            Parameterization & \(q_{\theta,t}^F(\cdot|x)\), \(q_{\phi,t}^B(\cdot|x)\) & \multicolumn{2}{c}{$R_{\theta,t}^F(x_t,y)$, $R_{\phi,t}^B(x_t,y)$} \\
            Network architecture & \multicolumn{3}{c}{2D CNN} \\
            Hidden channels & 256 & 256 & 384 \\
            Depth (layers) & \multicolumn{3}{c}{12} \\
            Kernel size & \multicolumn{3}{c}{$5\times5$} \\
            Batch normalization & \multicolumn{3}{c}{Yes} \\
            Residual connections & \multicolumn{3}{c}{Yes} \\
            Input encoding & scalar & scalar & one-hot \\
            \midrule
            \multicolumn{4}{c}{\textbf{Training}} \\
            \midrule
            Optimizer & \multicolumn{3}{c}{AdamW} \\
            Learning rate & $3\times10^{-4}$ & $3\times10^{-4}$ & $2\times10^{-4}$ \\
            Batch size & \multicolumn{3}{c}{256} \\
            Iterations & 10,000 & 10,000 & 20,000 \\
            Weight decay & 0 & 0 & $10^{-4}$ \\
            LR schedule & No & No & Cosine \\
            EMA decay & \multicolumn{3}{c}{0.995} \\
            OT pair & \multicolumn{3}{c}{Yes} \\
            \midrule
            \multicolumn{4}{c}{\textbf{Simulation and Estimation}} \\
            \midrule
            Step size $\Delta t$ & \multicolumn{3}{c}{0.05} \\
            Evaluation sample size & \multicolumn{3}{c}{2,000} \\
            Estimator & \multicolumn{3}{c}{BAR} \\
            \bottomrule
            \vspace{-2em}
        \end{tabular}%
    }
    \vspace{-1em}
\end{wraptable}
\textbf{Flow Matching (FM).}
\textit{FM} follows \hyperref[par:ctmc_case1]{Case 1} in \cref{sec:appendix_ctmc_feat_theory}, where the discrete bridge is constructed as $x_t \sim  \alpha_t \delta_{x_0}(x_t) + \beta_t \delta_{x_1}(x_t)$ with $\alpha_t = 1-t$ and $\beta_t = t$.
The method directly parameterizes the endpoint posterior distributions \(p_{1|t}(\cdot|x)\) and \(p_{0|t}(\cdot|x)\) using \(q_{\theta,t}^F(\cdot|x)\) and \(q_{\phi,t}^B(\cdot|x)\), respectively, and learns them via the cross-entropy loss in \cref{eq:appendix_case1_celoss_1,eq:appendix_case1_celoss_2}.
After training, the marginal generators are obtained according to \cref{eq:appendix_case1_marginal_RF,eq:appendix_case1_marginal_RB}.

\textbf{Generator Matching without Masking (GM-no mask).}
\textit{GM-no mask} employs the same discrete bridge of \textit{FM}, namely $x_t \sim  \alpha_t \delta_{x_0}(x_t) + \beta_t \delta_{x_1}(x_t)$ with $\alpha_t = 1-t$ and $\beta_t = t$, which induces the conditional generators defined in \cref{eq:appendix_case1_conditional_At,eq:appendix_case1_conditional_Bt}.
Unlike \textit{FM}, which parametrizes the endpoint posterior distributions, \textit{GM-no mask} directly parameterizes the marignal generators $R_t^F(x,y)$ and $R_t^B(x,y)$ as $R_{\theta,t}^F(x_t,y)$ and $R_{\phi,t}^B(x_t,y)$, respectively.
These generators are trained using the generator-matching loss in \cref{eq:appendix_case2_GMF,eq:appendix_case2_GMB}.

\textbf{Generator Matching with Masking (GM-w. mask).}
\textit{GM-w. mask} follows \hyperref[par:ctmc_case2]{Case 2} in \cref{sec:appendix_ctmc_feat_theory}, where the discrete bridge is constructed as $x_t \sim  \alpha_t \delta_{x_0}(x_t) + \beta_t \delta_{x_1}(x_t) + \gamma_t\delta_M(x_t)$ with $\alpha_t=(1-a_t)^2, \beta_t=a_t^2, \gamma_t=2a_t(1-a_t)$ and $a_t = t$, which induces the conditional generators defined in \cref{eq:appendix_case2_conditional_At,eq:appendix_case2_conditional_Bt}.
Similar to \textit{GM-no mask}, \textit{GM-w. mask} parameterizes the marignal generators directly and learns them using the generator-matching loss in \cref{eq:appendix_case2_GMF,eq:appendix_case2_GMB}.

\subsubsection{Hyperparameters and Settings for CTMC FEAT}

\textbf{Network.}
For all three CTMC FEAT variants, we use a convolutional neural network (CNN) that treats each spin configuration as a 2D spatial feature map.
The input spin state $x_t \in \{0,1\}^{L \times L}$ (or $\{0,1,M\}^{L \times L}$ for the masking variant) is reshaped into a grid and combined with a sinusoidal time embedding via channel concatenation before being passed through the convolutional backbone.
The backbone consists of depth-12 convolutional layers with kernel size $5 \times 5$, batch normalization, ReLU activations, and residual connections.
The network produces two output heads, one for the forward direction and one for the backward direction, via $1\times1$ convolutions.
For FM and GM-no mask, the input is encoded as a scalar (yielding 2 input channels), while for GM-w.\ mask the input is encoded as a one-hot vector over the extended vocabulary $\{0, 1, M\}$ (yielding 4 input channels).

\textbf{Training details.}
We train all models using the AdamW optimizer with a batch size of 256. Models are trained for 10,000 iterations on the $15\times 15$ and $25\times 25$ lattice, and for 20,000 iterations on the $32\times 32$ lattices.
For FM and GM-no mask, we use a learning rate of $3\times10^{-4}$ with no weight decay and no learning rate schedule.
For GM-w.\ mask, we use a learning rate of $2\times10^{-4}$, a weight decay of $10^{-4}$, and a cosine annealing learning rate schedule.
We maintain an exponential moving average (EMA) of the model weights with decay rate 0.995 for all variants, and use the EMA model for evaluation.
To facilitate training, we sample paired data $(x_0, x_1)$ via an optimal transport (OT) plan computed within each mini-batch, following \citet{he2025feat}.
Gradients are clipped to unit norm.

\textbf{Estimation details.}
We use 2,000 held-out samples from each state for evaluation.
The forward and backward trajectories are discretized with step size $\Delta t = 0.05$, corresponding to 20 discretization steps.
The free energy difference is estimated using the bidirectional importance-weighted estimator and refined with the Bennett Acceptance Ratio (BAR) method as discussed in \cref{sec:inf_est_deltaF}.
The hyperparameters are summarized in \Cref{tab:hyperparam_ctmc}.
All results are run three times, and the mean and standard deviation are calculated.

\subsection{AR FEAT: Lattice Ising Models}

\subsubsection{System}

We consider the same lattice Ising system as CTMC FEAT for the AR setting.
Unlike CTMC FEAT, which learns a transport between two target distributions, AR FEAT learns an autoregressive proposal distribution for each target:
\begin{align}
    q_\theta(x) = \prod_{i=1}^{L^2} q_\theta(x_i \mid x_{<i}),
\end{align}
where the sites are ordered in raster scan order.

We evaluate inverse temperature \(\beta \in \{0.2,0.4,0.6\}\), on lattices of size \(15\times15\), \(25\times25\), and \(32\times32\).
For each lattice size and temperature, we use the same data as the CTMC case.
Each dataset contains 100,000 samples, of which 98,000 are used for training and 2,000 are held out for evaluation.

\subsubsection{Hyperparameters and Settings for AR FEAT}

\textbf{Network.}
We use an autoregressive convolutional neural network to model the Ising distribution in raster order.
For a configuration \(x \in \{0,1\}^{L\times L}\), the model outputs one Bernoulli logit for each lattice site, where the logit at site \(i\) parameterizes \(q_\theta(x_i=1\mid x_{<i})\).
The architecture follows a PixelCNN-style \citep{van2016conditional} masked convolutional design: the first layer uses a type-A mask, so the current site and all future sites are hidden, while subsequent layers use type-B masks, which allow the current hidden representation but still exclude future sites.
The backbone has 12 masked convolutional layers, hidden width 384, kernel size \(5\times5\), and GELU activations.
In addition, we include explicit causal periodic-neighbor features so that periodic boundary information that is already available in raster order can be used without leaking future spins.

\textbf{Training details.}
We train the autoregressive model by teacher forcing, minimizing the cross-entropy,
All models are trained using AdamW with learning rate \(2\times10^{-4}\), batch size 128, and no weight decay.
We train for 5 epochs and clip gradients to norm 5.
For each setting, 98,000 samples are used for training, and 2,000 samples are held out for evaluation.

\textbf{Estimation details.}
After training, we estimate the absolute free energy using the trained autoregressive model as a tractable proposal.
We draw 2,000 samples from \(q_\theta\) by ancestral sampling, and also use 2,000 held-out samples from the target distribution.
We combine the forward and reverse estimates using Bennett's Acceptance Ratio (BAR), initialized from the average of the two one-sided estimates.
We report the estimated free energy normalized by the number of spins, \(F/L^2\), and in the main table scale it as \(F/L^2\times 10^3\).
All results are run three times, and the mean and standard deviation are calculated.

\subsection{Multi FEAT: Ising Fluids}\label{app_details_ising_fluid}

\subsubsection{System}
We consider a system with both continuous components (coordinates) and discrete components (spin state).
We take the system from \citep{omelyan2004ising}.
Precisely,  denote the particle positions
\(\mathbf r=(r_1,\ldots,r_N)\), and the spin  \(s_i\in\{-1,+1\}\), and pairwise
distances \(r_{ij}=\|r_i-r_j\|\). The target  energy is
\begin{align}
    U(\mathbf r,\mathbf s)
    =
    \sum_{1\leq i<j\leq N}
    \left[
        \phi_{\mathrm{sc}}(r_{ij})
        -
        I_R(r_{ij})
        -
        J(r_{ij})s_i s_j
    \right].
    \label{eq:ising_fluid_energy}
\end{align}
where
\begin{align}
    J(r)
    &=
    \frac{2(z_1\sigma)^2}{z_1\sigma+1}
    \frac{\epsilon_J\sigma}{r}
    \exp[-z_1(r-\sigma)],
    \label{eq:ising_fluid_J}
    \\
    I(r)
    &=
    \frac{2(z_2\sigma)^2}{z_2\sigma+1}
    \frac{\epsilon_I\sigma}{r}
    \exp[-z_2(r-\sigma)].
    \label{eq:ising_fluid_I}
\end{align}
and
\begin{align}
    \phi_{\mathrm{sc}}(r)
    =
    \begin{cases}
    4\epsilon_{\rm LJ}
    \left[
        \left(\dfrac{\sigma}{r}\right)^{12}
        -
        \left(\dfrac{\sigma}{r}\right)^6
    \right]
    +\epsilon_{\rm LJ},
    & r < 2^{1/6}\sigma,
    \\[6pt]
    0,
    & r\geq 2^{1/6}\sigma.
    \end{cases}
    \label{eq:ising_fluid_soft_core}
\end{align}
In our experiments, we use
\begin{align}
    \sigma=\epsilon_J=\epsilon_{\rm LJ}=z_1=z_2=\lambda=1,
    \qquad
    H=0,
\end{align}
and parameterize the nonmagnetic attraction by
\begin{align}
    R=\frac{\epsilon_J}{\epsilon_I},
    \qquad
    \epsilon_I=\frac{1}{R}.
\end{align}
Under this parameterization,
\begin{align}
    J(r)
    &=
    \frac{\exp[-(r-1)]}{r},
    \\
    I(r)
    &=
    \frac{1}{R}\frac{\exp[-(r-1)]}{r}.
\end{align}
We then vary the value $R$ to obtain different targets.

We also found that this target might easily diffuse when $R$ is large. Therefore, we add a harmonic potential to avoid such diffusion:
\begin{align}
    U_h(\mathbf{r}) = k\frac12 \|\bar{r} - r_i\|^2
\end{align}
where we set $k=0.1$ and $\bar{r}$ represents the center of the system.

We visualize some samples in \Cref{fig:IsingFluid_55}.

We estimate for the 55-particle systems.
We estimated $\Delta F$ between $R=0.5$ and $R=1.0$, $R=0.5$ and $R=1.5$ as well as $R=0.5$ and $R=2.0$.
For the reference value, we draw samples from intermediate targets $R=0.5, 0.6,  \cdots, 1.9, 2.0$, and run MBAR \citep{shirts2008statistically}.

\subsubsection{Hyperparameters and Settings}

For simplicity of learning, we did not directly bridge between the two targets. Instead, we learn transport from a Gaussian and mask distribution to the target with coordinates and $\pm 1$ spins. 

\textbf{Network.}
We use a mixed continuous-discrete EGNN to jointly model the particle positions and spin variables.
The model takes as input the continuous particle positions, the discrete spin tokens, and the time variable.
The discrete state uses two output tokens corresponding to the two spin values.
For the 55-particle system, token embeddings have dimension $32$, and the time variable is embedded using a $32$-dimensional sinusoidal time embedding.
The EGNN backbone uses 8 layers with hidden dimension 96 and SiLU activations.
We also include explicit spin features ($\pm1$), consisting of the current spin value and a binary indicator for whether the token is masked.
Edges are fully connected, with squared pairwise distances used as edge features.
The model outputs both a continuous vector for the particle positions and discrete logits for the spin variables.

\textbf{Training details.}
When one side is Gaussian and mask, this means that the forward process has been already fixed.
We learn the backward process directly by generator matching or score matching.
Here, we train the model using the sum of a continuous denoising score-matching loss and a discrete cross-entropy loss.
For the continuous variables, we corrupt the data as
\begin{align}
    x_t = (1-t)x_0 + t\epsilon,
\end{align}
where \(\epsilon\) is centered Gaussian noise, and train the model to predict the target vector field \(\epsilon-x_0\).
For the discrete variables, we use a masking corruption process corresponding to \citet{shi2024simplified}.
At time \(t\), each token is kept with probability
\begin{align}
    \alpha_t = 1 - \cos\left(\frac{\pi}{2}(1-t)\right),
\end{align}
and otherwise replaced by the mask token.
The discrete loss is a weighted cross-entropy loss on the masked tokens, with weight
\begin{align}
    \pi \tan\left(\frac{\pi}{2}(1-t)\right).
\end{align}

All models are trained with AdamW using learning rate \(1\times 10^{-4}\), batch size 64, and no weight decay.
We train for 200,000  iterations with gradient clipping at norm 1.
We maintain an exponential moving average of the model parameters with decay 0.995 and use the EMA model for evaluation.
We also scale the target by a factor of $2$ to make learning more stable.

\textbf{Estimation details.}
For free-energy estimation, we use the BAR estimator in \Cref{eq:bar} with 2,000 held-out samples to estimate the free energy difference between the centered Gaussian particle positions and fully masked spin variables, to the target distribution.
We discretize the path using 300 time steps.
For the stochastic continuous updates used in work estimation, we directly obtain the score from vector field, and then set the noise scale to \(\sigma_t=0.5\).  
For the discrete part, following the discussion in \Cref{app:mask_dm_simple}, we directly use the posterior density.

The free energy difference between two states is given by the difference between their own free energy difference against the centered Gaussian with fully masked spin.

\subsection{Momentum-Augmented FEAT: Gaussian Mixture Models}

\subsubsection{System}

We next consider a continuous-state experiment on Gaussian mixture models (GMMs).
The task is to estimate the free energy difference between two 40-dimensional GMM targets.
The first target contains 40 Gaussian components, while the second target contains 16 Gaussian components.
Both targets use location scaling 2.0 and log-variance scaling \(-3.0\), with different random seeds for the two mixtures.

Here, we transport the joint state
\begin{align}
    z_t = (x_t, y_t) \in \mathbb{R}^{40} \times \mathbb{R}^{40},
\end{align}
where \(y_t\) plays the role of momentum.

\subsubsection{Hyperparameters and Settings}

\textbf{Network.} We parameterize two neural networks.
The first network learns a momentum-space score \(s_\theta(x_t,y_t,t)\), and the second network learns a vector field \(v_\phi(x_t,y_t,t)\).
Both networks are MLPs with sinusoidal time embeddings, 5 hidden layers, hidden width 400, GELU activations.
The networks take the concatenated state \((x_t,y_t)\) and time \(t\) as input and output a 40-dimensional vector.

\textbf{Training details.} We use Adam to optimize both networks jointly.
We maintain an exponential moving average of both networks with decay rate of 0.995 and use the EMA networks for evaluation.
Gradients are clipped to norm 5.

\textbf{Estimation Details.} For free energy estimation, we simulate forward and backward trajectories in the augmented state space using an OA integrator.
We use 5,000 samples for evaluation.

For the 1D visualization on \Cref{fig:underdamped_vis_1D}, we use the same setting, but apply it to GMM in 1D.

\section{Licenses for Existing Assets}\label{licence}
Momentum FEAT was modified from FEAT's codebase: \url{https://github.com/jiajunhe98/FEAT}, released under the MIT License.

 CTMC FEAT was modified from \url{https://github.com/J-zin/DNFS}, and the computation of the reference value is adapted from \url{https://github.com/ml-jku/DiffUCO/blob/main/IsingTheoryBaselines/IsingTheory.py}.
Both are released under the MIT License.

\section{Visualizations and Additional Results}

\subsection{CTMC FEAT on Ising Models}

\begin{figure}[H]
    \centering
    \begin{minipage}{\linewidth}
        \begin{minipage}{0.03\linewidth}
            \centering
            \vspace{3mm}
            \rotatebox{90}{\small Data}
        \end{minipage}
        \begin{minipage}{0.95\linewidth}
            \centering
            \begin{minipage}{0.32\linewidth}
                \centering
                {\small Ising 15x15}
                \includegraphics[width=1.\linewidth]{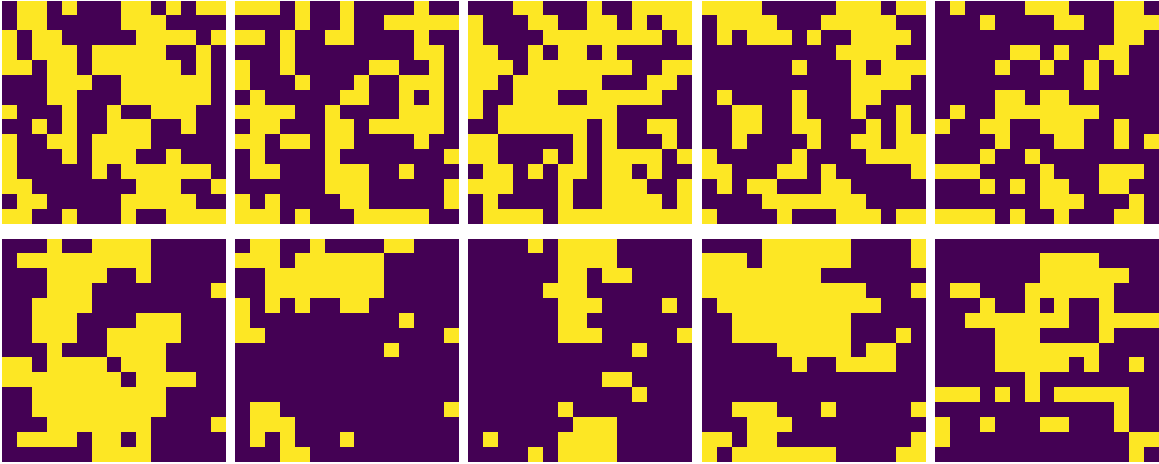}
            \end{minipage}
            \hfill
            \begin{minipage}{0.32\linewidth}
                \centering
                {\small Ising 25x25}
                \includegraphics[width=1.\linewidth]{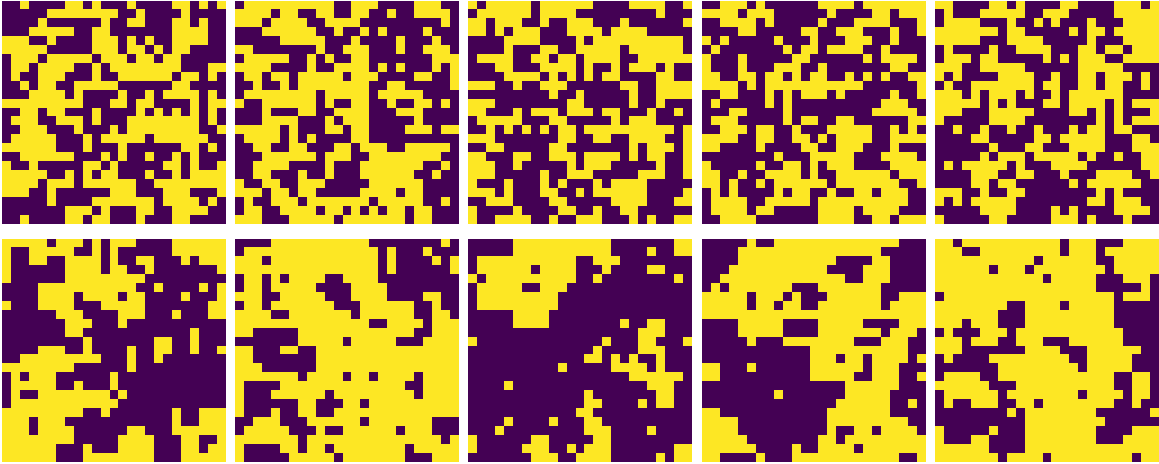}
            \end{minipage}
            \hfill
            \begin{minipage}{0.32\linewidth}
                \centering
                {\small Ising 32x32}
                \includegraphics[width=1.\linewidth]{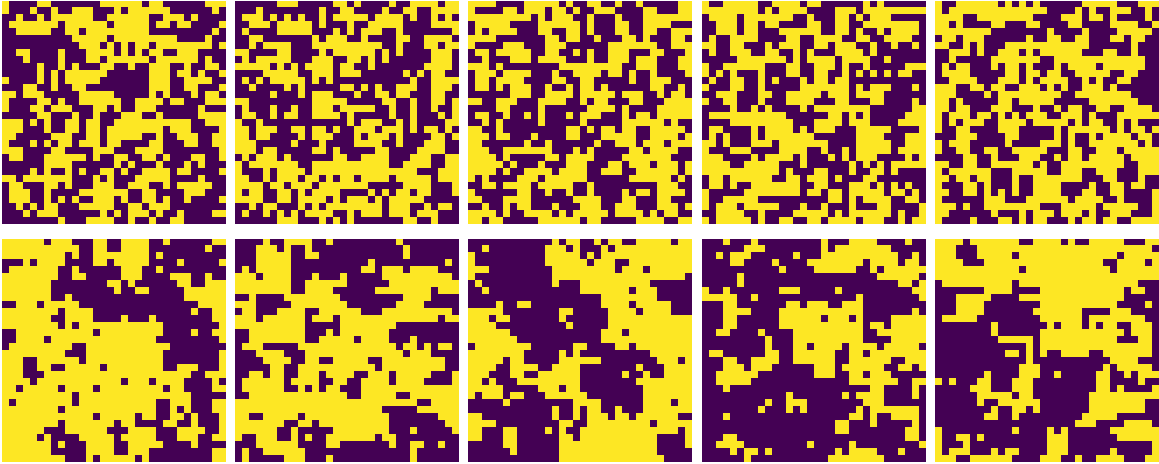}
            \end{minipage}
        \end{minipage}
    \end{minipage} \\
    \vspace{1mm}
    \begin{minipage}{\linewidth}
        \begin{minipage}{0.03\linewidth}
            \centering
            \rotatebox{90}{\small DFM}
        \end{minipage}
        \begin{minipage}{0.95\linewidth}
            \centering
            \begin{minipage}{0.32\linewidth}
                \centering
                \includegraphics[width=1.\linewidth]{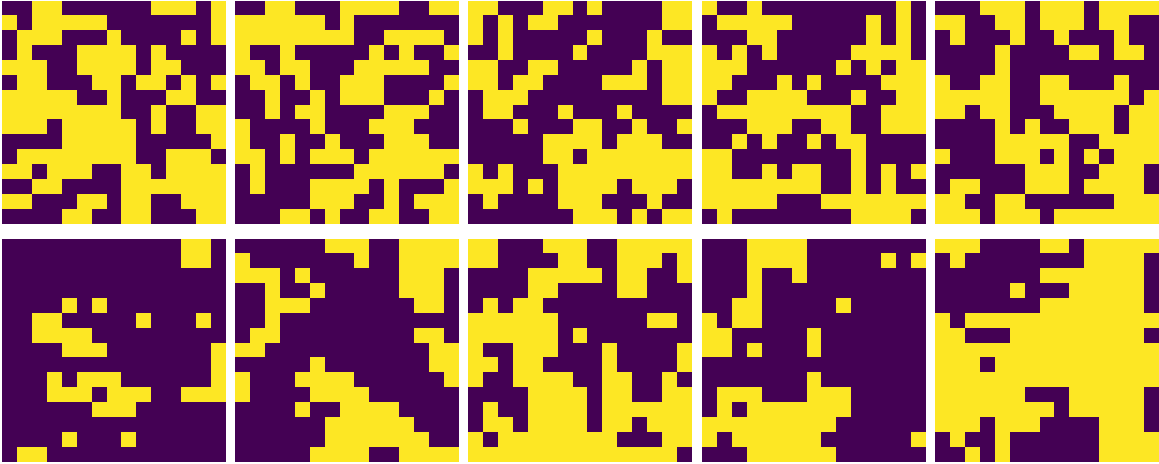}
            \end{minipage}
            \hfill
            \begin{minipage}{0.32\linewidth}
                \centering
                \includegraphics[width=1.\linewidth]{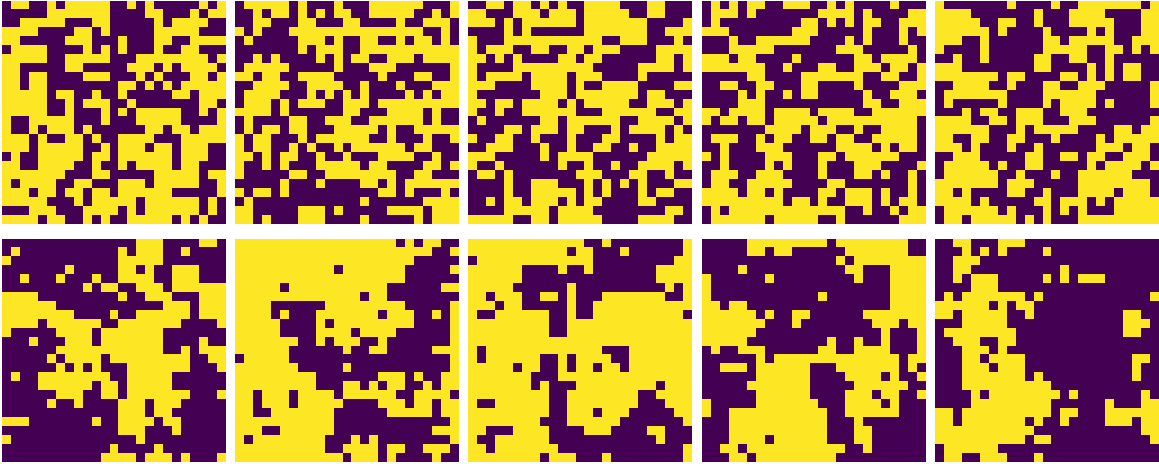}
            \end{minipage}
            \hfill
            \begin{minipage}{0.32\linewidth}
                \centering
                \includegraphics[width=1.\linewidth]{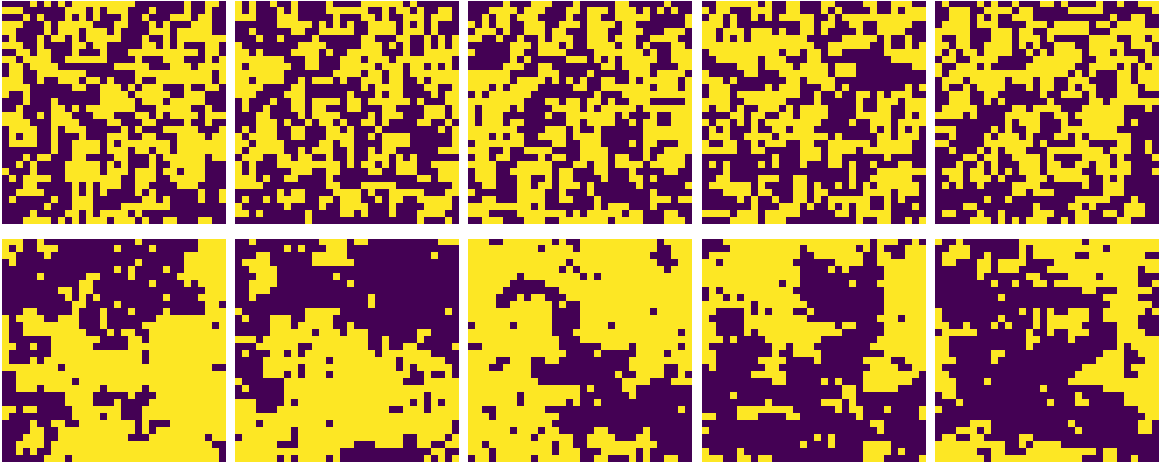}
            \end{minipage}
        \end{minipage}
    \end{minipage} \\
    \vspace{1mm}
    \begin{minipage}{\linewidth}
        \begin{minipage}{0.03\linewidth}
            \centering
            \rotatebox{90}{\small GM-no mask}
        \end{minipage}
        \begin{minipage}{0.95\linewidth}
            \centering
            \begin{minipage}{0.32\linewidth}
                \centering
                \includegraphics[width=1.\linewidth]{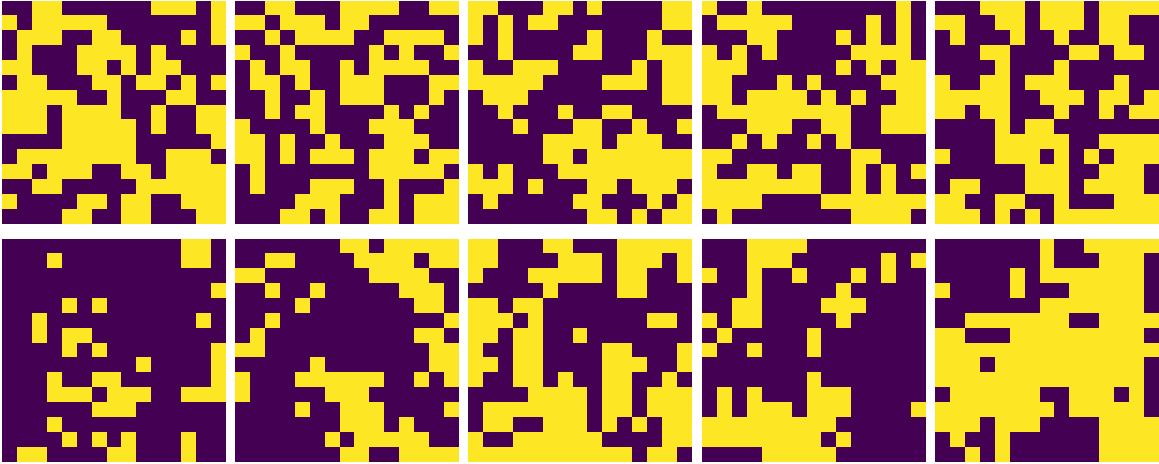}
            \end{minipage}
            \hfill
            \begin{minipage}{0.32\linewidth}
                \centering
                \includegraphics[width=1.\linewidth]{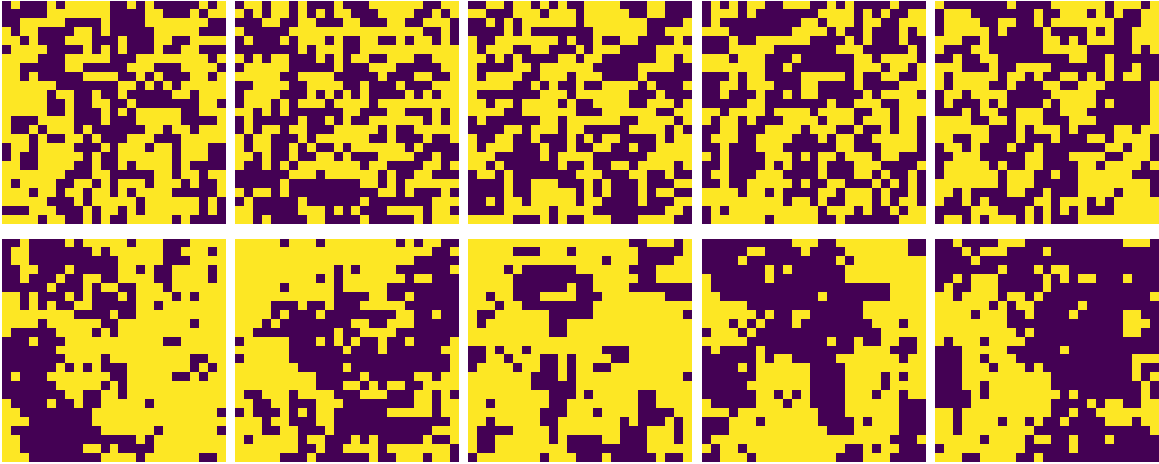}
            \end{minipage}
            \hfill
            \begin{minipage}{0.32\linewidth}
                \centering
                \includegraphics[width=1.\linewidth]{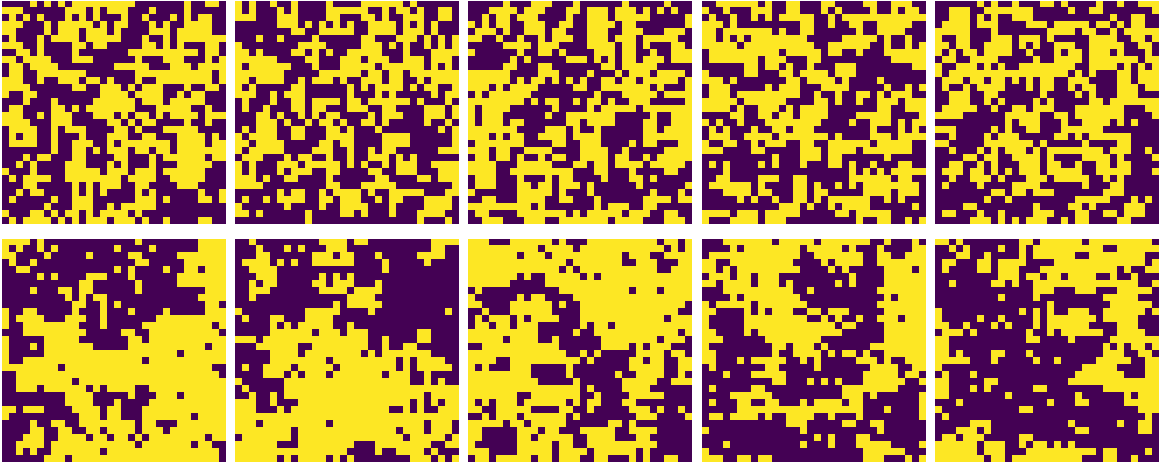}
            \end{minipage}
        \end{minipage}
    \end{minipage} \\
    \vspace{1mm}
    \begin{minipage}{\linewidth}
        \begin{minipage}{0.03\linewidth}
            \centering
            \rotatebox{90}{\small GM-w. mask}
        \end{minipage}
        \begin{minipage}{0.95\linewidth}
            \centering
            \begin{minipage}{0.32\linewidth}
                \centering
                \includegraphics[width=1.\linewidth]{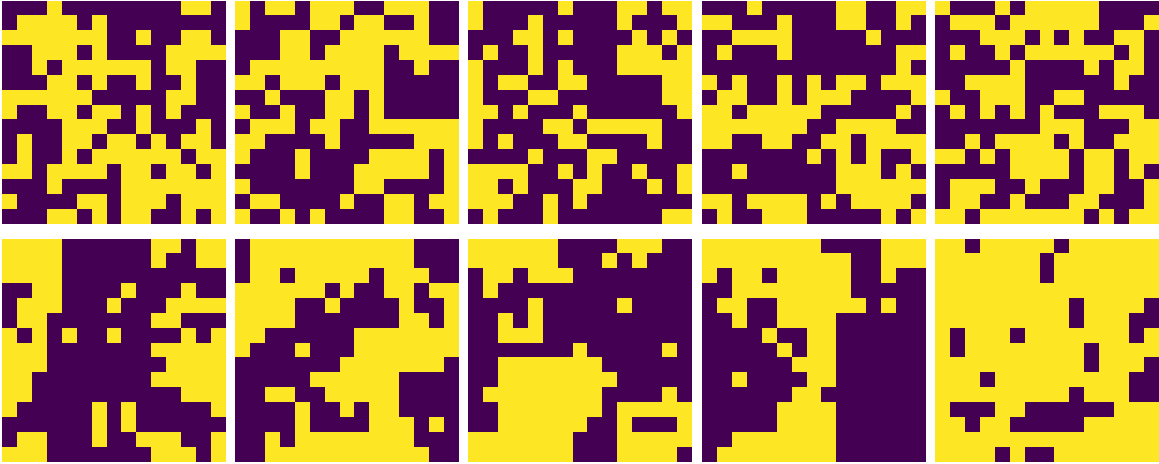}
            \end{minipage}
            \hfill
            \begin{minipage}{0.32\linewidth}
                \centering
                \includegraphics[width=1.\linewidth]{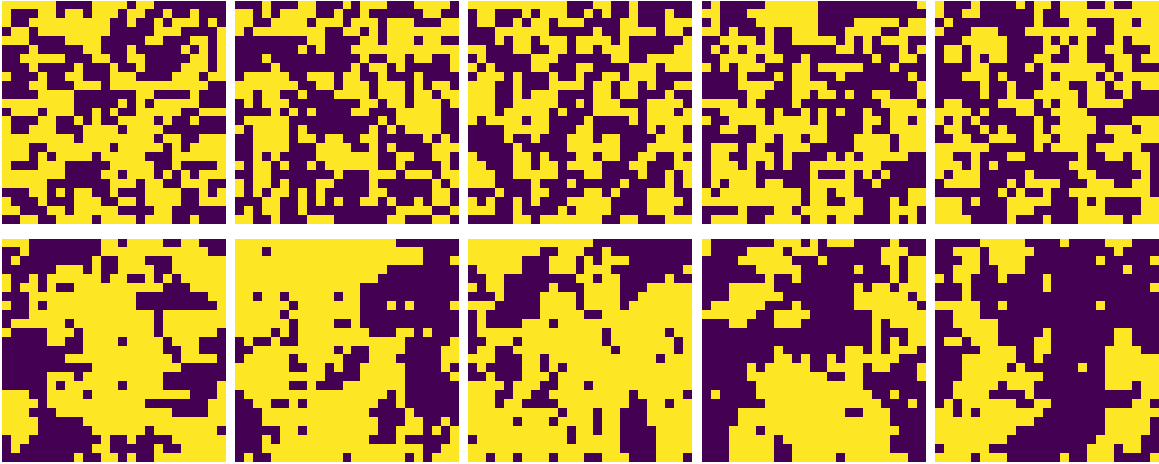}
            \end{minipage}
            \hfill
            \begin{minipage}{0.32\linewidth}
                \centering
                \includegraphics[width=1.\linewidth]{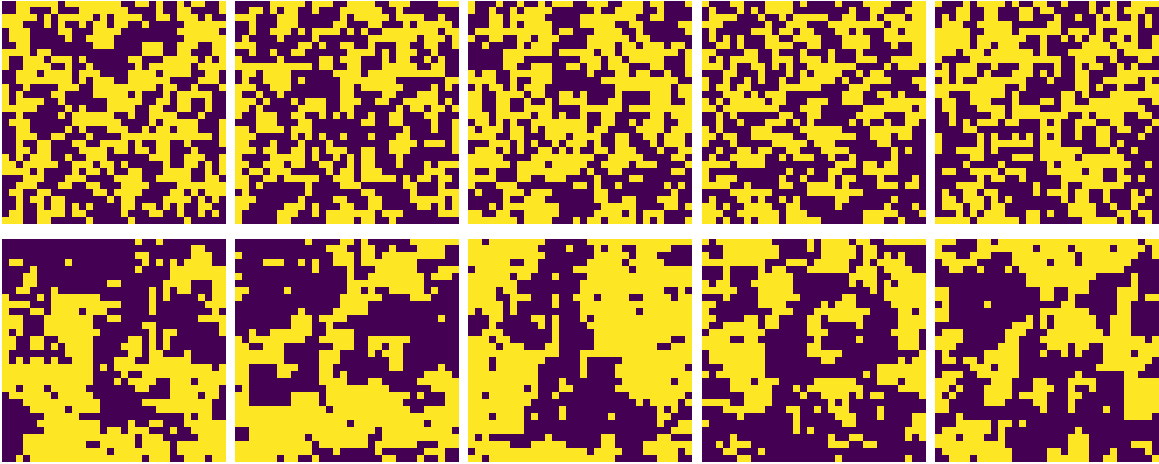}
            \end{minipage}
        \end{minipage}
    \end{minipage}
    \caption{Generated samples from CTMC FEAT variants vs.\ ground truth on Ising model 
    transport ($\beta = 0.2 \leftrightarrow 0.4$) across different lattice sizes. For each method, the two sub-rows show samples from backward and forward transport, both initialized from test data.}
    \label{fig:ising_samples_.2_.4}
\end{figure}

\begin{figure}[H]
    \centering
    \begin{minipage}{\linewidth}
        \begin{minipage}{0.03\linewidth}
            \centering
            \vspace{3mm}
            \rotatebox{90}{\small Data}
        \end{minipage}
        \begin{minipage}{0.95\linewidth}
            \centering
            \begin{minipage}{0.32\linewidth}
                \centering
                {\small Ising 15x15}
                \includegraphics[width=1.\linewidth]{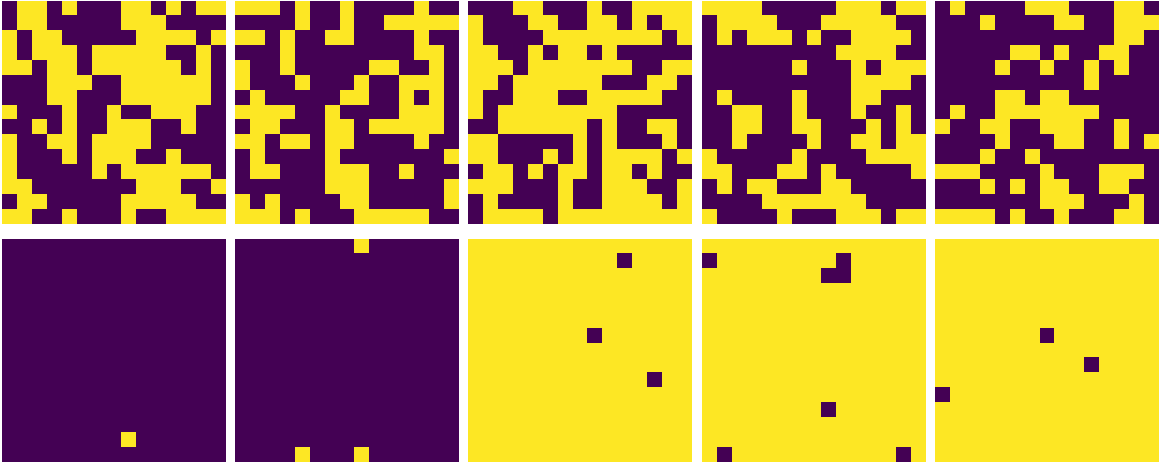}
            \end{minipage}
            \hfill
            \begin{minipage}{0.32\linewidth}
                \centering
                {\small Ising 25x25}
                \includegraphics[width=1.\linewidth]{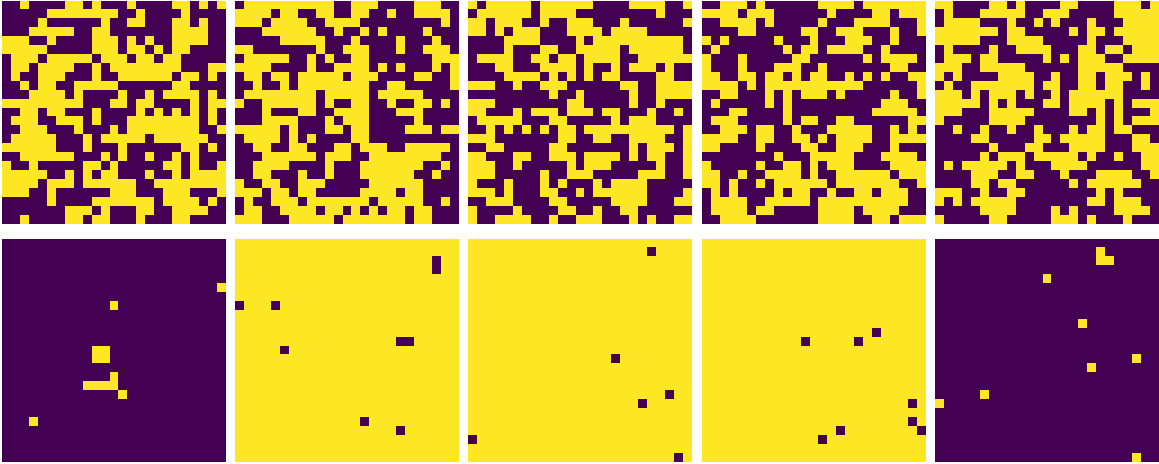}
            \end{minipage}
            \hfill
            \begin{minipage}{0.32\linewidth}
                \centering
                {\small Ising 32x32}
                \includegraphics[width=1.\linewidth]{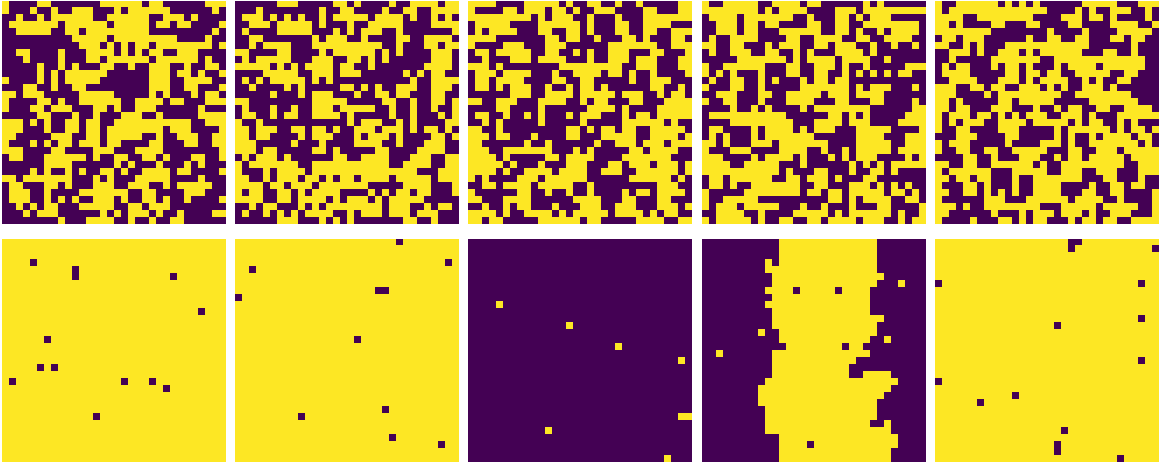}
            \end{minipage}
        \end{minipage}
    \end{minipage} \\
    \vspace{1mm}
    \begin{minipage}{\linewidth}
        \begin{minipage}{0.03\linewidth}
            \centering
            \rotatebox{90}{\small DFM}
        \end{minipage}
        \begin{minipage}{0.95\linewidth}
            \centering
            \begin{minipage}{0.32\linewidth}
                \centering
                \includegraphics[width=1.\linewidth]{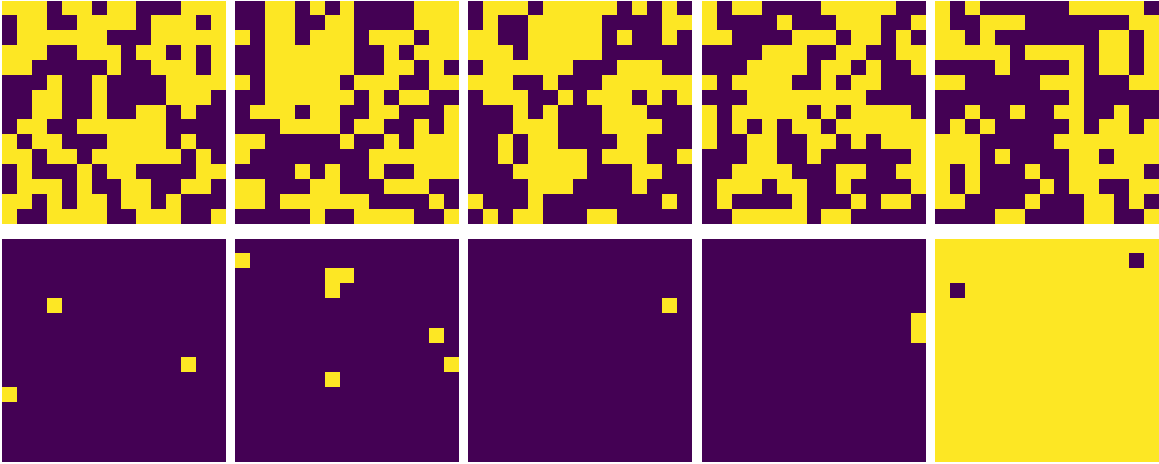}
            \end{minipage}
            \hfill
            \begin{minipage}{0.32\linewidth}
                \centering
                \includegraphics[width=1.\linewidth]{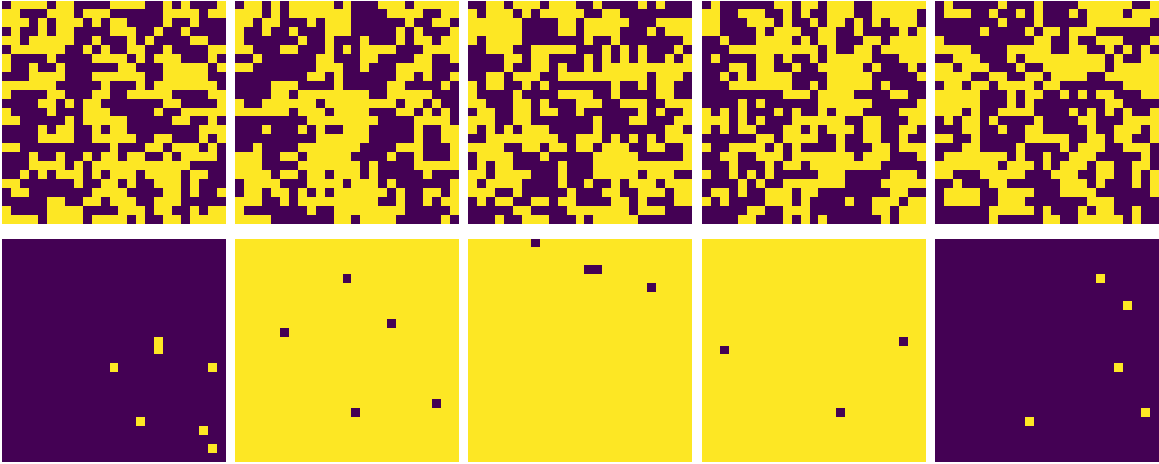}
            \end{minipage}
            \hfill
            \begin{minipage}{0.32\linewidth}
                \centering
                \includegraphics[width=1.\linewidth]{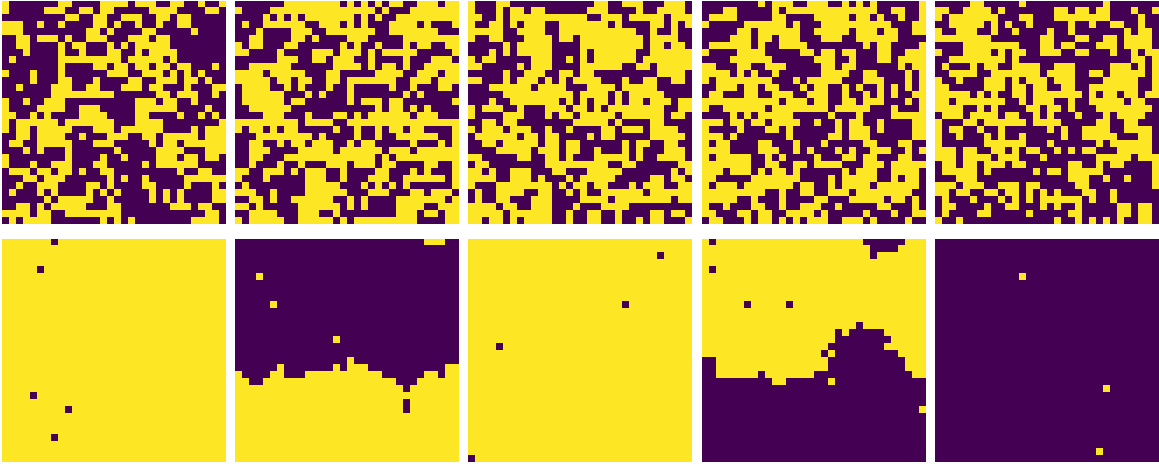}
            \end{minipage}
        \end{minipage}
    \end{minipage} \\
    \vspace{1mm}
    \begin{minipage}{\linewidth}
        \begin{minipage}{0.03\linewidth}
            \centering
            \rotatebox{90}{\small GM-no mask}
        \end{minipage}
        \begin{minipage}{0.95\linewidth}
            \centering
            \begin{minipage}{0.32\linewidth}
                \centering
                \includegraphics[width=1.\linewidth]{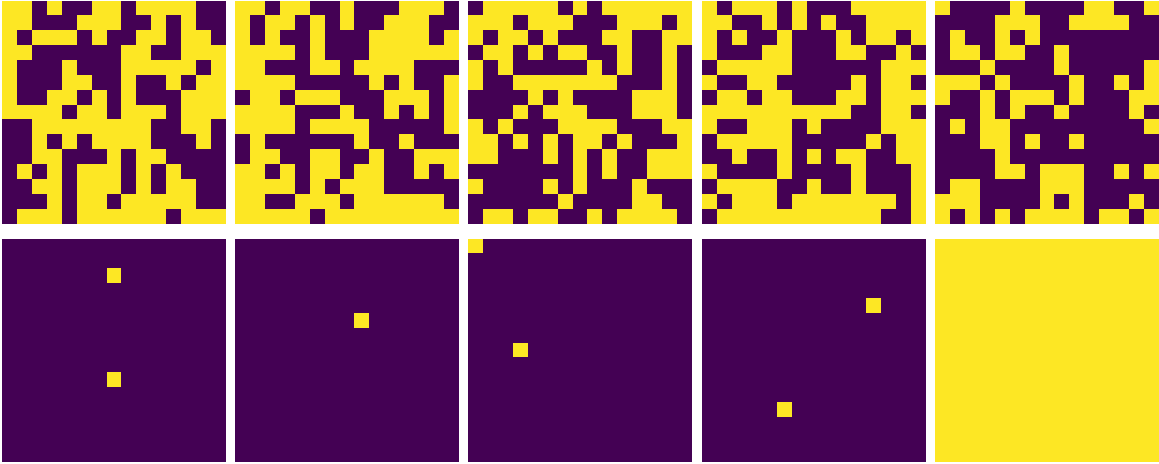}
            \end{minipage}
            \hfill
            \begin{minipage}{0.32\linewidth}
                \centering
                \includegraphics[width=1.\linewidth]{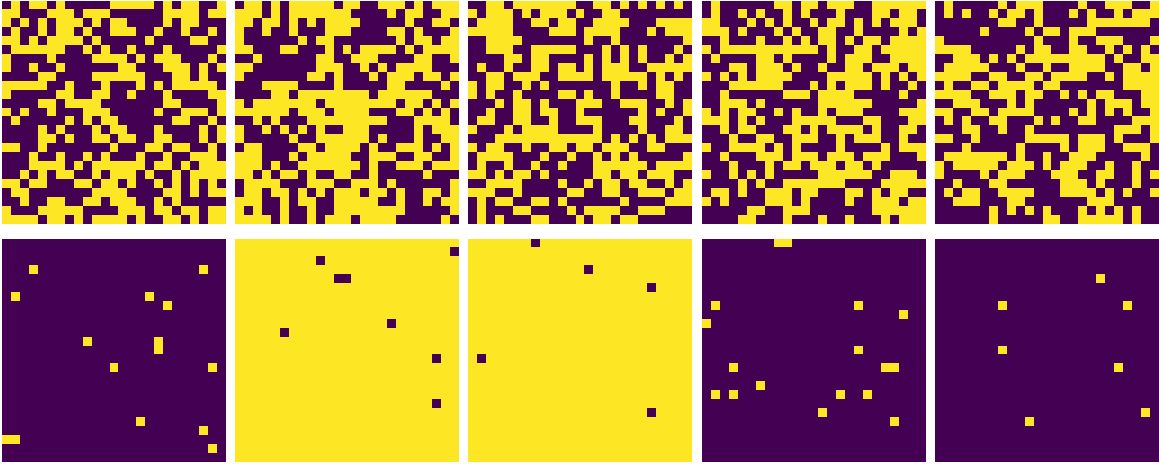}
            \end{minipage}
            \hfill
            \begin{minipage}{0.32\linewidth}
                \centering
                \includegraphics[width=1.\linewidth]{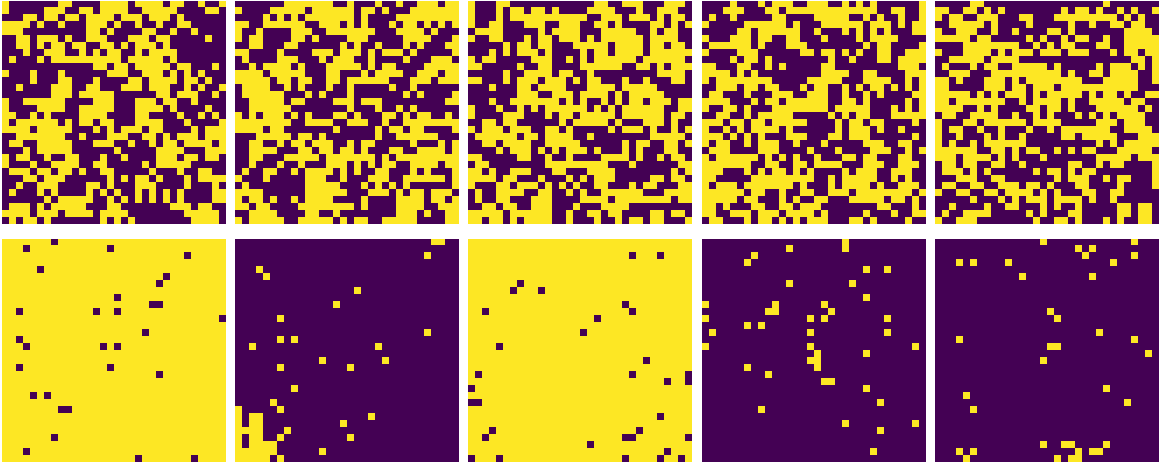}
            \end{minipage}
        \end{minipage}
    \end{minipage} \\
    \vspace{1mm}
    \begin{minipage}{\linewidth}
        \begin{minipage}{0.03\linewidth}
            \centering
            \rotatebox{90}{\small GM-w. mask}
        \end{minipage}
        \begin{minipage}{0.95\linewidth}
            \centering
            \begin{minipage}{0.32\linewidth}
                \centering
                \includegraphics[width=1.\linewidth]{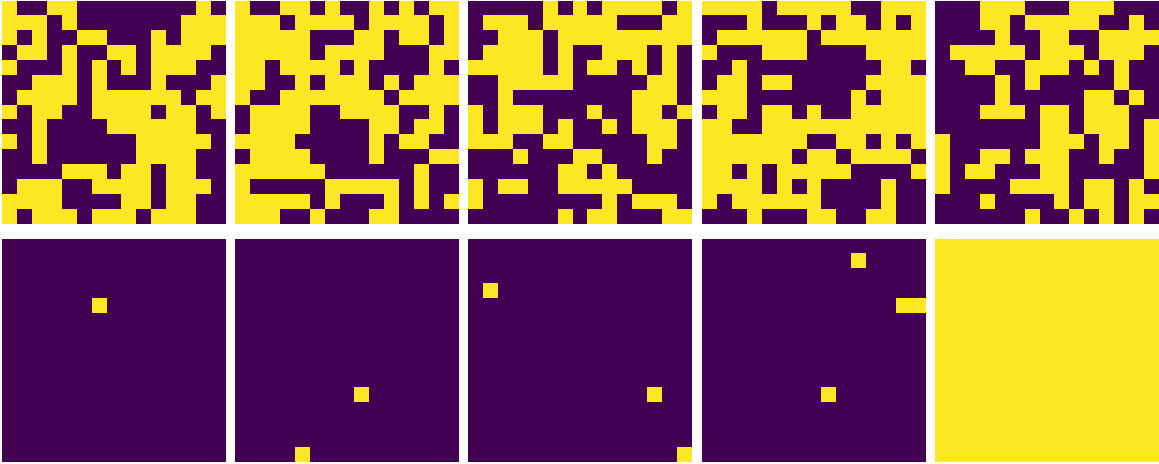}
            \end{minipage}
            \hfill
            \begin{minipage}{0.32\linewidth}
                \centering
                \includegraphics[width=1.\linewidth]{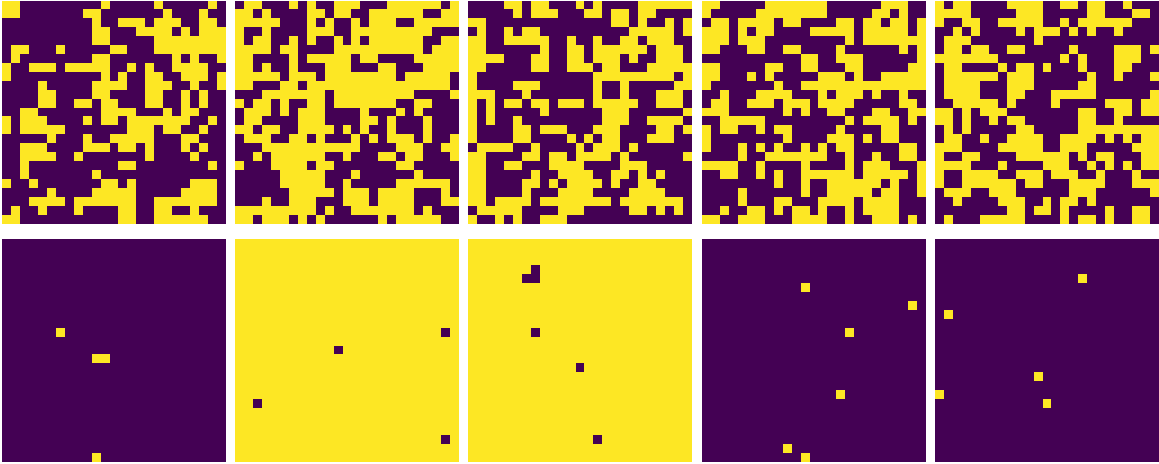}
            \end{minipage}
            \hfill
            \begin{minipage}{0.32\linewidth}
                \centering
                \includegraphics[width=1.\linewidth]{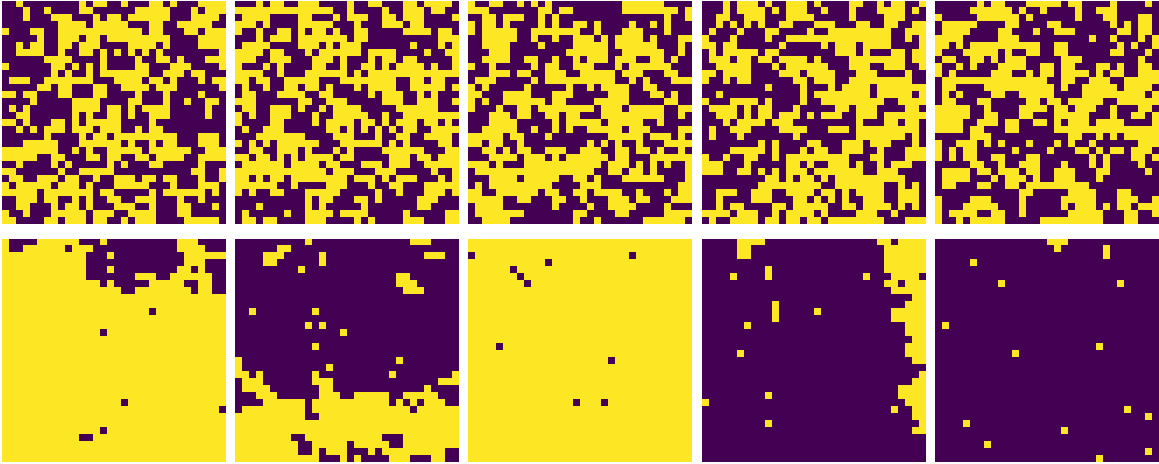}
            \end{minipage}
        \end{minipage}
    \end{minipage}
    \caption{Generated samples from CTMC FEAT variants vs.\ ground truth on Ising model 
    transport ($\beta = 0.2 \leftrightarrow 0.6$) across different lattice sizes. For each method, the two sub-rows show samples from backward and forward transport, both initialized from test data.}
    \label{fig:ising_samples_.2_.6}
\end{figure}

\begin{figure}[H]
    \centering
    \begin{minipage}{\linewidth}
        \begin{minipage}{0.03\linewidth}
            \centering
            \vspace{3mm}
            \rotatebox{90}{\small DFM}
        \end{minipage}
        \begin{minipage}{0.95\linewidth}
            \centering
            \begin{minipage}{0.32\linewidth}
                \centering
                \hspace{3mm} {\small Ising 15x15}
                \includegraphics[width=1.\linewidth]{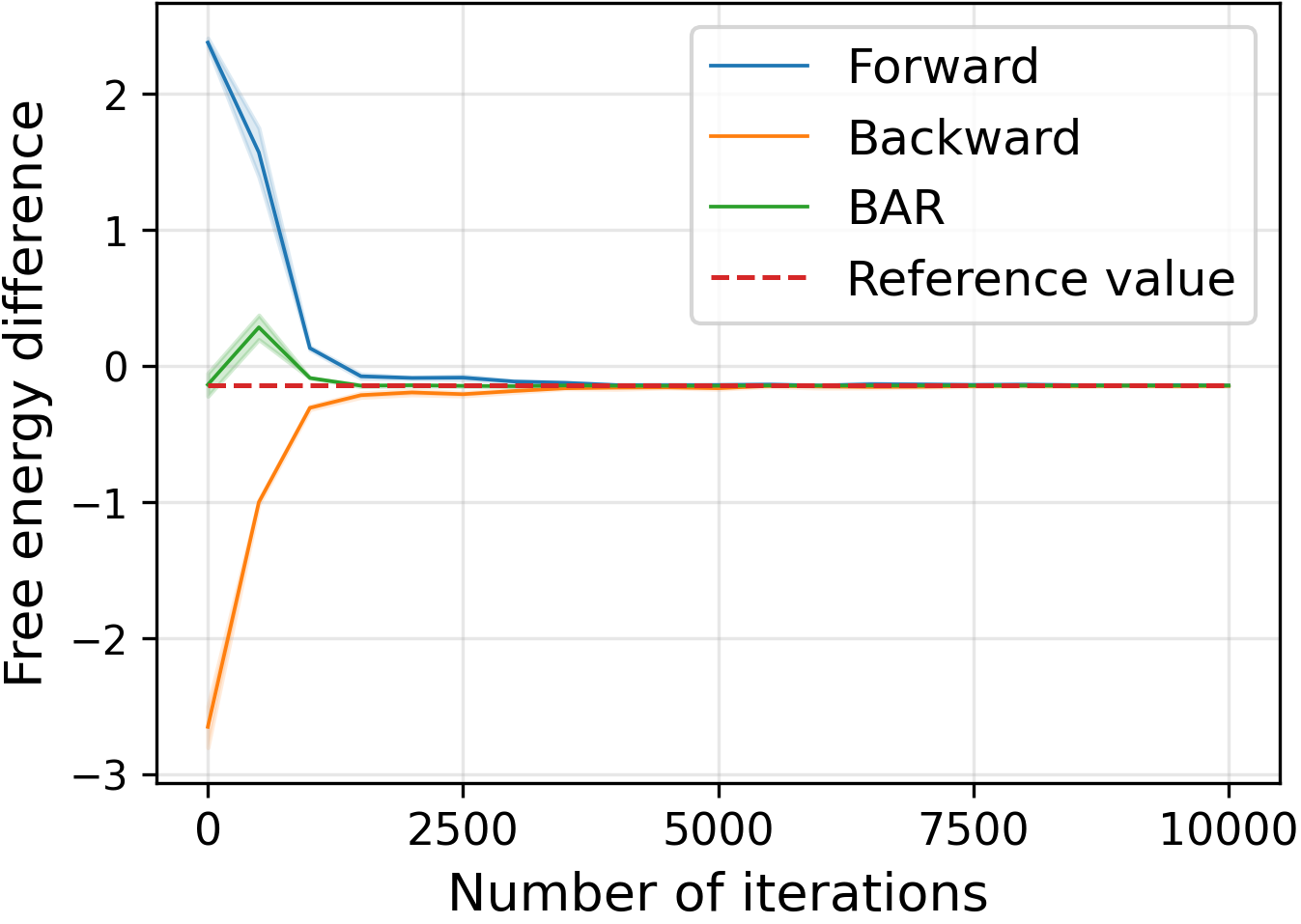}
            \end{minipage}
            \hfill
            \begin{minipage}{0.32\linewidth}
                \centering
                \hspace{3mm} {\small Ising 25x25}
                \includegraphics[width=1.\linewidth]{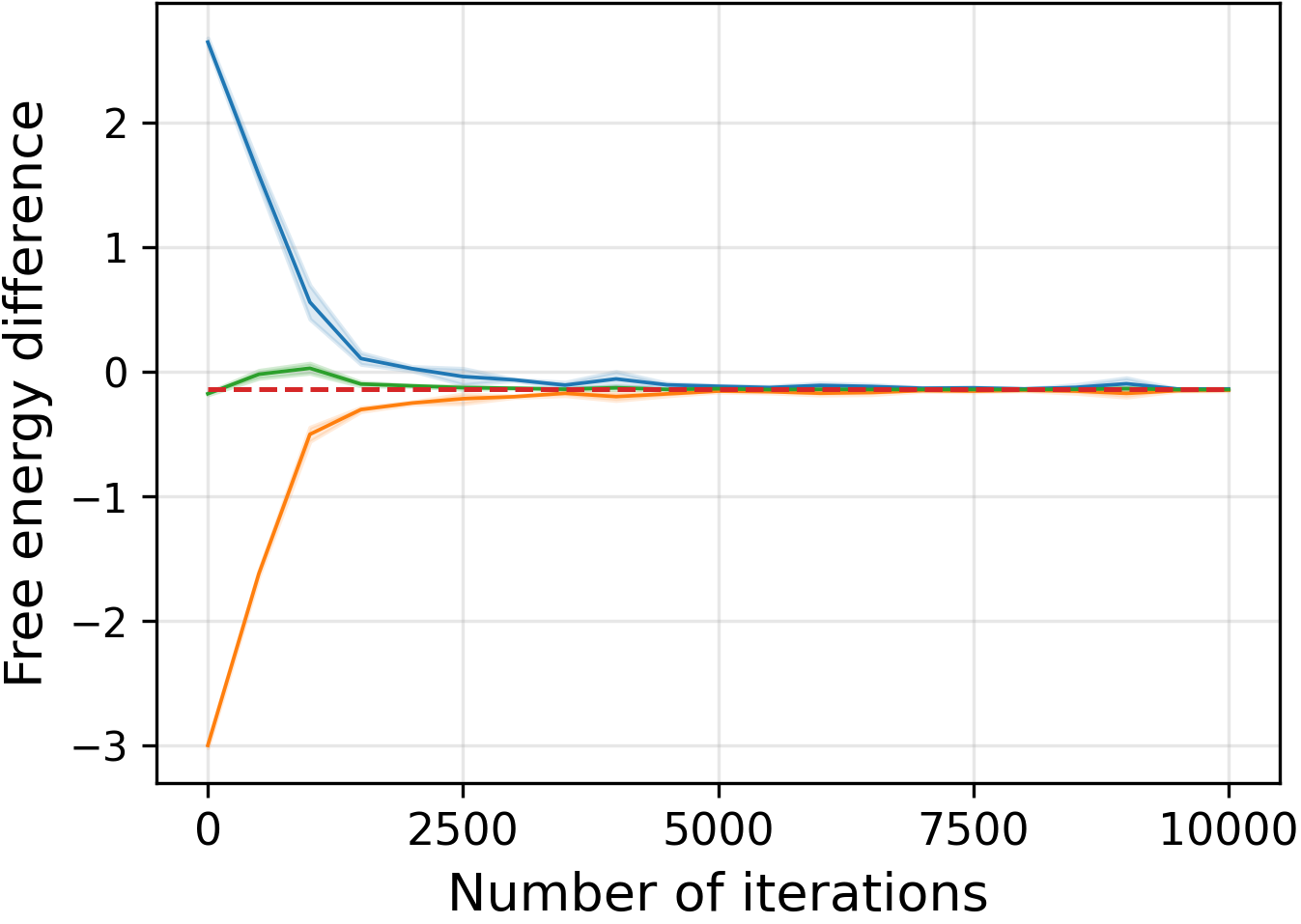}
            \end{minipage}
            \hfill
            \begin{minipage}{0.32\linewidth}
                \centering
                \hspace{3mm} {\small Ising 32x32}
                \includegraphics[width=1.\linewidth]{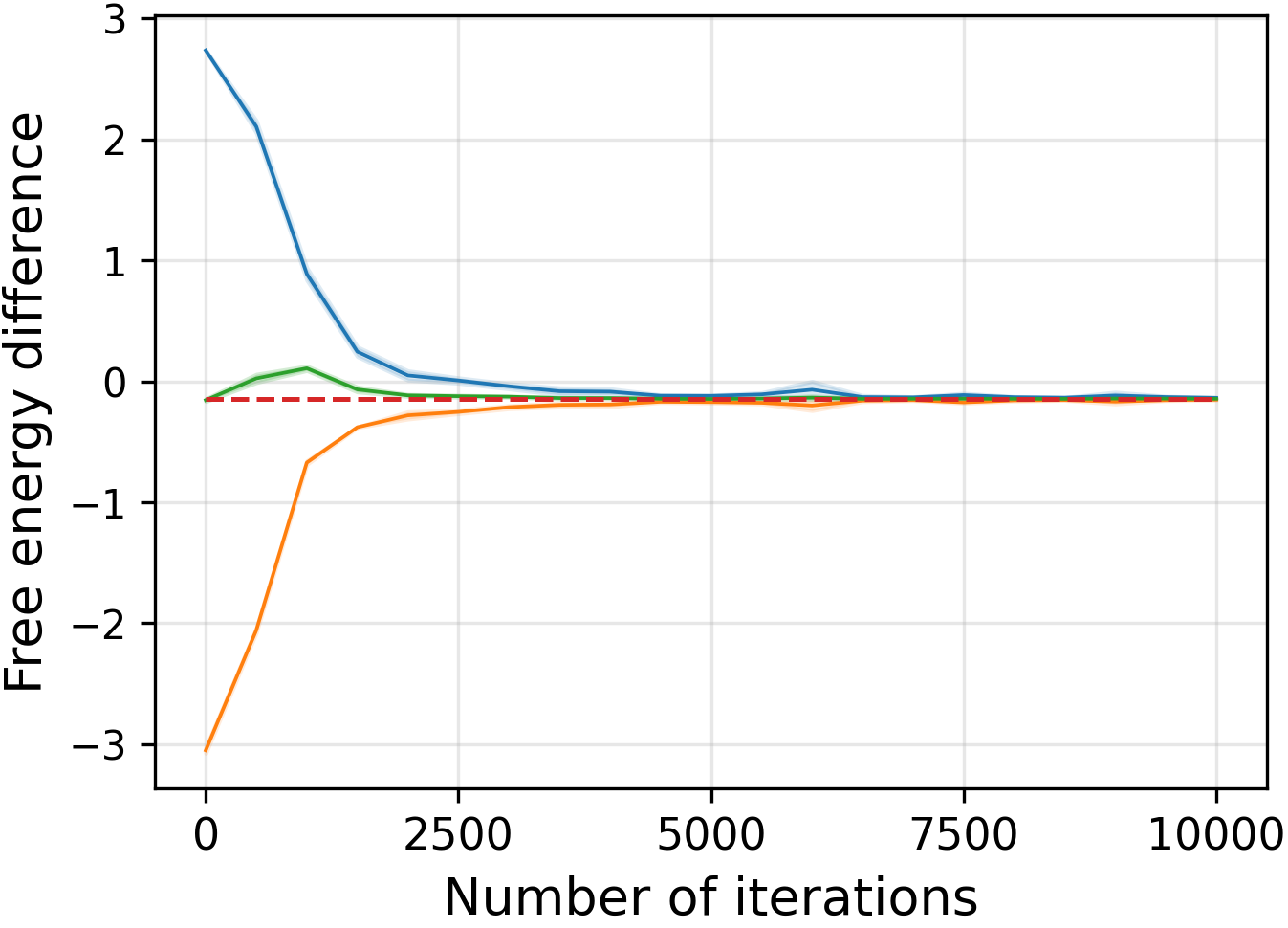}
            \end{minipage}
        \end{minipage}
    \end{minipage}
    \begin{minipage}{\linewidth}
        \begin{minipage}{0.03\linewidth}
            \centering
            \rotatebox{90}{\small GM-no mask}
        \end{minipage}
        \begin{minipage}{0.95\linewidth}
            \centering
            \begin{minipage}{0.32\linewidth}
                \centering
                \includegraphics[width=1.\linewidth]{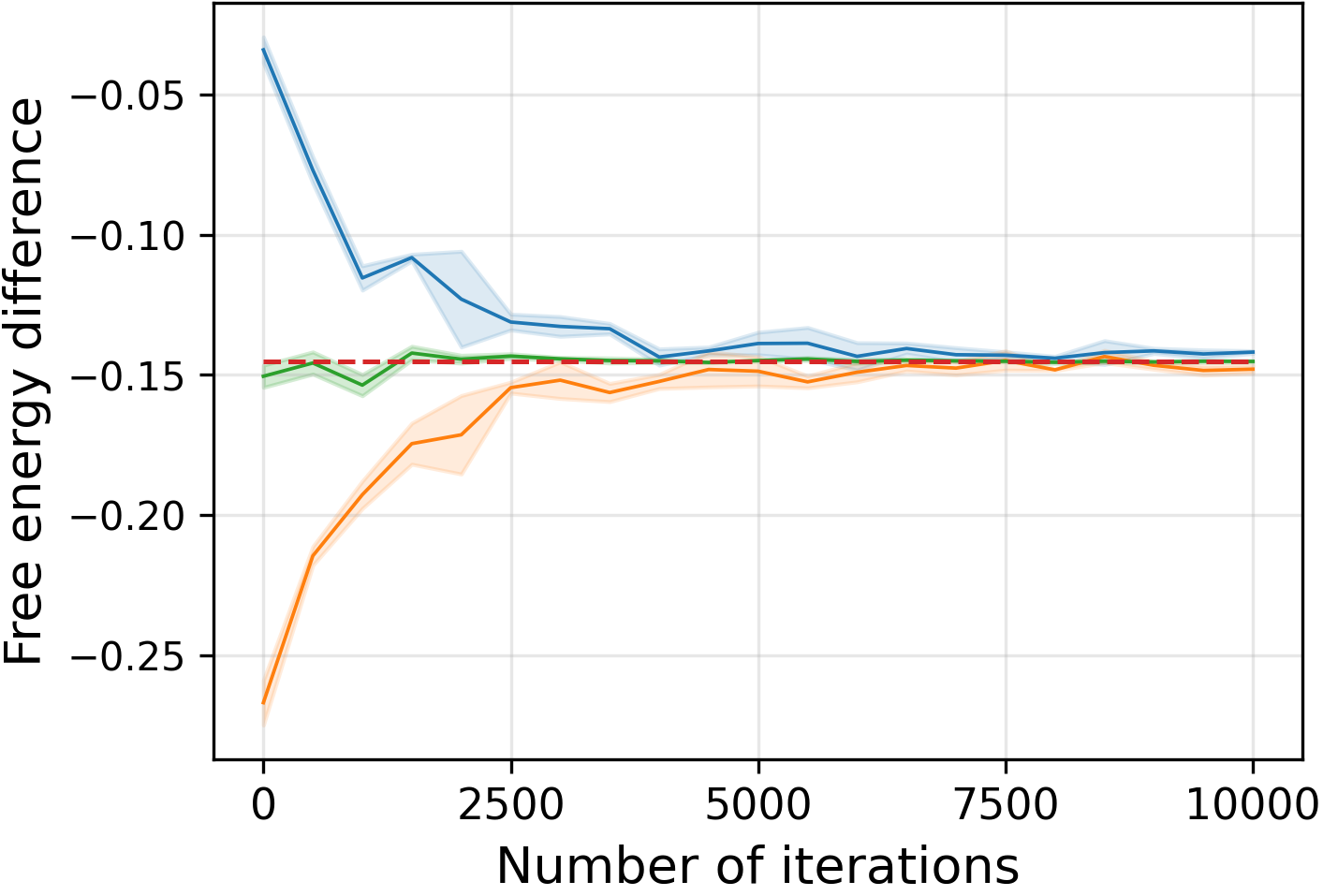}
            \end{minipage}
            \hfill
            \begin{minipage}{0.32\linewidth}
                \centering
                \includegraphics[width=1.\linewidth]{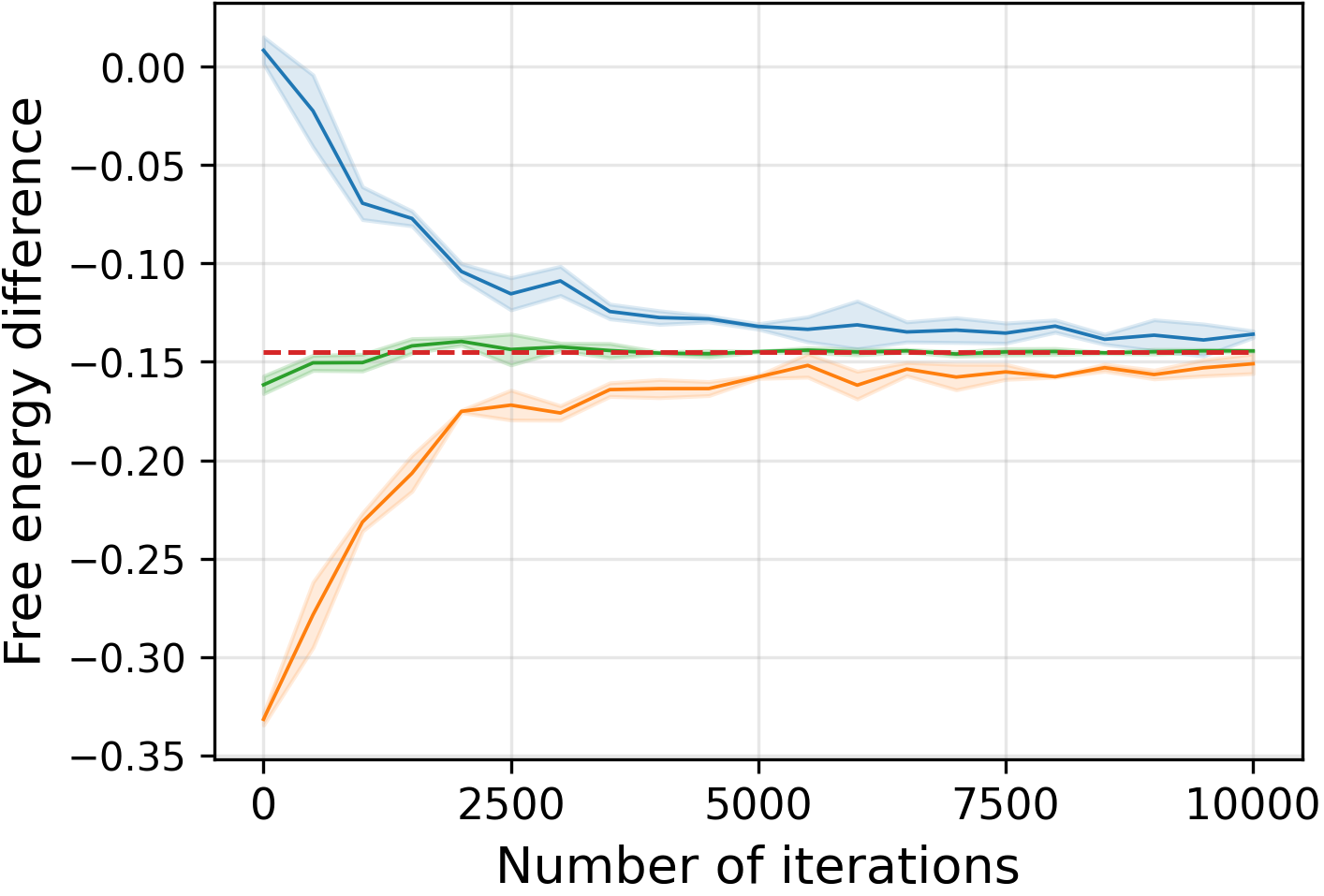}
            \end{minipage}
            \hfill
            \begin{minipage}{0.32\linewidth}
                \centering
                \includegraphics[width=1.\linewidth]{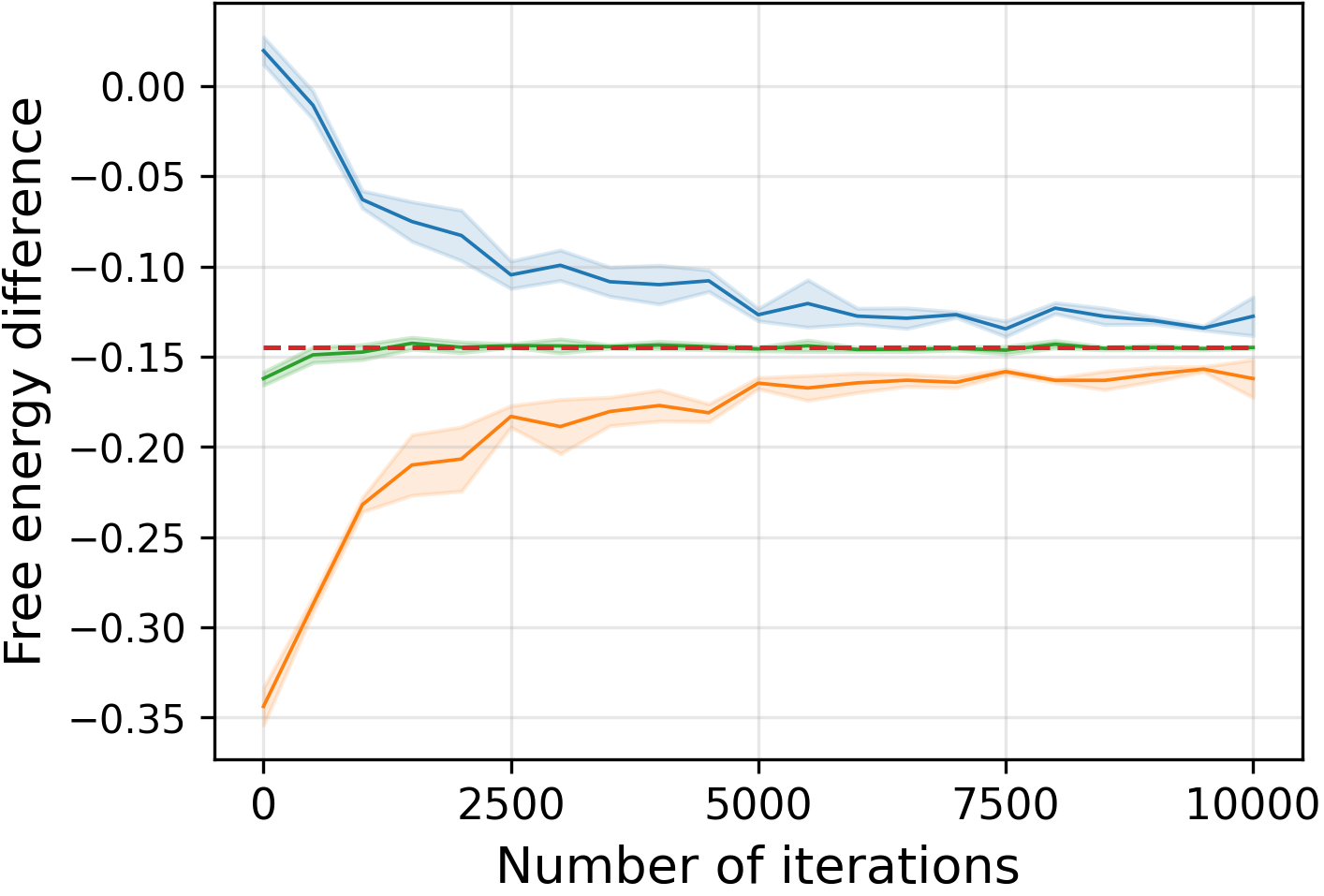}
            \end{minipage}
        \end{minipage}
    \end{minipage}
    \begin{minipage}{\linewidth}
        \begin{minipage}{0.03\linewidth}
            \centering
            \rotatebox{90}{\small GM-w. mask}
        \end{minipage}
        \begin{minipage}{0.95\linewidth}
            \centering
            \begin{minipage}{0.32\linewidth}
                \centering
                \includegraphics[width=1.\linewidth]{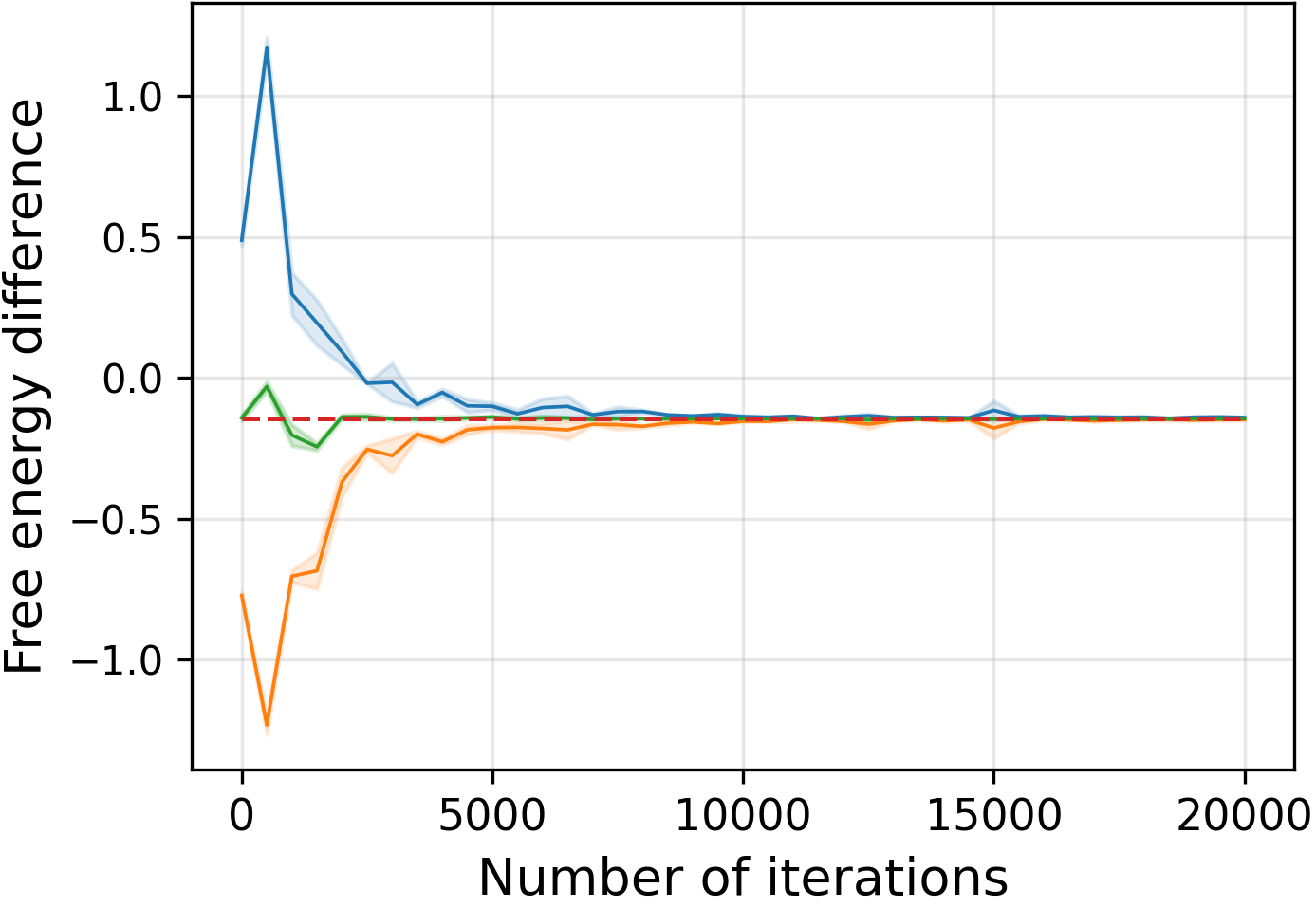}
            \end{minipage}
            \hfill
            \begin{minipage}{0.32\linewidth}
                \centering
                \includegraphics[width=1.\linewidth]{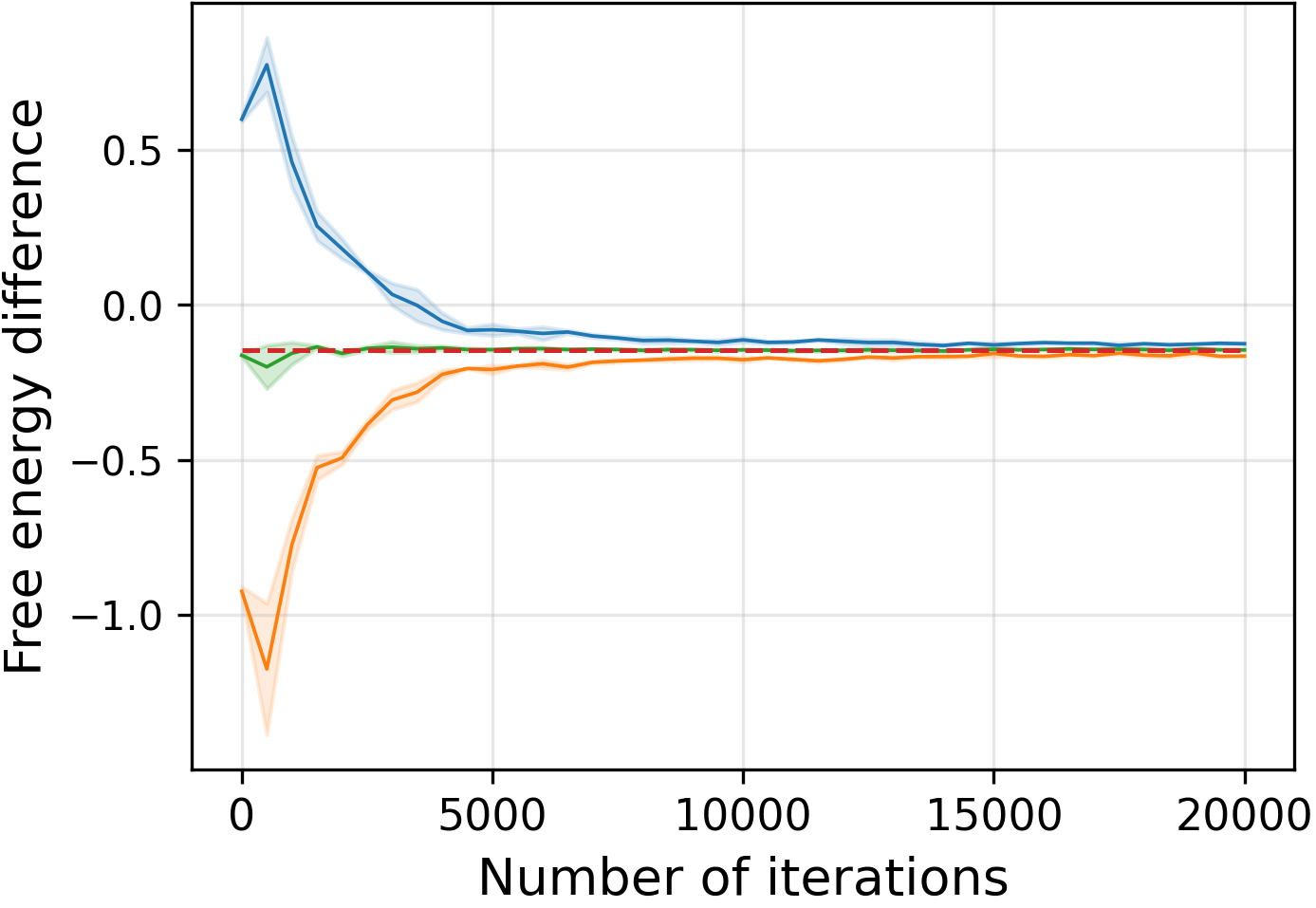}
            \end{minipage}
            \hfill
            \begin{minipage}{0.32\linewidth}
                \centering
                \includegraphics[width=1.\linewidth]{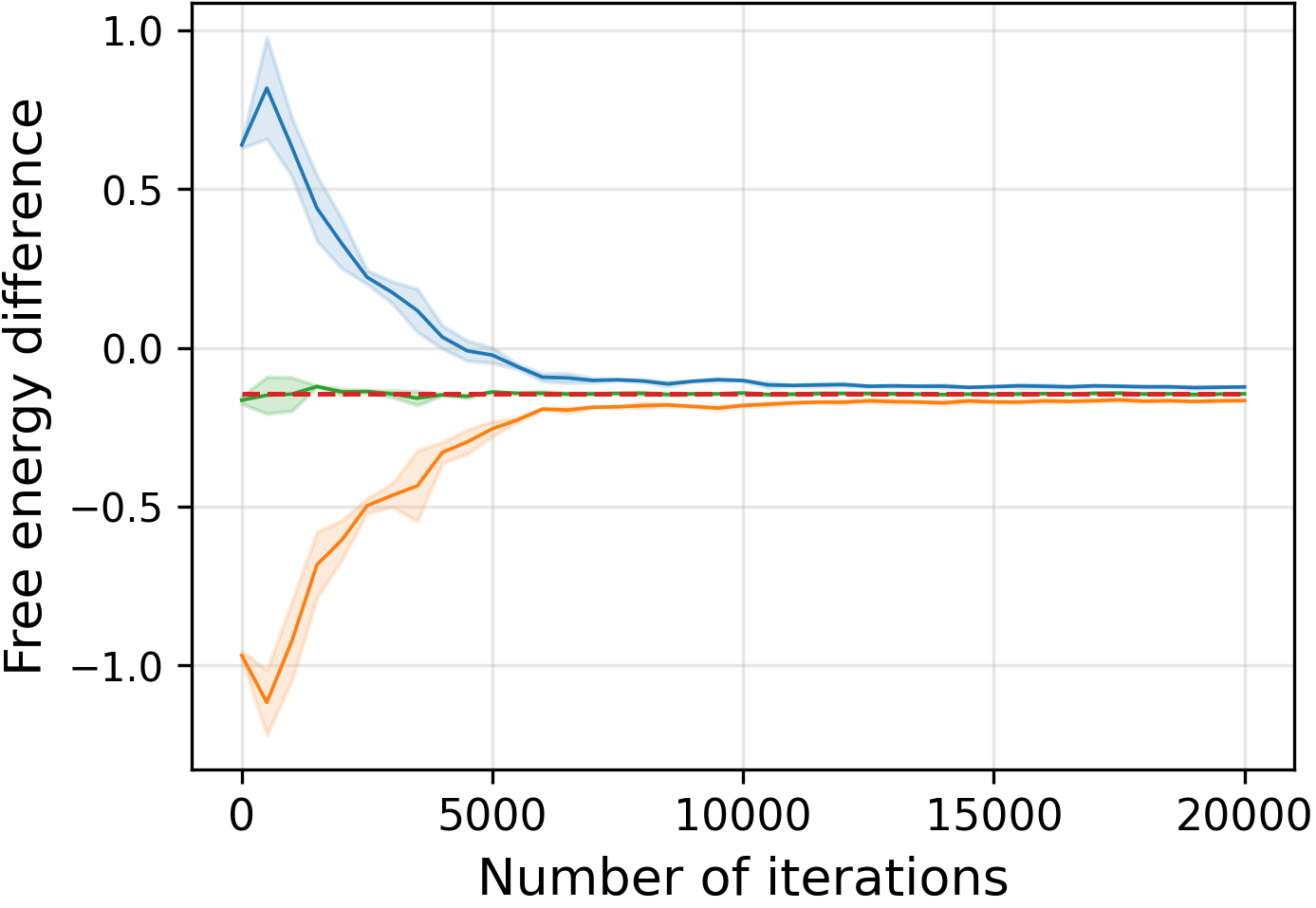}
            \end{minipage}
        \end{minipage}
    \end{minipage}
    \caption{Free energy difference estimates ($\widehat{\Delta F}_F$, $\widehat{\Delta F}_B$, 
    $\widehat{\Delta F}_\mathrm{BAR}$) over training iterations for CTMC FEAT variants 
    on Ising model transport ($\beta = 0.2 \leftrightarrow 0.4$) across lattice sizes. The dashed red line indicates the reference value.}
    \label{fig:deltaF_curve_.2_.4}
\end{figure}

\begin{figure}[H]
    \centering
    \begin{minipage}{\linewidth}
        \begin{minipage}{0.03\linewidth}
            \centering
            \vspace{3mm}
            \rotatebox{90}{\small DFM}
        \end{minipage}
        \begin{minipage}{0.95\linewidth}
            \centering
            \begin{minipage}{0.32\linewidth}
                \centering
                \hspace{3mm} {\small Ising 15x15}
                \includegraphics[width=1.\linewidth]{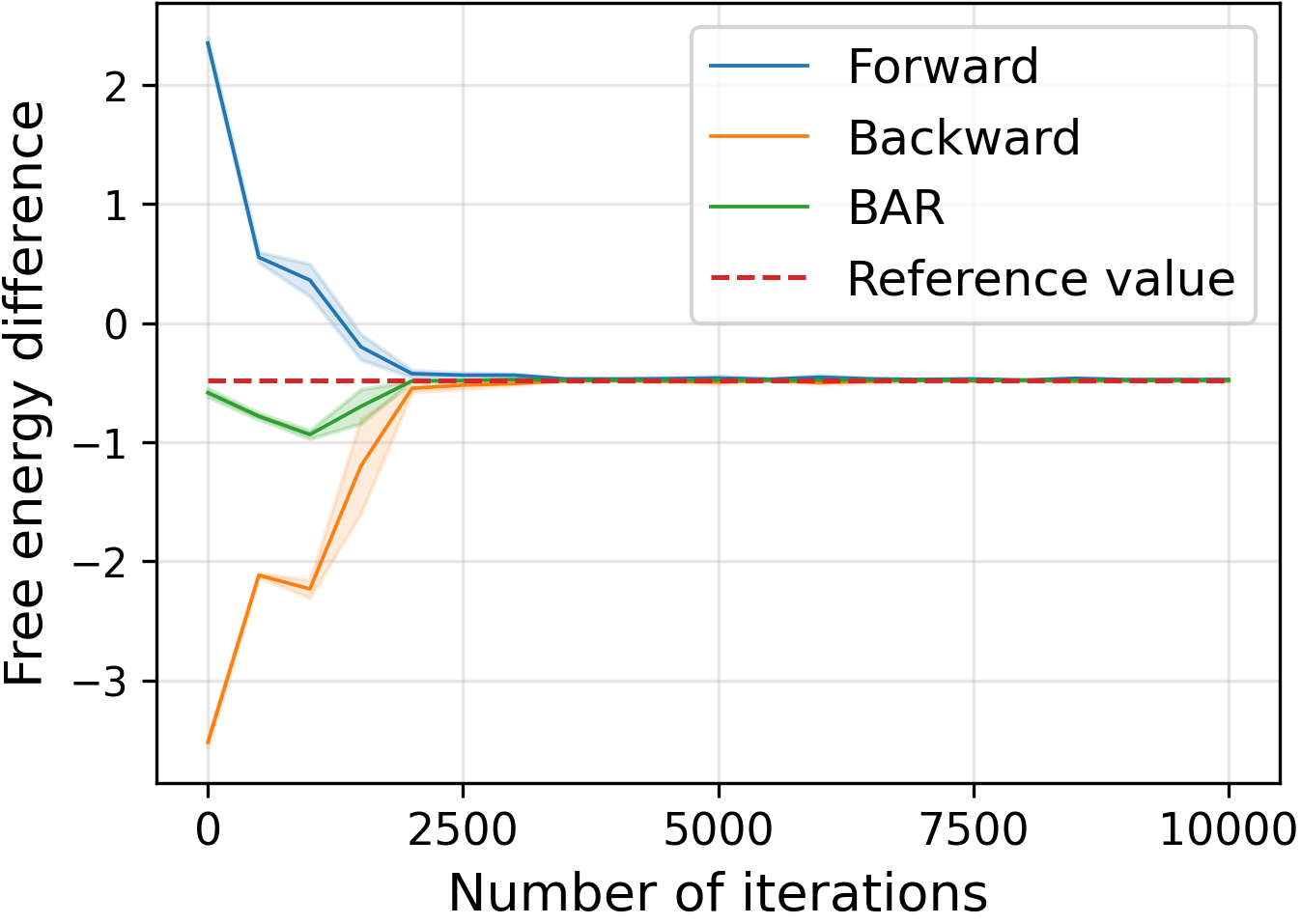}
            \end{minipage}
            \hfill
            \begin{minipage}{0.32\linewidth}
                \centering
                \hspace{3mm} {\small Ising 25x25}
                \includegraphics[width=1.\linewidth]{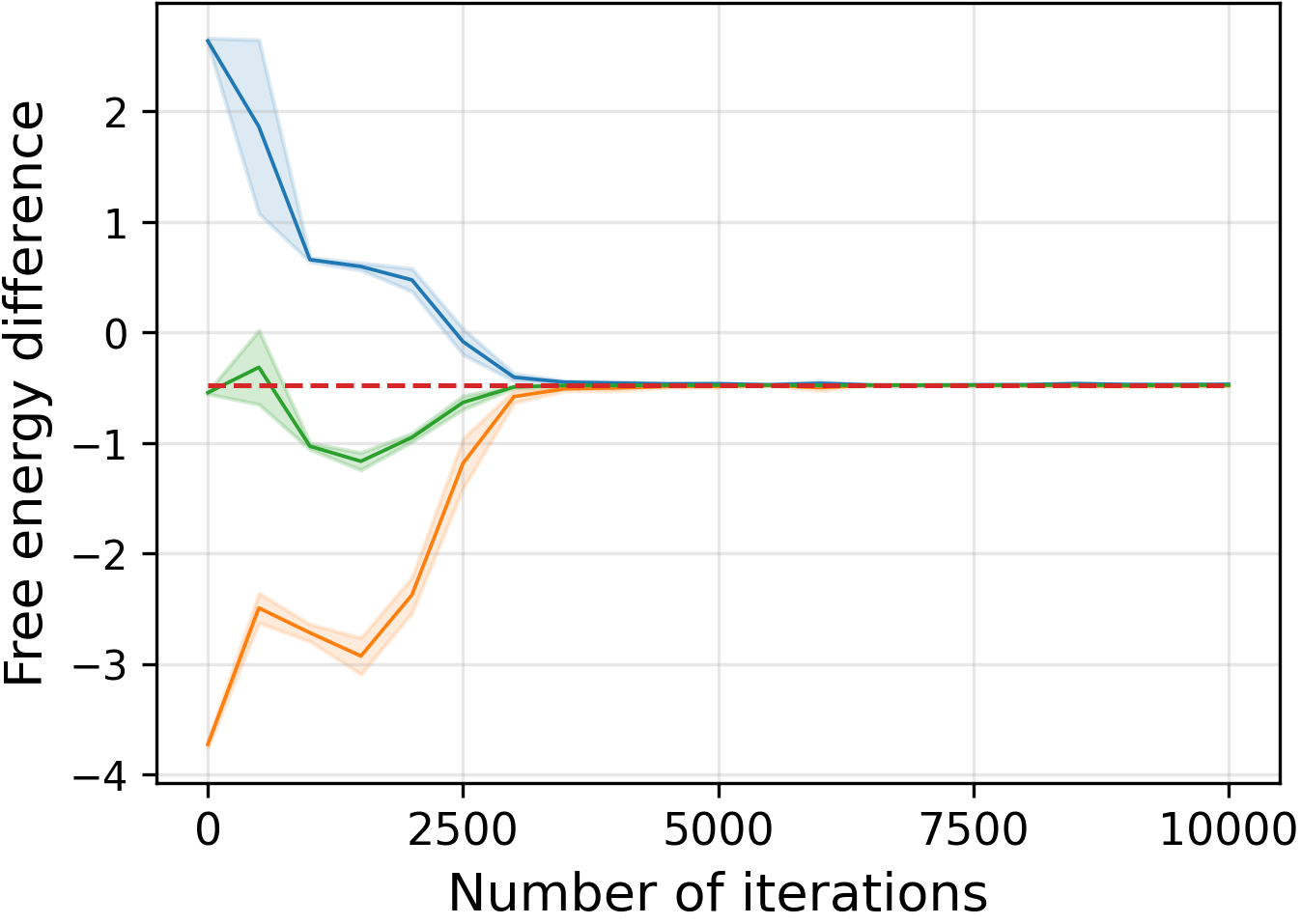}
            \end{minipage}
            \hfill
            \begin{minipage}{0.32\linewidth}
                \centering
                \hspace{3mm} {\small Ising 32x32}
                \includegraphics[width=1.\linewidth]{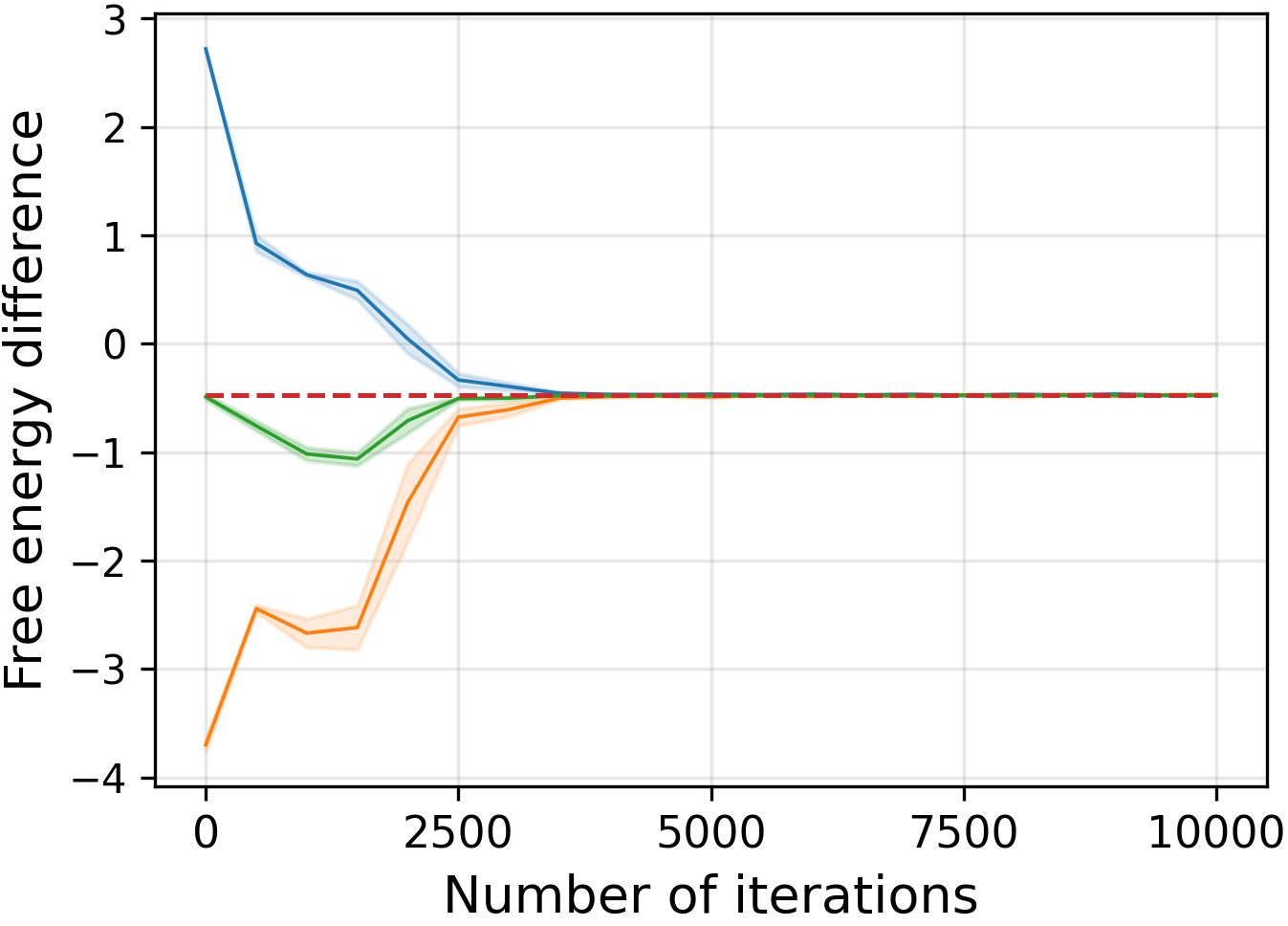}
            \end{minipage}
        \end{minipage}
    \end{minipage}
    \begin{minipage}{\linewidth}
        \begin{minipage}{0.03\linewidth}
            \centering
            \rotatebox{90}{\small GM-no mask}
        \end{minipage}
        \begin{minipage}{0.95\linewidth}
            \centering
            \begin{minipage}{0.32\linewidth}
                \centering
                \includegraphics[width=1.\linewidth]{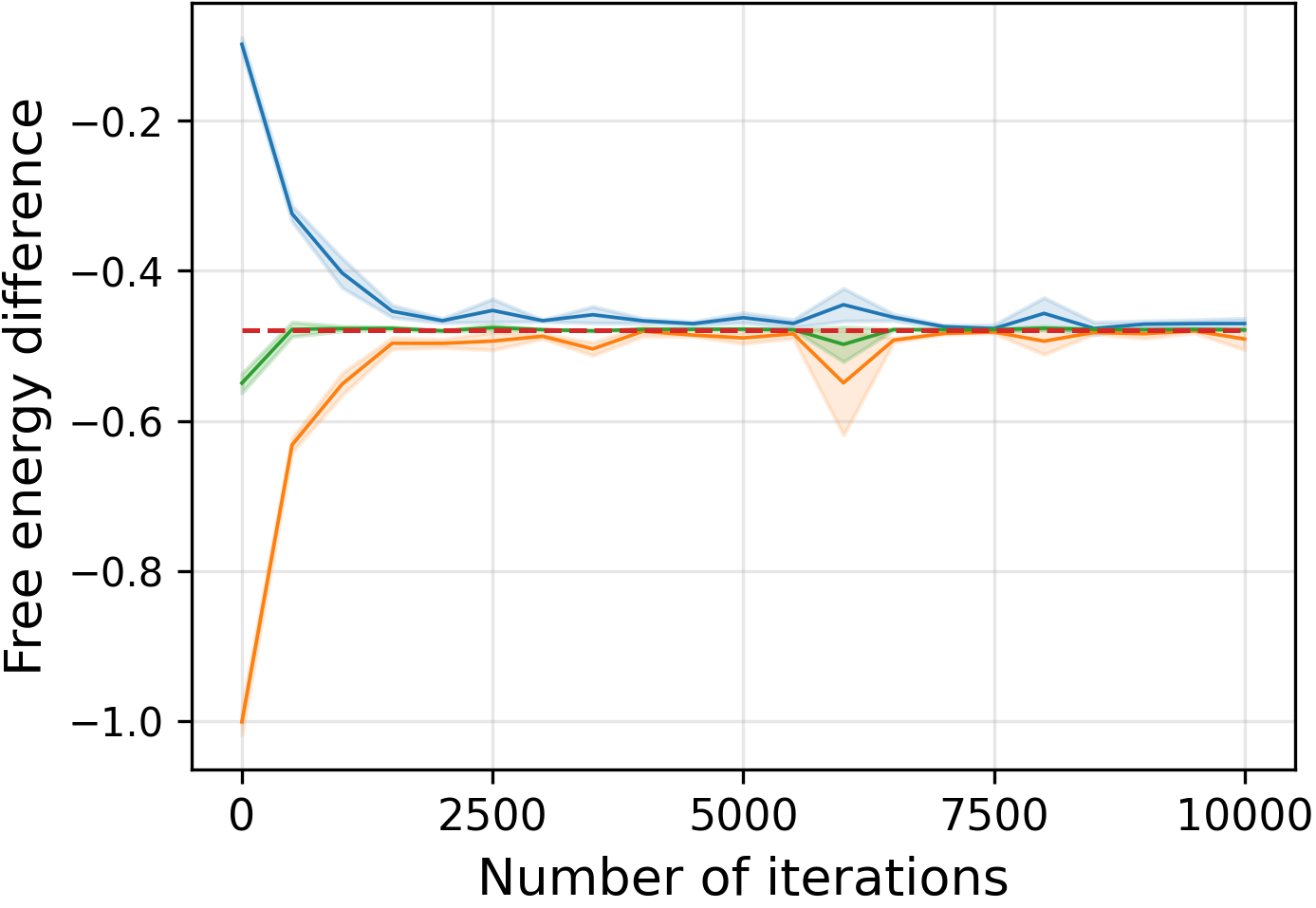}
            \end{minipage}
            \hfill
            \begin{minipage}{0.32\linewidth}
                \centering
                \includegraphics[width=1.\linewidth]{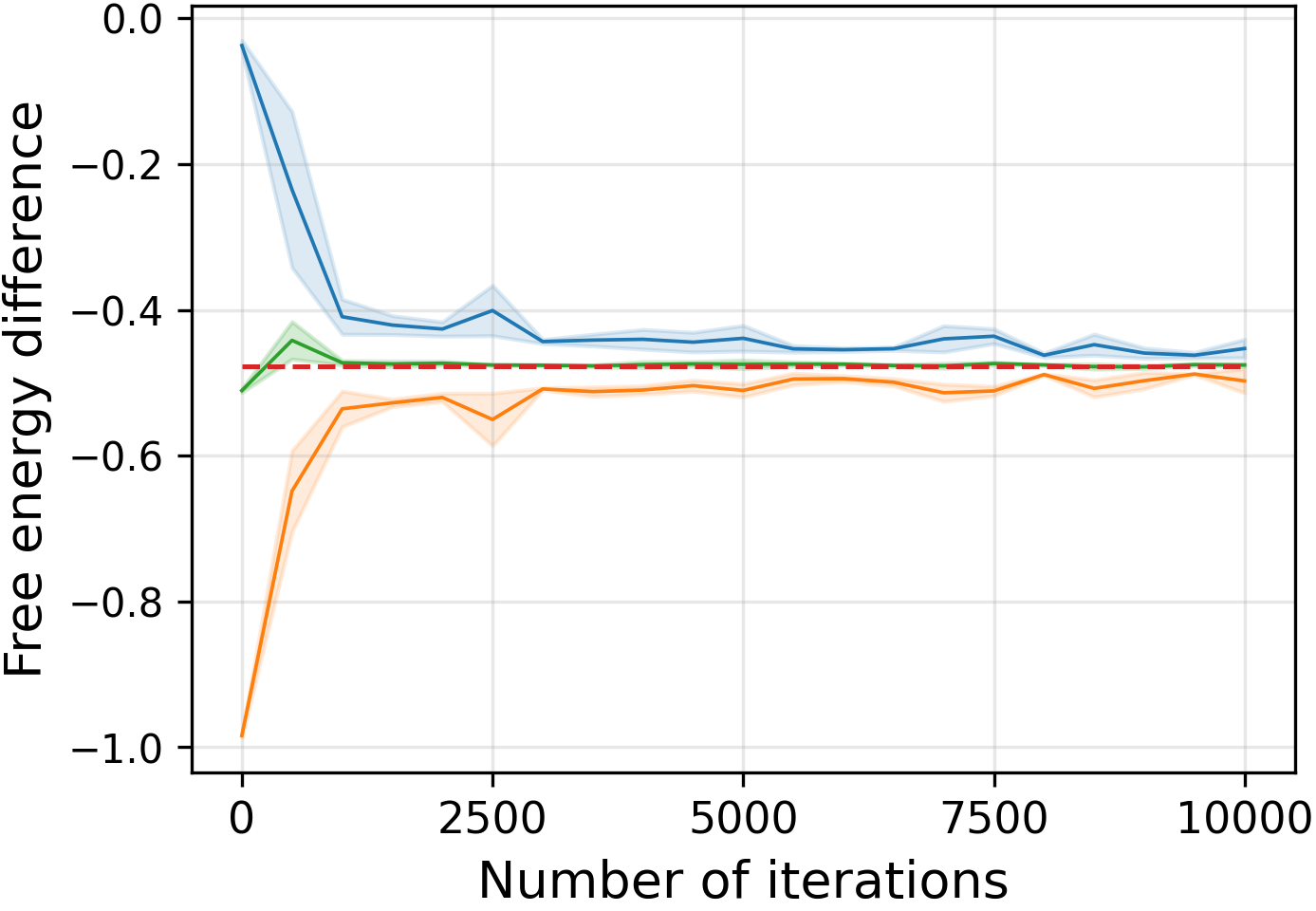}
            \end{minipage}
            \hfill
            \begin{minipage}{0.32\linewidth}
                \centering
                \includegraphics[width=1.\linewidth]{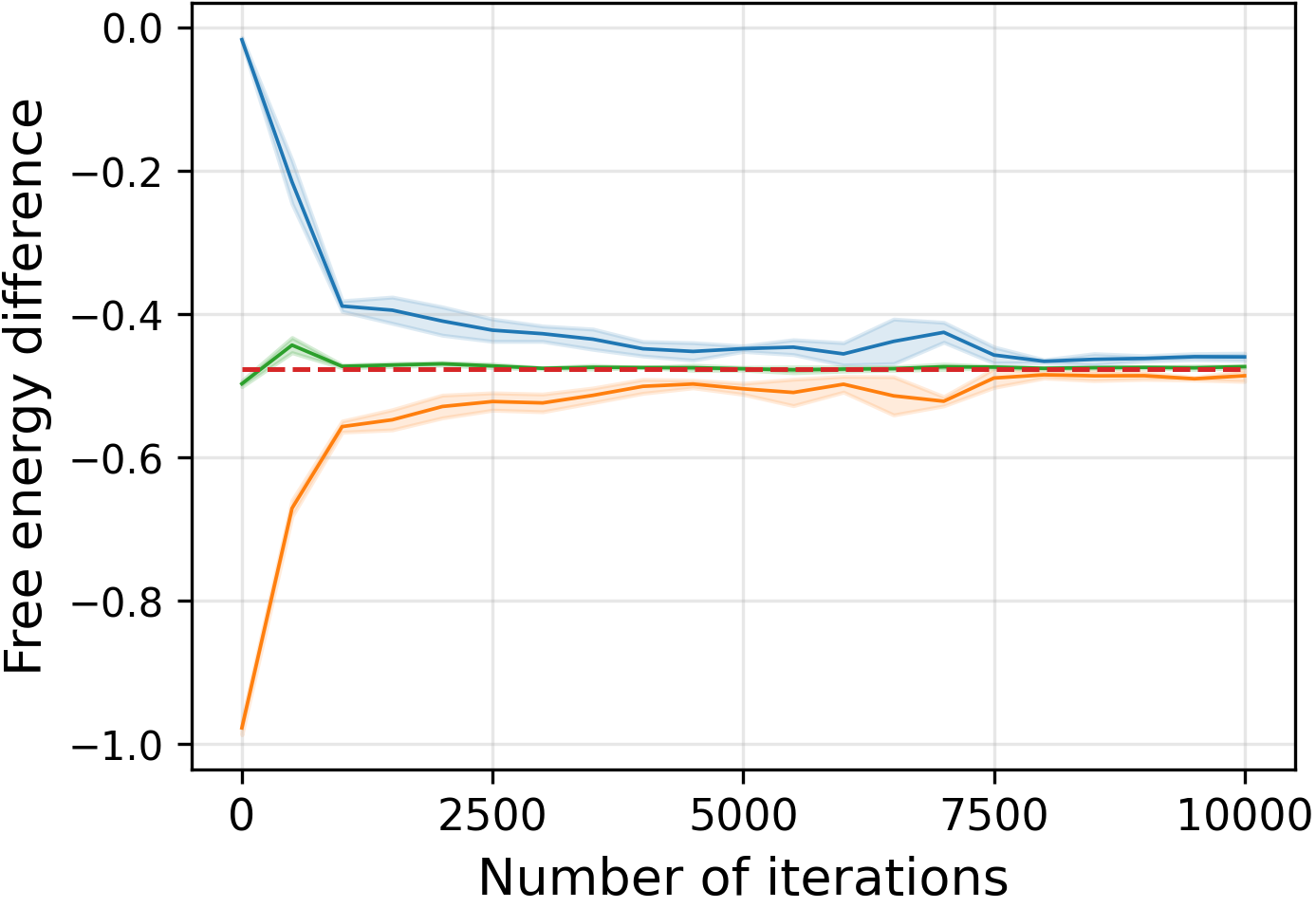}
            \end{minipage}
        \end{minipage}
    \end{minipage}
    \begin{minipage}{\linewidth}
        \begin{minipage}{0.03\linewidth}
            \centering
            \rotatebox{90}{\small GM-w. mask}
        \end{minipage}
        \begin{minipage}{0.95\linewidth}
            \centering
            \begin{minipage}{0.32\linewidth}
                \centering
                \includegraphics[width=1.\linewidth]{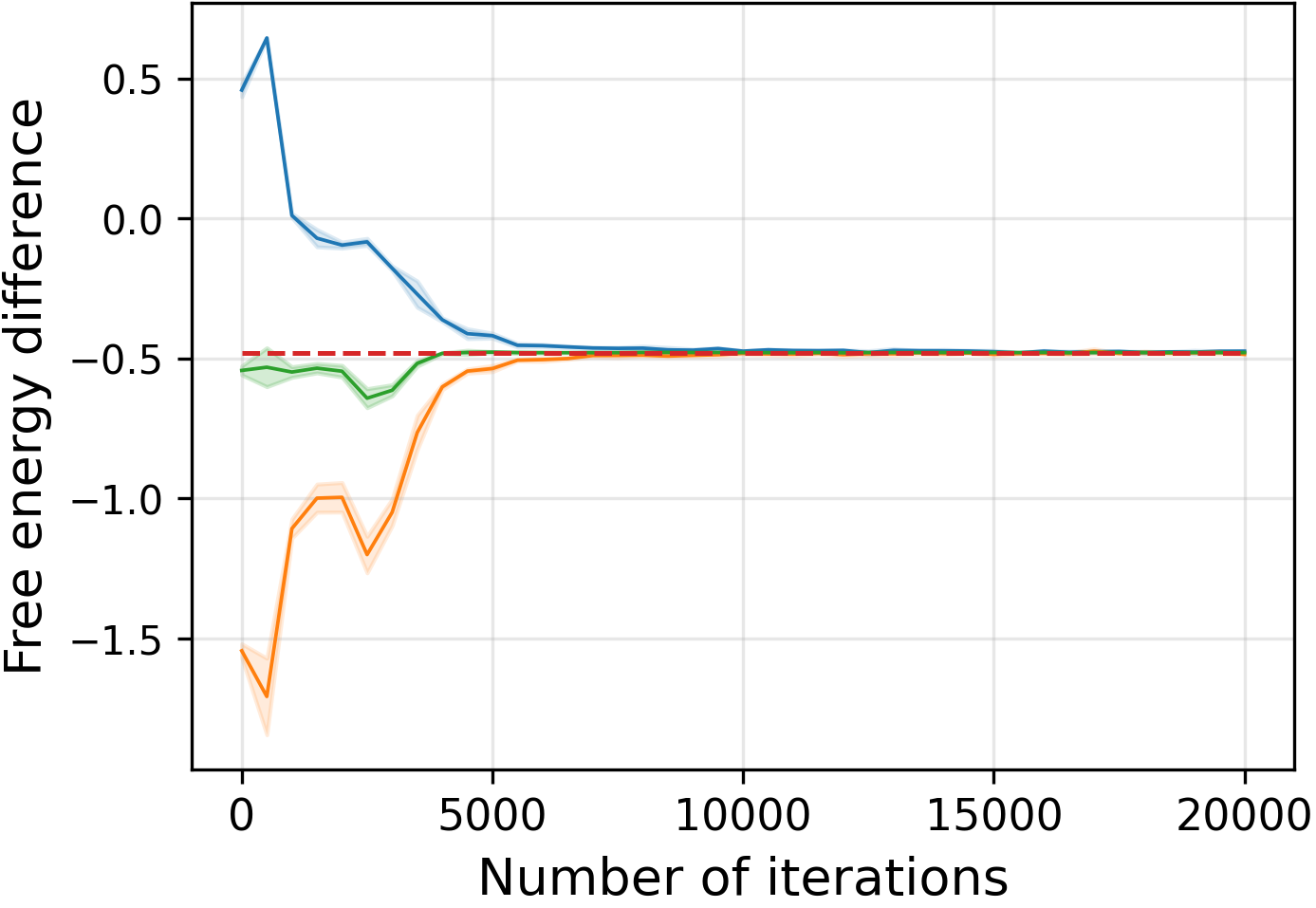}
            \end{minipage}
            \hfill
            \begin{minipage}{0.32\linewidth}
                \centering
                \includegraphics[width=1.\linewidth]{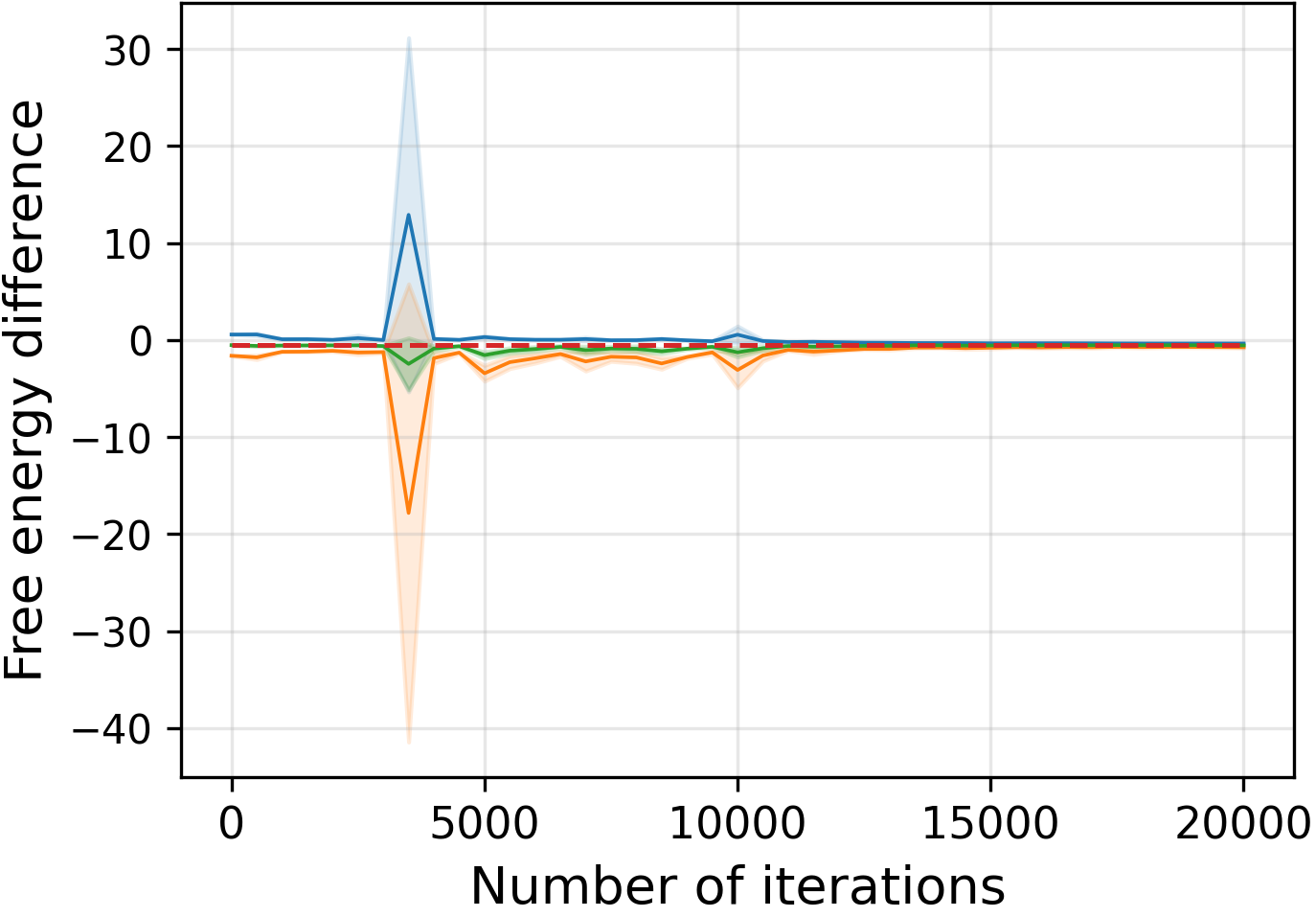}
            \end{minipage}
            \hfill
            \begin{minipage}{0.32\linewidth}
                \centering
                \includegraphics[width=1.\linewidth]{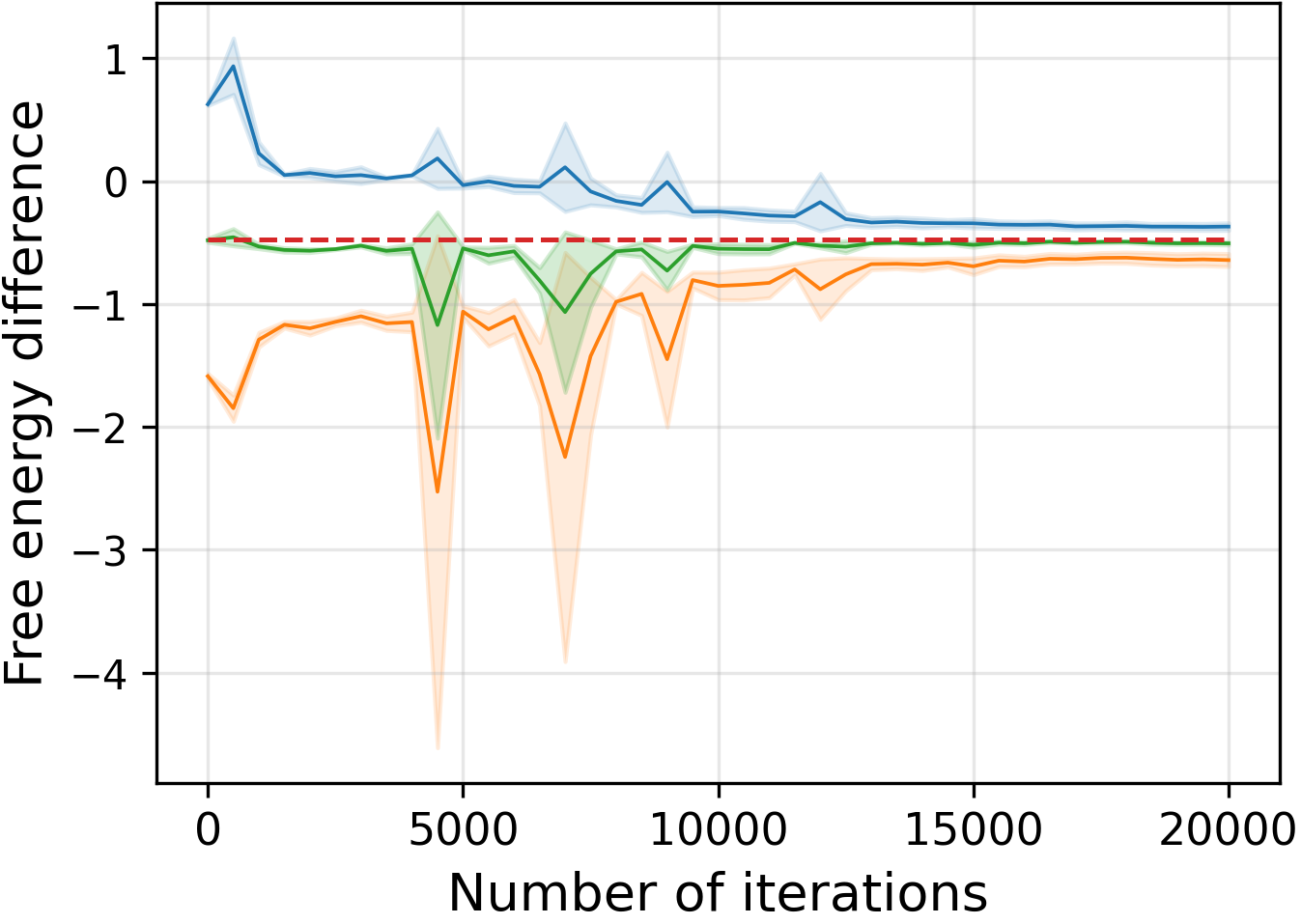}
            \end{minipage}
        \end{minipage}
    \end{minipage}
    \caption{Free energy difference estimates ($\widehat{\Delta F}_F$, $\widehat{\Delta F}_B$, 
    $\widehat{\Delta F}_\mathrm{BAR}$) over training iterations for CTMC FEAT variants 
    on Ising model transport ($\beta = 0.2 \leftrightarrow 0.6$) across lattice sizes. The dashed red line indicates the reference value.}
    \label{fig:deltaF_curve_.2_.6}
\end{figure}

\begin{table}[H]
\centering
\centering
\caption{Ising model using CTMC-FEAT $\beta=0.2  \leftrightarrow  0.6$.}
\label{tab:discrete_feat_0.2_0.6_estimated_value}
\scriptsize
\setlength{\tabcolsep}{4pt}
\renewcommand{\arraystretch}{0.88}
\resizebox{0.7\linewidth}{!}{%
\begin{tabular}{llcc}
\toprule
\textbf{Dim} $D$ & \textbf{Method} & \makecell{\textbf{Estimates}\\($\Delta F / D \times 10^3$)}& \makecell{\textbf{Reference value}\\($\Delta F / D \times 10^3$)} \\
\midrule
\multirow{3}{*}{15$\times$15 (225)}
    & DFM         & -478.68 {\tiny $\pm 0.19$} & \multirow{3}{*}{-478.68} \\
    & GM-no mask  & -478.81 {\tiny $\pm 1.03$} &  \\
    & GM-w.\ mask & -478.91 {\tiny $\pm 0.80$} &  \\
\midrule
\multirow{3}{*}{25$\times$25 (625)}
    & DFM         & -476.21 {\tiny $\pm 0.81$} & \multirow{3}{*}{-476.71} \\
    & GM-no mask  & -475.53 {\tiny $\pm 2.69$} &  \\
    & GM-w.\ mask & -539.61 {\tiny $\pm 35.3$} &  \\
\midrule
\multirow{3}{*}{32$\times$32 (1024)}
    & DFM         & -474.32 {\tiny $\pm 0.60$} & \multirow{3}{*}{-476.28} \\
    & GM-no mask  & -473.16 {\tiny $\pm 1.75$} &  \\
    & GM-w.\ mask & -506.83 {\tiny $\pm 15.9$} &  \\
\bottomrule
\end{tabular}%
}
\end{table}

\subsection{AR FEAT on Ising Model}

\begin{figure}[H]
    \centering
    \begin{minipage}{\linewidth}
        \begin{minipage}{0.95\linewidth}
            \centering
            \begin{minipage}{0.325\linewidth}
                \centering
                {\small Ising 15x15}
                \includegraphics[width=1.\linewidth]{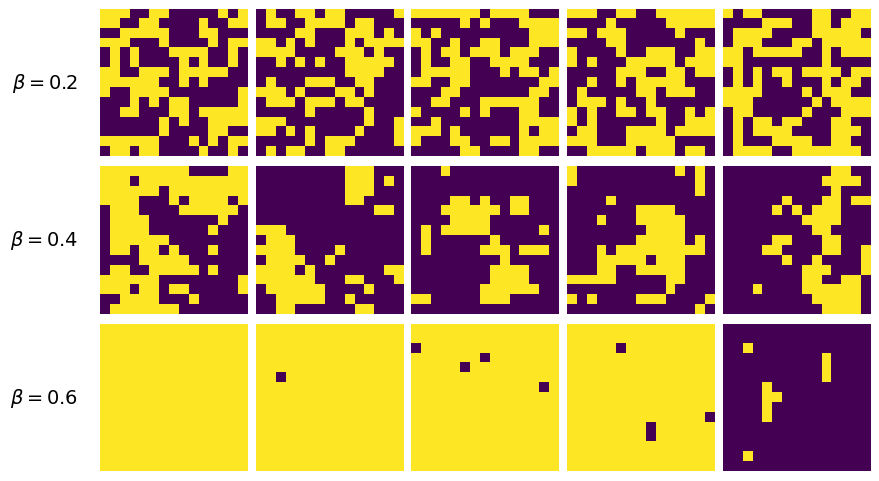}
            \end{minipage}
            \hfill
            \begin{minipage}{0.325\linewidth}
                \centering
                {\small Ising 25x25}
                \includegraphics[width=1.\linewidth]{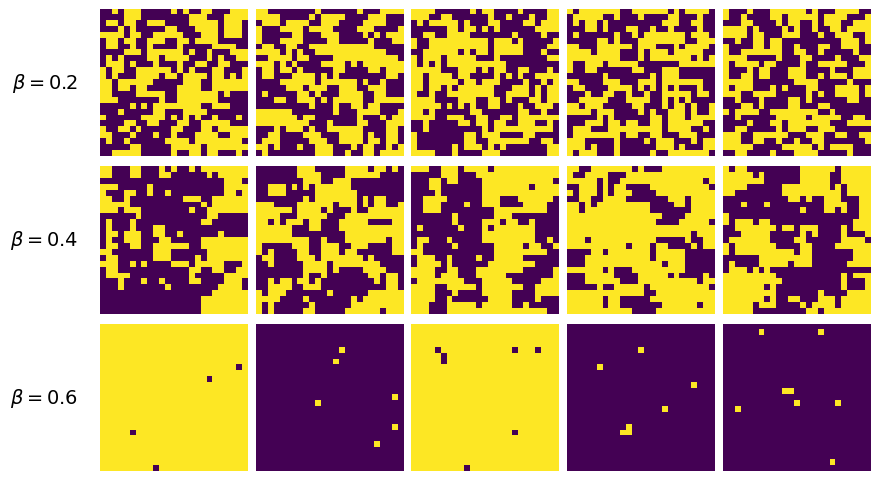}
            \end{minipage}
            \hfill
            \begin{minipage}{0.325\linewidth}
                \centering
                {\small Ising 32x32}
                \includegraphics[width=1.\linewidth]{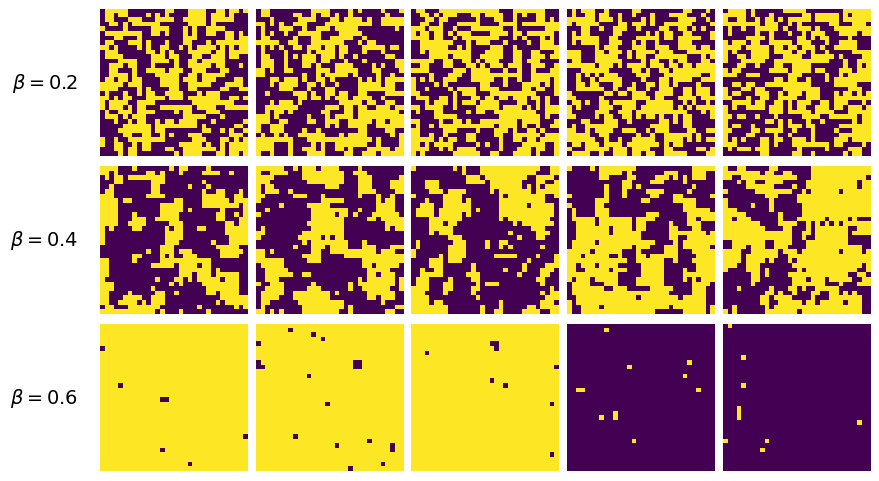}
            \end{minipage}
        \end{minipage}
    \end{minipage} 
    \caption{Samples generated by AR model at different $\beta$.}
\end{figure}
\subsection{Ising Fluid}

\begin{figure}[H]
    \centering
    \includegraphics[width=0.95\linewidth]{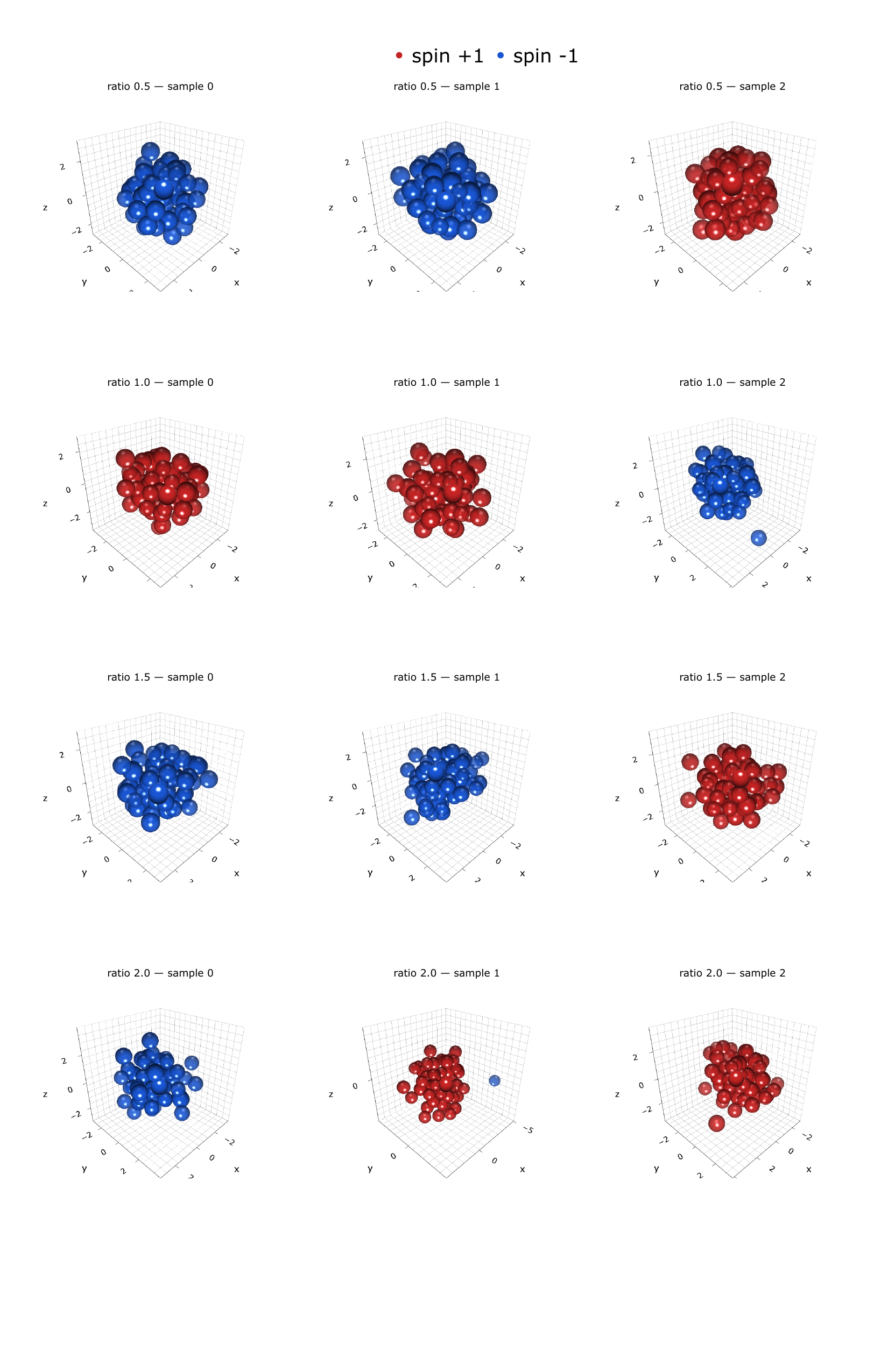}
    \caption{Visualization of 55 particles of Ising fluid (first row: $R=0.5$, second row: $R=1.0$, third row: $R=1.5$, last row: $R=2.0$). Particle size is fixed across plots. }
    \label{fig:IsingFluid_55}
\end{figure}
\newpage

\begin{figure}[H]
    \centering

    \begin{subfigure}[t]{\textwidth}
        \centering
        \includegraphics[width=\linewidth]{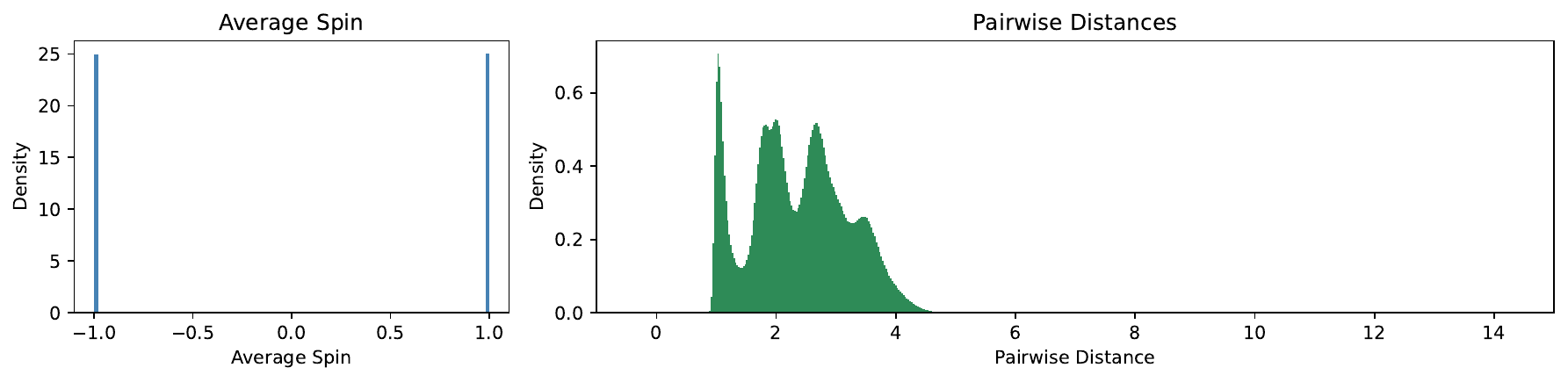}
        \caption{$R = 0.5$}
        \label{fig:plot-55-05}
    \end{subfigure}
    \hfill
    \begin{subfigure}[t]{\textwidth}
        \centering
        \includegraphics[width=\linewidth]{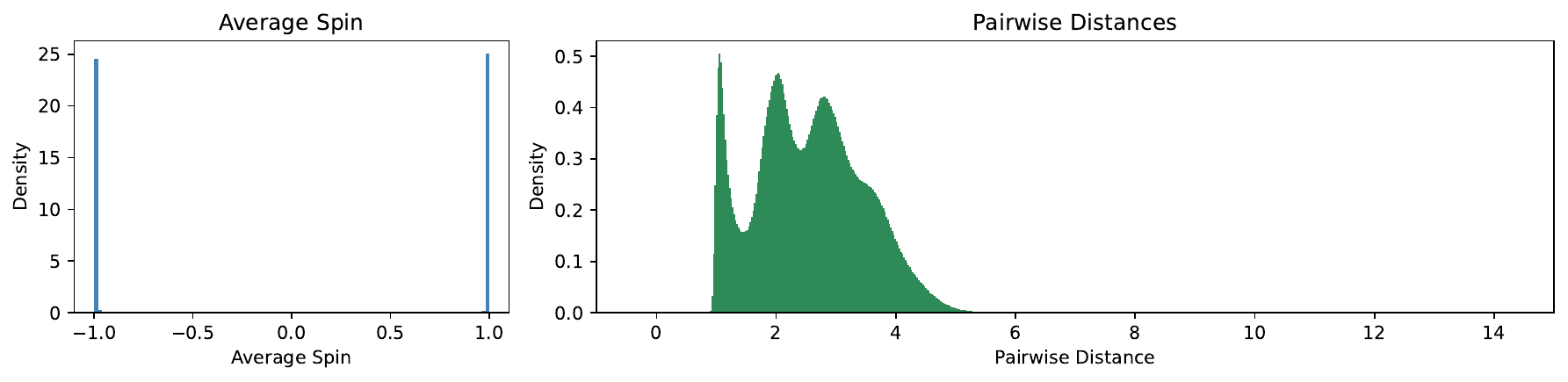}
        \caption{$R = 1.0$}
        \label{fig:plot-55-10}
    \end{subfigure}

    \vspace{0.5em}

    \begin{subfigure}[t]{\textwidth}
        \centering
        \includegraphics[width=\linewidth]{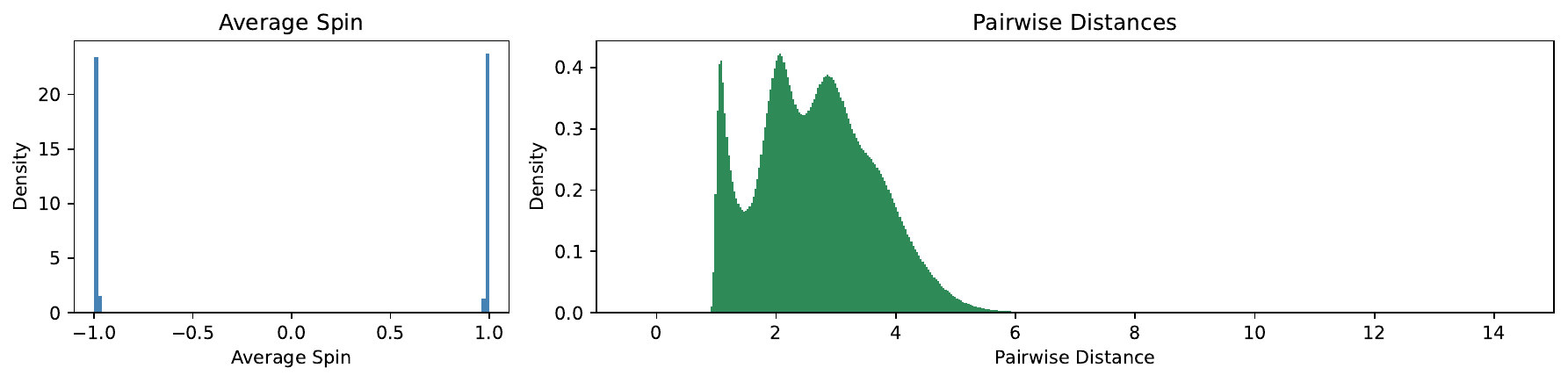}
        \caption{$R = 1.5$}
        \label{fig:plot-55-15}
    \end{subfigure}
    \hfill
    \begin{subfigure}[t]{\textwidth}
        \centering
        \includegraphics[width=\linewidth]{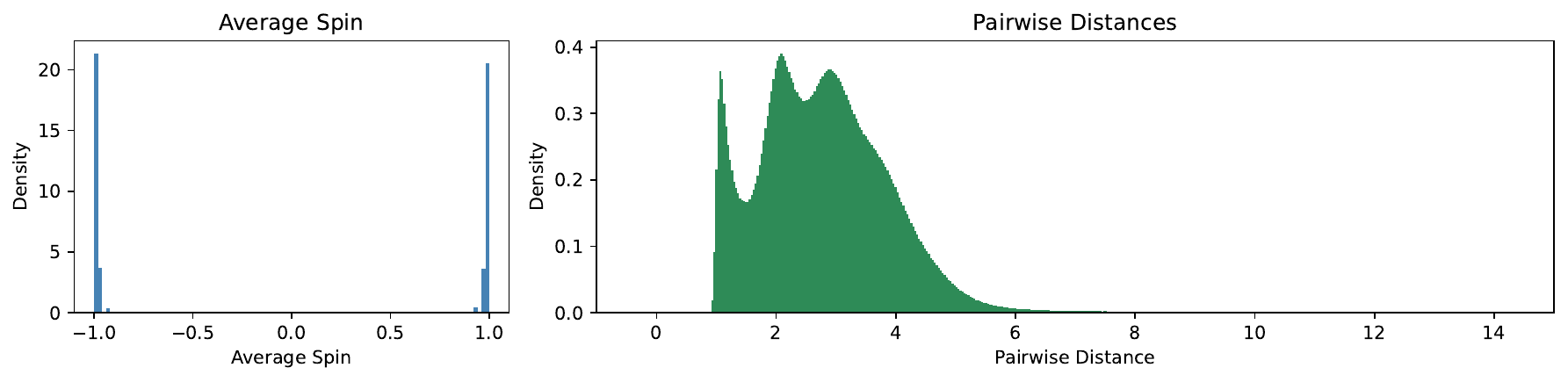}
        \caption{$R = 2.0$}
        \label{fig:plot-55-20}
    \end{subfigure}

    \caption{Average spin and pairwise distances for different values of $R$ for 55-particle Ising fluid.}
    \label{fig:spin-distance-lambda}
\end{figure}

\end{document}